\documentclass{article}

\usepackage{amssymb}
\usepackage{amsmath, amsthm, color}
\usepackage[colorlinks]{hyperref}
\usepackage{fullpage}
\usepackage{verbatim}
\usepackage{epstopdf}
\usepackage{lineno}
\newtheorem{theorem}{Theorem}

\newtheorem{corollary}[theorem]{Corollary}

\newtheorem{lemma}{Lemma}
\newtheorem{proposition}[theorem]{Proposition}

\newtheorem{definition}{Definition}

\newtheorem{assumption}{Assumption}

\theoremstyle{definition}
\newtheorem{remark}[theorem]{Remark}
\newtheorem{example}[theorem]{Example}

\newlength{\widebarargwidth}
\newlength{\widebarargheight}
\newlength{\widebarargdepth}
\DeclareRobustCommand{\widebar}[1]{%
  \settowidth{\widebarargwidth}{\ensuremath{#1}}%
  \settoheight{\widebarargheight}{\ensuremath{#1}}%
  \settodepth{\widebarargdepth}{\ensuremath{#1}}%
  \addtolength{\widebarargwidth}{-0.3\widebarargheight}%
  \addtolength{\widebarargwidth}{-0.3\widebarargdepth}%
  \makebox[0pt][l]{\hspace{0.3\widebarargheight}%
    \hspace{0.3\widebarargdepth}%
    \addtolength{\widebarargheight}{0.3ex}%
    \rule[\widebarargheight]{0.95\widebarargwidth}{0.1ex}}%
  {#1}}

\newcommand{\E}{\ensuremath{\mathbb{E}}}
\newcommand{\mprob}{\ensuremath{\mathbb{P}}}

\newcommand{\real}{\ensuremath{\mathbb{R}}}

\newcommand{\sigmatil}{\ensuremath{\widetilde{\sigma}}}

\newcommand{\sigmahat}{\ensuremath{\widehat{\sigma}}}
\newcommand{\betastar}{\ensuremath{\beta^*}}
\newcommand{\gammastar}{\ensuremath{\gamma^*}}
\newcommand{\betahat}{\ensuremath{\widehat{\beta}}}
\newcommand{\gammahat}{\ensuremath{\widehat{\gamma}}}

\newcommand{\supp}{\ensuremath{\operatorname{supp}}}

\newcommand{\inprod}[2]{\ensuremath{\langle #1 , \, #2 \rangle}}

\newcommand{\Xtil}{\ensuremath{\widetilde{X}}}

\newcommand{\opnorm}[1]{\left|\!\left|\!\left|{#1}\right|\!\right|\!\right|}

\usepackage{arxiv}
\usepackage{debugging}

\usepackage[utf8]{inputenc} % allow utf-8 input
\usepackage[T1]{fontenc}    % use 8-bit T1 fonts
\usepackage{hyperref}       % hyperlinks
\usepackage{url}            % simple URL typesetting
\usepackage{booktabs}       % professional-quality tables
\usepackage{amsfonts}       % blackboard math symbols
\usepackage{nicefrac}       % compact symbols for 1/2, etc.
\usepackage{microtype}      % microtypography
\usepackage{lipsum}
\usepackage{xcolor}
\usepackage{graphicx}
\title{Provable Training Set Debugging for Linear Regression
}

\author{
  Xiaomin Zhang \\
  Department of Computer Sciences\\
  University of Wisconsin-Madison \\
  \texttt{xzhang682@wisc.edu} \\
  \And
  Xiaojin Zhu \\
  Department of Computer Sciences\\
  University of Wisconsin-Madison \\
  \texttt{jerryzhu@cs.wisc.edu} \\
  \And
  Po-Ling Loh \\
  Department of Pure Mathematics and Mathematical Statistics \\ 
  University of Cambridge  \\
  \texttt{pll28@cam.ac.uk} \\
}

\begin{document}
% % keywords can be removed
% \keywords{First keyword \and Second keyword \and More}

\maketitle

\begin{abstract}
We investigate problems in penalized $M$-estimation, inspired by applications in machine learning debugging. Data are collected from two pools, one containing data with possibly contaminated labels, and the other which is known to contain only cleanly labeled points. We first formulate a general statistical algorithm for identifying buggy points and provide rigorous theoretical guarantees when the data follow a linear model. We then propose an algorithm for tuning parameter selection of our Lasso-based algorithm with theoretical guarantees. Finally, we consider a two-person ``game" played between a bug generator and a debugger, where the debugger can augment the contaminated data set with cleanly labeled versions of points in the original data pool. We develop and analyze a debugging strategy in terms of a Mixed Integer Linear Programming (MILP). Finally, we provide empirical results to verify our theoretical results and the utility of the MILP strategy.
\keywords{Robust Statistics \and Outlier Detection \and Tuning Parameter Selection \and Optimization}
\end{abstract}

\section{Introduction}\label{sec:intro}
\label{intro}
Modern machine learning systems are extremely sensitive to training set contamination. Since sources of error and noise are unavoidable in real-world data (e.g., due to Mechanical Turkers, selection bias, or adversarial attacks), an urgent need has arisen to perform automatic debugging of large data sets. Cadamuro et al.~\cite{cadamuro2016debugging} and Zhang et al.~\cite{zhang2018training} proposed a method called ``machine learning debugging'' to identify training set errors by introducing new clean data. Consider the following real-world scenario: Company \textit{A} collects movie ratings for users on a media platform, from which it learns relationships between features of movies and ratings in order to perform future recommendations. A competing company \textit{B} knows \textit{A}'s learning method and hires some users to provide malicious ratings. Company \textit{A} could employ a robust method for learning contaminated data---but in the long run, it would be more effective for company \textit{A} to \emph{identify} the adversarial users and prevent them from submitting additional buggy ratings in the future. This distinguishes debugging from classical learning. The debugging problem also assumes that company \textit{A} can hire an expert to help rate movies, from which it obtains a second trusted data set which is generally smaller than the original data set due to budget limitations.  In this paper, we will study a theoretical framework for the machine learning debugging problem in a linear regression setting, where the main goal is to identify bugs in the data. We will also discuss theory and algorithms for selecting the trusted data set.

Our \emph{first contribution} is to provide a rigorous theoretical framework explaining how to identify errors in the ``buggy" data pool. Specifically, we embed a squared loss term applied to the trusted data pool into the extended Lasso algorithm proposed by Nguyen and Tran~\cite{nguyen2013robust}, and reformulate the objective to better service the debugging task. Borrowing techniques from robust statistics~\cite{huber2009robust,she2011outlier,nguyen2013robust,foygel2014corrupted,slawski2017linear} and leveraging results on support recovery analysis~\cite{wainwright2009sharp,meinshausen2009lasso}, we provide sufficient conditions for successful debugging in linear regression. We emphasize that our setting, involving data coming from multiple pools, has not been studied in any of the earlier papers.

The work of Nguyen and Tran~\cite{nguyen2013robust} and Foygel and Mackey~\cite{foygel2014corrupted} (and more recently, Sasai and Fujisawa~\cite{sasai2020robust}) provided results for the extended Lasso with a theoretically optimal choice of tuning parameter, which depends on the unknown noise variance in the linear model. Our \emph{second contribution} is to discuss a rigorous procedure for tuning parameter selection which does not require such an assumption. Specifically, our algorithm starts from a sufficiently large initial tuning parameter that produces the all-zeros vector as an estimator. Assuming the sufficient conditions for successful support recovery are met, this tuning parameter selection algorithm is guaranteed to terminate with a correct choice of tuning parameter after a logarithmic number of steps. Note that when outliers exist in the training data set, it is improper to use cross-validation to select the tuning parameter due to possible outliers in the validation data set.

Our \textit{third contribution} considers how to design a second clean data pool, which is an  important but previously unstudied problem in machine learning debugging. We consider a two-player ``game" between a bug generator and debugger, where the bug generator performs adversarial attacks~\cite{chakraborty2018adversarial}, and the debugger applies Lasso-based linear regression to the augmented data set. 
On the theoretical side, we establish a sufficient condition under which the debugger can always beat the bug generator, and show how to translate this condition into a debugging strategy based on mixed integer linear programming. Our theory is only derived in the ``noiseless” setting; nonetheless, empirical simulations show that our debugging strategy also performs well in the noisy setting.
We experimentally compare our method to two other algorithms motivated by the machine learning literature, which involve designing two neural networks, one to correct labels and one to fit cleaned data~\cite{veit2017learning}; and a method based on semi-supervised learning that weights the noisy and clean datasets differently and employs a similarity matrix based on the graph Laplacian~\cite{fergus2009semi}.

The remainder of the paper is organized as follows: Section~\ref{sec:formu} introduces our novel framework for machine learning debugging using weighted $M$-estimators. Section~\ref{sec:supprec} provides theoretical guarantees for recovery of buggy data points. Section~\ref{sec:tune} presents our algorithm for tuning parameter selection and corresponding theoretical guarantees. Section~\ref{sec:active} discusses strategies for designing the second pool. Section~\ref{sec:exp} provides experimental results. Section~\ref{sec:conclusion} concludes the paper.

\textbf{Notation:} We write $\Lambda_{\min}(A)$ and $\Lambda_{\max}(A)$ to denote the minimum and maximum eigenvalues, respectively, of a matrix $A$. We use $Null(A)$ to denote the nullspace of $A$. For subsets of row and column indices $S$ and $T$, we write $A_{S,T}$ to denote the corresponding submatrix of $A$.
We write $\|A\|_{\max}$ to denote the elementwise $\ell_\infty$-norm, $\|A\|_2$ to denote the spectral norm, and $\|A\|_\infty$ to denote the $\ell_\infty$-operator norm. For a vector $v \in \real^n$, we write $\supp(v) \subseteq \{1, \dots, n\}$ to denote the support of $v$, and $\|v\|_{\infty} = \max|v_i|$ to denote the maximum absolute entry. We write $\|v\|_p$ to denote the $\ell_p$-norm, for $p \ge 1$. We write $\diag(v)$ to denote the $n \times n$ diagonal matrix with entries equal to the components of $v$.
For $S \subseteq \{1, \dots, n\}$, we write $v_S$ to denote the $|S|$-dimensional vector obtained by restricting $v$ to $S$. We write $[n]$ as shorthand for $\{1, \dots, n\}$.

\section{PROBLEM FORMULATION}
\label{sec:formu}

We first formalize the data-generating models analyzed in this paper.
Suppose we have observation pairs $\{(x_i, y_i)\}_{i=1}^n$ from the contaminated linear model
\begin{equation}
\label{EqnLinModel}
y_i = x_i^\top  \betastar + \gammastar_i + \epsilon_i, \qquad 1 \le i \le n,
\end{equation}
where $\betastar \in \real^p$ is the unknown regression vector, $\gammastar \in \real^n$ represents possible contamination in the labels, and the $\epsilon_i$'s are i.i.d.\ sub-Gaussian noise variables with variance parameter $\sigma^2$. We also assume the $x_i$'s are i.i.d.\ and $x_i \condind \epsilon_i$. This constitutes the ``first pool." Note that the vector $\gammastar$ is unknown and may be generated by some adversary. If $\gammastar_i = 0$, the $i^{\text{th}}$ point is uncontaminated and follows the usual linear model; if $\gammastar_i \neq 0$, the $i^{\text{th}}$ point is contaminated/buggy. Let $T : = \supp(\gammastar)$ denote the indices of the buggy points, and let $t := |T|$ denote the number of bugs.

We also assume we have a clean data set which we call the ``second pool." We observe $\{(\xtil_i,\ytil_i)\}_{i=1}^m$ satisfying
\begin{equation}
\ytil_i = \xtil_i^\top \betastar + \epsilontil_i, \qquad 1 \leq i \leq m,
\end{equation}
where the $\epsilontil_i$'s are i.i.d. sub-Gaussian noise variables with parameter $\sigmatil^2$. Let $L := \frac{\sigma}{\sigmatil}$, and suppose $L \geq 1$. Unlike the first pool, the data points in the second pool are all known to be uncontaminated. 

For notational convenience, we also use $X \in \real^{n \times p}$, $y \in \real^n$, and $\epsilon \in \real^m$ to denote the matrix/vectors containing the $x_i$'s, $y_i$'s, and $\epsilon_i$'s, respectively. Similarly, we define the matrices $\Xtil \in \real^{m \times p}, \ytil \in \real^m$, and $\epsilontil \in \real^m$. Note that $\betastar, \gammastar, T, t$, 
and the noise parameters $\sigma$ and $\sigmatil$ are all assumed to be unknown to the debugger. In this paper, we will work in settings where $X^\top X$ is invertible.

\paragraph{Goal:} Upon observing $\{(x_i, y_i)\}_{i=1}^n$, the debugger is allowed to design $m$ points $\Xtil$ in a stochastic or deterministic manner and query their corresponding labels $\ytil$, with the goal of recovering the support of $\gammastar$. 
We have the following definitions:
\begin{definition}
An estimator $\gammahat$ satisfies \textbf{subset support recovery} if $\supp(\gammahat) \subseteq \supp(\gammastar)$. It satisfies \textbf{exact support recovery} if $\supp(\gammahat) = \supp(\gammastar)$.
\end{definition}
In words, when $\gammahat$ satisfies subset support recovery, all estimated bugs are true bugs. When $\gammahat$ satisfies exact support recovery, the debugger correctly flags \emph{all} bugs. We are primarily interested in exact support recovery. 

\textbf{Weighted $M$-estimation Algorithm:} We propose to optimize the joint objective
\begin{align}
\label{EqnObj2}
\begin{split}
(\betahat, \gammahat) \in & \arg\min_{\beta \in \real^p, \gamma \in \real^n} \left\{\frac{1}{2n} \|y - X\beta - \gamma\|_2^2  + \frac{\eta}{2m} \|\ytil - \Xtil \beta\|_2^2 + \lambda \|\gamma\|_1\right\},
\end{split}
\end{align}
where the weight parameter $\eta > 0$ determines the relative importance of the two data pools. The objective function applies the usual squared loss to the points in the second pool and introduces the additional variable $\gamma$ to help identify bugs in the first pool. Furthermore, the $\ell_1$-penalty encourages $\gammahat$ to be sparse, since we are working in settings where the number of outliers is relatively small compared to the total number of data points. Note that the objective function~\eqref{EqnObj2} may equivalently be formulated as a weighted sum of $M$-estimators applied to the first and second pools, where the loss for the first pool is the robust Huber loss and the loss for the second pool is the squared loss (cf.\ Proposition~\ref{PropWeightedM} in Appendix~\ref{sec:app-additional-discussion}).

%%%%%%%%%%%%%%%%%%%%PUT THIS IN DISCUSSION?%%%%%%%%%%%%%%%%%%%%%%%%%

\textbf{Lasso Reformulation:} Recall that our main goal is to estimate (the support of) $\gammastar$ rather than $\betastar$. Thus, we will restrict our attention to $\gammastar$ by reformulating the objectives appropriately. We first introduce some notation: Define the stacked vectors/matrices
\begin{equation}
\label{EqnStack}
X' = \begin{pmatrix} X \\ \sqrt{\frac{\eta n}{m}} \Xtil
\end{pmatrix},  y' = \begin{pmatrix} y \\ \sqrt{\frac{\eta n}{m}} \ytil \end{pmatrix}, 
\epsilon' = \begin{pmatrix} \epsilon \\ \sqrt{\frac{\eta n}{m}} \tilde{\epsilon} \end{pmatrix},
\end{equation}
where $X' \in \real^{(m+n) \times p}$ and $y', \epsilon'  \in \real^{m+n}$. For a matrix $A$, let $P_A=A(A^\top A)^{-1}A^\top $ and $P_A^\perp=I - A(A^\top A)^{-1}A^\top $ denote projection matrices onto the range of the column space of $A$ and its orthogonal complement, respectively. For a matrix $S \subseteq [n]$, let $M_S$ denote the $(n+m) \times |S|$ matrix with $i^{\text{th}}$ column equal to the canonical vector $e_{S(i)}$. Thus, right-multiplying by $M_S$ truncates a matrix to only include columns indexed by $S$. We have the following useful result:

\begin{proposition}
\label{prop:obj-reform}
The objective function
\begin{equation}
\label{EqnPenGamma}
\gammahat \in \arg\min_{\gamma \in \real^n} \Big\{\frac{1}{2n} \|P_{X'}^\perp y' - P_{X'}^\perp M_{[n]} \gamma\|_2^2 + \lambda \|\gamma\|_1\Big\}
\end{equation}
shares the same solution for $\gammahat$ with the objective function~\eqref{EqnObj2}.
\end{proposition}
Proposition~\ref{prop:obj-reform}, proved in Appendix~\ref{AppSecFormu}, translates the joint optimization problem~\eqref{EqnObj2} into an optimization problem only involving the parameter of interest $\gamma$. We provide a discussion regarding the corresponding solution $\betahat$ in Appendix~\ref{sec:app-additional-discussion} for the interested reader. 
Note that the optimization problem~\eqref{EqnPenGamma} corresponds to linear regression of the vector/matrix pairs $(P_{X'}^\perp y', P_{X'}^\perp M_{[n]})$ with a Lasso penalty, inspiring us to borrow techniques from high-dimensional statistics. 

%%%%%

\section{SUPPORT RECOVERY}
\label{sec:supprec}

The reformulation~\eqref{EqnPenGamma} allows us to analyze the machine learning debugging framework through the lens of Lasso support recovery. The three key conditions we impose to ensure support recovery are provided below. Recall that we use $M_T$ to represent the truncation matrix indexed by $T$.
%We use $X_T$ to denote the submatrix of $X$ with rows indexed by $T$, and define $X_{T^c}$ analogously. 

\begin{assumption}[Minimum Eigenvalue]\label{C1}
Assume that there is a positive number $b'_{\min}$ such that 
\begin{equation}\label{C1-2}
  \Lambda_{\min}\left( M_T^\top P_{X'}^\perp M_T\right) \geq b'_{\min}.
\end{equation}
\end{assumption}
\begin{assumption}[Mutual Incoherence]\label{C2}
Assume that there is a number $\alpha' \in [0,1)$ such that
\begin{equation}\label{C2-2}
  \| M_{T^c}^\top P_{X'}^\perp M_T (M_T^\top P_{X'}^\perp M_T)^{-1} \|_\infty \leq \alpha'.
\end{equation}
\end{assumption}
\begin{assumption}[Gamma-Min]\label{C3}
Assume that
\begin{equation}\label{C3-2}
\begin{split}
  \min_{i \in T} |\gammastar_i| > G' & := \| (M_T^\top P_{X'}^\perp M_T)^{-1}M_T^\top P_{X'}^\perp \epsilon'\|_\infty + n\lambda\opnorm{(M_T^\top P_{X'}^\perp M_T)^{-1}}_{\infty}.
\end{split}
\end{equation}
\end{assumption}

Assumption~\ref{C1} comes from a primal-dual witness argument~\cite{wainwright2009sharp} to guarantee that the minimizer $\gammahat$ is unique. 
Assumption~\ref{C2} measures a relationship between the sets $T^c$ and $T$, indicating that the large number of nonbuggy covariates (i.e., $T^c$) cannot exert an overly strong effect on the subset of buggy covariates~\cite{ravikumar2010high}. To aid intuition, consider an orthogonal design, where $X = \begin{bmatrix} c  I_{[t], [p]} \\ c' I_{p \times p} \end{bmatrix}$ and $\Xtil = c'' I_{p \times p}$, for some $t < p$, and $c, c', c'' > 0$. We use the notation $I_{[t], [p]}$ to denote a submatrix of $I_{p \times p }$ with rows indexed by the set $[t]$. Suppose the first $t$ points are bugs, and for simplicity, let $\eta = m/n$. Then the mutual incoherence condition requires $c < c' + \frac{(c'')^2}{c'}$, meaning that in every direction $e_i$, the component of buggy data cannot be too large compared to the nonbuggy data and the clean data. 
Assumption~\ref{C3} lower-bounds the minimum absolute value of elements of $\gamma$. Note that $\lambda$ is chosen based on $\epsilon'$, so the right-hand expression is a function of $\epsilon'$. This assumption indeed captures the intuition that the  signal-to-noise ratio, $\frac{\min_{i\in T}{|\gammastar_i|}}{\sigma}$, needs to be sufficiently large. 

We now provide two general theorems regarding subset support recovery and exact support recovery. 

\begin{theorem}[Subset support recovery]
\label{subsetmainthm} 
Suppose $P_{X'}^\perp$ satisfies Assumptions~\ref{C1} and~\ref{C2}. If the tuning parameter satisfies
\begin{align}\label{tdplambdacond1}
\begin{split}
\lambda & \geq \frac{2}{1-\alpha'} \Big\|M_{T^c} P_{X'}^\perp \Big(I - P_{X'}^\perp M_T(M_T^\top P_{X'}^\perp M_T)^{-1}M_T^\top P_{X'}^\perp\Big)\frac{\epsilon'}{n}\Big\|_\infty,
\end{split}
\end{align}
then the objective~\eqref{EqnPenGamma} has a unique optimal solution $\gammahat$, satisfying $\supp(\gammahat) \subseteq \supp(\gammastar)$ and $\infnorm{\gammahat - \gammastar} \leq G'$.
\end{theorem}

\begin{theorem}[Exact support recovery]\label{eaxtsetmainthm}
In addition to the assumptions in Theorem~\ref{subsetmainthm}, suppose Assumption~\ref{C3} holds. Then we have a unique optimal solution $\gammahat$, which satisfies exact support recovery.
\end{theorem}

Note that we additionally need Assumption~\ref{C3} to guarantee exact support recovery. This follows the aforementioned intuition regarding the assumption. In particular, recall that $\epsilon$ and $\epsilontil$ are sub-Gaussian vectors with parameters $\sigma^2$ and $\sigma^2/L$, respectively, where $L \geq 1$ (i.e., the clean data pool has smaller noise). The minimum signal strength $\min_{i\in T}{|\gammastar_i|}$ needs to be at least $\Theta(\sigma \sqrt{\log n})$, since $\E\left[\max_{i \in [n]} |\epsilon_i|\right] \leq \sigma \sqrt{2 \log(2n)}$. Intuitively, if $\min_{i\in T}{|\gammastar_i|}$ is of constant order, it is difficult for the debugger to distinguish between random noise and intentional contamination.

We now present two special cases to illustrate the theoretical benefits of including a second data pool. Although Theorems~\ref{subsetmainthm} and~\ref{eaxtsetmainthm} are stated in terms of \emph{deterministic} design matrices and error vectors $\epsilon$ and $\epsilontil$, the assumptions can be shown to hold with high probability in the example. We provide formal statements of the associated results in Appendix~\ref{app:ortho-design} and Appendix~\ref{app:subgaussian-design}.

\begin{example}[Orthogonal design]
\label{ExaOrth}
Suppose $Q$ is an orthogonal matrix with columns $q_1, q_2, \dots, q_p$, and consider the setting where $X_T = RQ^\top \in \real^{t \times p}$ and $X_{T^c} = FQ^\top \in \real^{p \times p}$, where  $R = \left[\diag(\{r_i\}_{i=1}^t) \mid  \bm{0}_{t\times (p-t)}\right]$ and $F = \diag(\{f_i\}_{i=1}^p)$. Thus, points in the contaminated first pool correspond to orthogonal vectors. 
Similarly, suppose the second pool consists of (rescaled) columns of $Q$, so $\Xtil = WQ^\top \in \real^{m \times p}$, where $W = \diag(\{w_i\}_{i=1}^p)$. (To visualize this setting, one can consider $Q = I$ as a special case.) 
The mutual incoherence parameter is $\alpha' = \max_{1 \leq i \leq t}\left|\frac{r_i f_i}{f_i^2 + \eta \frac{n}{m}w_i^2}\right|$. Hence, $\alpha' < 1$ if the weight of a contaminated point dominates the weight of a clean point in any direction, e.g., when $|r_i| > |f_i|$ and $w_i = 0$; in contrast, if the second pool includes clean points $w_iq_i$ with sufficiently large $|w_i|$, we can guarantee that $\alpha' < 1$. Furthermore,
\begin{multline*}
G' \approx \sigma\left(\sqrt{2\log t}+c\right)\sqrt{1+\max_{1\leq i \leq t}\frac{r_i^2(Lf_i^2+\frac{\eta n}{m}w_i^2)}{L(f_i^2 + \frac{\eta n}{m}w_i^2)^2}} 
+ \frac{2\sigma}{1-\alpha'} \left(\sqrt{\log 2(n-t)} +C\right)\left(1+\max_{1\leq i \leq t}\frac{r_i^2}{f_i^2+\frac{\eta n}{m}w_i^2}\right)
\end{multline*}
for some constant $C$. It is not hard to verify that $G'$ decreases by adding a second pool. 
Further note that the behavior of the non-buggy subspace, $\mathrm{span}\{q_{t+1}, \dots, q_p\}$, is not involved in any conditions or conclusions. Thus, our key observation is that the theoretical results for support recovery consistency only rely on the addition of second-pool points in  buggy directions. 

\end{example}

\begin{example}[Random design]Consider a random design setting where the rows of $X$ and $\Xtil$ are drawn from a common sub-Gaussian distribution with covariance $\Sigma$. The conditions in Assumptions~\ref{C1}--\ref{C3} are relaxed in the presence of a second data pool when $n$ and $m$ are large compared to $p$: First, $b_{\min}'$ increases by adding a second pool. Second, $\alpha' \approx \frac{\|X_{T^c}\Sigma^{-1}X_T\|_\infty}{n-t+\eta n}$, so the mutual incoherence parameter also decreases by adding a second pool. Third,
\begin{equation*}
G' \approx \frac{2\sigma \sqrt{\log t}}{b'_{\min}} + \frac{2\sigma}{1-\alpha'}\max\left\{1,\sqrt{\frac{\eta n}{mL}}\right\} \left\|(I_{t\times t} - \frac{X_T\Sigma^{-1}X_T^\top}{n+\eta n})^{-1}\right\|_\infty,
\end{equation*}
where $X_T$ and $X_{T^c}$ represent the submatrices of $X$ with rows indexed by $T$ and $T^c$, respectively. Note that the one-pool case corresponds to $\eta = 0$ and $\left\|(I_{t\times t} - \frac{X_T\Sigma^{-1}X_T^\top}{n+\eta n})^{-1}\right\|_\infty < \left\|(I_{t\times t} - \frac{X_T\Sigma^{-1}X_T^\top}{n})^{-1}\right\|_\infty$, so if we choose $\eta \leq \frac{mL}{n}$, then $G'$ decreases by adding a second pool.
Therefore, all three assumptions are relaxed by having a second pool, making it easier to achieve exact support recovery.
\end{example}
We also briefly discuss the three assumptions with respect to the weight parameter $\eta$: Increasing $\eta$ always relaxes the eigenvalue and mutual incoherence conditions, so placing more weight on the second pool generally helps with subset support recovery. However, the same trend does not necessarily hold for exact recovery. This is because a larger value of $\eta$ causes the lower bound~\eqref{tdplambdacond1} on $\lambda$ to increase, resulting in a stricter gamma-min condition. Therefore, there is a tradeoff for selecting $\eta$.

%%%%%

\section{TUNING PARAMETER SELECTION}\label{sec:tune}

A drawback of the results in the previous section is that the proper choice of tuning parameter depends on a lower bound~\eqref{tdplambdacond1} which cannot be calculated without knowledge of the unknown parameters $(T, \alpha', \epsilon')$. The tuning parameter $\lambda$ determines how many outliers a debugger detects; if $\lambda$ is large, then $\gammahat$ contains more zeros and the algorithm detects fewer bugs. A natural question arises: \textit{In settings where the conditions for exact support recovery hold, can we select a data-dependent tuning parameter that correctly identifies all bugs?}
In this section, we propose an algorithm which answers this question in the affirmative.

\subsection{Algorithm and Theoretical Guarantees}
\label{sec:alg-choice-lambda}

Our tuning parameter selection algorithm is summarized in Algorithm~\ref{alg:choice-lambda}, which searches through a range of parameter values for $\lambda$, starting from a large value $\lambda_u$ and then halving the parameter on each successive step until a stopping criterion is met. The intuition is as follows: First, let $\lambda^*$ be the right-hand expression of inequality~\eqref{tdplambdacond1}. Suppose that for any value in $I = [\lambda^*, 2 \lambda^*]$, support recovery holds. Then given $\lambda_u > \lambda^*$, the geometric series $\Lambda = \left\{\lambda_u, \frac{\lambda_u}{2}, \frac{\lambda_u}{4}, \dots \right\}$ must contain at least one correct parameter for exact support recovery since $\Lambda \cap I  \neq \emptyset$, guaranteeing that the algorithm stops. As for the stopping criterion, let $X_S$ denote the submatrix of $X$ with rows indexed by $S$ for $T^c \subseteq S \subseteq [n]$. We have $P_{X_S}^\perp \overset{|S| \to \infty}{\longrightarrow} \left(1-\frac{p}{|S|)}\right)I$ under some mild assumptions on $X$, in which case $P_{X_S}^\perp y_S \rightarrow \left(1-\frac{p}{|S|}\right)(\gammastar_S + \epsilon_S)$. When $\lambda$ is large and the conditions hold for subset support recovery but not exact recovery, we have $S \cap T \neq \emptyset$, so
\begin{equation*}
\min |P_{X_S}^\perp y_S| \geq \left(1-\frac{p}{|S|}\right)\left(\min |\gammastar_T| - \max_{i\in [n]} |\epsilon_i|\right).
\end{equation*}
In contrast, when $S = T^c$, we have
\begin{equation*}
\min |P_{X_S}^\perp y_S| \leq \left(1-\frac{p}{|S|}\right)\max_{i\in [n]} |\epsilon_i|.
\end{equation*}
When $\min |\gammastar_T|$ is large enough, the task then reduces to choosing a proper threshold to distinguish the error $|\epsilon_{T^c}|$ from the bug signal $|\gammastar_T|$, which occurs when the threshold is chosen between $\max_i|\epsilon_i|$ and $\min_{i\in T}|\gammastar_i| - \max_i|\epsilon_i|$.

\begin{algorithm}
\caption{Regularizer selection}
\label{alg:choice-lambda}
\hspace*{\algorithmicindent} \textbf{Input: $\lambda_u, \bar{c}$} \\
\hspace*{\algorithmicindent} \textbf{Output: $\hat{\lambda}^k$} 
\begin{algorithmic}[1]
\State $C = 1, k = 1, \hat{\lambda}^{k}=\lambda_u$.
\While{$C=1$}
    \State $\gammahat^k \in \arg\min_{\gamma \in \real^n} \big\{\frac{1}{2n} \|P_{X'}^\perp y' - P_{X'}^\perp M_{[n]} \gamma\|_2^2 + \hat{\lambda}^k \|\gamma\|_1\big\}.$
    \State Let $X^{(k)}, y^{(k)}$ consist of $x_i,y_i$ such that $ i \notin \supp(\gammahat^k)$. Let $l^{(k)}$ be the length of $y^{(k)}$.
    \State $\sigmahat = \frac{l^{(k)}}{l^{(k)}-p} \cdot \mbox{median}\left(\left|P_{X^{(k)}}^\perp y^{(k)}\right|\right)$.
    \State $C = 0$ if $\| P_{X^{(k)}}^\perp y^{(k)}\|_\infty \leq \frac{5}{2} \bar{c}^{-1} \sqrt{\log 2n}\,\sigmahat $.
    \State $k=k+1, \hat{\lambda}^k = \hat{\lambda}^{k-1}/2$.
\EndWhile
\end{algorithmic}
\end{algorithm}

With the above intuition, we now state our main result concerning exact recovery guarantees for our algorithm. Recall that the $\epsilon_i$'s are sub-Gaussian with parameter $\sigma^2$.

Let $c_t := \frac{t}{n} < \frac{1}{2}$ denote the fraction of outliers. We assume knowledge of a constant $\bar{c}$ that satisfies $c_t + \mathbb{P}[|\epsilon_i| \leq \bar{c} \sigma] < \frac{1}{2}$. Note that a priori knowledge of $\bar{c}$ is a less stringent assumption than knowing $\sigma$, since we can always choose $\bar{c}$ to be close to zero. For instance, if we know the $\epsilon_i$'s are Gaussian, we can choose $\bar{c} < \mathrm{erf}^{-1}(\frac{1}{2} - c_t)$; in practice, we can usually estimate $c_t$ to be less than $\frac{1}{3}$, so we can take $\bar{c} = \mathrm{erf}^{-1}(\frac{1}{6})$. As shown later, the tradeoff is that having a larger value of $\bar{c}$ provides the desired guarantees under weaker requirements on the lower bound of $\min_{i \in T} |\gammastar_i|$. Hence, if we know more about the shape of the error distribution, we can be guaranteed to detect bugs of smaller magnitudes. We will make the following assumption on the design matrix:

\begin{assumption}\label{assump:X-choice-lambda} 
There exists a $p \times p$ positive definite matrix $\Sigma$, with bounded minimum and maximum eigenvalues, such that for all $X^{(k)}$ appearing in the while loop of Algorithm~\ref{alg:choice-lambda}, we have
\begin{align}
\label{eq:linfcov-eq:EqnSigmaEig}
\left\|\frac{X^{(k)} \Sigma^{-1} X^{(k) \top} }{p} - I\right\|_{\max} & \le c\max\left\{\sqrt{\frac{\log l^{(k)}}{p}},\frac{\log l^{(k)}}{p} \right\}, \qquad
\norm{\frac{X^{(k) \top}X^{(k)}}{l^{(k)}} - \Sigma} & \le \frac{\lambda_{\min}(\Sigma)}{2},
\end{align}
where $l^{(k)}$ is the number of rows of the matrix $X^{(k)}$ and $c$ is a universal constant.
\end{assumption}

This assumption is a type of concentration result, which we will show holds w.h.p.\ in some random design settings in the following proposition:
\begin{proposition}\label{prop:Xassumption-holds}
Suppose the $x_i$'s are i.i.d.\ and satisfy any of the following additional conditions:
\begin{itemize}
\item[(a)] the $x_i$'s are Gaussian and the spectral norm of the covariance matrix is bounded;
\item[(b)] the $x_i$'s are sub-Gaussian with mean zero and independent coordinates, and the spectral norm of the covariance matrix is bounded; or
\item[(c)] the $x_i$'s satisfy the convex concentration property.
\end{itemize}
Then Assumption~\ref{assump:X-choice-lambda} holds with probability at least $1 - O(n^{-1})$.
\end{proposition}

The $\Sigma$ matrix can be chosen as the covariance of $X$. In fact, Assumption~\ref{assump:X-choice-lambda} shows that $P_{X^{(k)}}^\perp$ is approximately a scalar matrix. We now introduce some additional notation: For $\nu > 0$, define $c_\nu$ and $C_\nu$ such that $\nu = \mathbb{P}[|\epsilon_i| \leq c_\nu \sigma]$ and $\nu = \mathbb{P}[|\epsilon_i| \geq C_\nu \sigma]$. We write $G'(\lambda)$ to denote the function of $\lambda$ in the right-hand expression of inequality~\eqref{C3-2}. Proofs of the theoretical results in this section are provided in Appendix~\ref{AppSecTune}.

\begin{theorem}\label{thm:choice-lambda}
Assume $\nu$ is a constant satisfying $\nu + c_t < \frac{1}{2}$. 
Suppose Assumption~\ref{assump:X-choice-lambda}, the minimum eigenvalue condition, and the mutual incoherence condition hold. If
\begin{align}\label{eq:tuning-para-lower-bound-n}
n \geq \max \left\{ \left[\frac{24}{c_\nu}\right]^{\frac{1}{c_n}},  \left[\frac{C \log 2n}{1-c_t}(p^2+ \log^2 n)\right]^{\frac{1}{1-2c_n}}\right\},
\end{align}
where $C$ is an absolute constant, and 
\begin{align}\label{eq:tuning-para-lower-bound}
\begin{split}
\min_{i \in T}|\gammastar_i| & > \max\Bigg\{G'(2\lambda^*),4\sqrt{\log (2n)}\sigma,  \frac{5}{4}\sqrt{\log (2n)}\frac{c_\nu +5C_\nu}{\bar{c}}\sigma\Bigg\}, \\
\|\gammastar\|_\infty & \le \frac{\sqrt{C} c_\nu}{16\sqrt{2}}\sqrt{1-c_t} \sqrt{ \log 2n}\, \frac{n^{1/2+c_n}}{t} \sigma,
\end{split}
\end{align}
for some $c_n \in (0,\frac{1}{2})$,
then Algorithm~\ref{alg:choice-lambda} with inputs $\bar{c} < c_\nu$ and $\lambda_u \ge \lambda^*$ will return a feasible $\lambdahat$ in at most $\log_2\left(\frac{\lambda_u}{\lambda^*}\right)$ iterations such that the Lasso estimator $\gammahat$ based on $\lambdahat$ satisfies $\supp(\gammahat) = \supp(\gammastar)$, with probability at least
\begin{equation*}
1 - \frac{3\log_2\left(\frac{\lambda_u}{\lambda^*}\right)}{n-t} - 2\log_2\left(\frac{\lambda_u}{\lambda^*}\right) \exp\left(-2\left(\frac{1}{2}-c_t-\nu\right)^2n\right).
\end{equation*}
\end{theorem}
Theorem~\ref{thm:choice-lambda} guarantees exact support recovery for the output of Algorithm~\ref{alg:choice-lambda} without knowing $\sigma$. Note that compared to the gamma-min condition~\eqref{C3-2} with $\lambda = \lambda^*$, the required lower bound~\eqref{eq:tuning-para-lower-bound} only differs by a constant factor.
In fact, the constant 2 inside $G'(2\lambda^*)$ can be replaced by any constant $c > 1$, but Algorithm~\ref{alg:choice-lambda} will then update $\hat{\lambda}^k = \hat{\lambda}^{k - 1}/c$ and require $\log_c \left(\frac{\lambda_u}{\lambda^*}\right)$ iterations. Further note that larger values of $c_t$ translate into a larger sample size requirement, as $n = \Omega\left(\frac{1}{1-c_t}\right)$ for $c_n$ being close to 0. A limitation of the theorem is the upper bound on $\|\gammastar\|_\infty$, where $t$ needs to be smaller than $n$ in a nonlinear relationship. Also, $n$ is required to be $\Omega(p^2)$. These two conditions are imposed in our analysis in order to guarantee that $P_{X_S}^\perp y_S \rightarrow \left(1-\frac{p}{|S|}\right)(\gammastar_S + \epsilon_S)$. We now present a result indicating a practical choice of $\lambda_u$:

\begin{corollary}\label{cor:lambda-upper-bound-1}
Define
\begin{equation*}
\lambda(\sigma) := \frac{8\max\{1,\sqrt{\frac{\eta n}{Lm}}\}}{1-\alpha'}\sqrt{\log 2(n-t)}\frac{\|P_{X,T^c}^\perp\|_2}{n} \cdot c\sigma.
\end{equation*}
Suppose Assumption~\ref{assump:X-choice-lambda}, the minimum eigenvalue condition, and the mutual incoherence condition hold. Also assume conditions~\eqref{eq:tuning-para-lower-bound-n} and~\eqref{eq:tuning-para-lower-bound} hold when replacing $\lambda^*$ by $\lambda(\sigma)$. Taking the input $\lambda_u = \frac{2\|M_{[n]} P^\perp_{X'} y'\|_\infty}{n}$,
Algorithm~\ref{alg:choice-lambda} outputs a parameter $\lambdahat$ in $O(\log n)$ iterations which provides exact support recovery, with probability at least
\begin{multline*}
1 - \frac{4\left(c'\log_2 n + \max\left\{0,\frac{1}{2}\log_2 \frac{\eta n}{mL}\right\} \right)}{n-t} - 2\left(c'\log_2 n +  \frac{1}{2}\max\left\{0,\log_2 \frac{\eta n}{mL}\right\}\right)e^{-2\left(\frac{1}{2}-c_t-\nu\right)^2n}.
\end{multline*}
\end{corollary}

Note that $\lambda_u$ can be calculated using the observed data set. 
Further note that the algorithm is guaranteed to stop after $O(\log n)$ iterations, meaning it is sufficient to test a relatively small number of candidate parameters in order to achieve exact recovery.

\section{STRATEGY FOR SECOND POOL DESIGN}
\label{sec:active}

We now turn to the problem of designing a clean data pool. In the preceding sections, we have discussed how a second data pool can aid exact recovery under sub-Gaussian designs. In practice, however, it is often unreasonable to assume that new points can be drawn from an entirely different distribution.
Specifically, recall the movie rating example discussed in Section~\ref{sec:intro}: The expert can only rate movies in the movie pool, say $\{x_i\}_{i=1}^n$, whereas an arbitrarily designed $\xtil$, e.g., $\xtil = x_1 / 2$, is unlikely to correspond to an existing movie. Thus, we will focus on devising a debugging strategy where the debugger is allowed to choose points for the second pool which have the same covariates as points in the first pool.

In particular, we consider this problem in the ``worst" case: suppose a bug generator can generate any $\gammastar \in \Gamma := \{\gamma \in \real^n : \supp(\gamma)| \leq t\}$ and add it to the correct labels $X\betastar$. We will also suppose the bug generator knows the debugger's strategy. The debugger attempts to add a second data pool which will ensure that all bugs are detected regardless of the choice of $\gammastar$. Our theory is limited to the noiseless case, where $y = X\betastar + \gammastar$ and $\ytil = \Xtil \betastar$; the noisy case is studied empirically in Section~\ref{sec:active-exp}.

\subsection{Preliminary Analysis}\label{subsec:active-analysis}

We denote the debugger's choice by $\xtil_i = X^\top e_{\nu(i)}$, for $i \in [m]$, where $e_{\nu(i)} \in \real^{n}$ is a canonical vector and $\nu: [m] \rightarrow [n]$ is injective. In matrix form, we write $\Xtil = X_{D}$, where $D \subseteq [n]$ represents the indices selected by the debugger. Assume $m < p$, so the debugger cannot simply use the clean pool to obtain a good estimate of $\beta$. 
In the noiseless case, we can write the debugging algorithm as follows:
\begin{align}
\label{eq:noiseless-twopool-optimization}
\begin{split}
\min_{\beta\in \real^p,\gamma \in \real^n} & \|\gamma\|_1 \\
\mbox{subject to } & y = X\beta + \gamma, \ \ytil = \Xtil \beta.
\end{split}  
\end{align}

Similar to Proposition~\ref{prop:obj-reform}, given a $\gamma$, we can pick $\beta$ to satisfy the constraints, specifically $\beta = \left(X^\top X + \Xtil^\top \Xtil \right)^{-1} \left(X^\top (y-\gamma)+ \Xtil^\top \ytil\right)$.
Eliminating $\beta$, we obtain the optimization problem
\begin{align}\label{eq:active-objective-gamma-y}
\begin{split}
\min_{\gamma \in \real^n}\ & \|\gamma\|_1 \\
\mbox{subject to } & \begin{bmatrix}y \\ \ytil\end{bmatrix} = \begin{bmatrix}X \\ \Xtil\end{bmatrix} \left(X^\top X + \Xtil^\top \Xtil \right)^{-1} \left(X^\top (y-\gamma)+ \Xtil^\top \ytil\right) + \begin{bmatrix}\gamma \\ \vec{0} \end{bmatrix}.
\end{split}
\end{align}

Before presenting our results for support recovery, we introduce some definitions. 
Define the cone set $\mathbb{C}(K)$ for some subset $K \subseteq [n]$ and $|K| = t$:
\begin{align}\label{eq:cone}
\mathbb{C}(K) := \left\{\Delta \in \real^n: \|\Delta_{K^c}\|_1 \leq \|\Delta_K\|_1\right\}.
\end{align}
Further let $\mathbb{C}^A = \cup_{K \subseteq [n], |K| = t} \mathbb{C}(K)$, and define
\begin{equation*}
\Pbar(D) = \begin{bmatrix}I - X\left(X^\top X +  X_D^\top X_D \right)^{-1}X^\top \\ X_D\left(X^\top X + X_D^\top X_D \right)^{-1}X^\top \end{bmatrix}.
\end{equation*}

\begin{theorem}\label{thm:active-recover}
Suppose
\begin{equation}
\label{EqnRSN}
Null(\Pbar(D)) \cap \mathbb{C}^A = \{\vec{0}\}.
\end{equation}
Then a debugger who queries the points indexed by $D$ cannot be beaten by any bug generator who introduces at most $t$ bugs.
\end{theorem}

Theorem~\ref{thm:active-recover} suggests that equation~\eqref{EqnRSN} is a sufficient condition for support recovery for an omnipotent bug generator who knows the subset $D$. As a debugger, the consequent goal is to find such a subset $D$ which makes equation~\eqref{EqnRSN} true. Whether such a $D$ exists and how to find it will be discussed in Section~\ref{sec:optimal-debugger}. 

\begin{remark}
When $m = n$, we can verify that $Null(\Pbar(D)) = \{\vec{0}\}$, which implies that equation~\eqref{EqnRSN} always holds. Indeed, in this case, we can simply take $\Xtil = X$ and solve for $\betastar$ explicitly to recover $\gammastar$.
\end{remark}

\begin{remark}\label{remark:sign-thm-active}
As stated in Theorem~\ref{thm:active-recover}, equation~\eqref{EqnRSN} is a sufficient condition for support recovery. In fact, it is an if-and-only-if condition for signed support recovery: When equation~\eqref{EqnRSN} holds, $\sign(\gammahat) = \sign(\gammastar)$; and when it does not hold, the bug generator can find a $\gammastar$ with $\supp(\gammastar) \le t$ such that $\sign(\gammahat) \neq \sign(\gammastar)$.
\end{remark}

\begin{remark}
\label{remark:active-condition}
We can also write $Null(\Pbar(D))$ as
\begin{equation*}
\{u \in \real^n  \mid \exists v \in \real^p, \mbox{ s.t. } u = Xv, X_Dv = 0\}.
\end{equation*}
Let $\betahat = \betastar + v $ for some vector $v \in \real^p$. From the constraint-based algorithm, we obtain 
\begin{align*}
y_T & = X_T(\betastar + v) + \gammahat_T, \\
y_{T^c} & = X_{T^c} (\betastar + v) + \gammahat_{T^c},\ y_D = X_D(\betastar + v),
\end{align*}
which implies that $\gammahat_T = \gammastar_T - X_Tv$ and $\gammahat_{T^c} = -X_{T^c} v ,\ X_D v = 0$. Let $u = Xv$. Then we obtain $\gammahat = \gammastar - u$.
As can be seen, equation~\eqref{EqnRSN} requires that $u = \vec{0}$, which essentially implies $\gammahat = \gammastar$, and thus $\supp(\gammahat) = \supp(\gammastar)$.
\end{remark}

%%%%%%%%%%%%%

\subsection{Optimal Debugger via MILP}
\label{sec:optimal-debugger}

The above analysis is also useful in practice for providing a method for designing $\Xtil$. Consider the following optimization problem:

\begin{subequations}
\label{eq:bug-generator}
\begin{equation}
\label{eq:bug-generator-objective} 
\max_{K \subseteq [n], |K| \leq t, u \in \real^n, v \in \real^d} \|u_K\|_1 - \|u_{K^c}\|_1, 
\end{equation}
\begin{equation}\label{eq:bug-generator-constraint}
\mbox{subject to } u = Xv, X_D v = 0, \|u\|_\infty \leq 1. 
\end{equation}
\end{subequations}

If the problem~\eqref{eq:bug-generator} has the unique solution $(u,v) = (\vec{0}, \vec{0})$, then a debugger who queries the points indexed by $D$ cannot be beaten by a bug generator who introduces at most $t$ bugs.

Based on this argument, we can construct a bilevel optimization problem for the debugger to solve by further minimizing the objective~\eqref{eq:bug-generator-objective} with respect to $D \subseteq [n]$ such that $|D| \le m$.
The optimization problem can then be transformed into a minimax MILP: 
\begin{align}\label{eq:milp}
\begin{split}
\min_{\xi \in \{0,1\}^n} & \max_{\substack{a,a^+,a^-\in \real^n, \\ u, u^+, u^- \in \real^n, v \in \real^d,\\ z,w \in \{0,1\}^n}}  \sum_{j=1}^n a_j^+ - a_j^-, \\
\mbox{subject to } & \Big\{ u = Xv, u = u^+ - u^-, u^+, u^- \geq 0, \\
& \quad a = u^+ + u^-,  u^+ \leq z, \ u^- \leq (\mathds{1}_n-z), \\
& \quad \sum_{i=1}^n w_i \leq t, a^+ \leq M w, \ a^- \leq M(\mathds{1}_n-w),  \\
& \quad a = a^+ + a^-, a^+ \geq 0, a^- \geq 0, \\
& \quad \sum_{i=1}^n \xi_i \leq m, u \leq (\mathds{1}_n - \xi), u \geq - (\mathds{1}_n - \xi).\Big\}
\end{split}
\end{align} 

%%%%%%%%%%%%%%%%%%% PUT THIS INTO THE PROOF %%%%%%%%%%%%%%%%%%%%%

\begin{theorem}[MILP for debugging]\label{thm:milp}
If the optimization problem~\eqref{eq:milp} has the unique solution $(u,v) = (\vec{0}, \vec{0})$, then the debugger can add $m$ points indexed by $D = \supp(\xi)$ to achieve support recovery.
\end{theorem}

\begin{remark}
For more information on efficient algorithms for optimizing minimax MILPs, we refer the reader to the references~\cite{tang2016class},~\cite{xu2014exact}, and~\cite{zeng2014solving}.
\end{remark}

\section{EXPERIMENTS}
\label{sec:exp}
In this section, we empirically validate our Lasso-based debugging method for support recovery. The section is organized as follows:
\begin{itemize}
\item Subsection~\ref{sec:exp-gaussian-design}, corresponding to Section~\ref{sec:supprec}, contains a number of experiments which investigate the performance of our proposed debugging formulation.
\item Subsection~\ref{sec:experiments-tuning-lambda}, corresponding to Section~\ref{sec:tune}, studies the proposed tuning parameter selection procedure.
\item Subsection~\ref{sec:experiments-clean-points} studies the Lasso-based debugging method with a clean data pool, including the proposed MILP algorithm from Section~\ref{sec:active}.
\end{itemize}
We also compare our proposed method to alternative methods motivated by existing literature.

We begin with an outline of the experimental settings used in most of our experiments:
\begin{enumerate}
\item[S1] Generate the feature design matrix $X \in \real^{n \times p}$ by sampling each row i.i.d.\ from $\mathcal{N}(\vec{0}_p, I_{p\times p})$.
\item[S2] Generate $\betastar \in \real^p$, where each entry $\betastar_i$ is drawn i.i.d.\ from $Unif(-1,1)$.
\item[S3] Generate $\epsilon \in \real^n$, where each entry $\epsilon_i$ is drawn i.i.d.\ from $\mathcal{N}(0,\sigma^2)$. 
\item[S4]\label{step4} Generate the bug vector $\gammastar \in \real^n$, where we draw $\gammastar_i = (10\sqrt{\log(2n)}\sigma + Unif(0,10))\cdot Bernoulli(\pm 1, 0.5)$ for $i \in [t]$ and take $\gammastar_i = 0$ for the remaining positions. 
\item[S5] Generate the labels by $y = X\betastar + \epsilon + \gammastar$.
\end{enumerate}
These five steps produce a synthetic dataset $(X,y)$; we will specify the particular parameters $(n,p,t, \sigma)$ in each task. If we use a real dataset, the first step changes to [S1']:
\begin{enumerate}
\item[S1'] Given the whole data pool $X_{real}$, uniformly sample $n$ data points from it to construct $X$.
\end{enumerate}

In the plot legends, we will refer to our Lasso-based debugging method as ``debugging." We may also invoke a postprocessing step on top of debugging, called ``debugging + postprocess," which first runs the Lasso optimization algorithm to obtain $\gammahat$ and an estimated support set $\hat{T}$, then removes the points $(X_{\hat{T},\cdot}, y_{\hat{T}})$ and runs ordinary least squares on the remaining points to obtain $\betahat$.

\subsection{Support Recovery}
\label{sec:exp-gaussian-design}
In this section, we design two experiments. The first experiment investigates the influence of the fraction of bugs $c_t:=\frac{t}{n}$ on the three assumptions imposed in our theory and the resulting recovery rates. We will vary the design of $X$ using different datasets. 
The second experiment compares debugging with four alternative regression methods, using the precision-recall metric. Note that we will take the tuning parameter $\lambda = 2\frac{\sqrt{\log 2(n-t)}}{n}$ for these experiments, since the other outlier detection methods we use for comparison do not propose a way to perform parameter tuning. We will explore the performance of the proposed algorithm for parameter selection in the next subsection.

\subsubsection{Number of Bugs vs.\ Different Measurements}
Our first experiment involves four different datasets with different values of $n$ and $c_t$. We track the performance of the three assumptions (Assumptions~\ref{C1}--\ref{C3}) and the subset/exact recovery rates, which measure the fraction of experiments which result in subset/exact recovery. The first dataset is generated using the synthetic mechanism described at the beginning of Section~\ref{sec:exp}, with $p=15$. The other three datasets are chosen from the UCI Machine Learning Repository: Combined Cycle Power Plant\footnote{\url{http://archive.ics.uci.edu/ml/datasets/Combined+Cycle+Power+Plant}}, temperature forecast\footnote{\url{http://archive.ics.uci.edu/ml/datasets/Bias+correction+of+numerical+prediction+model+temperature+forecast}}, and YearPredictionMSD\footnote{\url{http://archive.ics.uci.edu/ml/datasets/YearPredictionMSD}}. They are all associated to regression tasks, with varying feature dimensions (4, 21, and 90, respectively). In the temperature forecast dataset, we remove the attribute of station and date from the original dataset, since they are discrete objects. For each of the UCI datasets, after randomly picking $n$ data points from the entire data pool, we normalize the subsampled dataset according to $X_{\cdot,j} = \frac{X_{\cdot,j} - \frac{1}{n}\sum_{i \in [n]} X_{i,j}}{std[X_{\cdot,j} ]}$, where std represents the standard deviation. 

\begin{figure}[htp!]
\centering
\begin{subfigure}[b]{0.32\textwidth}
\centering
\includegraphics[width=\textwidth]{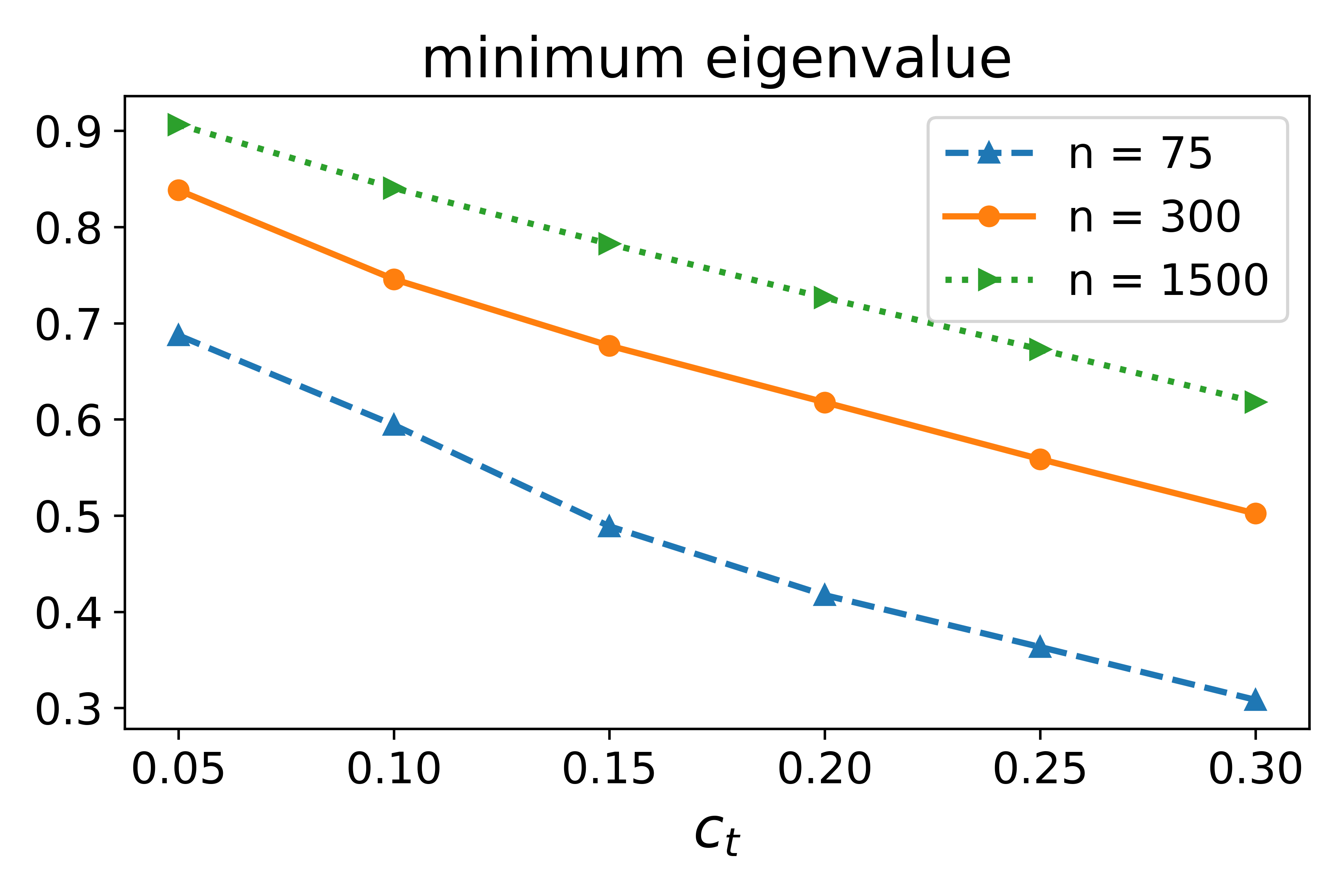}
\caption{Synthetic dataset}
\label{fig:conditions_synthetic_mineig}
\end{subfigure}
~
\begin{subfigure}[b]{0.32\textwidth}
\centering
\includegraphics[width=\textwidth]{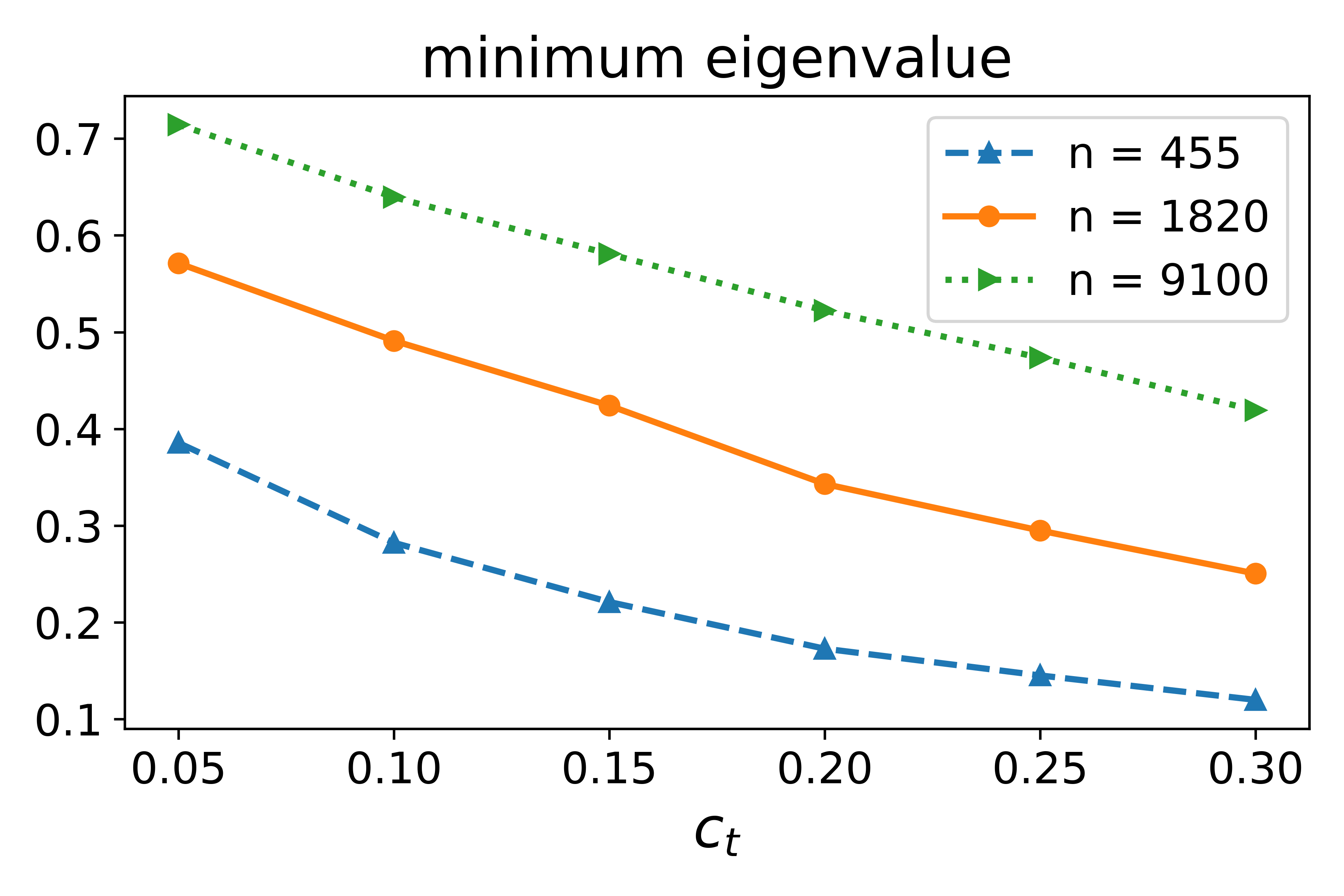}
\caption{YearPredictionMSD dataset}
\label{fig:conditions_msd_mineig}
\end{subfigure}
\vskip\baselineskip
\begin{subfigure}[b]{0.32\textwidth}
\centering
\includegraphics[width=\textwidth]{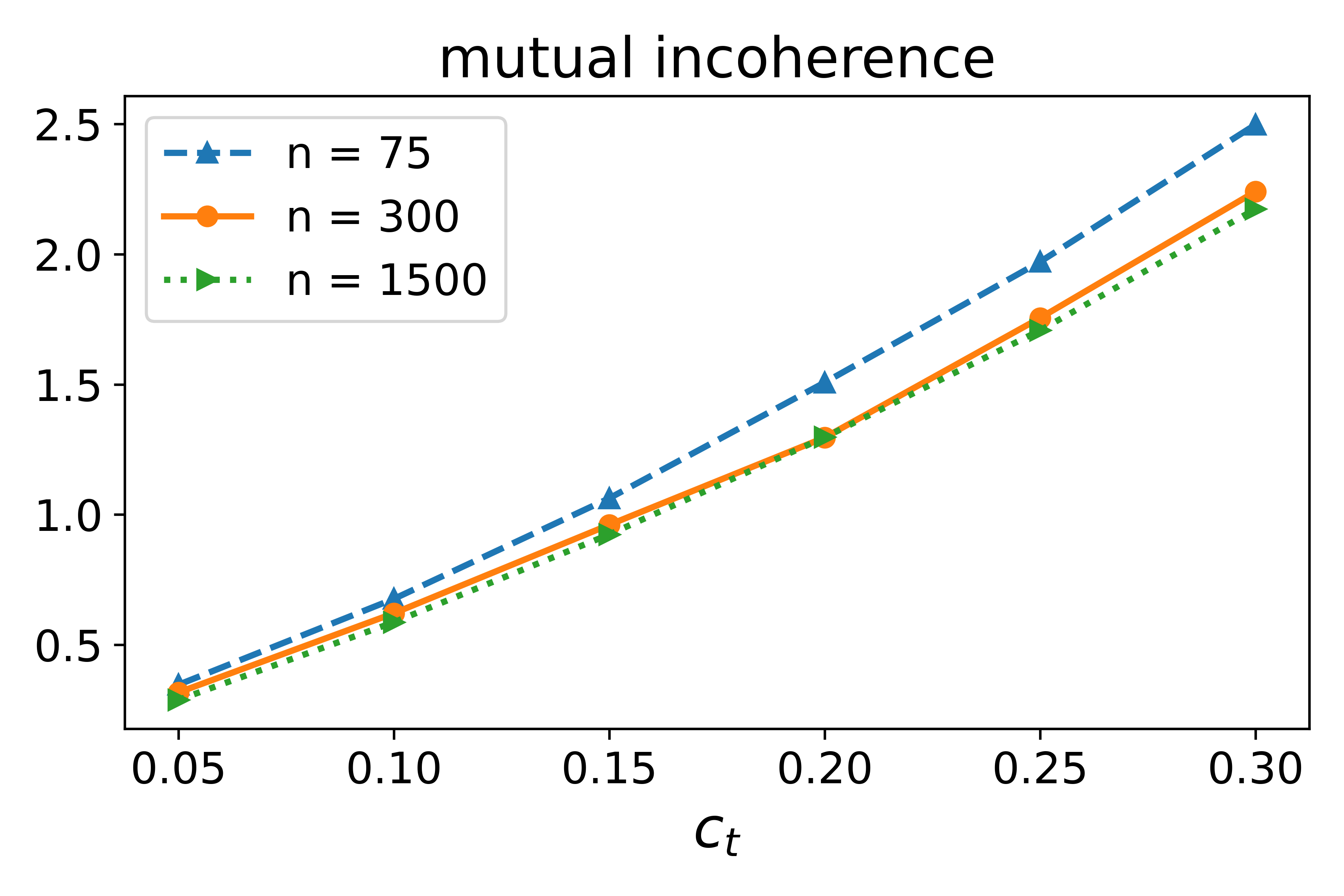}
\caption{Synthetic dataset}
\label{fig:conditions_synthetic_mutinc}
\end{subfigure}
~
\begin{subfigure}[b]{0.32\textwidth}
\centering
\includegraphics[width=\textwidth]{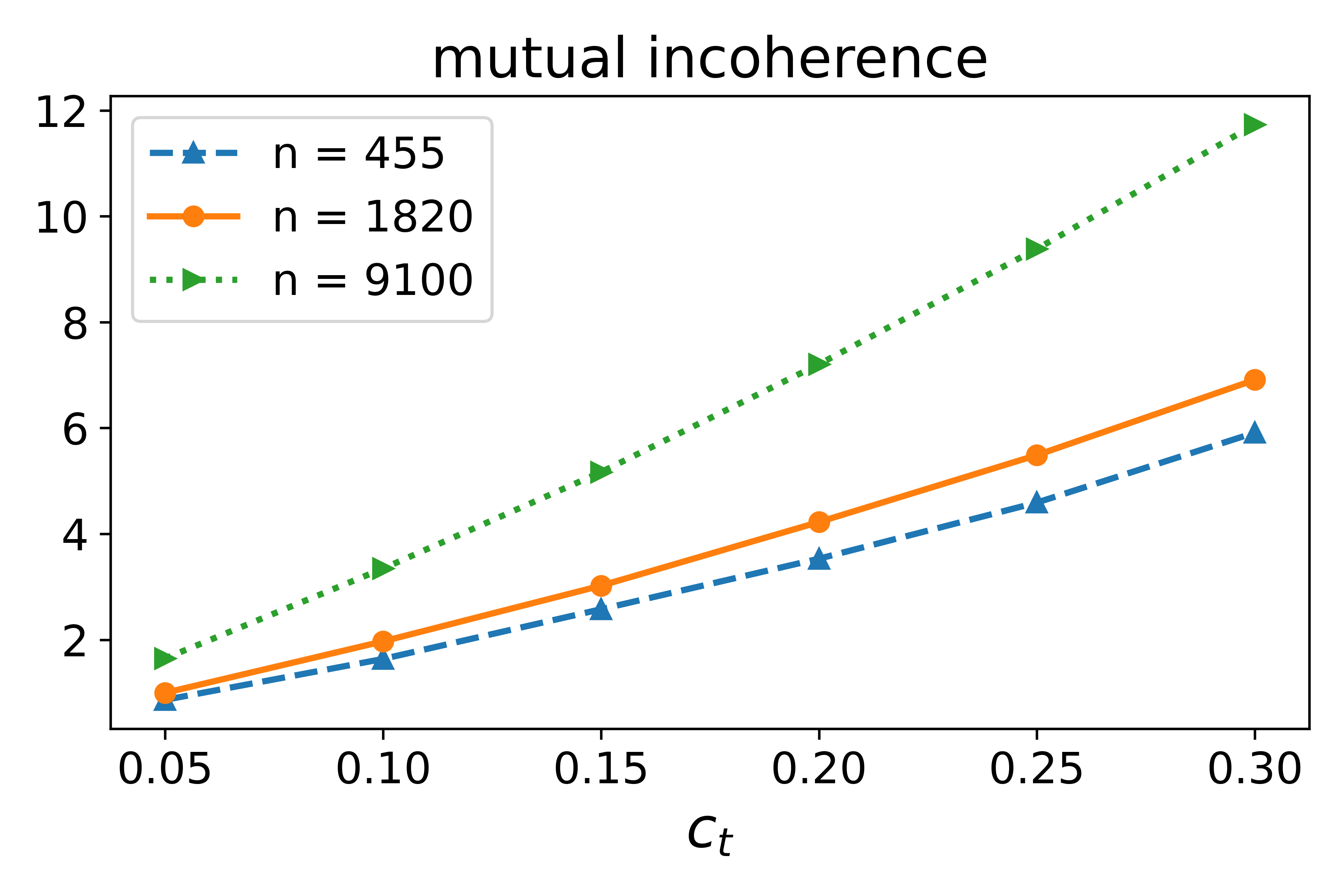}
\caption{YearPredictionMSD dataset}
\label{fig:conditions_msd_mutinc}
\end{subfigure}
\vskip\baselineskip
\begin{subfigure}[b]{0.32\textwidth}
\centering
\includegraphics[width=\textwidth]{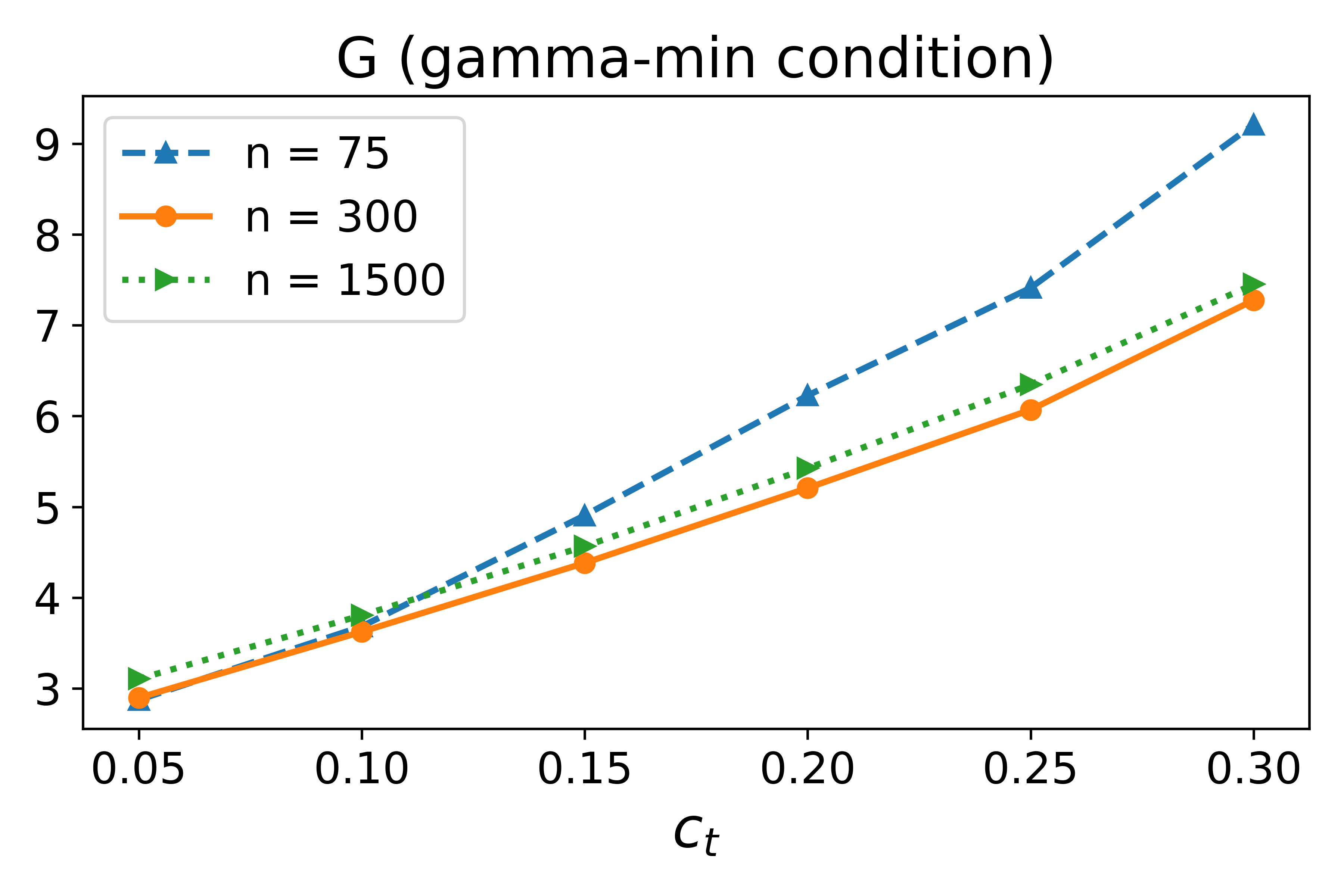}
\caption{Synthetic dataset}
\label{fig:conditions_synthetic_mingammaG}
\end{subfigure}
~
\begin{subfigure}[b]{0.32\textwidth}
\centering
\includegraphics[width=\textwidth]{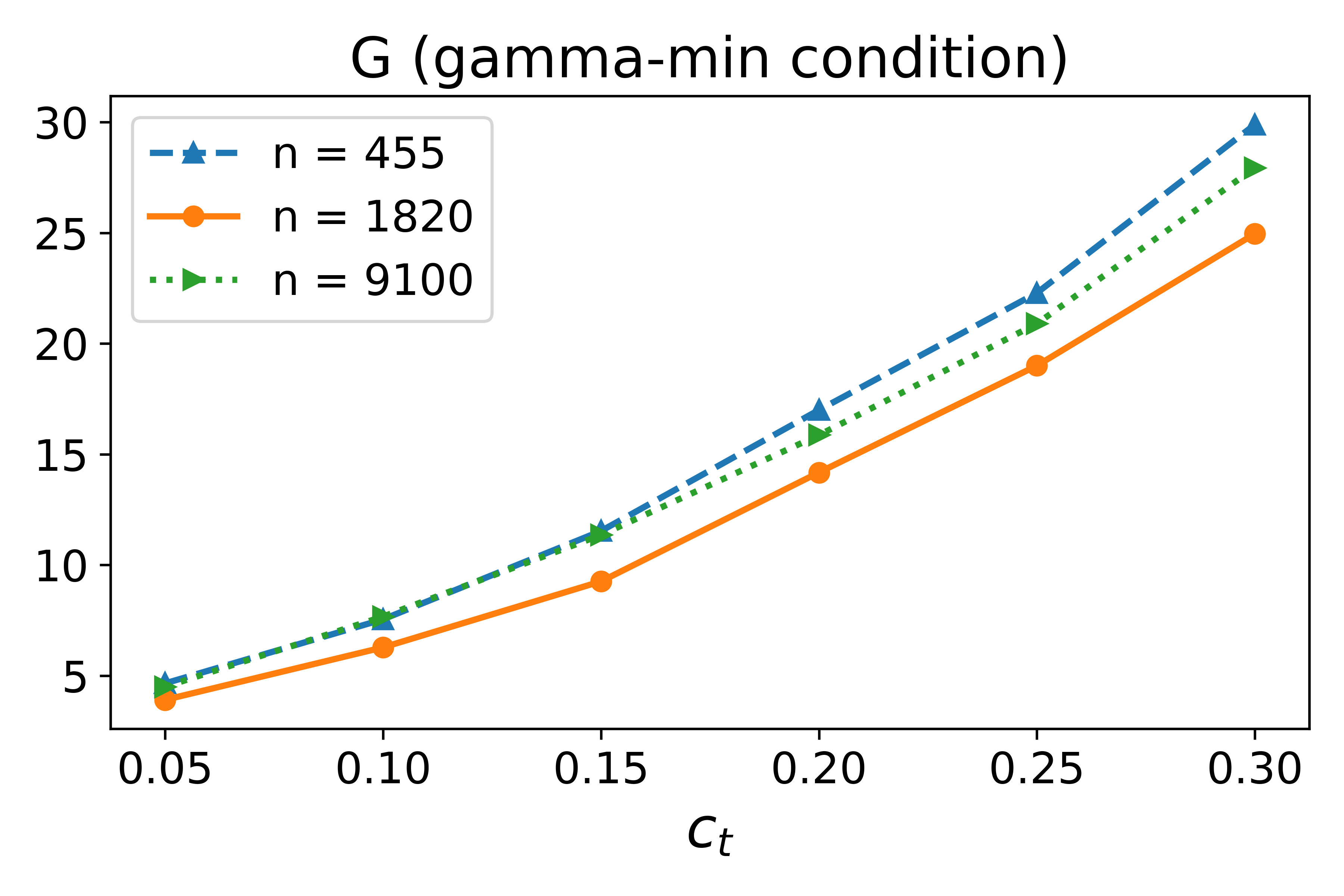}
\caption{YearPredictionMSD dataset}
\label{fig:conditions_msd_mingammaG}
\end{subfigure}
\vskip\baselineskip
\begin{subfigure}[b]{0.32\textwidth}
\centering
\includegraphics[width=\textwidth]{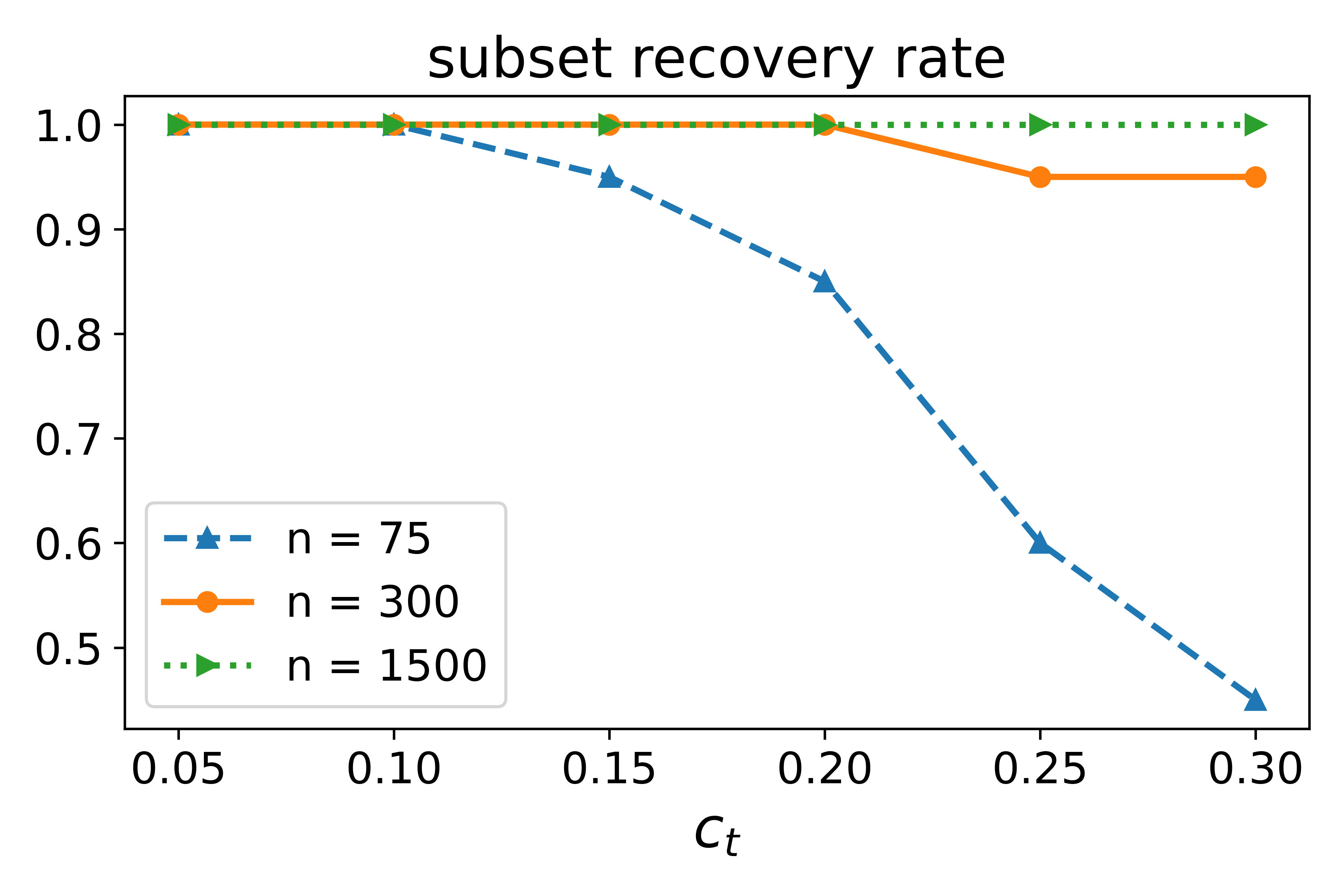}
\caption{Synthetic dataset}
\label{fig:conditions_synthetic_subrec}
\end{subfigure}
~
\begin{subfigure}[b]{0.32\textwidth}
\centering
\includegraphics[width=\textwidth]{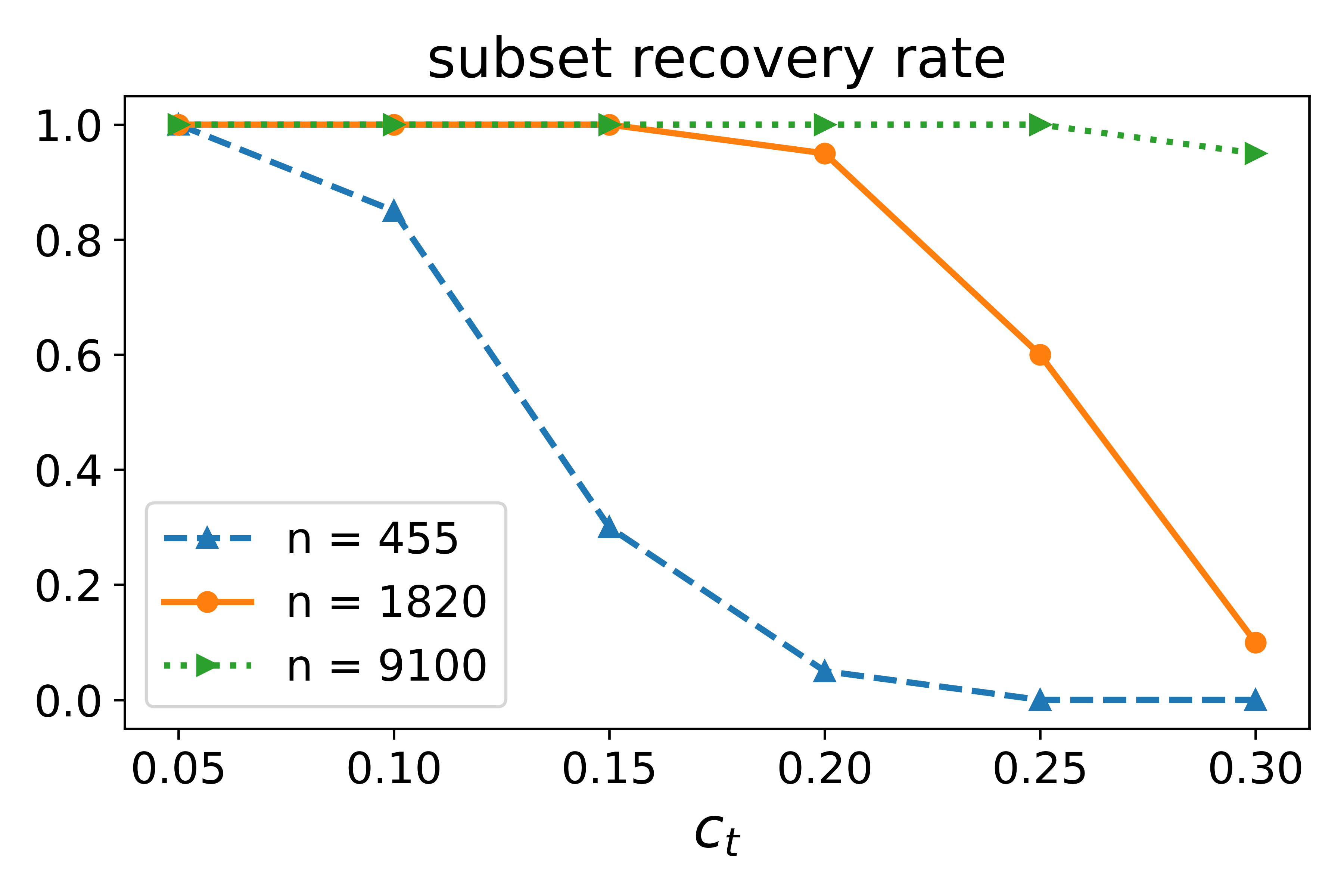}
\caption{YearPredictionMSD dataset}
\label{fig:conditions_msd_subrec}
\end{subfigure}
\vskip\baselineskip
\begin{subfigure}[b]{0.32\textwidth}
\centering
\includegraphics[width=\textwidth]{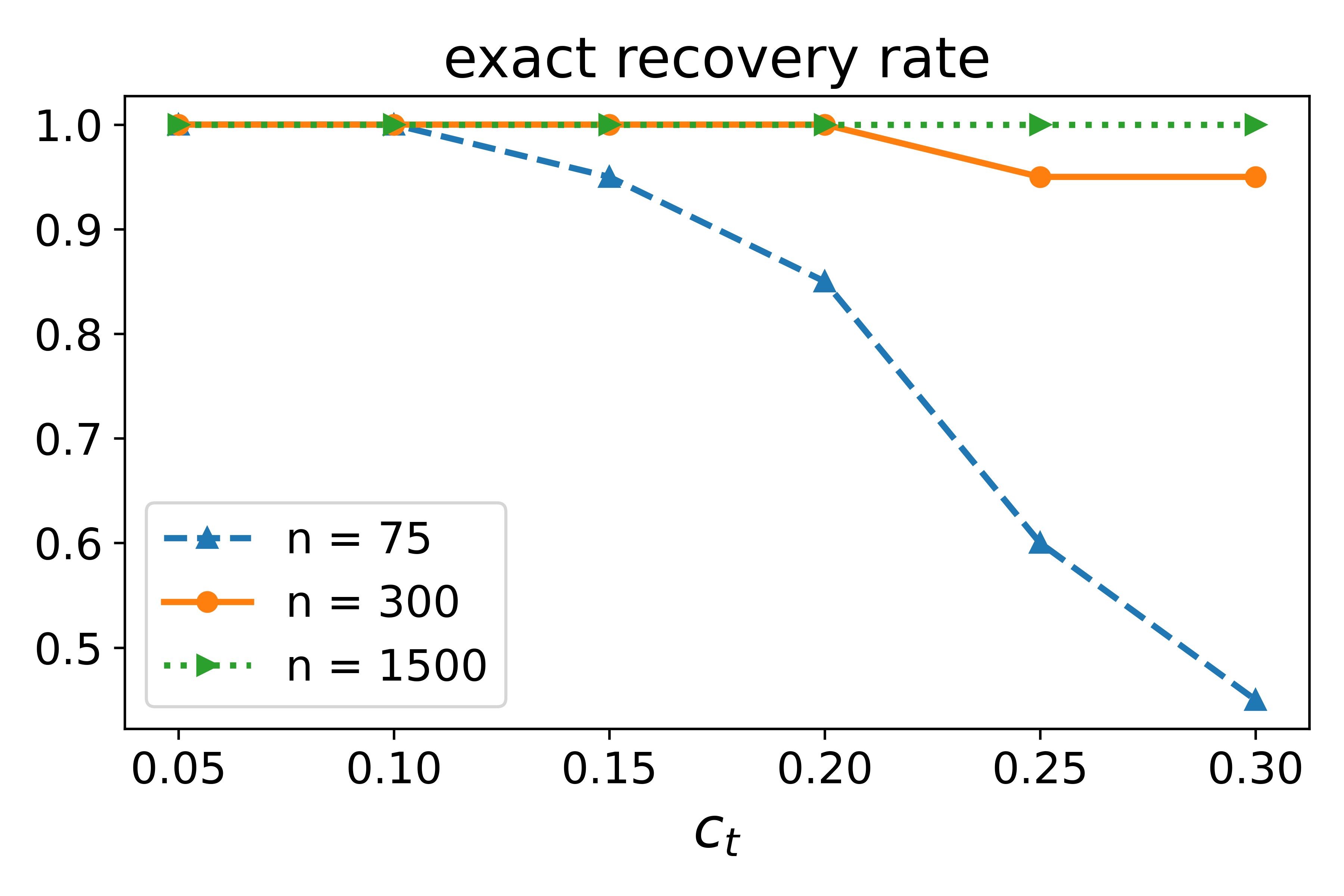}
\caption{Synthetic dataset}
\label{fig:conditions_synthetic_exactrec}
\end{subfigure}
~
\begin{subfigure}[b]{0.32\textwidth}
\centering
\includegraphics[width=\textwidth]{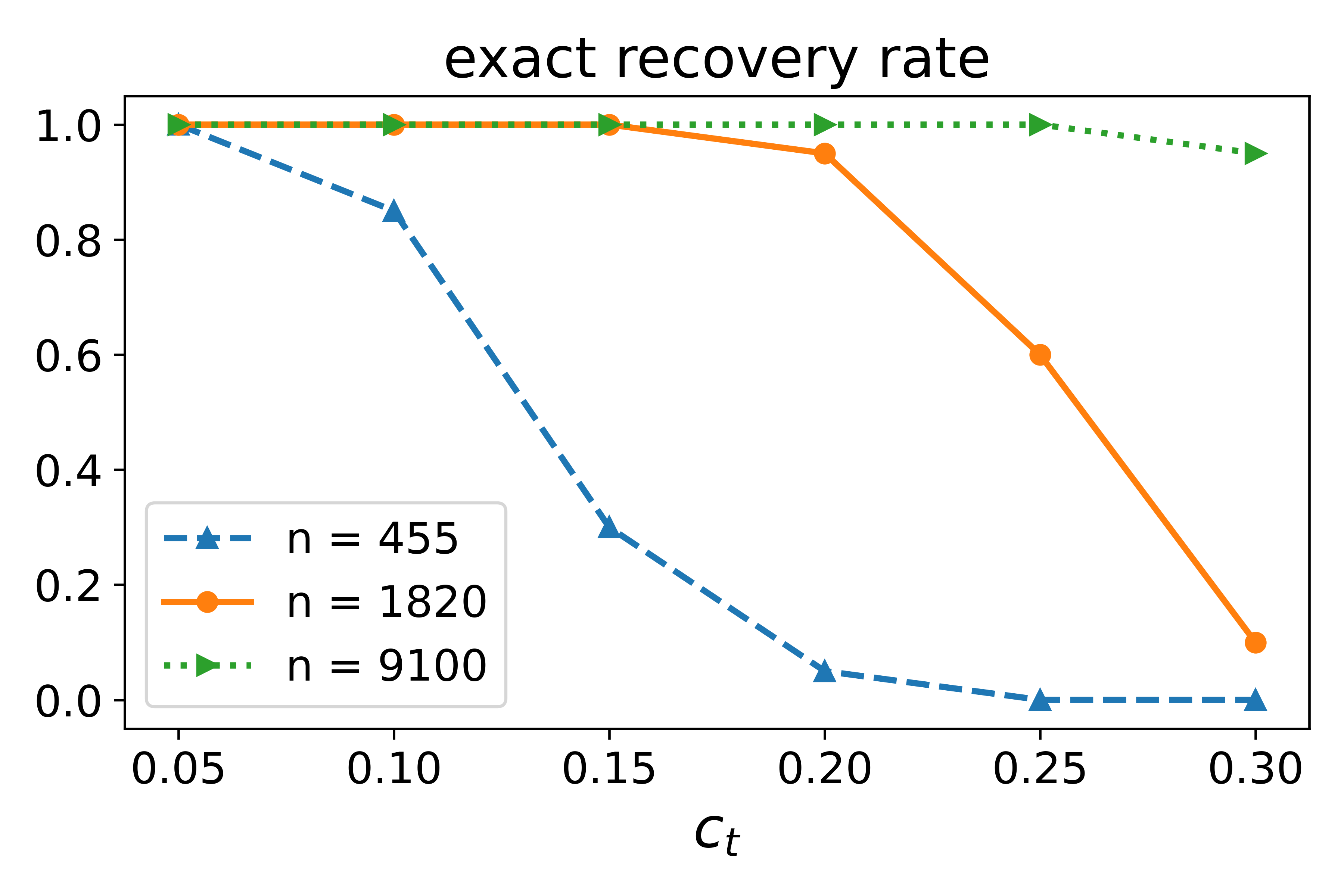}
\caption{YearPredictionMSD dataset}
\label{fig:conditions_synthetic_exactrec}
\end{subfigure}
\end{figure}%
%%%%%%%%%%%%%%%%%%%%%%%%%%%%%%%%%%%%%%%%%%%%%%%%
\begin{figure}[htp!]\ContinuedFloat
\centering
\begin{subfigure}[b]{0.32\textwidth}
\centering
\includegraphics[width=\textwidth]{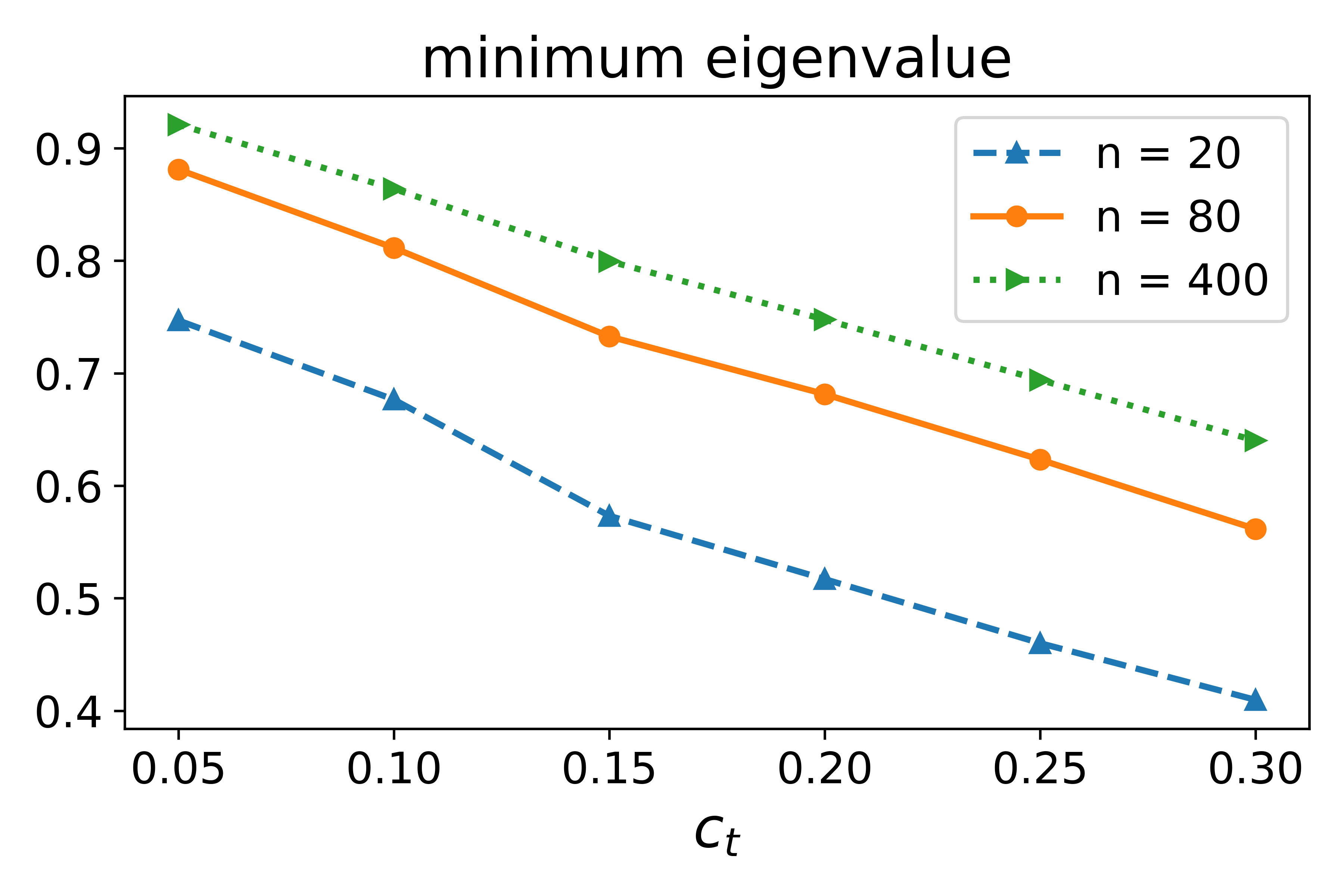}
\caption{Combined Cycle Power Plant dataset}
\label{fig:conditions_ccpp_mineig}
\end{subfigure}
~
\begin{subfigure}[b]{0.32\textwidth}
\centering
\includegraphics[width=\textwidth]{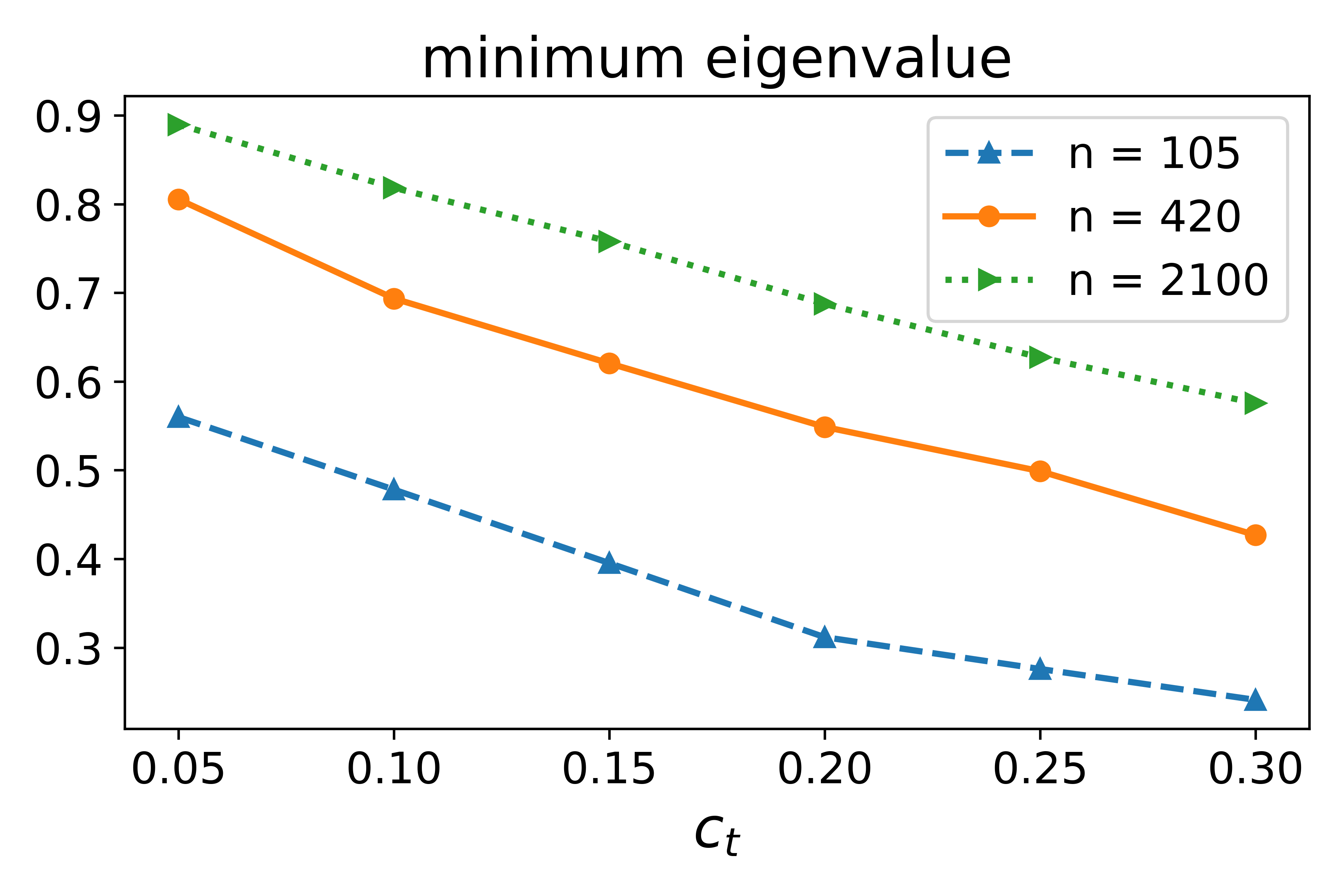}
\caption{Temperature forecast dataset}
\label{fig:conditions_bias_mineig}
\end{subfigure}
\vskip\baselineskip
\begin{subfigure}[b]{0.32\textwidth}
\centering
\includegraphics[width=\textwidth]{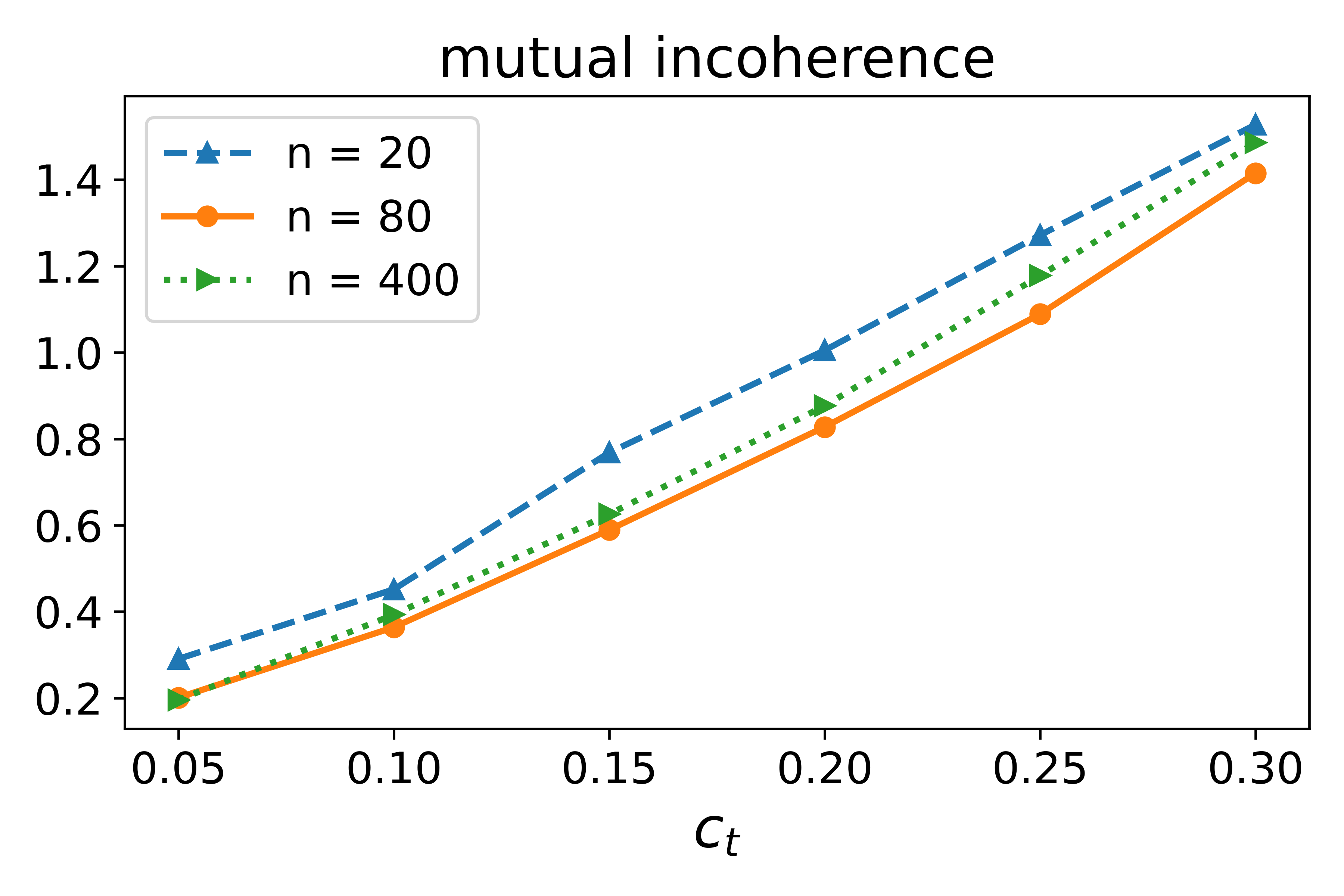}
\caption{Combined Cycle Power Plant dataset}
\label{fig:conditions_ccpp_mutinc}
\end{subfigure}
~
\begin{subfigure}[b]{0.32\textwidth}
\centering
\includegraphics[width=\textwidth]{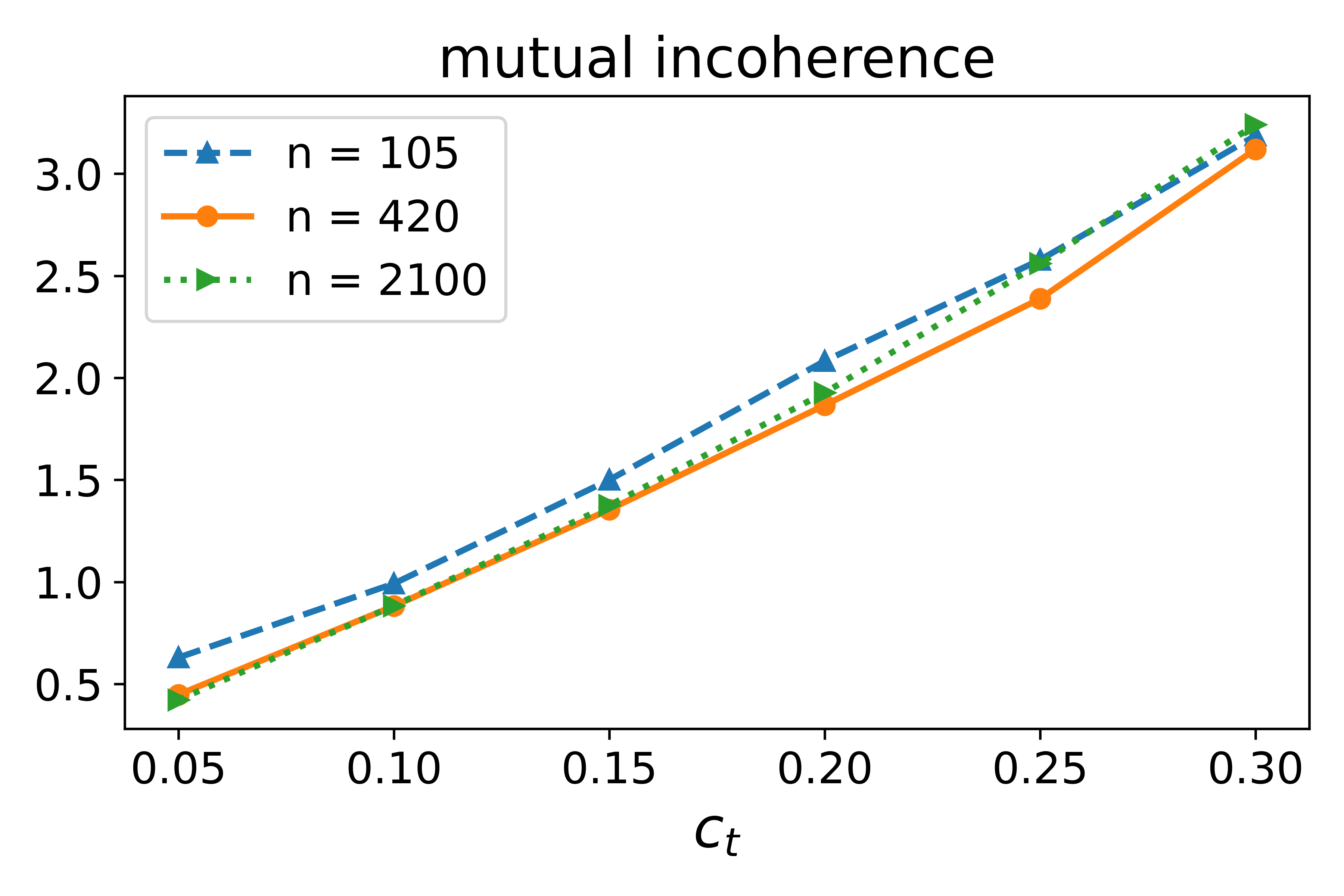}
\caption{Temperature forecast dataset}
\label{fig:conditions_bias_mutinc}
\end{subfigure}
\vskip\baselineskip
\begin{subfigure}[b]{0.32\textwidth}
\centering
\includegraphics[width=\textwidth]{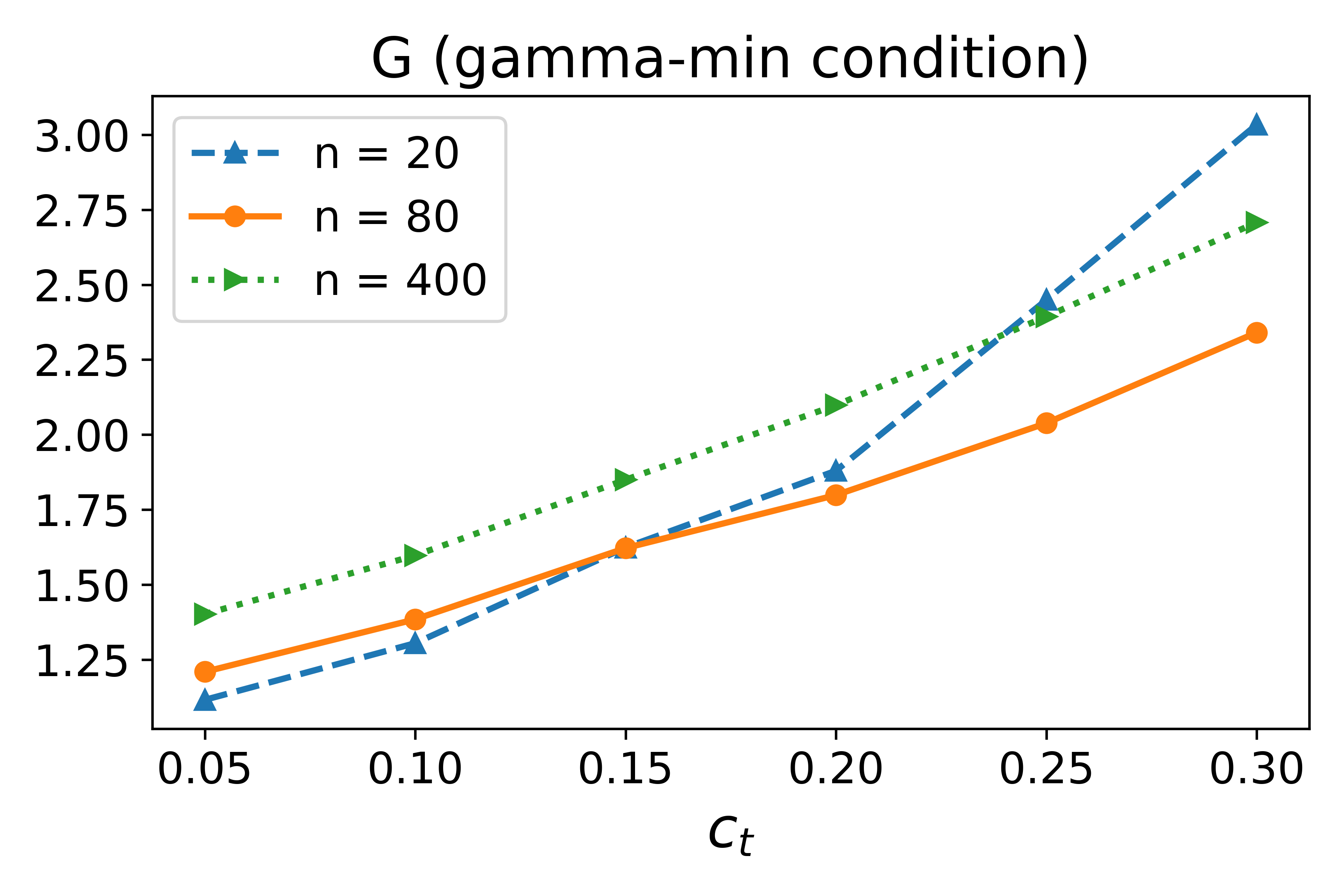}
\caption{Combined Cycle Power Plant dataset}
\label{fig:conditions_ccpp_mingammaG}
\end{subfigure}
~
\begin{subfigure}[b]{0.32\textwidth}
\centering
\includegraphics[width=\textwidth]{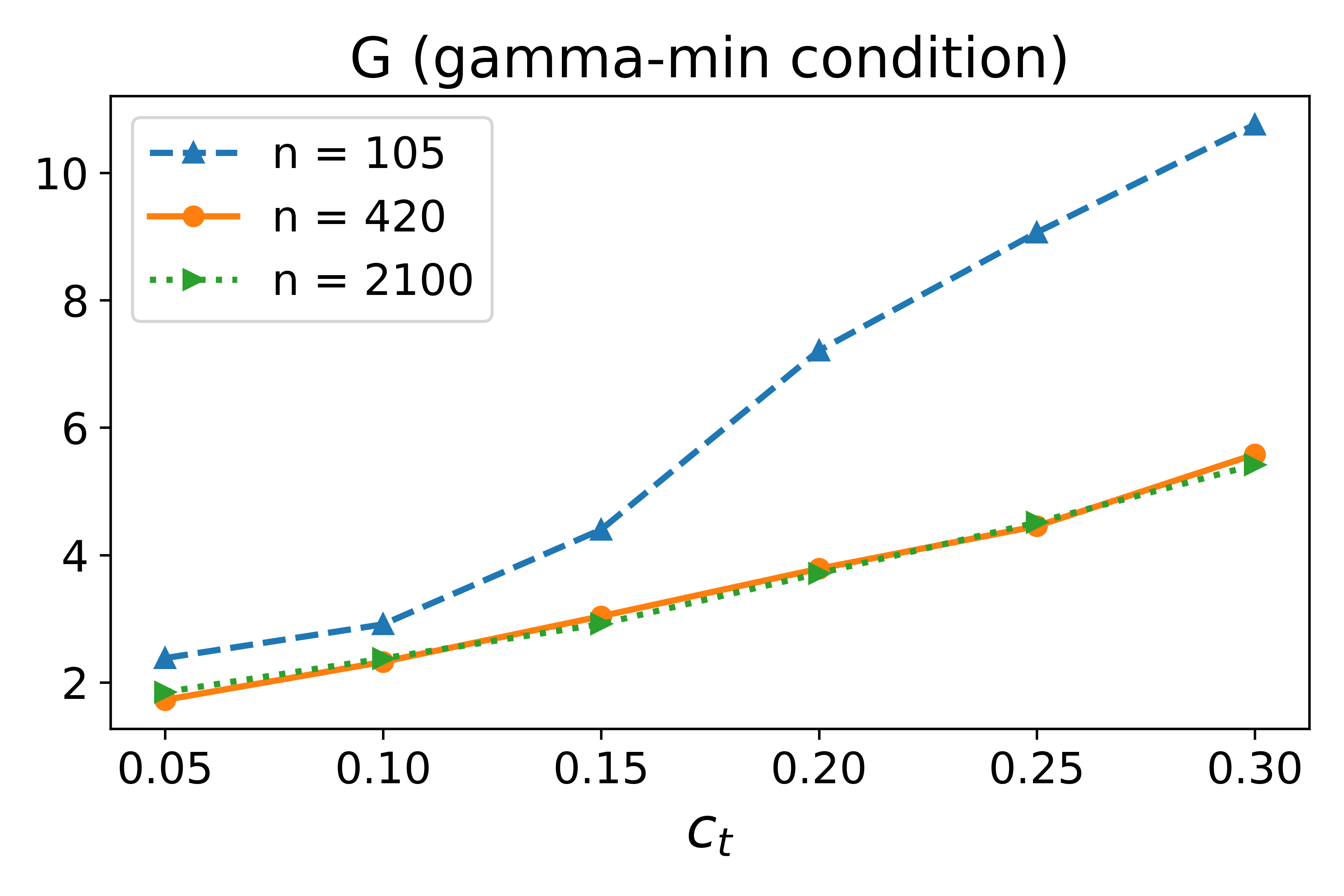}
\caption{Temperature forecast dataset}
\label{fig:conditions_synthetic_mingammaG}
\end{subfigure}
\vskip\baselineskip
\begin{subfigure}[b]{0.32\textwidth}
\centering
\includegraphics[width=\textwidth]{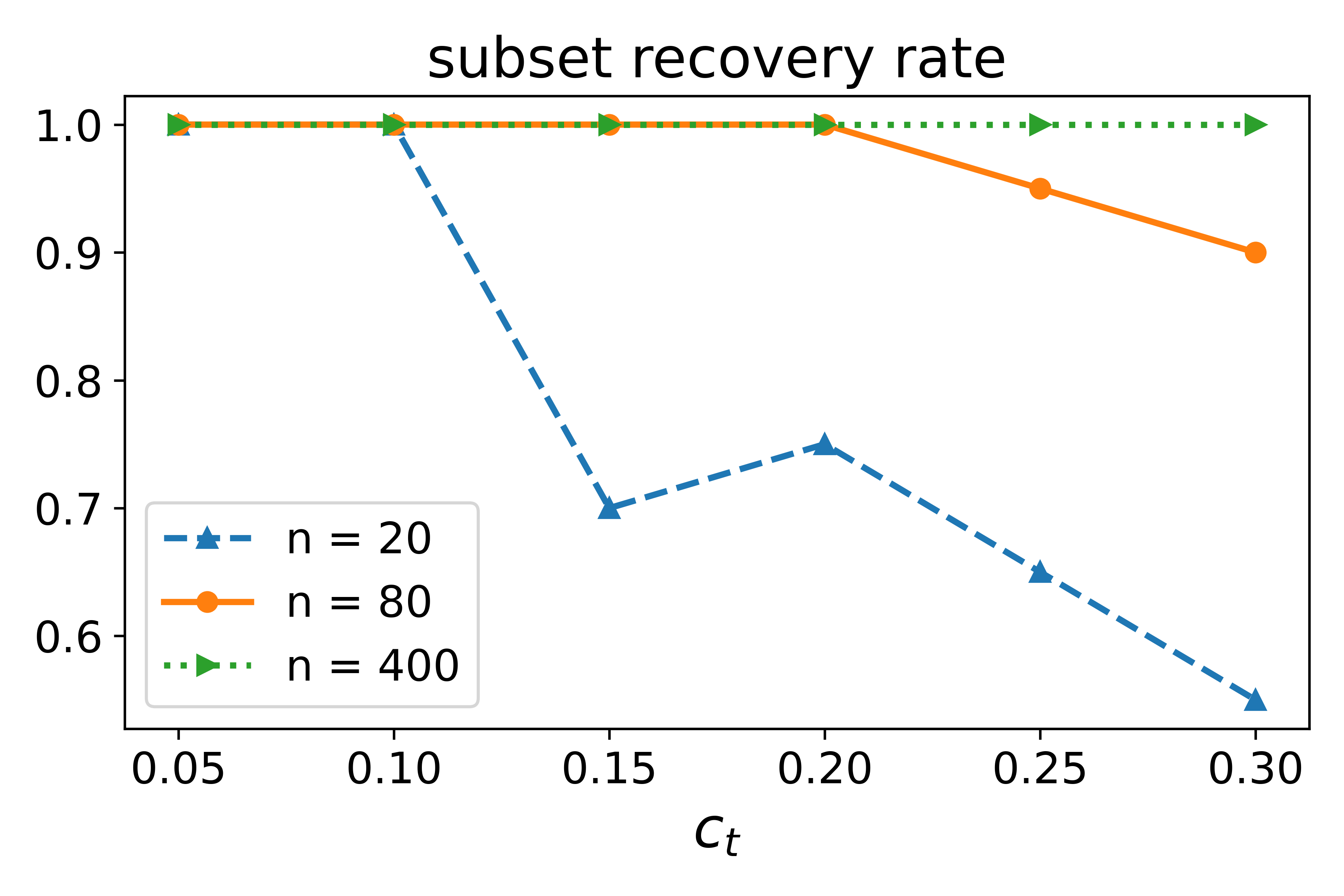}
\caption{Combined Cycle Power Plant dataset}
\label{fig:conditions_ccpp_subrec}
\end{subfigure}
~
\begin{subfigure}[b]{0.32\textwidth}
\centering
\includegraphics[width=\textwidth]{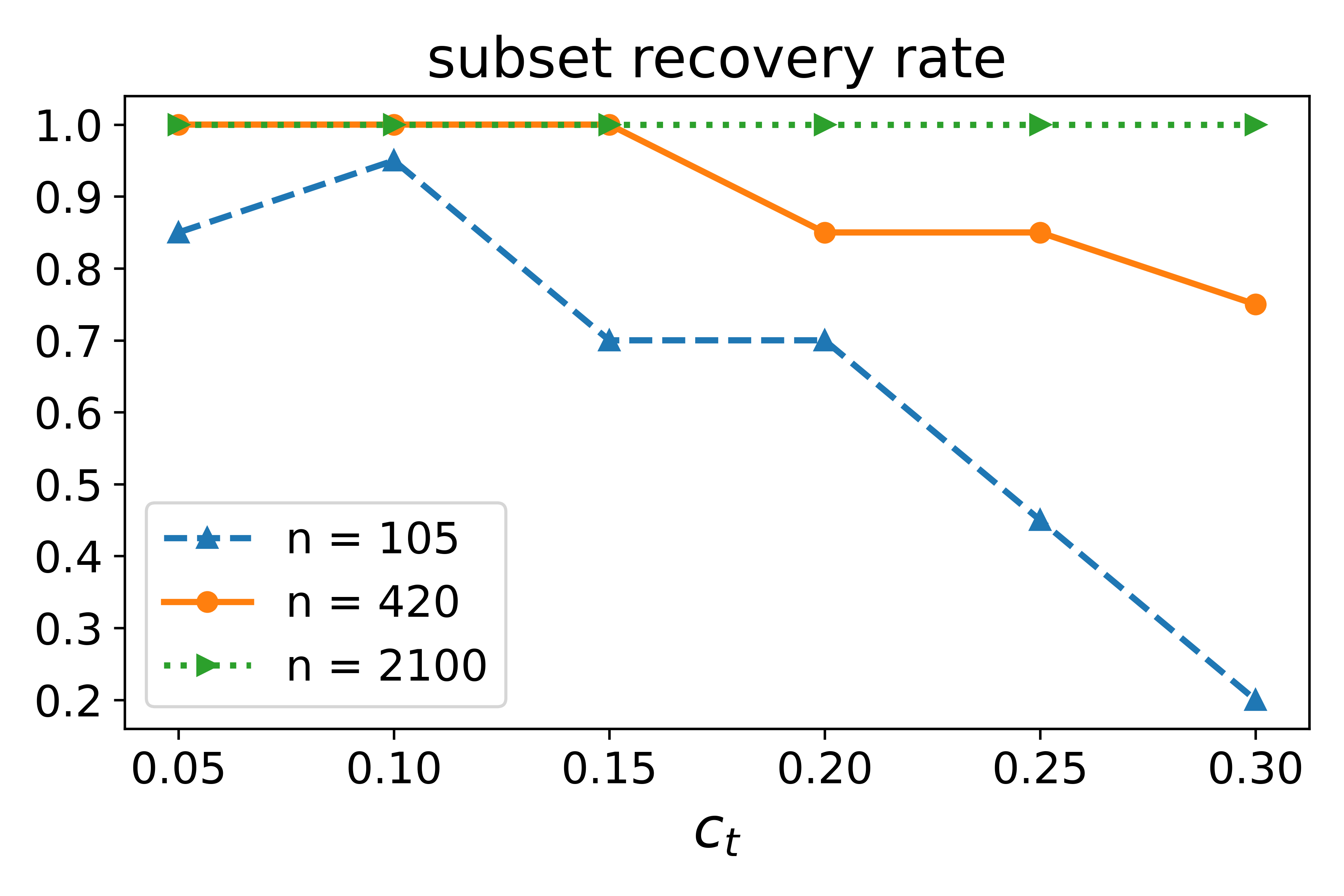}
\caption{Temperature forecast dataset}
\label{fig:conditions_bias_subrec}
\end{subfigure}
\vskip\baselineskip
\begin{subfigure}[b]{0.32\textwidth}
\centering
\includegraphics[width=\textwidth]{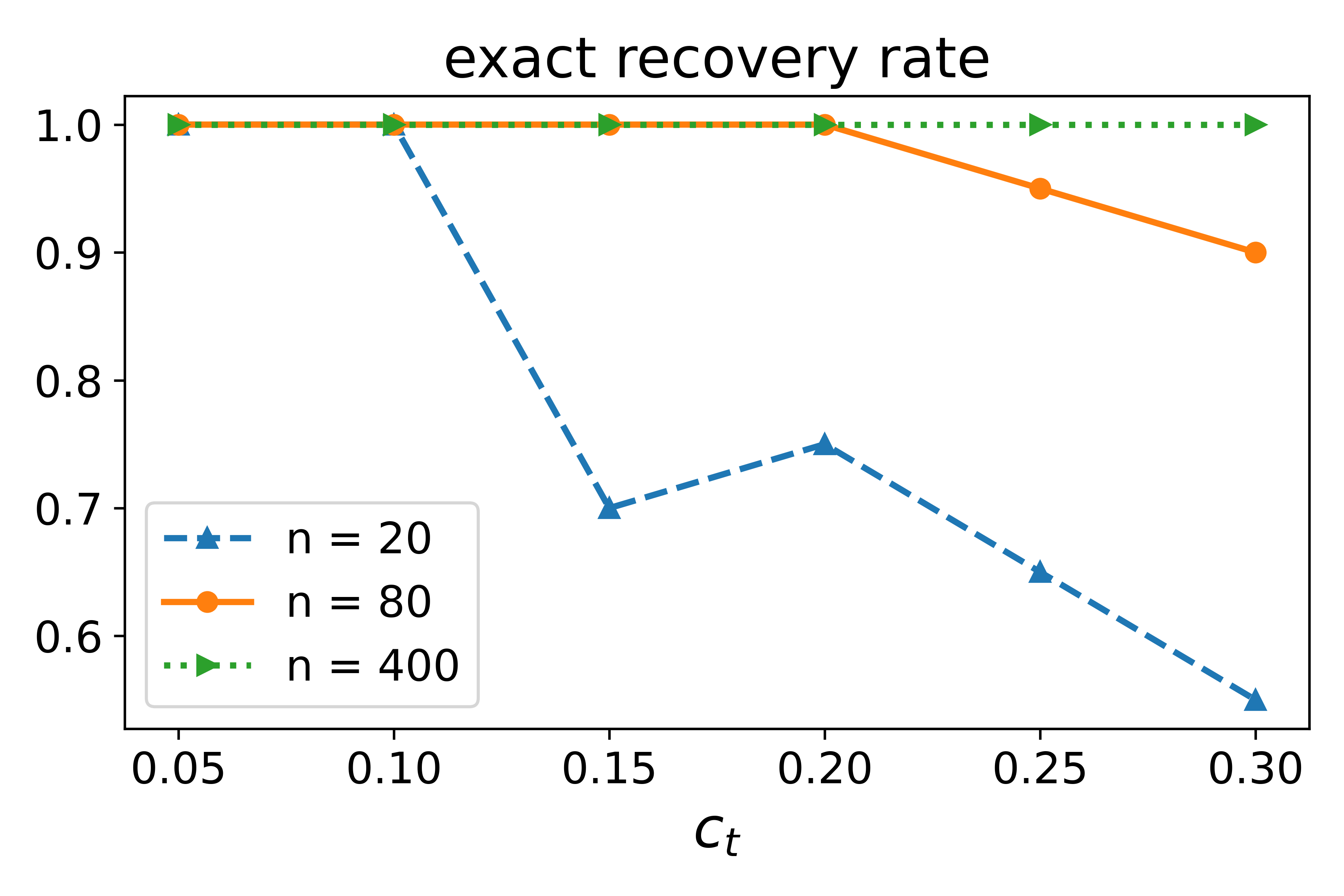}
\caption{Combined Cycle Power Plant dataset}
\label{fig:conditions_ccpp_exactrec}
\end{subfigure}
~
\begin{subfigure}[b]{0.32\textwidth}
\centering
\includegraphics[width=\textwidth]{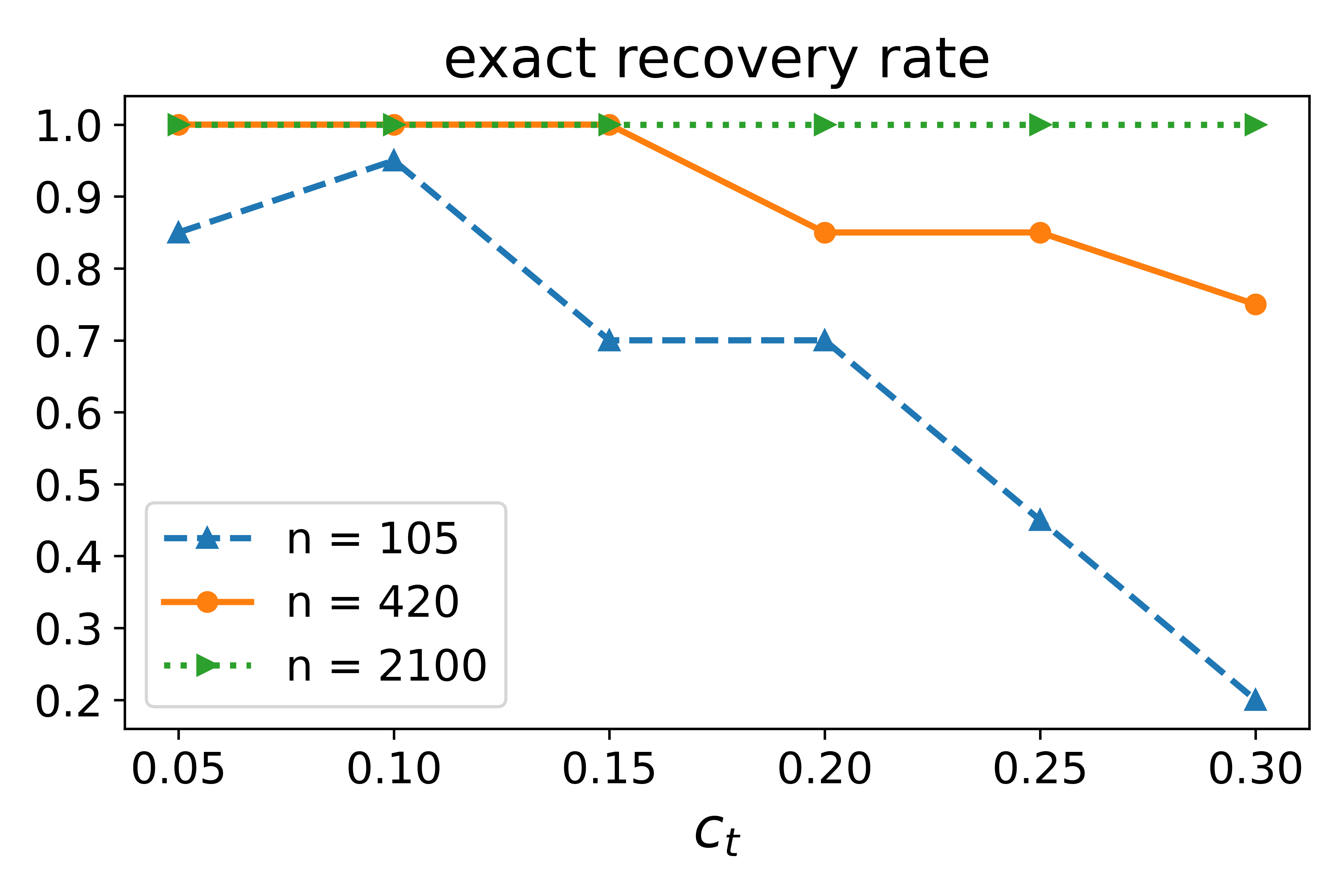}
\caption{Temperature forecast dataset}
\label{fig:conditions_synthetic_exactrec}
\end{subfigure}
\caption{Five Measurements on Four Datasets. Three different $n$'s are of values $5p,\ 20p$, and $100p$. The variance $\sigma$ is set to 0.1. The tuning parameter is set to $\lambda = 2\frac{\sqrt{\log 2(n-t)}}{n}$. Each dot is an average value of 20 random trials.} 
\label{fig:t_versus_assumptions}
\end{figure}
The results are displayed in Figure~\ref{fig:t_versus_assumptions}. For the minimum eigenvalue assumption, a key observation from all datasets is that the minimum eigenvalue becomes larger (improves) as $n$ increases, and becomes smaller as $c_t$ increases. For the mutual incoherence assumption, the synthetic dataset satisfies the condition with less than 15\% outliers. The Combined Cycle Power Plant dataset has mutual incoherence close to 1 when $c_t$ is approximately 20\%-25\%, and the mutual incoherence condition of the YearPredictionMSD dataset approaches 1 when $c_t$ is approximately 5\%. Therefore, we see that the validity of the assumption highly depends on the design of $X$. For the gamma-min condition, as $c_t$ increases, we need more obvious (larger $\min_i|\gammastar_i|$) outliers. Finally, with larger $n$ and smaller $c_t$, the subset/exact recovery rate improves. 

\subsubsection{Effectiveness for Recovery}

The second experiment compares our debugging method to other proposed methods in the robust statistics literature. We compare our method with the Fast LTS~\cite{rousseeuw2006computing}, E-lasso~\cite{nguyen2013robust}, Simplified $\Theta$-IPOD~\cite{she2011outlier}, and Least Squares methods. E-lasso is similar to our formulation, except it includes an additional penalty with $\beta$. The Simplified $\Theta$-IPOD method iteratively uses hard thresholding to eliminate the influence of outliers. For the experimental setup, we generate synthetic data with $n = 2000, t = 200, p = 15$, and $\sigma = 0.1$, but replace step [S4] by one of the following mechanisms for generating $\gammastar$:
\begin{enumerate}
\item We generate $\gammastar_i, i \in T$ by $Bernoulli(\pm 1,0.5) \cdot (10\sqrt{\log (2n) \sigma} + Unif(0,10))$.
\item We generate $\beta'$ elementwise from $Unif(-10,10)$ and take $\gammastar_i = x_i^\top (\beta'-\betastar), i \in T$.
\end{enumerate}
The first adversary is random, whereas the second adversary aims to attack the data by inducing the learner to fit another hyperplane. The precision/recall for Fast LTS and Least Squares are calculated by running the method once and applying various thresholds to clip $\gammahat$. For the other three methods, we apply different tuning parameters, compute precision/recall for each result, and finally combine them to plot a macro precision-recall curve.

In the left panel of Figure~\ref{fig:compare_subrecovery_pr}, Least Squares and Fast LTS reach perfect AUC, while the other three methods have slightly lower scores.
%This may result in the way we generate precisions and recalls. The unperfect AUC of the debugging, E-lasso and Simplified $\Theta$-IPOD could come from some mistakes along the tuning parameter paths;
In the right panel of Figure~\ref{fig:compare_subrecovery_pr}, we see that debugging, E-lasso, and Fast LTS perform comparably well, and slightly better than Simplified $\Theta$-IPOD.
%This advantage could locate at one more correct position estimation of support of $\gammastar$.
Not surprisingly, Least Squares performs somewhat worse, since it is not a robust procedure.

\begin{figure}[htp!]
\centering
\begin{subfigure}[b]{0.425\textwidth}
\centering
\includegraphics[width=\columnwidth]{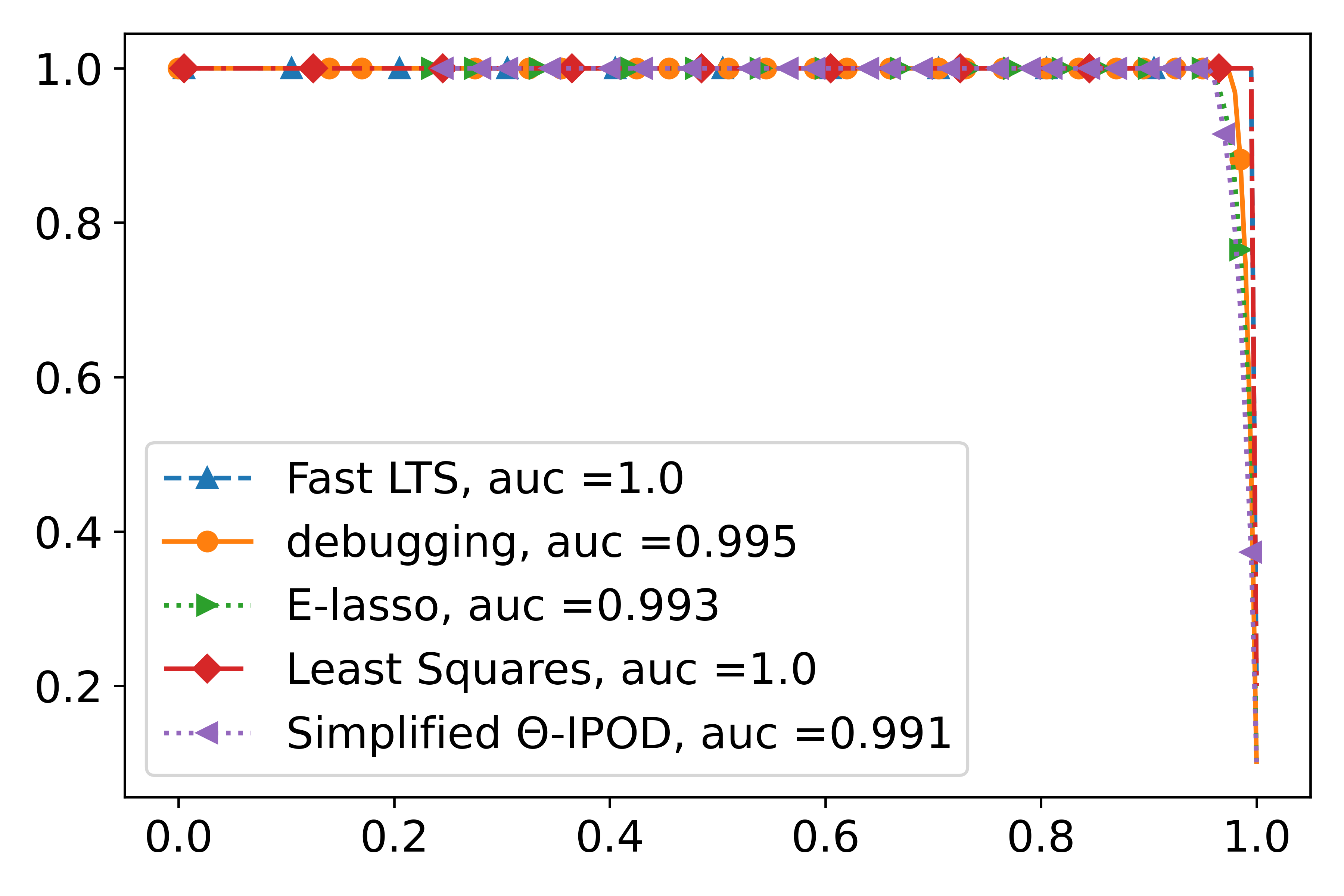}
% \label{fig:compare_subrecovery_pr_uniform}
\end{subfigure}
~
\begin{subfigure}[b]{0.425\textwidth}
\centering
\includegraphics[width=\columnwidth]{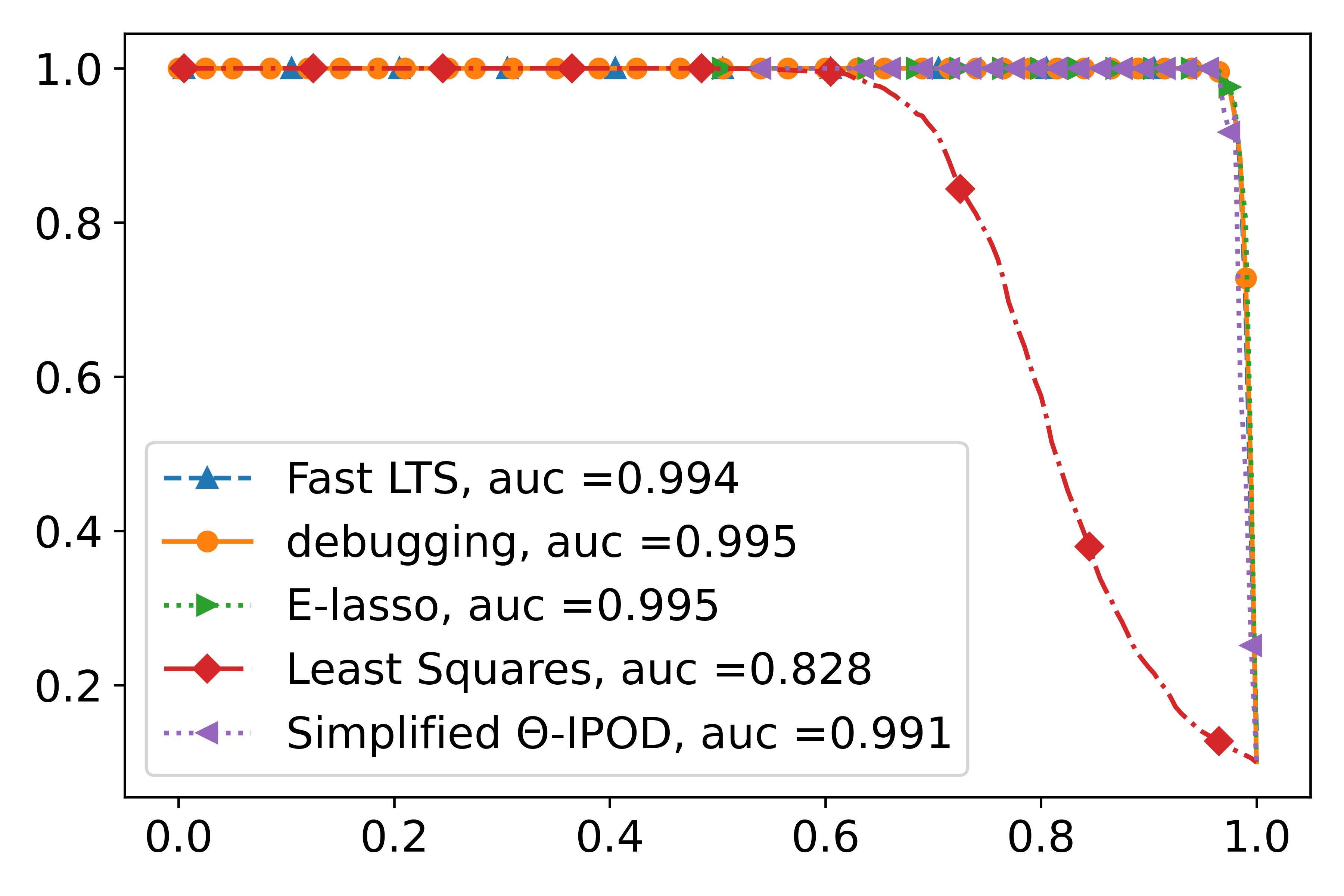}
% \label{fig:compare_subrecovery_pr_beta2}
\end{subfigure}
\caption{Precision Recall Curves over Different Regression Methods. The two plots correspond to the two settings described in the text for generating $\gamma^*$. To better view the curves, we only show the dots for every $c$ positions, where $c$ is an interger and different for different methods.}
\label{fig:compare_subrecovery_pr}
\end{figure}

\subsection{Tuning Parameter Selection}\label{sec:experiments-tuning-lambda}

We now present two experimental designs for tuning parameter selection. The first experiment runs Algorithm~\ref{alg:choice-lambda} for both one- and two-pool cases. We will present the recovery rates for a range of $n$'s and $c_t$'s, showing the effectiveness of our algorithm in a variety of situations. The second experiment compares Algorithm~\ref{alg:choice-lambda} in one- and two-pool cases to cross-validation, which is a popular alternative for parameter tuning. Our results indicate that Algorithm~\ref{alg:choice-lambda} outperforms cross-validation in terms of support recovery performance. 

We begin by describing the method used to generate the second data pool. Given the first data pool $(X,y)$ and the ground-truth parameters $(\betastar, \sigma)$, we describe two pipelines to generate the second pool. The first pipeline checks $m$ random points of the first pool, with steps [T1-T3]:
\begin{enumerate}
\item[T1] Select $m$ points uniformly at random from the first pool to construct $\Xtil$ for the second pool.
\item[T2] Generate $\epsilontil \in \real^m$, where each entry $\epsilontil_i$ is drawn i.i.d.\ from $\mathcal{N}(0,\sigma^2/L)$. 
\item[T3] Generate the labels by $\ytil = \Xtil\betastar + \epsilontil$.
\end{enumerate}
When the debugger is able to query features of clean points from a distribution $\mathcal{P}_X$, we can use a second pipeline, where [T1] is replaced by [T1']:
\begin{enumerate}
\item[T1'] Independently draw $m$ points from $\mathcal{P}_X$ to construct $\Xtil$.
\end{enumerate}

\subsubsection{Verification of Algorithm~\ref{alg:choice-lambda}}

We use the default procedure for generating the synthetic dataset, with parameters $p = 15$, $\sigma=0.1$, and $t = c_t n$, where $c_t$ ranges from 0.05 to 0.4 in increments of 0.05. In all cases, we input $\bar{c} = 0.2$ and $\lambda_u = \frac{2\|P_{X}^\perp y \|_\infty}{n}$ in Algorithm~\ref{alg:choice-lambda}.

\begin{figure}[htp!]
\centering
\begin{subfigure}[b]{0.45\textwidth}
\centering
\includegraphics[width=\textwidth]{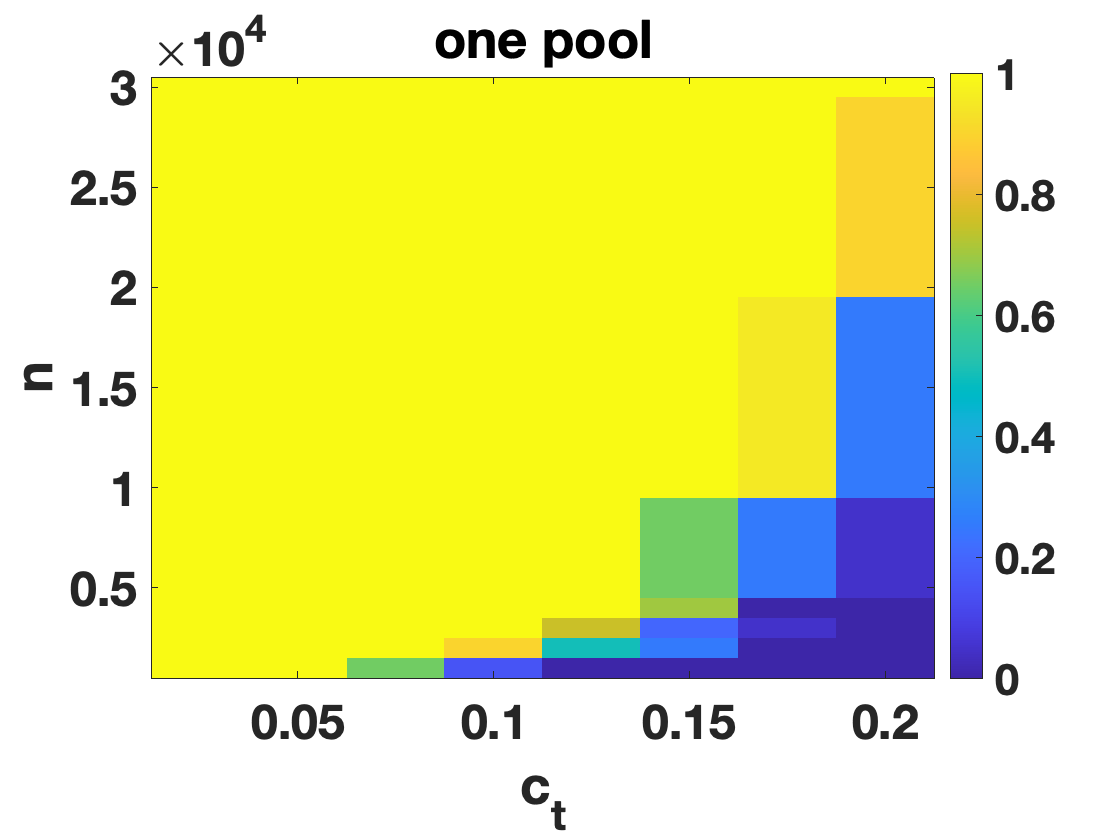}
\end{subfigure}
~
\begin{subfigure}[b]{0.45\textwidth}
\centering
\includegraphics[width=\textwidth]{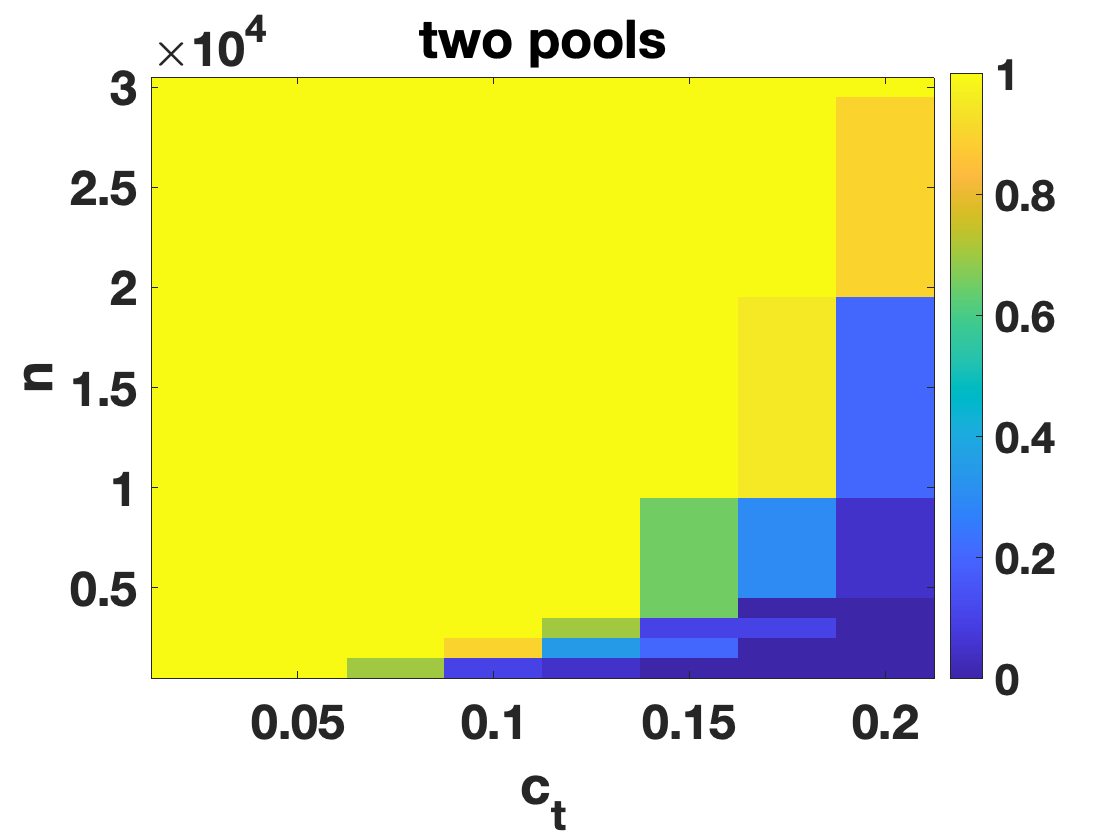}
\end{subfigure}
\caption{Exact Recovery Rate over 20 Trials. The recovery rate is shown in different cases varying by fraction of outliers $c_t$ and $n$. The left subfigure is for one-pool case and the right subfigure is for two-pool case. We set $m = 100, L = 5$ for the second pool.}
\label{fig:choice-lambda} 
\end{figure}

Figure~\ref{fig:choice-lambda} displays the results for $n \in \{1,2,3,4,5,10,20,30\}\cdot 10^3$. 
First, we see that Algorithm~\ref{alg:choice-lambda} achieves exact support recovery in all 20 trials in the yellow area.
Second, the exact recovery rate increases with increasing $n$ and decreasing $c_t$, showing that the algorithm is particularly useful for large-scale data sets. This trend can also be seen from the requirement on $n$ imposed in Theorem~\ref{thm:choice-lambda}. In particular, we see that the contour curve for the exact recovery rate matches the curve of $\left(1-c_t\right)^{-\frac{1}{1-2c_n}}$ for some constant $c_n \in (0,\frac{1}{2})$. 
However, a downside of Algorithm~\ref{alg:choice-lambda} is that it does not fully take advantage of the second pool in the two-pool case, as the left panel and the right panel display similar results.
%Note that Algorithm~\ref{alg:choice-lambda} does not take $\frac{m}{n}$ into consideration and we will pay attention to this factor in our further develop and to see how $m$ can help with tuning parameter selection.

\subsubsection{Effectiveness of Tuning Parameter Selection}

We now compare our method for tuning parameter selection to cross-validation.
%and the E-lasso~\cite{nguyen2013robust}.
We also use the postprocessing step described at the beginning of the section. Four measurements are presented, including two recovery rates, the $\ell_2$-error of $\betahat$, and the runtime. In both the one- and two-pool cases, we use our default methods for generating synthetic data, and we set $\bar{c} = 0.2$ for all the experiments.

The cross-validation method for the one-pool case splits the dataset into training and testing datasets with the ratio of $8:2$, then selects $\lambda$ with the smallest test error, $\|X_{test} \betahat - y_{test}\|_2$.
The procedure for the two-pool case is to run the Lasso-based debugging method with a list of candidate $\lambda$'s and test it on the second pool. Finally, we select the $\lambda$ value with the smallest test error, $\|\Xtil \betahat - \ytil\|_2$. We use 15 candidate values for $\lambda$, spaced evenly on a log scale between $10^{-6}$ and $
\lambda_u = \frac{2\|P_{X}^\perp y \|_\infty}{n}$.

Figure~\ref{fig:compare_tuning_method_onepool} compares the results in the one-pool case. We note that cross-validation does not perform very well for all the measurements except $\|\betahat - \betastar\|_2$. Specifically, it does not work at all for subset support recovery, since cross-validation tends to choose very small $\lambda$ values. For the $\ell_2$-error, we see that for small values of $c_t$, our algorithm can select a suitable choice of $\lambda$, so that after removing outliers, we can fit the remaining points very well. This is why the debugging + postprecessing methods gives the lowest error. As $c_t$ increases, our debugging method shows poorer performance in terms of support recovery, resulting in larger $\ell_2$-error for $\betahat$. Although cross-validation seems to perform well, carefully designed adversaries may still destroy the good performance of cross-validation, since its test dataset could be made to contain numerous buggy points. 
 
\begin{figure}[htp!]
\centering
\begin{subfigure}[b]{0.425\textwidth}
\centering
\includegraphics[width=\textwidth]{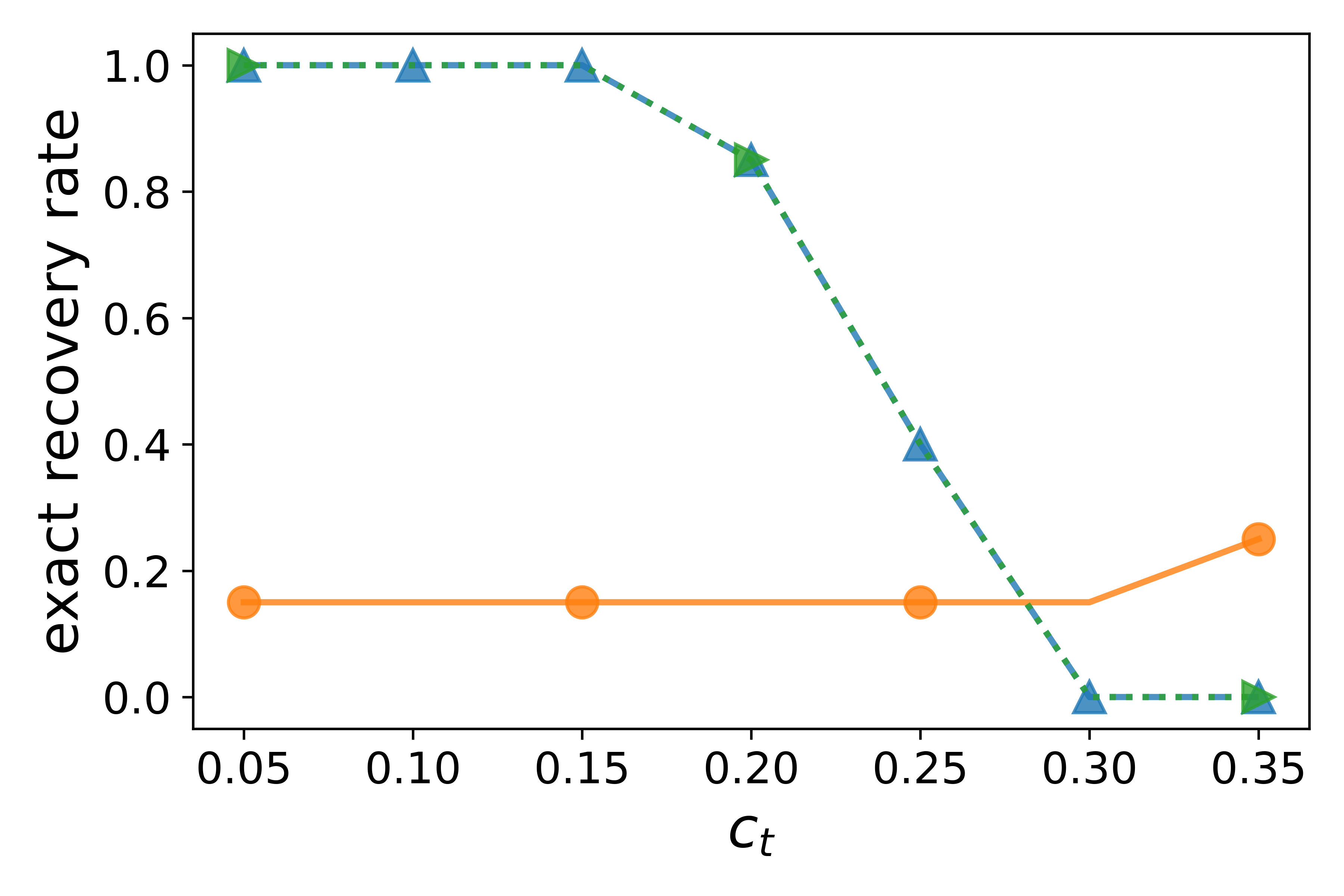} 
% \caption{.}
% \label{fig:compare_recovery_exactrec}
\end{subfigure}
~
\begin{subfigure}[b]{0.425\textwidth}
\centering
\includegraphics[width=\textwidth]{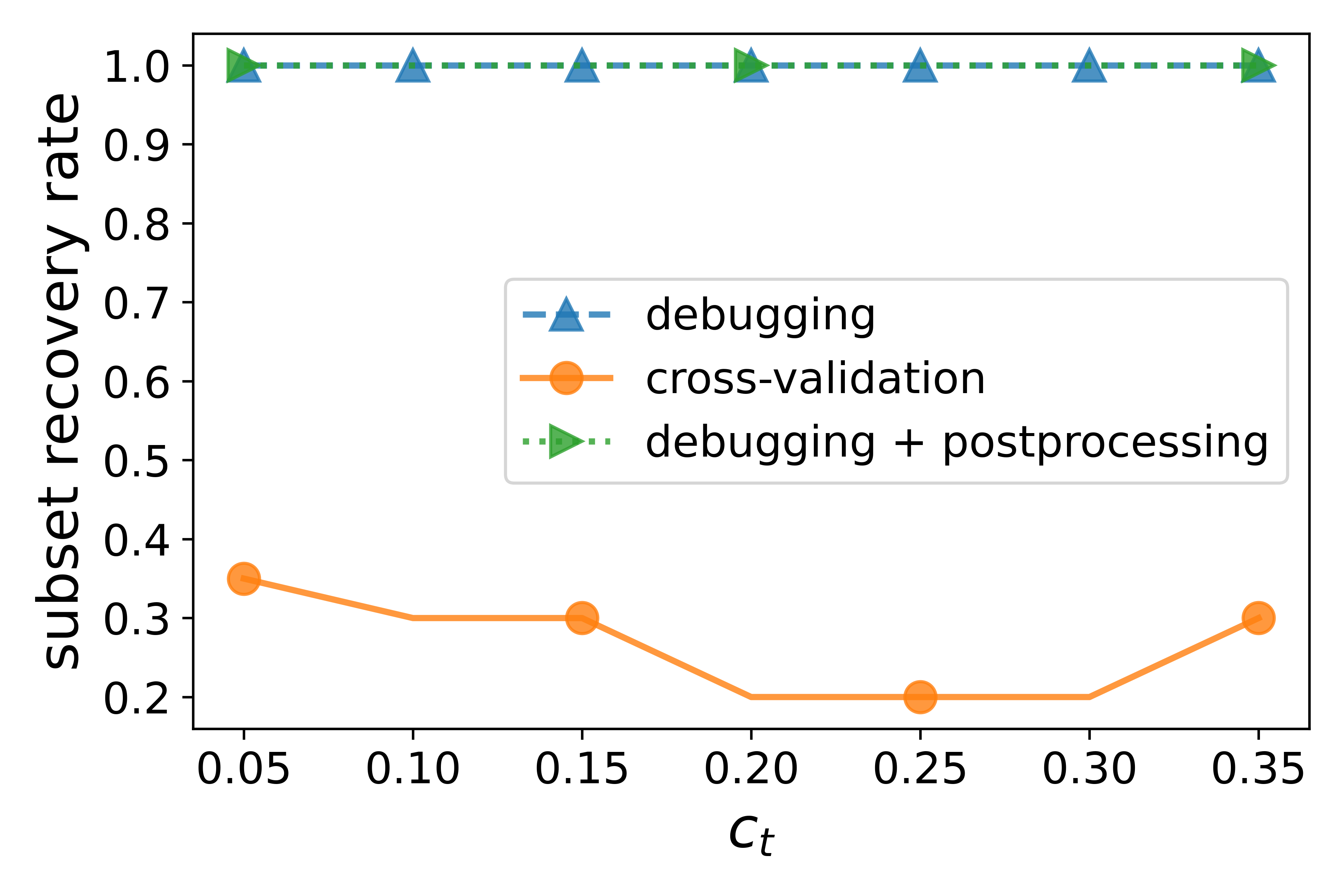}
% \caption{.}
% \label{fig:compare_recovery_subrec}
\end{subfigure}
\vskip\baselineskip
\begin{subfigure}[b]{0.425\textwidth}
\centering
\includegraphics[width=\textwidth]{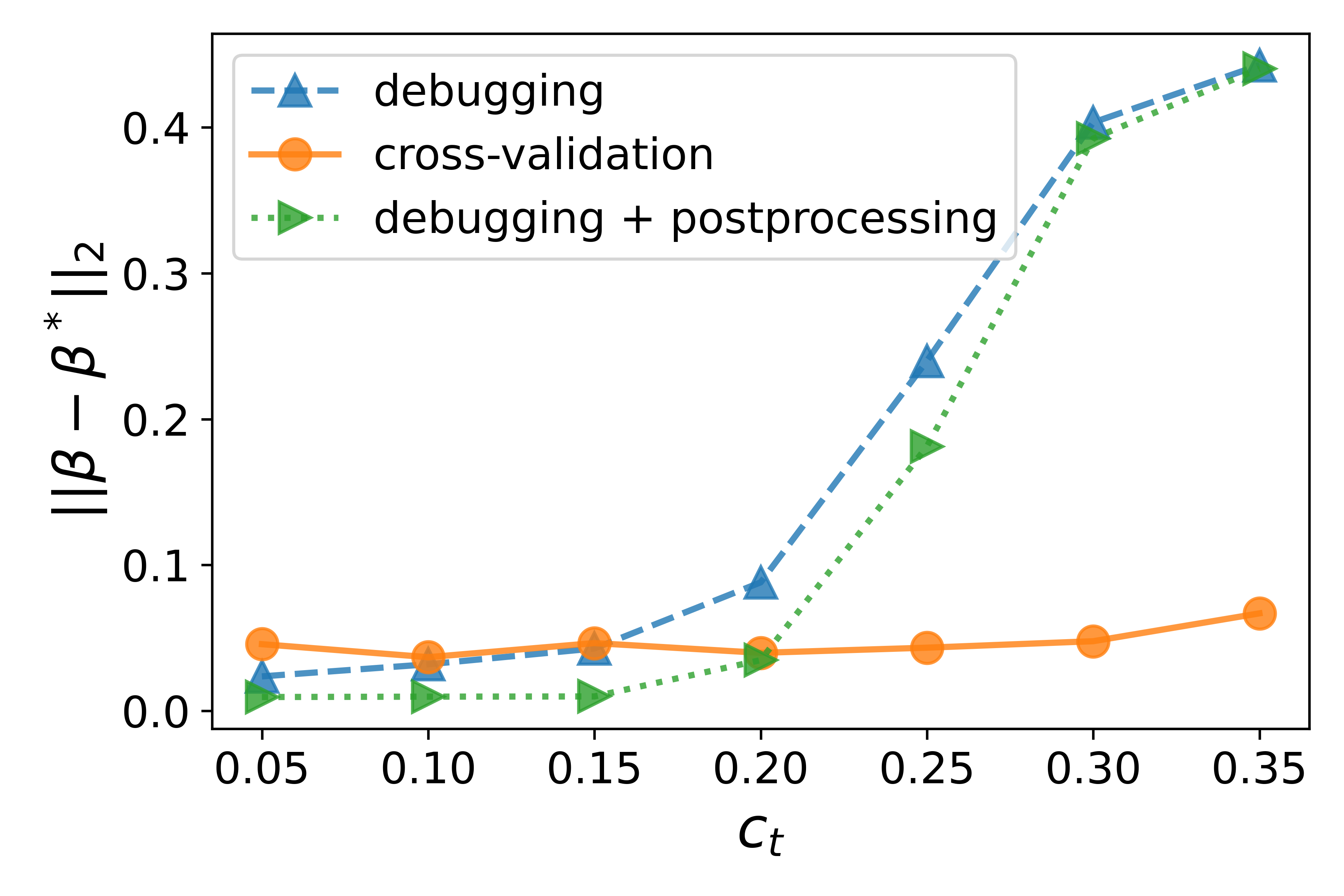}
% \caption{.}
% \label{fig:compare_recovery_beta}
\end{subfigure}
~
\begin{subfigure}[b]{0.425\textwidth}
\centering{}
\includegraphics[width=\textwidth]{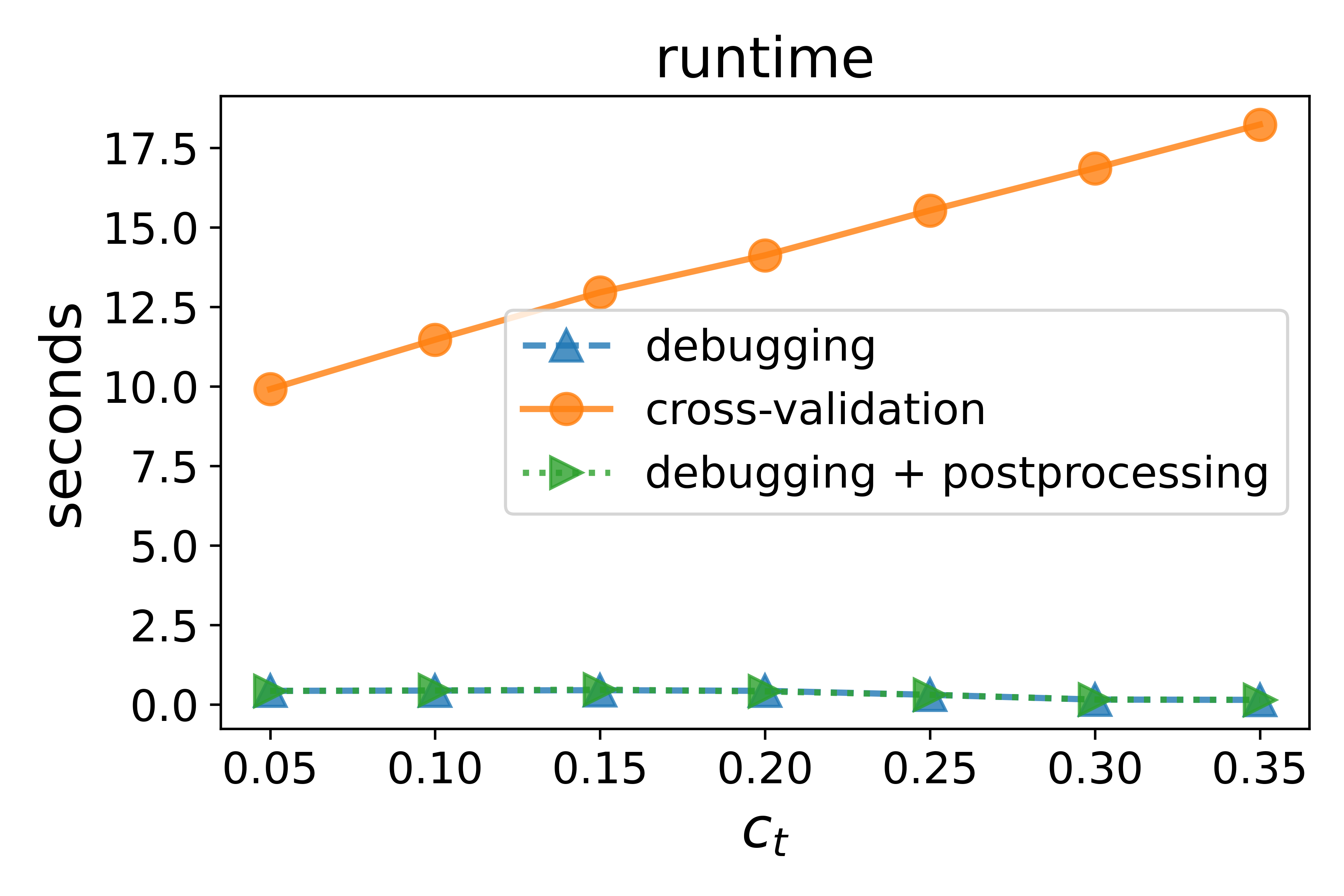}
% \caption{.}
% \label{fig:compare_recovery_time}
\end{subfigure}
\caption{Effectiveness of Tuning Parameter Selection (One Pool). Each dot is the average result of 20 random trials. We set $n = 2000, p = 15$, and $\sigma=0.1$.}
\label{fig:compare_tuning_method_onepool}
\end{figure}

\begin{figure}[htp!]
\centering
\begin{subfigure}[b]{0.425\textwidth}
\centering
\includegraphics[width=\textwidth]{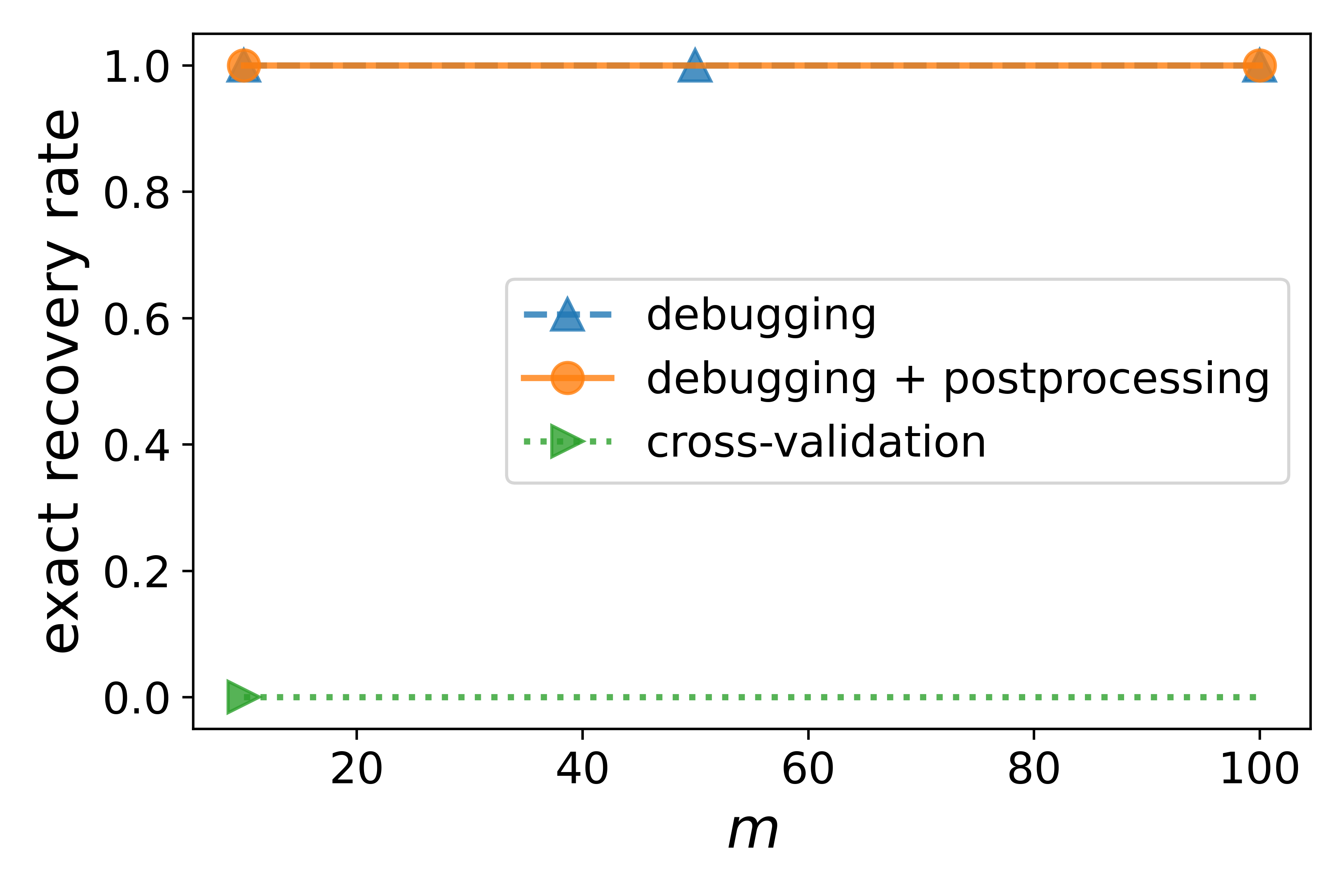}
% \caption{.}
% \label{fig:compare_recovery_clean_exactrec}
\end{subfigure}
~
\begin{subfigure}[b]{0.425\textwidth}
\centering
\includegraphics[width=\textwidth]{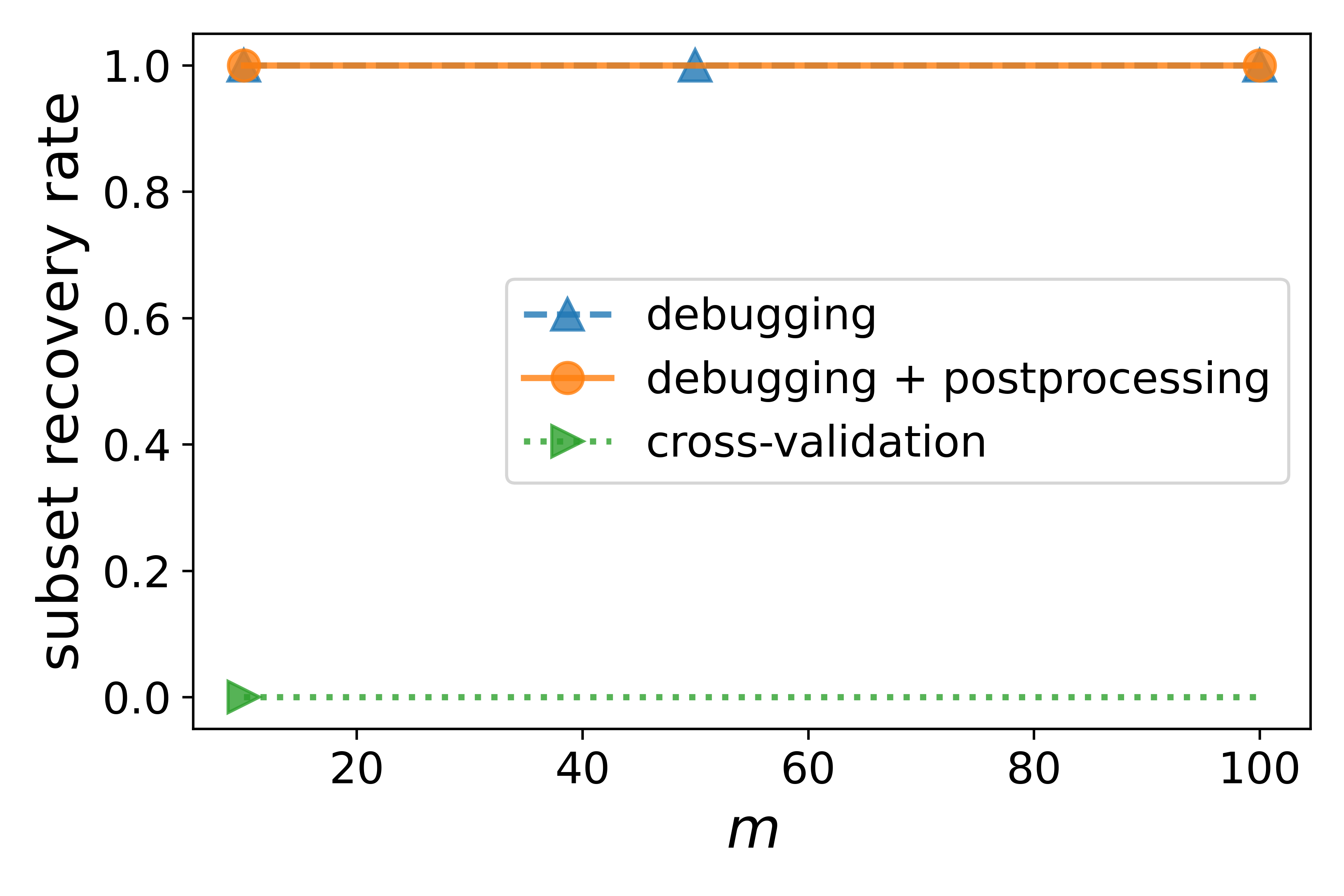}
% \caption{.}
% \label{fig:compare_recovery_clean_subrec}
\end{subfigure}
\vskip\baselineskip
\begin{subfigure}[b]{0.425\textwidth}
\centering
\includegraphics[width=\textwidth]{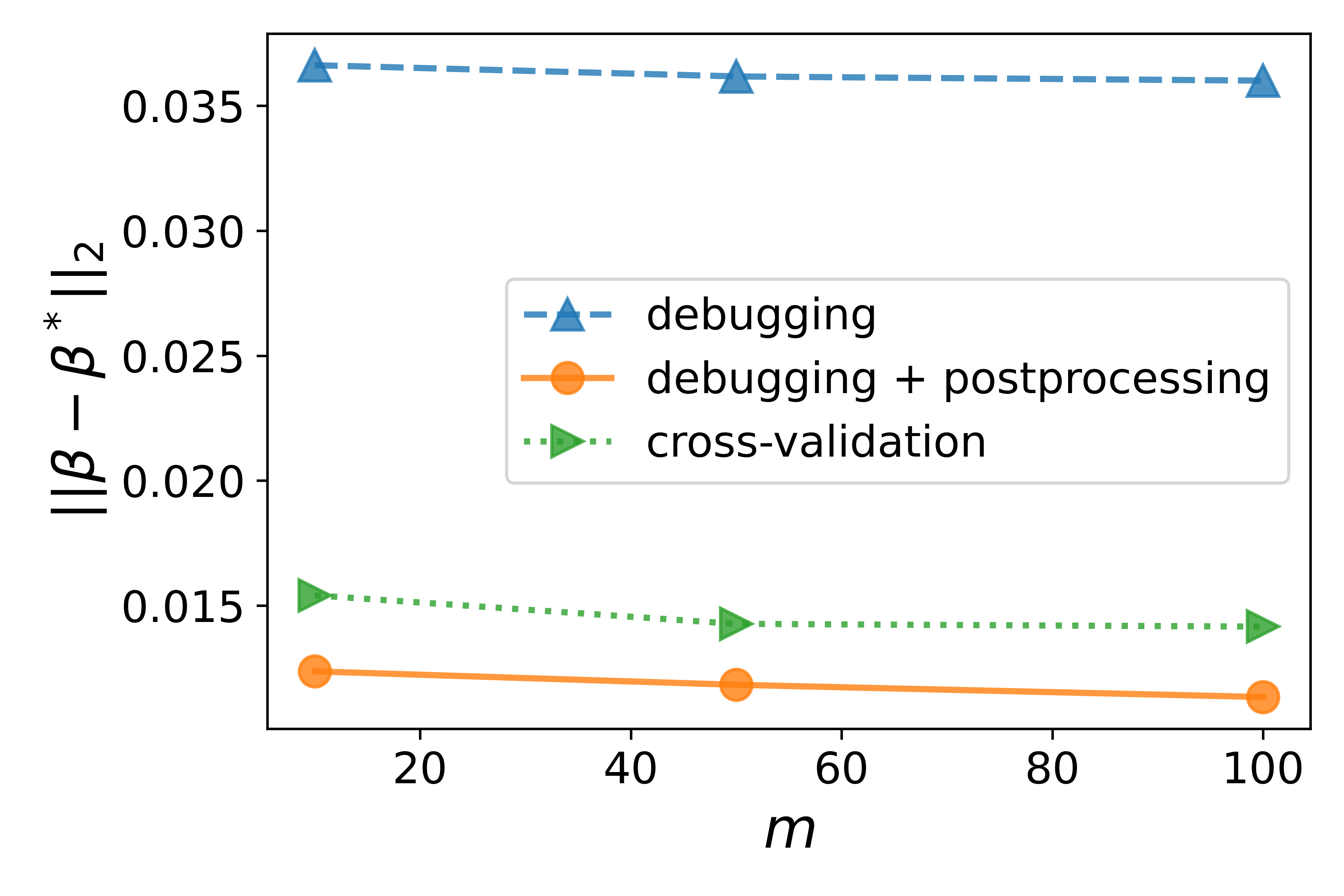}
% \caption{.}
% \label{fig:compare_recovery_clean_beta} 
\end{subfigure}
~
\begin{subfigure}[b]{0.425\textwidth}
\centering
\includegraphics[width=\textwidth]{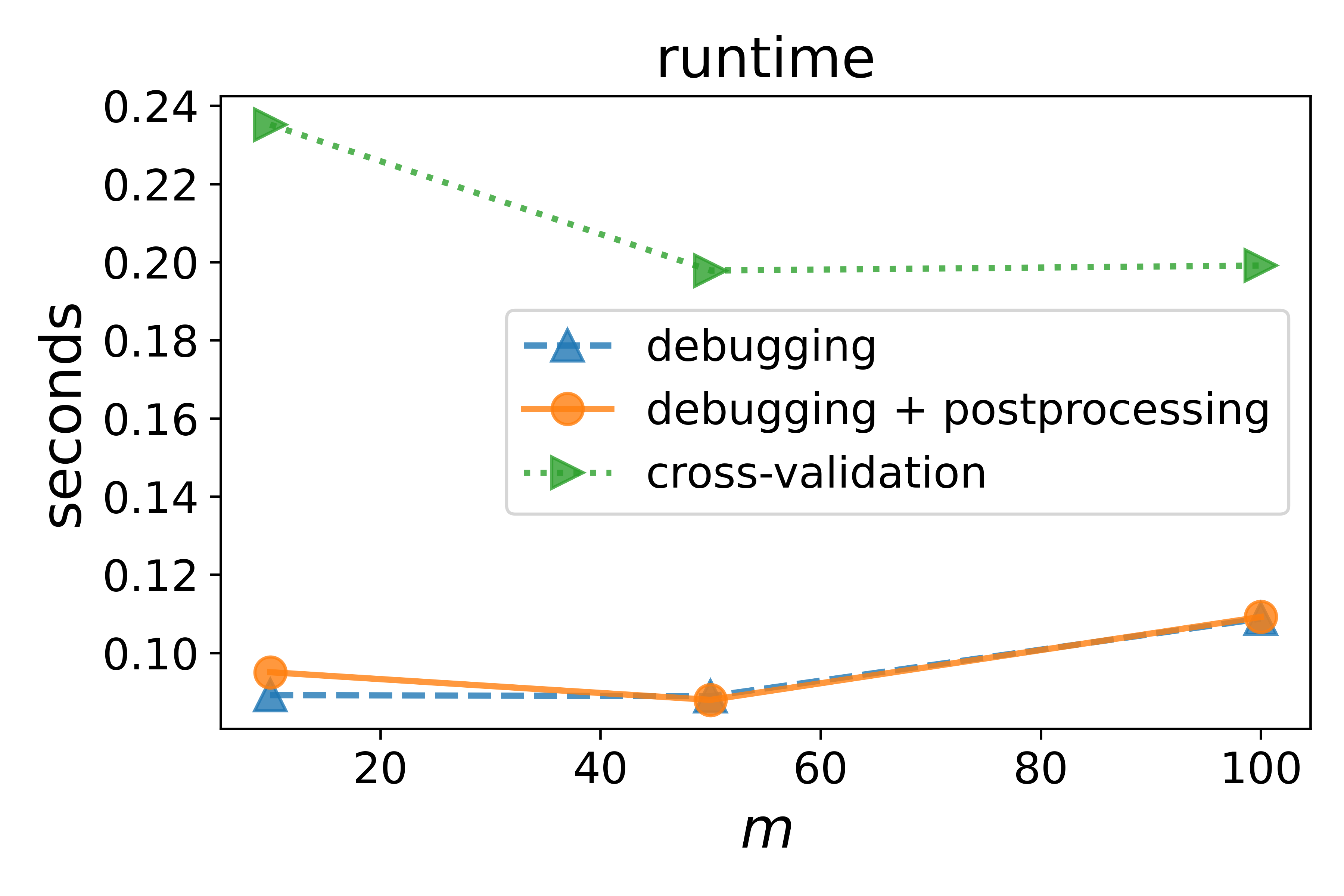}
% \caption{.}
% \label{fig:compare_recovery_clean_time}
\end{subfigure}
\caption{Effectiveness on Tuning Parameter Selection (Two Pools). Each dot is the average result of 20 random trials. We set $n = 1000, p = 15, t = 100, L = 5$, and $\sigma=0.1$.}
\label{fig:compare_tuning_method_clean}
\end{figure}

Figure~\ref{fig:compare_tuning_method_clean} displays the results for the two-pool  experiments, which are qualitatively similar to the results of the one-pool experiments. We emphasize that our method works well for support recovery; furthermore, the methods exhibit comparable performance in terms of the $\ell_2$-error. The slightly larger error of our debugging method can be attributed to the bias which arises from using an $\ell_1$-norm instead of an $\ell_0$-norm.

\subsection{Experiments with Clean Points}
\label{sec:experiments-clean-points}

We now focus on debugging methods involving a second clean pool. We have three experimental designs: First, we study the influence of $m$ on support recovery. Second, we compare debugging with alternative methods suggested in the literature. Third, we study the performance of our proposed MILP debugger, where we compare it to three other simple strategies. Different strategies for selecting clean points correspond to changing step [T1] in the setup described above.

\subsubsection{Number of Clean Points vs.\ Exact Recovery}
In this subsection, we present two experiments involving synthetic and YearPredictionMSD datasets, respectively ,to see how $m$ affects the exact recovery rate. Recall that the pipeline for generating the first pool is described at the beginning of Section~\ref{sec:exp}. For the second pool, we use steps [T'1, T2, T3] for the synthetic dataset, where we assume $\mathcal{P}_X$ is standard Gaussian. We take steps [T1-T3] for YearPredictionMSD to check the sample points in the first pool. 

\begin{figure}[htp!]
\centering
\begin{subfigure}[b]{0.45\textwidth}
\centering
\includegraphics[width=\textwidth]{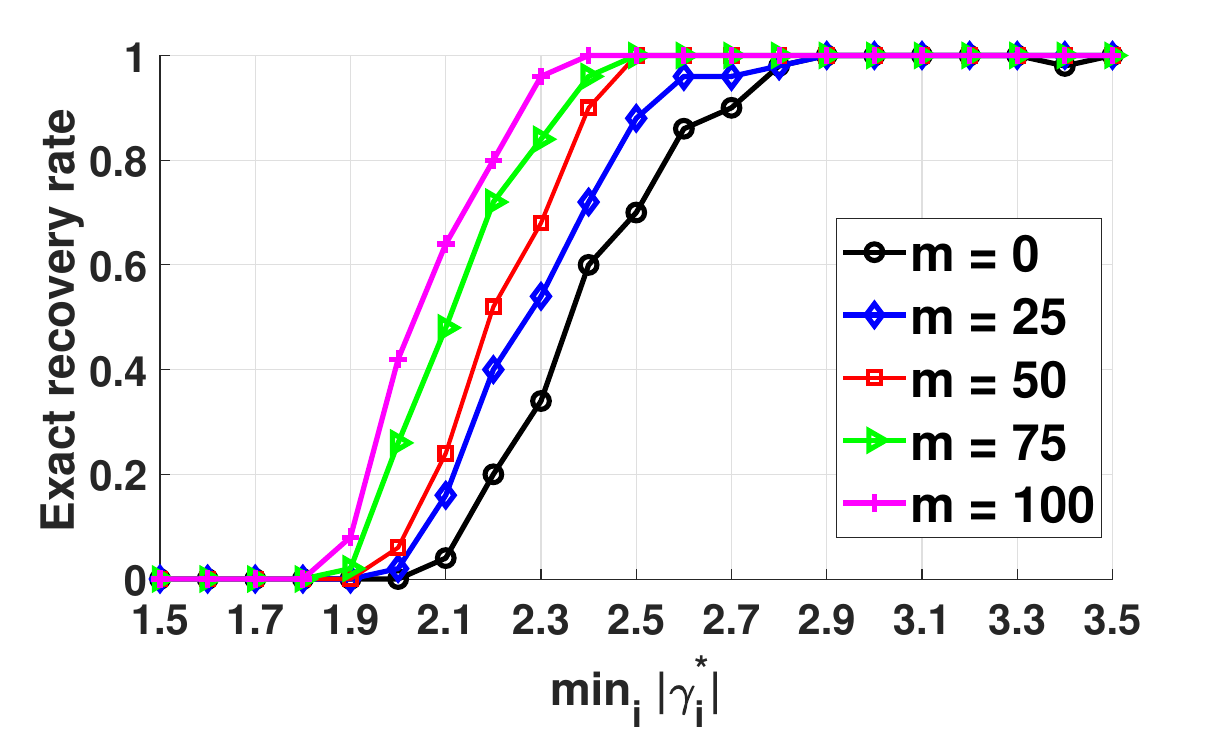}
\end{subfigure}
~
\begin{subfigure}[b]{0.45\textwidth}
\centering
\includegraphics[width=\textwidth]{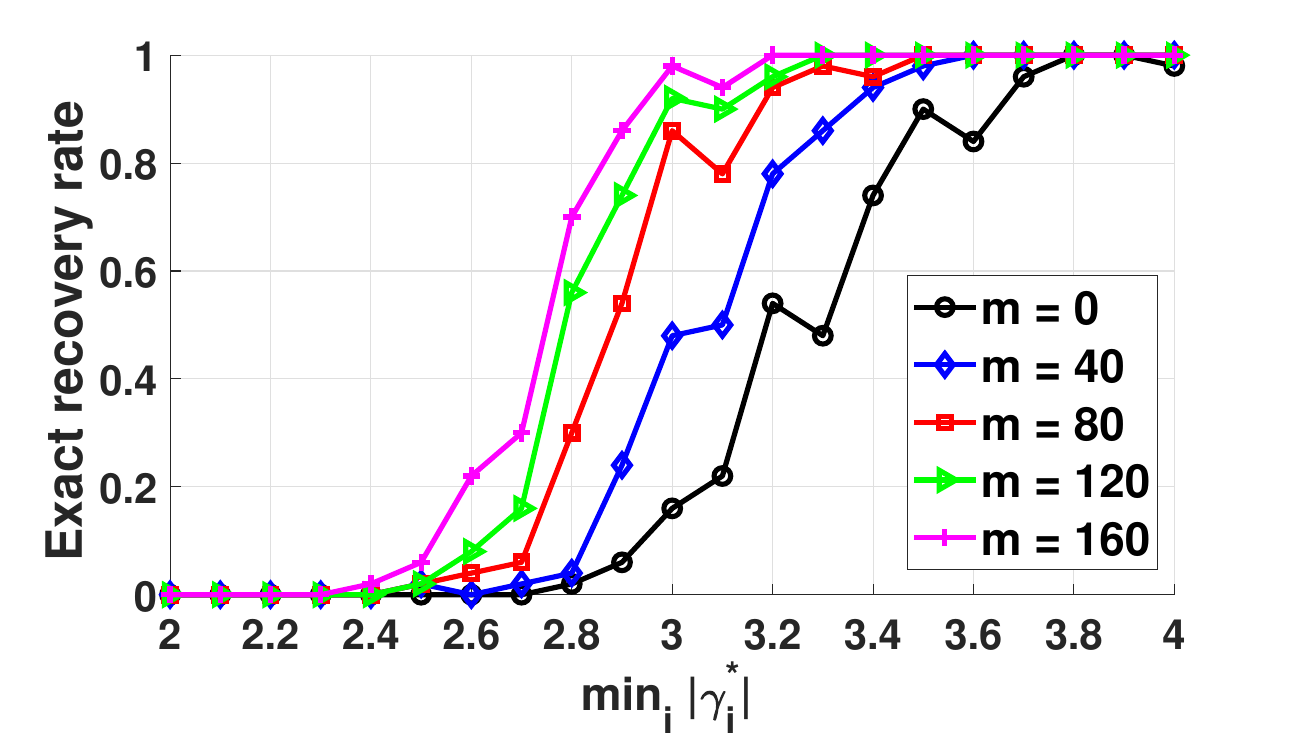}
\end{subfigure}
\caption{Minimal Gamma vs.\ Exact Recovery Rate on Synthetic Data. 
We run 50 trials for each dot and compute the average.} 
\label{fig:GammaMinvsExactRec}
% \caption{Gamma-min Condition versus Exact Recovery Rate on Real Data. We generate $\gamma_i$ as the sum of $c\sigma$ and a random variable $Unif(0,\sigma)$. We regard $\min_i |\gammastar_i|$ as $c$ (actually, it is a bit larger than $c\sigma$). Then we run the optimization for 50 times for each dot.}
\end{figure}

Recall that the YearPredictionMSD dataset is designed to predict the release year of a song from audio features. The dataset consists of 515,345 songs, each with 90 audio features. Therefore, for both experiments, we set $n = 500, t = 50, p = 90, \sigma = 0.1$, and $L = 10$, and take $\lambda = 2.5\frac{\sqrt{\log (n-t)}}{n}$. 

From Figure~\ref{fig:GammaMinvsExactRec}, we see that the phenomena are similar for the two different design matrices. In particular, increasing the number of clean points helps with exact recovery. For instance, in the left subfigure, for $m = 0$, when $\min_i |\gammastar_i| > 2.9$, the exact recovery rate goes to 1. For $m = 100$, the exact recovery rate goes to 1 when $\min_i |\gammastar_i| > 2.4$. Also, the slope of the curve for larger $m$ is sharper. Thus, adding a second pool helps relax the gamma-min condition.

\subsubsection{Comparisons to Methods with Clean Points}

In this experiment, we compare the debugging method for two pools with other methods suggested by the machine learning literature. We generate synthetic data using the default first-pool setup with $n = 1000, p = 15, t = 100$, and $\sigma=0.1$, and we run [T1-T3] to generate the second pool using different values of $m$. For our proposed debugging method, we use Algorithm~\ref{alg:choice-lambda} to select the tuning parameter. We compare the following methods: (1) debugging + postprocessing, (2) least squares, (3) simplified noisy neural network, and (4) semi-supervised eigvec. 
%Recall that the postprocessing step on debugging is to remove the estimated buggy points and get the least squares solution on the remaining data points, which concatenate estimated nonbuggy points and clean points.
The least squares solution is applied using $\left\{\begin{pmatrix} X \\ \Xtil \end{pmatrix}, \begin{pmatrix} y \\ \ytil \end{pmatrix}\right\}$. 

The simplified noisy neural network method borrows an idea from Veit et al.~\cite{veit2017learning}, which is designed for image classification tasks for a datasets with noisy and clean points. This work introduced two kinds of networks and combines them together: the ``Label Cleaning Network," used to correct the labels, and the ``Image Classifier," which classifies images using CNN features as inputs and corrected labels as outputs. Each of them is associated with a loss, and the goal is to minimize the sum of the losses.  Let $w\in \real, \beta_1 \in \real^d$, and $\beta_2 \in \real^d$ be the variables to be optimized. For our linear regression setting, we design the ``Label Cleaning Network'' by defining $\hat{c}_i = y_i w - x_i^\top \beta_1$ as the corrected labels for both noisy and clean datasets. Then we define the loss $\mathcal{L}_{clean} = \sum_{i \in cleanset} |\tilde{y}_i - y_i w - x_i^\top \beta_1|$. The ``Image Classifier" is modified to the regression setting using predictions of $x_i^\top \beta_2$ and the squared loss. Therefore, the classification loss can be formalized as $\mathcal{L}_{classify} = \sum_{i \in cleanset} (x_i^\top \beta_2 - \tilde{y}_i)^2 + \sum_{i \in noisy set} (x_i^\top \beta_2 - \hat{c}_i)$. Together, the optimization problem becomes 
\[
\min_{\stackrel{\beta_1 \in \real^d, \beta_2 \in \real^d}{w \in \real}} \sum_{i \in clean set} \{(x_i^\top \beta_2 - \tilde{y}_i)^2 + |\tilde{y}_i - w y_i - x_i^\top \beta_1|\} +\sum_{i \in noisy set} (x_i^\top \beta_2 - w y_i - x_i^\top \beta_1)^2. 
\]
We use gradient descent to do the optimization, and initialize it with $w = 0$ and $\beta_1 = \beta_2 = \beta_{ls}$. The optimizer $\betahat_2$ is used for further predictions. We then calculate $\gammahat = y - X\betahat_2$. For gradient descent, we will validate multiple step sizes and choose the one with the best performance on the squared loss of the clean pool.

The method ``semi-supervised eigvec" is from Fergus et al.~\cite{fergus2009semi}, and is designed for the semi-supervised classification problem. It also contains an experimental setting that involves noisy and clean data. To further apply the ideas in our linear regression setting, we make the following modifications: Define the loss function as
$$J(f) = f^\top L f + \left(f-\begin{pmatrix} y \\ \tilde{y}
\end{pmatrix}\right)^\top \Lambda \left(f-\begin{pmatrix} y \\ \tilde{y}
\end{pmatrix}\right),$$
where $L = D-W (\varepsilon)$ is the graph Laplacian matrix and $\Lambda$ is a diagonal matrix whose diagonal elements are $\Lambda_{ii} = \lambda$ for clean points and $\Lambda_{ii} = \frac{\lambda}{c}$ for noisy points. In the classification setting, $f\in \real^{n+m}$ is to be optimized. The idea is to constrain the elements of $f$ by injecting smoothness/similarity using the Laplacian matrix $L$. Since we assume the linear regression model, we can further plug in $f = \begin{pmatrix} X \\ \tilde{X}\end{pmatrix}\beta$. Our goal is then to estimate $\beta$ by minimizing $J(\beta)$. As suggested in the original paper, we use the range of values $\varepsilon \in [0,1,1,5], c \in [1,10,50]$, and $\lambda \in [1,10,100]$. We will evaluate all 36 possible combinations and pick the one with the smallest squared loss on the clean pool.

\begin{figure}[htp!]
\centering
\begin{subfigure}[b]{0.425\textwidth}
\centering
\includegraphics[width=\textwidth]{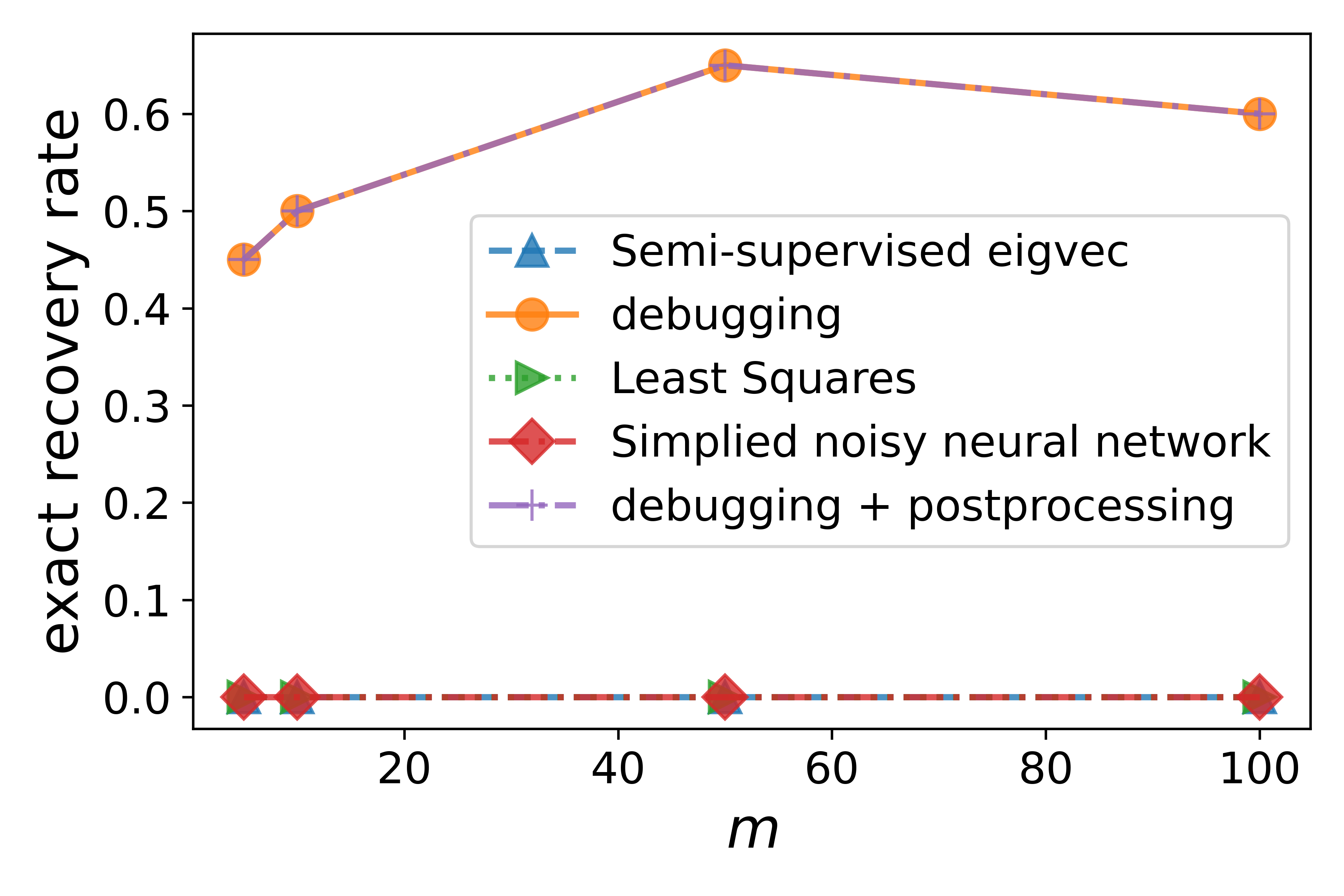}
% \caption{.}
\label{fig:compare_clean_exactrec}
\end{subfigure}
~
\begin{subfigure}[b]{0.425\textwidth}
\centering
\includegraphics[width=\textwidth]{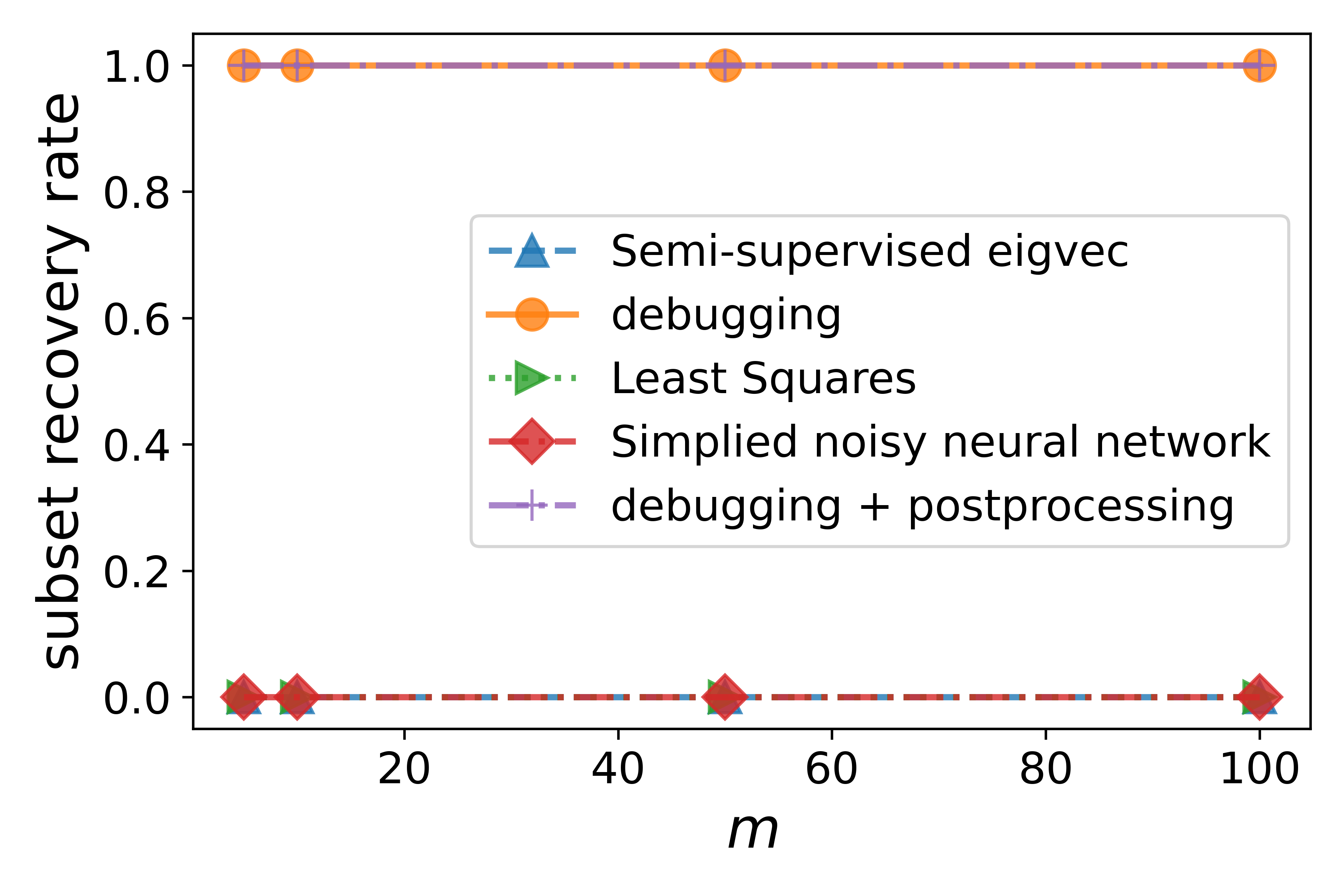}
% \caption{.}
\label{fig:compare_cleansubrec}
\end{subfigure}
\vskip\baselineskip
\begin{subfigure}[b]{0.425\textwidth}
\centering
\includegraphics[width=\textwidth]{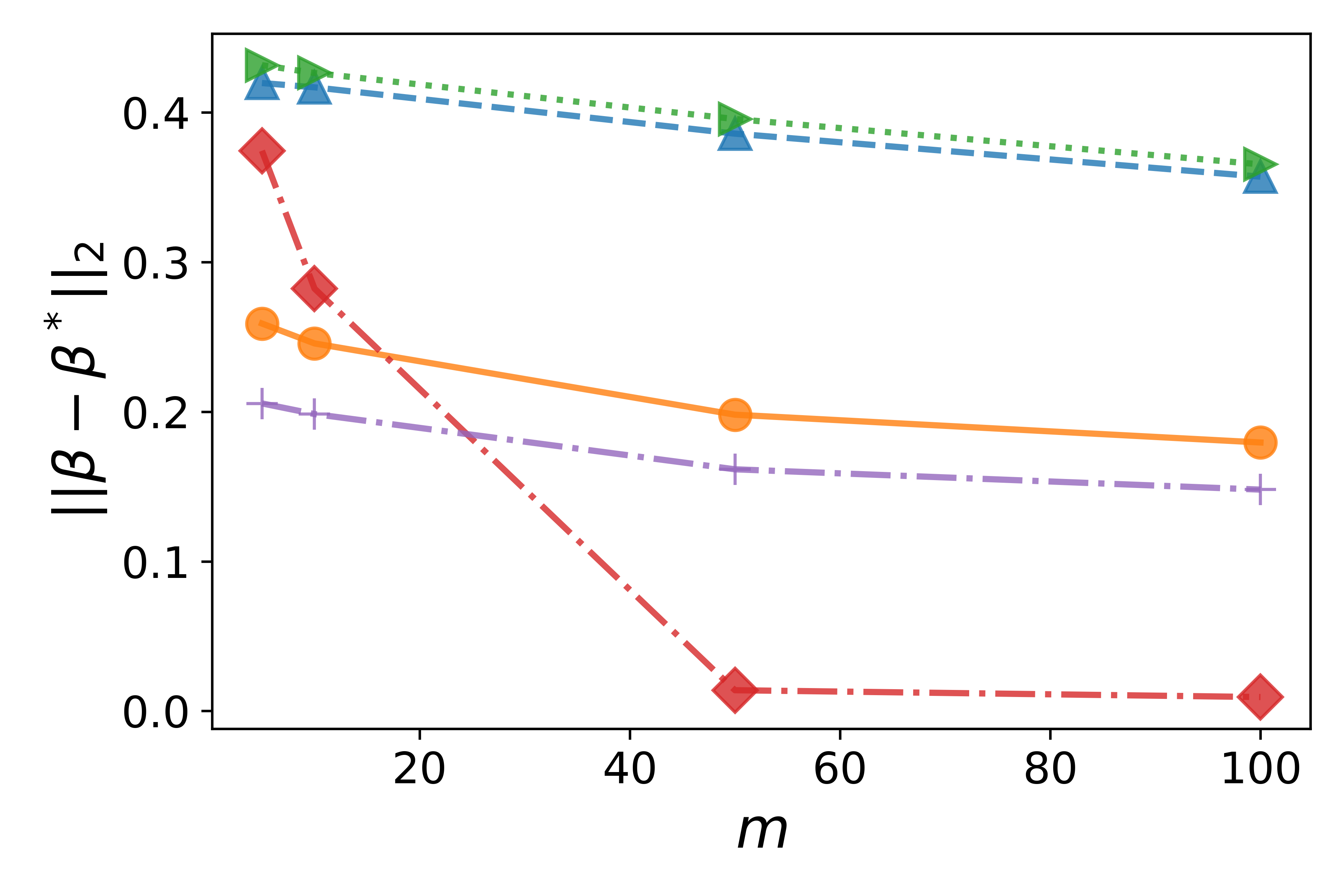}
% \caption{.}
\label{fig:compare_clean_beta}
\end{subfigure}
~
\begin{subfigure}[b]{0.425\textwidth}
\centering
\includegraphics[width=\textwidth]{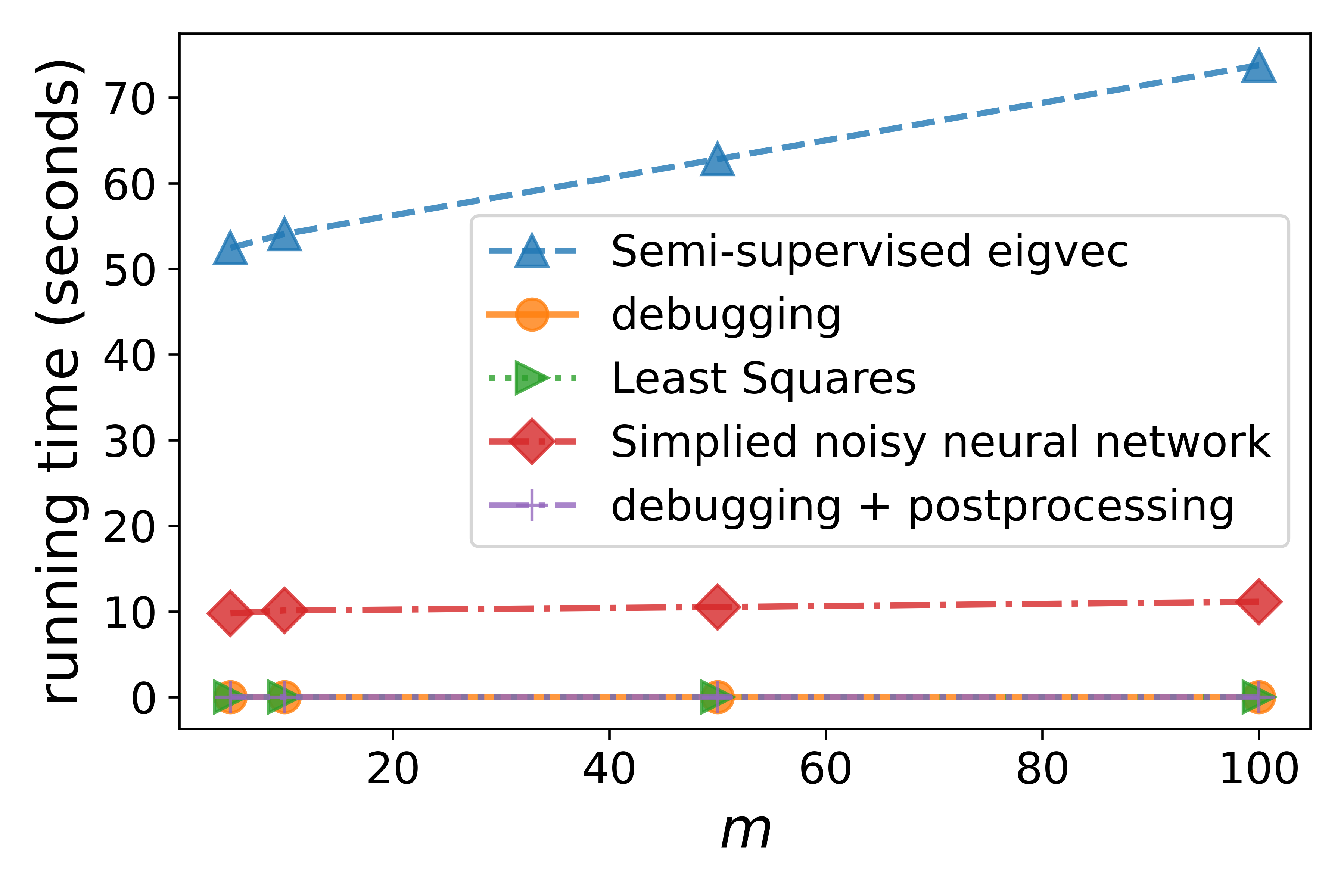}
% \caption{.}
\label{fig:compare_clean_time}
\end{subfigure}
\caption{Comparison to Methods involving Clean Points. Each dot is the average result of 20 random trials. We use the synthetic data setting, with $n = 500, p = 15, \sigma=0.1,t = 0.1n$, and $\min_i |\gammastar_i| = 10 \sqrt{\log 2n} \sigma$. The clean data pool is randomly chosen from the first pool without replacement; we query the labels of these chosen points.}
\label{fig:compare_clean_method}
\end{figure}

The results are shown in Figure~\ref{fig:compare_clean_method}. We observe that only the debugging method is effective for support recovery, as we have carefully designed our method for this goal. The method from Veit et al.~\cite{veit2017learning} works best in terms of $\ell_2$-error of $\beta$, especially when $m$ is large. The semi-supervised method, like least squares, does not perform well, possibly because it does not consider replacing/removing the influence of the noisy dataset.

\subsubsection{Effectiveness on Second Pool Design}
\label{sec:active-exp}

We now provide experiments to investigate the design of the clean pool, corresponding to Section~\ref{sec:active}. We use the Concrete Slump dataset\footnote{\url{https://archive.ics.uci.edu/ml/datasets/Concrete+Slump+Test}}, where $p = 7$. We limit our study to small datasets, since the runtime of the MILP optimizer is quite long.
%Nonetheless, we believe the MILP optimizer can be accelerated by the experts in this area (\cite{tang2016class}).  
We report the performance of the MILP debugging method in both noiseless and noisy settings. In our experiments, we compare the performance of the MILP debugger to a random debugger and a natural debugging method: adding high-leverage points into the second pool. In other words, D.milp selects $m$ clean points to query from running the MILP~\eqref{eq:milp}; D.leverage selects the $m$ points with the largest values of $x_i^\top(X^\top X)^{-1}x_i$; and D.random randomly chooses $m$ points from the first pool without replacement. After choosing the clean pool, the debugger applies the Lasso-based algorithm. In Zhang et al.~\cite{zhang2018training}, all the second pool points are chosen either randomly or artificially. Therefore, we may consider D.random as an implementation of the method in Zhang et al.~\cite{zhang2018training}, which will be compared to our D.milp.

In the noiseless setting, we define $\betastar$ to be the least squares solution computed from all data points. We randomly select $n$ data points as the $x_i$'s. For D.milp and D.leverage, since the bug generator knows their strategies or the selected $D$, it generates bugs according to the optimization problem~\eqref{eq:bug-generator}. Let $T \subseteq [n]$ be the index set of the $t$ largest $|u_i|$'s, for $i = 1,\dots, n$. The bug generator takes $\gammastar_T = u_T$ if the solution $u$ is nonzero, and otherwise randomly generates a subset $T$ of size $t$ to create $\gammastar_T = \vec{1}$. Thus, $y_i = x_i^\top \betastar + \gammastar$. For D.onepool, the bug generator follows the above description with $D = \emptyset$. The orange bars indicate whether the bug generator succeeds in exact recovery in the one-pool case. For D.random, the bug generator generates bugs using the same mechanism as for D.onepool. Note the above bug generating methods are the ``worst'' in the sense of signed support recovery: The debuggers run~\eqref{eq:active-objective-gamma-y} using their selected $X_D$. From Figure~\ref{fig:noiseless-compare-optimal-vs-leverage}, there is an obvious advantage of D.milp over D.onepool and D.leverage. This suggests improved performance of our MILP algorithm. D.random is sometimes successful even when $n$ and $t$ are small because the bug generator cannot control the randomness, but it performs worse than D.milp overall.

\begin{figure}[h!]
\centering
\begin{minipage}[c]{0.48\textwidth}
\centering
\includegraphics[width = \linewidth]{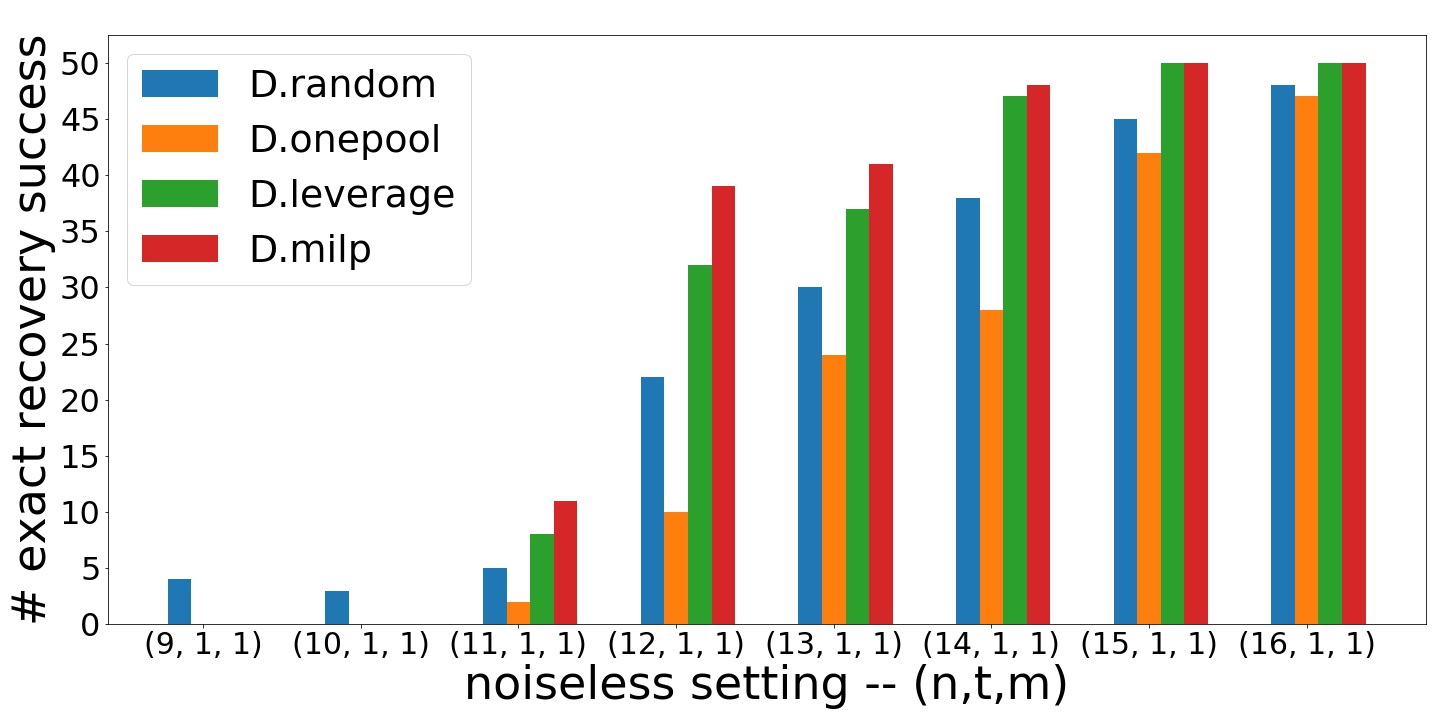}
\end{minipage}
~
\begin{minipage}[c]{0.48\textwidth}
\centering
\includegraphics[width = \linewidth]{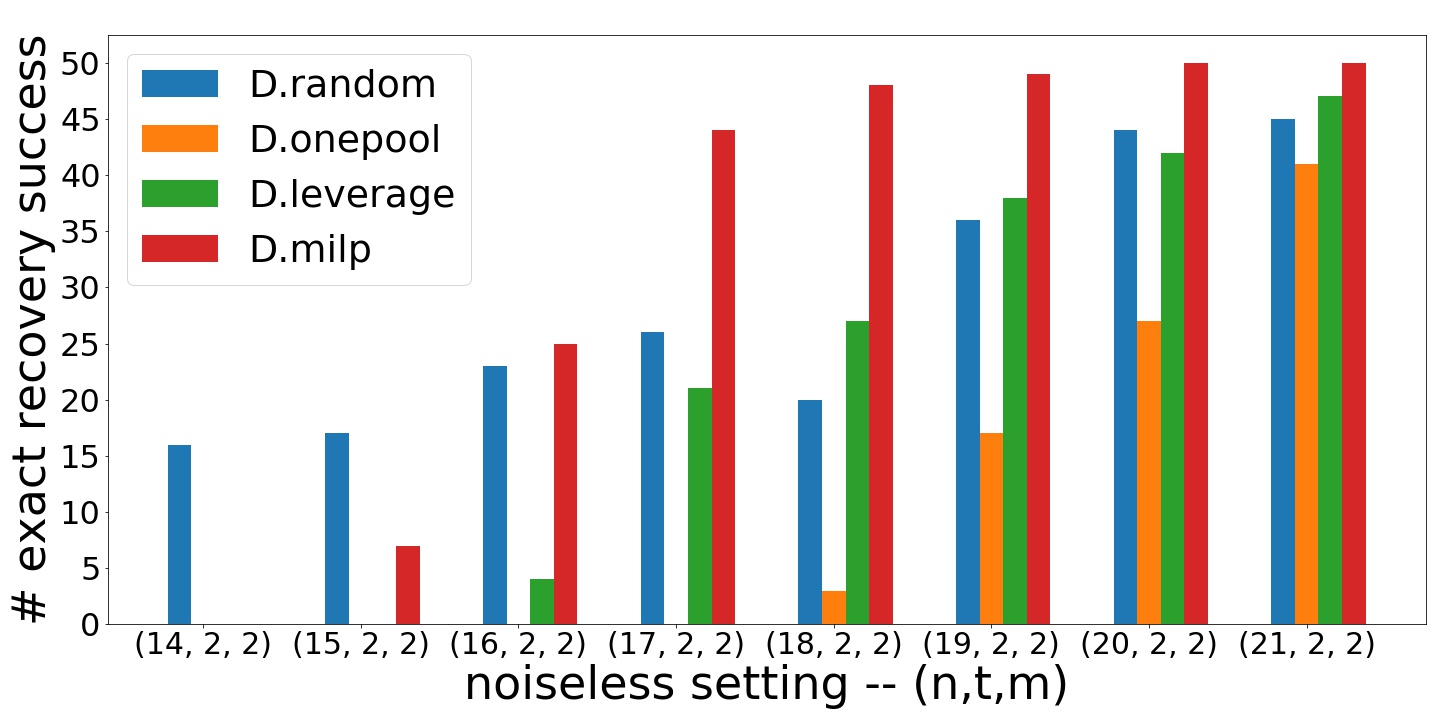}
\end{minipage}
\caption{Comparison between D.milp and other debugging strategies in noiseless settings. Each setting is an average over 50 random trials.}
\label{fig:noiseless-compare-optimal-vs-leverage}
\end{figure}

In the noisy setting, we define $\betastar$ to be the least squares solution computed using the entire data set. We randomly select $n$ data points as the $x_i$'s. For D.milp and D.leverage, since the bug generator knows their strategies or the selected $D$, it generates bugs via the optimization problem~\eqref{eq:bug-generator}: taking $\gammastar_T = u_T$ if the solution $u$ is nonzero for $T$ being the indices of the largest $t$ elements of $|u|$, and otherwise randomly generating a subset $T$ of size $t$ to create $\gammastar_T = \vec{1}$. Thus, $y_i = x_i^\top \betastar + \gammastar + \mathcal{N}(0,0.01)$. Note that having $\gammastar_T = u_T$ if the solution $u$ is nonzero gives incorrect signed support recovery, which is proved in Appendix~\ref{app:sign}. This is related to what we have claimed in Remark~\ref{remark:sign-thm-active} above. For D.onepool, the bug generator follows the above description with $D = \emptyset$. The orange bars indicate whether the bug generator succeeds in exact recovery in the one-pool case. For D.random, since it is not deterministic, the bug generator does not know $D$ and acts in the same way as in the one-pool case. Note that the above bug generating methods are the ``worst'' in the sense of signed support recovery. From Figure~\ref{fig:noisy-compare-optimal-vs-leverage}, there is an obvious advantage of D.milp over D.onepool and D.leverage. Our theory only guarantees the success of D.milp in the \emph{noiseless} setting, so the experimental results for the noisy setting are indeed encouraging.

\begin{figure}[htp!]
\centering
\begin{minipage}[c]{0.48\textwidth}
\centering
\includegraphics[width = \linewidth]{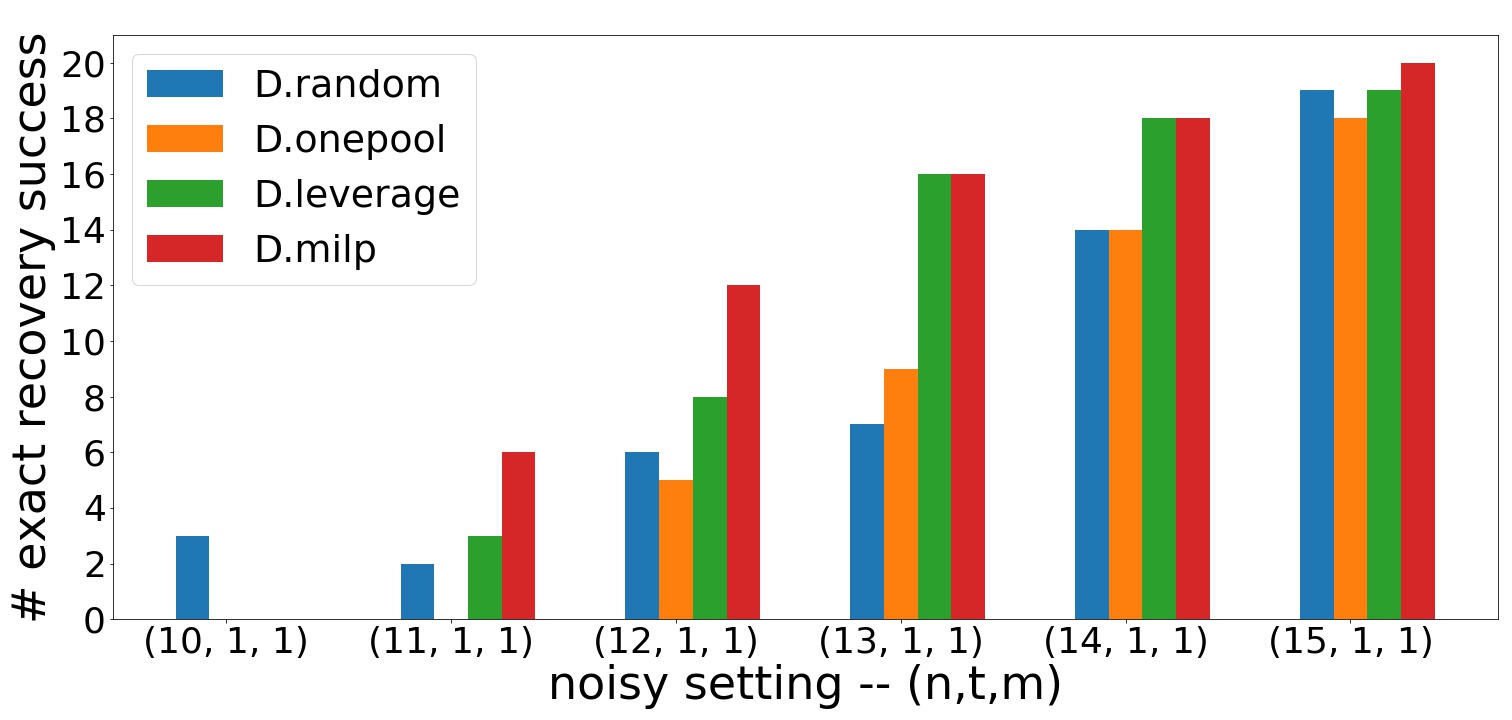}
\end{minipage}
~
\begin{minipage}[c]{0.48\textwidth}
\centering
\includegraphics[width = \linewidth]{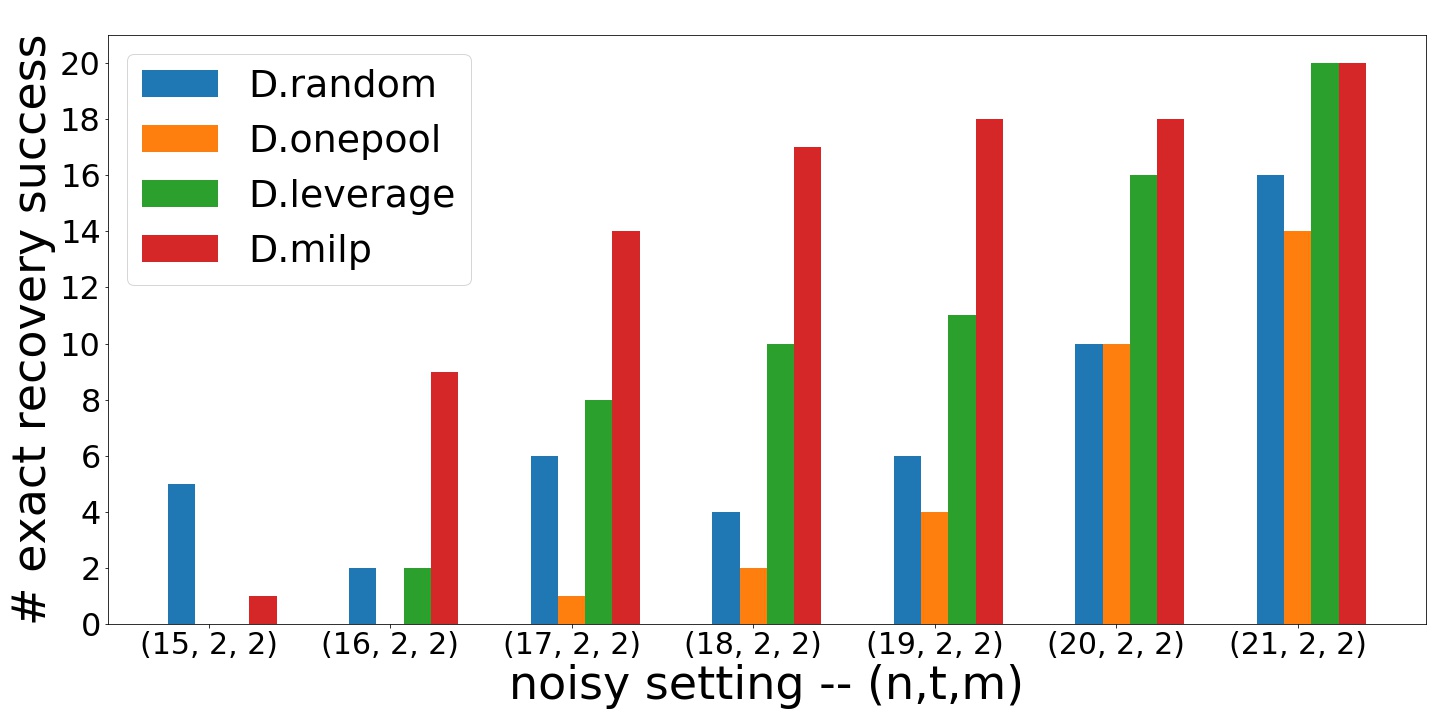}
\end{minipage}
\caption{Comparison between MILP Strategy and Others. In each setting, we run 20 random simulations.}
\label{fig:noisy-compare-optimal-vs-leverage}
\end{figure}

\textbf{Debugging in practice:}
The algorithm for minimax optimization has been executed by running all $n\choose m$ possible choices of clean points for the outer loop; for each outer loop, we then run the inner maximization. For optimal debugging in practice, i.e., $n,t$, and $m$ being large, some recent work provides methods for efficiently solving the minimax MILP~\cite{tang2016class}. 
Note that the MILP debugger can be easily combined to other heuristic methods: one can run the MILP, and if there is a nonzero solution, we can follow it to add clean points. Otherwise, we can switch to other methods, such as choosing random points or high-leverage points.

\section{CONCLUSION}
\label{sec:conclusion}
We have developed theoretical results for machine learning debugging via $M$-estimation and discussed sufficient conditions under which support recovery may be achieved. As shown by our theoretical results and illustrative examples, a clean data pool can assist debugging. We have also designed a tuning parameter algorithm which is guaranteed to obtain exact support recovery when the design matrix satisfies a certain concentration property. Finally, we have analyzed a competitive game between the bug generator and the debugger, and analyzed a mixed integer optimization strategy for the debugger. Empirical results show the success of the tuning parameter algorithm and proposed debugging strategy.

Our work raises many interesting future directions. First, the question of how to optimally choose the weight parameter $\eta$ remains open. Second, although we have mentioned several efficient algorithms for bilevel mixed integer programming, we have not performed a thorough comparison of these algorithms for our specific problem. Third, although our MILP strategy for second pool design has been experimentally found to be effective in a noisy setting, we do not have corresponding theoretical guarantees. Fourth, our proposed debugging strategy is a one-shot method, and designing adaptive methods for choosing the second pool constitutes a fascinating research direction. Finally, the analysis of our tuning parameter algorithm suggests that a geometrically decreasing series might be used as a grid choice for more general tuning parameter selection methods, e.g., cross validation---in practice, one may not need to test candidate parameters on a large grid chosen linearly from an interval. Lastly, it would be very interesting to extend the ideas in this work to regression or classification settings where the underlying data do not follow a simple linear model.

\appendixwithtoc
\newpage

The supplmentary materials is organized as follows: Section~\ref{sec:app-additional-discussion} presents some additional discussions on $\beta$. Section~\ref{AppSecFormu}, Section~\ref{AppSuppRec}, Section~\ref{AppSecTune} and Section~\ref{AppSecActive} mainly provide proofs respectively for problem reformulation and support recovery, tuning parameter selection and strategy for second pool selection. They may also include additional discussions and formal statements as referred in the main text.

\section{Additional Discussions}\label{sec:app-additional-discussion}
We present more miscellaneous discussions here to readers who may care about $\beta$.

\textbf{\emph{Debugging connection to $\beta$.}} Throughout this paper, we have focused on estimating $\gamma$ for the purpose of debugging. A result concerning how the second pool can be used to obtain a better estimate of $\beta$ is as follows: 
\begin{proposition}\label{act1beta}
Let $X = USV^\top$ and $\Xtil = \Stil V_0^\top$. Let $m < p$. It holds that
\begin{align}
\|V_0(\betahat - \betastar)\|_2 \leq \frac{c_1 \sigma  \sqrt{m}}{\sqrt{L}\sigma_{\min}(\tilde{S})} + \lambda n \|\tilde{S}^{-2} V_0 V S U z_{\gammahat}\|_2,
\end{align}
where $z_{\gammahat}$ is the subgradient of $\|\gammahat\|_1$.
\end{proposition}
\begin{proof}[Proof of Proposition \ref{act1beta}]
Recall the objective function \eqref{EqnObj2} is 
\begin{align*}
(\betahat, \gammahat) \in \arg\min_{\substack{\beta \in \real^p,\\ \gamma \in \real^n}} \Bigg\{\frac{1}{2n} \|y - X\beta - \gamma\|_2^2 
+ \frac{\eta}{2m} \|\ytil - \Xtil \beta\|_2^2 + \lambda \|\gamma\|_1\Bigg\}.
\end{align*}
By KKT conditions of the objective function, 
\begin{align}\label{kk0}
\begin{split}
\nabla_{\beta} &= - \frac{1}{n} X^\top (y-X\betahat -\gammahat) - \frac{\eta}{m}\Xtil^\top (\ytil - \Xtil \betahat) = 0;\\
\nabla_{\gamma} &= -\frac{1}{n} (y - X\betahat - \gammahat) + \lambda \partial |\gammahat| = 0.
\end{split}
\end{align}
Plug $y = X\betastar + \gammastar + \epsilon$ and $\ytil = \Xtil \betastar + \epsilontil$ into \eqref{kk0} we obtain
\begin{subequations}
\begin{equation}\label{kk1}
- \left(\frac{1}{n}X^\top X + \frac{\eta}{m}\Xtil^\top \Xtil \right)(\betastar - \betahat) - \frac{1}{n} X^\top (\gammastar - \gammahat) -\frac{1}{n}X^\top \epsilon - \frac{\eta}{m}\Xtil^\top \epsilontil = 0;
\end{equation}
\begin{equation}\label{kk2}
-\frac{1}{n} X(\betastar - \betahat) -\frac{1}{n}(\gammastar - \gammahat) - \frac{1}{n}\epsilon + \lambda \partial |\gammahat| = 0. 
\end{equation}
\end{subequations}
Mutiply $X^\top$ on \eqref{kk2} and plug it into \eqref{kk1} we get 
\begin{equation}
\Xtil^\top \Xtil(\betahat - \betastar) = \lambda \frac{m}{\eta}X^\top \partial |\gammahat| + \Xtil \epsilontil.
\end{equation}
Given that $\Xtil = \Stil V_0^\top$,
$$\Stil^\top \Stil V_0^\top (\betahat - \betastar) = \lambda \frac{m}{\eta}V_0^\top X^\top \partial |\gammahat| + V_0^\top V_0 \Stil \epsilontil.$$
Plugging into the SVD of $X = US V^\top$, we have
\begin{align*}
\begin{split}
\left\|V_0^\top (\betahat - \betastar)\right\|_2 &\leq \lambda \frac{m}{\eta} \left\|(\Stil^\top \Stil)^{-1}V_0^\top X^\top \partial |\gammahat| \right\|_2 + \|(\Stil^\top \Stil)^{-1}\Stil\|\|\epsilontil\|_2\\
&\leq \lambda \frac{m}{\eta} \left\|(\Stil^\top \Stil)^{-1}V_0^\top V S U^\top \partial |\gammahat| \right\|_2 + c_1\frac{\sqrt{m}\sigma}{\sqrt{L}\sigma_{\min}(\Stil)}\\
&\leq \lambda \frac{m}{\eta} \left\|(\Stil^\top \Stil)^{-1}V_0^\top V S U^\top\right\|_2 \sqrt{n} + c_1\frac{\sqrt{m}\sigma}{\sqrt{L}\sigma_{\min}(\Stil)}\\
& \leq c\sigma\sqrt{\frac{\log n}{n}}\frac{m}{\eta} \left\|(\Stil^\top \Stil)^{-1}S_0^{1/2}\right\|_2 + c_1\frac{\sqrt{m}\sigma}{\sqrt{L}\sigma_{\min}(\Stil)},
\end{split}
\end{align*}
with probability at least $1 - \exp(-cm)$. The second step is because $\sigmatil$ has subgaussian parameter $\sigma^2/L$.
\end{proof}
Note that when $\Stil$ is chosen large enough, then $
\|V_0(\betahat - \betastar)\|_2$ is controlled to a small number. 
Besides, if the subspace $V_0$ contains the buggy subspace of $X_T$, then $\|y_T-y_T^*\|_2$ is well controlled and we can spot the contaminated points. This, together with the orthogonal design we will discuss in Section~\ref{app:ortho-design}, suggests that a successful debugging strategy may be obtained by producing a carefully chosen interaction between the non-buggy subspace (augmented using a second pool of clean data points) and the buggy subspace.

\textbf{\emph{Related work \cite{she2011outlier}.}} Without the second pool,~\cite{she2011outlier} demonstrated the equivalence of the solution $\betahat$ to the joint optimization of the objective~\eqref{EqnObj2} over $(\beta, \gamma)$ to the optimum of a regression $M$-estimator in $\beta$ with the Huber loss. This motivates the question of whether the optimizer $\betahat$ of the objective~\eqref{EqnObj2} may similarly be viewed as the optimum of an $M$-estimation problem. 
\begin{proposition}
\label{PropWeightedM}
The solution $\betahat$ of the joint optimization problem~\eqref{EqnObj2} is the unique optimum of the following weighted $M$-estimation problem:
\begin{equation}
\label{EqnWeightedM}
\min_{\beta \in \real^p} \Big\{\frac{1}{n}\sum_{i=1}^n \ell_{n\lambda}\left(y_i - x_i^\top \beta\right) + \frac{\eta}{2m} \|\ytil - \Xtil \beta\|_2^2\Big\}.
\end{equation}
\end{proposition}

\begin{proof}
Recall the definition of the Huber loss function:
\begin{equation*}
\ell_k(u) = \begin{cases}
\lambda |u| - \frac{k^2}{2}, & \text{if } |u| > k, \\
\frac{u^2}{2}, & \text{if } |u| < k.
\end{cases}
\end{equation*}
We will show the desired equivalence via the KKT conditions for both objective functions. Taking gradients with respect to $\beta$ and $\gamma$ for the original objective function~\eqref{EqnObj2}, we obtain the following system of equations:
\begin{align}
0 & = \frac{X^\top X}{n} \beta - \frac{X^\top (y- \gamma)}{n} + \eta \left(\frac{\Xtil^\top \Xtil}{m} \beta - \frac{\Xtil^\top  \ytil}{m}\right), \label{EqnKKT1} \\
0 & = \frac{\gamma}{n} - \frac{y - X\beta}{n} + \lambda \sign(\gamma). \label{EqnKKT2}
\end{align}
The second equation~\eqref{EqnKKT2} has a unique solution, given by the soft-thresholding function:
\begin{equation*}
\gamma = \thresh_{n\lambda}\left(y-X\beta\right),
\end{equation*}
where for scalars $u, k \in \real$, we have
\begin{equation*}
\thresh_k(u) = \begin{cases}
u - \lambda\sign(u), & \text{if } |u| \ge k, \\
0, & \text{if } |u| < k,
\end{cases}
\end{equation*}
and $\thresh_k$ acts on vectors componentwise. Plugging back into equation~\eqref{EqnKKT1}, we obtain
\begin{align}
\label{EqnFirst}
0 = X^\top  \left(\frac{X\beta - y}{n} + \frac{1}{n} \thresh_{n\lambda}\left(y-X\beta\right)\right)
+ \eta \left(\frac{\Xtil^\top  \Xtil}{m} \beta - \frac{\Xtil^\top  \ytil}{m}\right).
\end{align}
We now consider the KKT conditions for the weighted $M$-estimator~\eqref{EqnWeightedM}. Taking a gradient with respect to $\beta$, we obtain
\begin{equation}
\label{EqnSecond}
0 = -\sum_{i=1}^n \ell_{n\lambda}'\left(y_i - x_i^\top  \beta\right) \frac{x_i}{n} + \eta\left(\frac{\Xtil^\top  \Xtil}{m} \beta - \frac{\Xtil^\top  \ytil}{m}\right).
\end{equation}
The key is to note that
\begin{equation*}
u -\ell'_{n\lambda}(u)  = \thresh_{n\lambda}(u),
\end{equation*}
so
\begin{align*}
- \ell_{n\lambda}'\left(y_i - x_i^\top  \beta\right)\frac{1}{n} 
= \frac{x_i^\top  \beta - y_i}{n} + \frac{1}{n} \thresh_{n\lambda} \left(y_i - x_i^\top  \beta\right),
\end{align*}
from which we may infer the equivalence of equations~\eqref{EqnFirst} and~\eqref{EqnSecond}. This concludes the proof.
\end{proof}
The proposition also illustrates that the objective uses Huber loss to get the robust estimation $\betahat$, and then imply the estimation $\gammahat$. Therefore, estimations of $\beta$ and $\gamma$ complement each other. Our reformulation more relies on giving a direct analysis of $\gamma$ and its support.

\section{Appendix for Section~\ref{sec:formu}}\label{AppSecFormu}

We show reformulation of the objective function in this section.

\begin{proof}[Proof of Proposition~\ref{prop:obj-reform}]
Using the notation~\eqref{EqnStack}, we can translate~\eqref{EqnObj2} into 
\begin{equation}\label{eq:reformu-1}
(\betahat, \gammahat) \in \arg\min_{\beta, \gamma} \left\{\frac{1}{2n} \left\|y' - X'\beta - \begin{bmatrix}\gamma \\ \vec{0}_m\end{bmatrix}\right\|_2^2 + \lambda \|\gamma\|_1\right\},
\end{equation}

First note that we can split $y'-X'\beta-\begin{bmatrix}\gamma \\ \vec{0}_m\end{bmatrix}$ into two parts by projecting onto the column space of $X'$ and the perpendicular space:
\begin{align*}
\begin{split}
\left\|y' - X'\beta - \begin{bmatrix}\gamma \\ \vec{0}_m\end{bmatrix}\right\|_2^2 & = \left\|P_{X'}\left(y' -  X'\beta - \begin{bmatrix}\gamma \\ \vec{0}_m\end{bmatrix}\right)\right\|_2^2 + \left\|P_{X'}^\perp \left(y' - X'\beta - \begin{bmatrix}\gamma \\ \vec{0}_m\end{bmatrix}\right)\right\|_2^2 \\ &= \left\|P_{X'}\left(y' - X'\beta - \begin{bmatrix}\gamma \\ \vec{0}_m\end{bmatrix}\right)\right\|_2^2 + \left\|P_{X'}^\perp\left(y' - \begin{bmatrix}\gamma \\ \vec{0}_m\end{bmatrix}\right)\right\|_2^2.
\end{split}
\end{align*}
For any value of $\gammahat$, we can choose $\betahat$ such that $\left\|P_{X'}\left(y' - X'\betahat - \begin{bmatrix}\gamma \\ \vec{0}_m\end{bmatrix}\right)\right\|_2^2 = 0$, simply by taking $\betahat = (X'^\top X')^{-1} X'^\top \left(y' - \begin{bmatrix}\gammahat \\ \vec{0}_m\end{bmatrix}\right)$. Hence, we get 
\begin{align*}
\left\|y' - X'\beta - \begin{bmatrix}\gammahat \\ \vec{0}_m\end{bmatrix}\right\|_2^2 = \left\|P_{X'}^\perp\left(y' - \begin{bmatrix}\gammahat \\ \vec{0}_m\end{bmatrix}\right)\right\|_2^2 = \left\|P_{X'}^\perp y' - \Pbar\gammahat\right\|_2^2,
\end{align*}
and~\eqref{eq:reformu-1} becomes
\begin{align*}
\begin{split}
\gammahat & \in \frac{1}{2n}\left\|P_{X'}^\perp y' - \Pbar\gammahat\right\|_2^2 + \lambda \|\gammahat\|_1, \\
\betahat & = (X'^\top X')^{-1} X'^\top \left(y' - \begin{bmatrix}\gammahat \\ \vec{0}_m\end{bmatrix}\right).
\end{split}
\end{align*}
Therefore, the two optimization problems share the same solution for $\gammahat$.
\end{proof}

\section{Appendix for Section~\ref{sec:supprec}}
\label{AppSuppRec}

\paragraph{Notations in appendix:} We write $P_{X',TT}^\perp$ to represent the submatrix of $P_{X'}^\perp$ with rows and column indexed by $T$. We write $P_{X',T\cdot}^\perp$ to represent the submatrix of $P_{X'}^\perp$ with rows indexed by $T$ and $P_{X',\cdot T}^\perp$ to represent the submatrix of $P_{X'}^\perp$ with columns indexed by $T$. For simplicity, let $\Pbar = P_{X'}^\perp M_{[n]}$. We slightly abuse notation by using $\Pbar_{T}$ and $\Pbar_{T^c}$ to denote $\Pbar_{\cdot T}$ and $\Pbar_{\cdot T^c}$, respectively. 

In this appendix, we provide proofs and additional details for the results in Section~\ref{sec:supprec}. The proofs for fixed design are in Section~\ref{app:proofmainthms}. We discuss orthogonal design in Section~\ref{app:ortho-design} and sub-Gaussian design in Section~\ref{app:subgaussian-design}. In particular, we use the two special designs to better understand the three assumptions and see how having a clean pool helps with the support recovery. We will call one-pool case the setting with only contaminated pool and call two-pool case the setting with both data pools.

\subsection{Proofs of Theorem~\ref{subsetmainthm} and Theorem~\ref{eaxtsetmainthm}}
\label{app:proofmainthms}

\begin{proof}[Proof of Theorem~\ref{subsetmainthm}]
We follow the usual Primal Dual Witness argument for support recovery in linear regression, which contains the following steps~\cite{wainwright2009sharp}:
\begin{itemize}
    \item[1.] Set $\gammahat_{T^c} = 0$.
    \item[2.] Solve the oracle subproblem for $(\gammahat_T, \hat{z}_{T})$:
    \begin{align}\label{oraclesub}
        \gammahat_T \in \arg \min_{\gamma \in \real^t} \left\{\frac{1}{2n}\norm{Ay' - B\gamma}^2 + \lambda \lnorm{\gamma}\right\},
    \end{align}
    and choose $\hat{z}_T \in \partial{\lnorm{\gammahat_T}}$. In the one data pool case, we have $A = P_{X,\cdot T}^{\perp}$ and $B = P_{X,\cdot T}^{\perp}$; in the two data pool case, we have $A = P_{X',\cdot T}^\perp$ and $B = \Pbar_T$.
    \item[3.] Solve $\hat{z}_{T^c}$ via the zero-subgradient equation, and check whether the strict dual feasibility condition holds: $\infnorm{\zhat_{T^c}} < 1$.
\end{itemize}
As in the usual Lasso analysis~\cite{wainwright2009sharp}, under the eigenvalue condition~\eqref{C1-2}, $(\gammahat_T, 0) \in \real^n$ is the unique optimal solution of the Lasso, where $\gammahat_T$ is the solution obtained by solving the oracle subproblem~\eqref{oraclesub}.

The focus of our current analysis is to verify the conditions under which the strict dual feasibility condition holds. The KKT conditions for equation~\eqref{EqnPenGamma} may be rewritten as
\begin{align}\label{PDW1}
    \Pbar_T^\top\Pbar_T(\hat{\gamma}_{T} - \gamma^*_{T}) - \Pbar_T^\top P_{X'}^\perp \epsilon' + n \lambda \hat{z}_T = 0,\\
    \Pbar_{T^c}^\top\Pbar_T(\hat{\gamma}_{T} - \gamma^*_{T}) -  \Pbar_{T^c}^\top P_{X'}^\perp \epsilon' + n \lambda  \hat{z}_{T^c} = 0, \label{PDW2}
\end{align}
where $\hat{z}_T \in \partial\lnorm{\gammahat_T}, \hat{z}_{T^c} \in \partial\lnorm{\gammahat_{T^c}}$.

We will use the following equations to simplify terms later:
\begin{align*}
\Pbar_T^\top\Pbar_T = (P_{X'}^{\perp \top} P_{X'}^\perp)_{TT}, \quad
\left(\begin{array}{c} \Pbar_T^\top P_{X'}^\perp \epsilon' \\ \Pbar_{T^c}^\top P_{X'}^\perp \epsilon' \end{array}\right) = \Pbar^{\top} P_{X'}^\perp \epsilon' = \Pbar^\top \epsilon' = \left(\begin{array}{c} \Pbar_T^\top \epsilon' \\ \Pbar_{T^c}^\top \epsilon' \end{array}\right).
\end{align*}
Since $\Pbar_T^\top\Pbar_T$ is invertible by condition \eqref{C1-2}, we can multiply equation~\eqref{PDW1} by $\left(\Pbar_T^\top\Pbar_T\right)^{-1}$ on the left to obtain
\begin{align}
\label{gammaerr}
    \hat{\gamma}_{T} - \gamma^*_{T} =  (\Pbar_T^\top\Pbar_T)^{-1}\Pbar_T^\top \epsilon' -n  \lambda (\Pbar_T^\top\Pbar_T)^{-1}\hat{z}_T.
\end{align}
Plugging this into equation~\eqref{PDW2}, we then obtain
\begin{align*}
    \hat{z}_{T^c} = 
    -\frac{1}{n \lambda}\Pbar_{T^c}^\top\Pbar_T \left[(\Pbar_T^\top\Pbar_T)^{-1}\Pbar_T^\top\epsilon' - n \lambda(\Pbar_T^\top\Pbar_T)^{-1}\hat{z}_T \right]  + \frac{1}{n \lambda }\Pbar_{T^c}^\top\epsilon',
\end{align*}
or
\begin{align}\label{proof-formu1}
    \hat{z}_{T^c} = \underbrace{\Pbar_{T^c}^\top\Pbar_T(\Pbar_T^\top\Pbar_T)^{-1}\hat{z}_T}_{\mu} 
    + \underbrace{\Pbar_{T^c}^\top\Big(I - \Pbar_T(\Pbar_T^\top\Pbar_T)^{-1}\Pbar_T^\top\Big)\frac{\epsilon'}{n \lambda }}_{V_{T^c}}.
\end{align}
We need to show that $\|\hat{z}_{T^c}\|_\infty < 1$. \\
Note that condition \eqref{C2-2} gives us
\begin{align*}
    \exists \alpha' \in [0,1),\ \| \mu \|_\infty = \max_{j \in T^c} \|\Pbar_j^\top \Pbar_T(\Pbar_T^\top\Pbar_T)^{-1}\|_1 \leq \alpha'.
\end{align*}
Furthermore, since
\begin{align*}
\lambda \geq \frac{1}{1-\alpha'} \left\|\Pbar_{T^c}^\top\Big(I - \Pbar_T(\Pbar_T^\top\Pbar_T)^{-1}\Pbar_T^\top\Big)\frac{\epsilon'}{n}\right\|_\infty,
\end{align*} 
we have
\begin{align*}
\|V_{T^c}\|_\infty \leq \frac{1-\alpha'}{2}.
\end{align*}
Combining these inequalities, we obtain strict dual feasibility:
\begin{align*}
\|\hat{z}_{T^c}\|_\infty \leq \|\mu\|_\infty + \|V_{T^c}\|_\infty < 1.
\end{align*}
In addition, applying the triangle inequality to the RHS of equation~\eqref{gammaerr}, we obtain
\begin{align*}
    G' = \|(\Pbar_T^\top\Pbar_T)^{-1}\Pbar_T^\top \epsilon'\|_\infty + n \lambda \|(\Pbar_T^\top\Pbar_T)^{-1}\hat{z}_T\|_\infty  \geq \|\hat{\gamma}_{T} - \gamma^*_{T}\|_\infty.
\end{align*}
This concludes the proof.
\end{proof}
%%%%%

\begin{proof}[Proof of Theorem~\ref{eaxtsetmainthm}]
Note that
\begin{align*}
    \forall i \in T,\quad |\gamma^*_i| - |\hat{\gamma}_i| \leq \|\hat{\gamma}_{T} - \gamma^*_{T}\|_\infty \le G',
\end{align*}
where the last inequality uses Theorem~\ref{subsetmainthm}. Thus, if condition \eqref{C3-2} also holds, we have
\begin{align*}
\forall i \in T, \quad |\hat{\gamma}_i| \geq \min_{i\in T}|\gamma^*_i| - \|\hat{\gamma}_{T} - \gamma^*_{T}\|_\infty \geq \min_{i\in T}|\gamma^*_i| - G'> 0,
\end{align*}
concluding the proof.
\end{proof}

\subsection{Orthogonal design}\label{app:ortho-design}
\subsubsection{Main results for orthogonal design}
In this section, we focus on a special case, where our data have an orthogonal property. Let $X = \begin{bmatrix} RQ^\top \\ FQ^\top \end{bmatrix} \in \real^{(t+p) \times p}, \Xtil = WQ^\top \in \real^{p \times p}$, where $Q$ is an orthogonal matrix with columns $q_1, q_2, \cdots, q_p$, $F,\, W$ are diagonal matrices with diagonals $f_i$'s and $w_i$'s separately ($i \in [p]$), and $R =
\begin{bmatrix}
  \begin{matrix}
  r_1 & 0 & 0 & 0 \\
  0 & r_2 & 0 & 0\\
  0 & 0 & \cdots & 0\\
  0 & 0 & 0 & r_t
  \end{matrix}
  & \rvline & \bigzero_{t\times (p-t)} \\
\end{bmatrix}.$
We assume for all $i \in [p]$, $r_i \neq 0, f_i \neq 0$. Consider the first $t$ points are buggy and the rest $p$ points are nonbuggy, i.e.,  $X_T = RQ^\top \in \real^{t \times p}, X_{T^c} = FQ^\top \in \real^{p \times p}$.

Applying Theorems~\ref{subsetmainthm} and \ref{eaxtsetmainthm}, we obtain Propositions~\ref{prop:orthogonaldesign_subset_recovery} and \ref{prop:orthogonaldesign_exact_recovery}. 
% The proofs of the propositions are contained in Appendix~\ref{AppProofOrth}.

\begin{proposition}
\label{prop:orthogonaldesign_subset_recovery}
In the one-pool case, suppose we choose 
\begin{align}\label{lamone_orthogonal}
\lambda \geq \frac{2\sigma}{n(1-\alpha)} \left(\sqrt{\log 2(n-t)}+C\right),
\end{align} for some constant $C > 0$, and 
\begin{align}\label{subset_one_pool}
\alpha = \max_{1 \leq i \leq t}\left|\frac{r_i}{f_i}\right| < 1.
\end{align}
Then the contaminated pool is capable of achieving subset support recovery with probability at least $1 - e^{-\frac{C^2}{2}}$.

In the two-pool case, suppose we choose 
{\small\begin{align}
\label{lamtwo_orthogonal}
\lambda \geq \frac{2\sigma}{n(1-\alpha')} \max\left\{1, \sqrt{\frac{\eta n}{mL}}\right\}\left(\sqrt{\log 2(n-t)}+C'\right),
\end{align}}
for some constant $C' > 0$, and 
\begin{eqnarray}
\label{subset_two_pool}
\alpha' = \max_{1 \leq i \leq t}\left|\frac{r_i f_i}{f_i^2 + \eta \frac{n}{m}w_i^2}\right| < 1.
\end{eqnarray}

Then adding clean points will achieve subset support recovery with probability at least $1 - e^{-\frac{C'^2}{2}}$.
\end{proposition}

As stated in Theorems~\ref{subsetmainthm} and \ref{eaxtsetmainthm}, to ensure exact recovery, we also need to impose a gamma-min condition. This leads to the following proposition:

\begin{proposition}
\label{prop:orthogonaldesign_exact_recovery}
In the one-pool case, suppose inequality~\eqref{subset_one_pool} holds. If also
{\small\begin{align}\label{g3}
\min_{1 \leq i \leq t}|\gammastar_i| > \sigma(\sqrt{2\log t} + c) \max_{1 \leq i \leq t}\sqrt{1+\frac{r_i^2}{f_i^2}} \ + \frac{2\sigma}{1-\alpha} \left(\sqrt{\log 2(n-t)}+C\right)\left(1+\max_{1 \leq i \leq t}\frac{r_i^2}{f_i^2}\right),
\end{align}}
then there exists a $\lambda$ to achieve exact recovery, with probability at least $1 - 2e^{-\frac{c^2}{2}} - e^{-\frac{C^2}{2}}$.

In the two-pool case, suppose $\eta \leq \frac{mL}{n}$, and inequality~\eqref{subset_two_pool} holds. If also
{\small\begin{align}\label{g4}
\begin{split}
\min_{1 \leq i \leq t} |\gammastar_i| & \geq \sigma(\sqrt{2\log t}+c)\sqrt{1+\max_{1\leq i \leq t}\frac{r_i^2(Lf_i^2+\frac{\eta n}{m}w_i^2)}{L(f_i^2 + \frac{\eta n}{m}w_i^2)^2}}  + \frac{2\sigma}{1-\alpha'} \left(\sqrt{\log 2(n-t)} +C\right)\left(1+\max_{1\leq i \leq t}\frac{r_i^2}{f_i^2+\frac{\eta n}{m}w_i^2}\right),
\end{split}
\end{align}}
then there exists a $\lambda$ to achieve exact recovery, with probability at least $1 - 2e^{-\frac{c^2}{2}} - e^{-\frac{C^2}{2}}$.
\end{proposition}

Compare~\eqref{subset_one_pool} and~\eqref{subset_two_pool}. Mutual incoherence is decreased from $\frac{r_i^2}{f_i^2}$ to $\frac{r_i^2}{f_i^2 + \frac{\eta n}{m}w_i^2}$. 
Compare~\eqref{g3} and~\eqref{g4}. The second $\max$ term, $\underset{1 \leq i \leq t}{\max}\frac{r_i^2}{f_i^2} \geq \underset{1 \leq i \leq t}{\max}\frac{r_i^2(Lf_i^2+\frac{\eta n}{m}w_i^2)}{L(f_i^2 + \frac{\eta n}{m}w_i^2)^2}$, because
\begin{equation*}
\underset{1 \leq i \leq t}{\max}\frac{r_i^2}{f_i^2} \geq \underset{1 \leq i \leq t}{\max}\frac{r_i^2(f_i^2+\frac{\eta n}{m}w_i^2)}{(f_i^2 + \frac{\eta n}{m}w_i^2)^2} \geq \underset{1 \leq i \leq t}{\max}\frac{r_i^2(Lf_i^2+\frac{\eta n}{m}w_i^2)}{L(f_i^2 + \frac{\eta n}{m}w_i^2)^2}
\end{equation*}
when $L \geq 1$. Also note that $\frac{1}{1 - \alpha} > \frac{1}{1-\alpha'}$. Altogether, the requirement of $\min_{i \in [t]}|\gammastar_i|$ is weakened by introducing clean points.
Thus, we see that the mutual incoherence improves in two-pool setting. The gamma-min condition imposes a lower bound of $\Omega\left(\sqrt{\log (n-t)}\right)$ on the signal-to-noise ratio, $\frac{\min_{i \in [t]} |\gammastar_i|}{\sigma}$, and including second pool reduces the prefactor.

As can be seen, we want $|w_i|$ to be sufficiently large compared to $|f_i|$. However, if $|w_i|$ is bounded, we may instead ensure support recovery by repeating points. In this section, we discuss the effect of repeating points and determine the number of points needed to guarantee correct support recovery.
Suppose
\begin{equation*}
W = \begin{bmatrix}\vec{w_1} & \vec{0} & \cdots & \vec{0}\\ \vec{0} & \vec{w_2} & \cdots & \vec{0} \\ \vdots & \vdots & \ddots & \vdots\\ \vec{0} & \vec{0} & \cdots & \vec{w_p}\end{bmatrix},
\end{equation*}
where $\vec{w_i} = [w_{i1}, \dots, w_{il_i}]^\top$. For the $i^{\text{th}}$ direction $q_i$, we have $k_i$ repeated points with respective weights $w_{i1}, w_{i2}, \dots, w_{il_i}$.

\begin{proposition}
\label{RepeatCleanPoints}
Suppose the scale of clean data points is bounded by $w_B$.
%Adding one data point may not be enough to satisfy the three conditions.
Using $w_{i1}, \dots, w_{il_i}$, where $l_i = \left\lceil \left({\frac{|w_i|}{w_B}}\right)^2\right\rceil$ and $|w_{ij}| = w_B,\ \forall j \in [l_i]$, achieves the same effect on Conditions \ref{C1}, \ref{C2}, and \ref{C3} as adding a single point with scale $w_i$.
\end{proposition}
From Proposition~\ref{RepeatCleanPoints}, we see that to correctly identify the bugs, we can also query multiple points in the same direction if the leverage of a single additional point is not large enough.

\subsubsection{Proofs for orthogonal design}
In this section, we first simplify the three conditions, and then provide the proofs of Propositions~\ref{prop:orthogonaldesign_subset_recovery}, \ref{prop:orthogonaldesign_exact_recovery}, and~\ref{RepeatCleanPoints}.

In the one-pool case, we have
\begin{align*}
\begin{split}
P_{X,TT}^\perp &= I_{t \times t} - X_T(X^\top X)^{-1}X_T^\top \\
& = I_{t\times t} - R(R^\top R + F^\top F)^{-1}R^\top \\
&= diag\left(\frac{f_1^2}{r_1^2+f_1^2},\cdots,\frac{f_t^2}{r_t^2+f_t^2} \right).
\end{split}
\end{align*}
Note that $P_{X,TT}^\perp$ is a diagonal matrix. Thus, the eigenvalues are immediately obtained and  
\begin{align*}
\lambda_{\min}(P_{X,TT}^\perp) = \min_{1 \leq i \leq t}{\frac{f_i^2}{r_i^2 + f_i^2}}  = \min_{1 \leq i \leq t}{\frac{1}{\left(\frac{r_i}{f_i}\right)^2 + 1}} = \frac{1}{\max_{1 \leq i \leq t}\left(\frac{r_i}{f_i}\right)^2 + 1}.
\end{align*}
The condition that $P_{X, TT}^\perp$ is invertible is therefore equivalent to the condition that $f_i \neq 0$ for all $i$. Assuming this is true, we have
\begin{align*}
\begin{split}
P_{X,T^cT}^\perp(P_{X, TT}^\perp)^{-1} &= -F(R^\top R + F^\top F)^{-1} R^\top  \cdot (I_{t\times t} - R(R^\top R + F^\top F)^{-1}R^\top)^{-1} \\
& = \begin{bmatrix} diag\left(-\frac{r_1}{f_1},\cdots, -\frac{r_t}{f_t} \right)_{t\times t} \\ \bm{0}_{(p-t)\times t} \end{bmatrix}.
\end{split}
\end{align*}
The mutual incoherence condition can then be written in terms of the quantity
\begin{align*}
  \opnorm{P_{X,T^cT}^\perp(P_{X, TT}^\perp)^{-1}}_{\infty}  = \max_{1 \leq i \leq t} \left|\frac{r_i}{f_i}\right|  = \max_{1 \leq i \leq t} \left|\frac{r_if_i}{f_i^2}\right|. 
\end{align*}
Note that the mutual incoherence condition also implies that $f_i \neq 0,\ \forall i$, since the mutual incoherence parameter will otherwise go to infinity. \\
The remaining condition is the gamma-min condition. Note that the upper bound on the $\ell_\infty$-error of $\gamma$ consists of two parts:
\begin{align*}
  \| \gammahat - \gammastar \|_\infty \leq \|(P_{X,TT}^\perp)^{-1}(P_{X,T\cdot}^\perp) \epsilon \|_{\infty}
   + n\lambda\opnorm{(P_{X,TT}^\perp)^{-1}}_{\infty}.
\end{align*}
Regarding $P_{X, T\cdot}^\perp $ as two blocks, $\left(P_{X, TT}^\perp,\ P_{X, TT^c}^\perp \right)$, we have
\begin{align*}
 \|(P_{X,TT}^\perp)^{-1}(P_{X,T\cdot}^\perp) \epsilon \|_{\infty} = \left\|\begin{pmatrix} I & (P_{X,TT}^\perp)^{-1}P_{X,TT^c}^\perp\end{pmatrix}\epsilon \right\|_\infty.
 \end{align*}
Altogether, we see that
\begin{align*}
G =  \max_{1 \leq i \leq t}\left|\epsilon_i- \frac{r_i}{f_i}\epsilon_{i+t}\right| + n\lambda\left(\max_{1 \leq i \leq t}\left\{\frac{r_i^2}{f_i^2}\right\} +1\right).
\end{align*}
To summarize, the minimum eigenvalue condition becomes 
\begin{subequations}
\begin{equation}\label{eigenvaluecon_orthogonal_one_pool}
\begin{split}
    \lambda_{\min}(P_{X,TT}^\perp) &= \frac{1}{\max_{1 \leq i \leq t}\left(\frac{r_i}{f_i}\right)^2 + 1} > 0;
\end{split}
\end{equation}
the mutual incoherence condition becomes
\begin{equation}\label{mutualincohcon_orthogonal_one_pool}
    \opnorm{P_{X,T^cT}^\perp(P_{X, TT}^\perp)^{-1}}_{\infty}  = \max_{1 \leq i \leq t} \left|\frac{r_i}{f_i}\right| = \alpha \in [0,1);
\end{equation}
and the gamma-min condition becomes
\begin{equation}\label{gammamincon_orthogonal_one_pool}
    \min_{1 \leq i \leq t}|\gammastar_i| \geq G = \max_{1 \leq i \leq t}|\epsilon_i- \frac{r_i}{f_i}\epsilon_{i+t}| + n\lambda\left(\max_{1 \leq i \leq t}\left\{\frac{r_i^2}{f_i^2}\right\} +1\right).
\end{equation}
\end{subequations}

Similar calculations show that in the two-pool case, the minimum eigenvalue condition becomes
\begin{subequations}
\begin{equation}\label{eigenvaluecon_orthogonal_two_pool}
    \lambda_{\min}(P_{X',TT}^\perp) = \min_{1 \leq i \leq t}{\frac{f_i^2 + \frac{\eta n}{m}w_i^2}{r_i^2 + f_i^2 + \frac{\eta n}{m}w_i^2}} = \frac{1}{\max_{i \in [t]}\frac{r_i^2}{f_i^2+\frac{\eta n}{m}w_i^2} + 1}  > 0;
\end{equation}
the mutual incoherence condition becomes
\begin{equation}\label{mutualincohcon_orthogonal_two_pool}
\begin{split}
    \opnorm{P_{X',T^cT}^\perp(P_{X', TT}^\perp)^{-1}}_{\infty} &= \max_{1 \leq i \leq t} \left|\frac{r_if_i}{f_i^2+\frac{\eta n}{m}w_i^2}\right| = \alpha' \in [0,1);
\end{split}
\end{equation}
and the gamma-min condition becomes
\begin{equation}\label{gammamincon_orthogonal_two_pool}
\min_{1 \leq i \leq t}|\gammastar_i| \geq G',
\end{equation}
where
\begin{equation*}
    G' =\max_{1 \leq i \leq t}\left|\epsilon_i - \frac{r_if_i}{f_i^2 + \frac{\eta n}{m}w_i^2}\epsilon_{i+t} - \frac{\sqrt{\frac{\eta n}{m}}r_iw_i}{f_i^2 + \frac{\eta n}{m}w_i^2}\epsilontil_{i}\right| + n\lambda\left(\max_{1 \leq i \leq t}\left\{\frac{r_i^2}{f_i^2 + \frac{\eta n}{m}w_i^2}\right\} +1\right).
\end{equation*}
\end{subequations}

Here is the proof of Proposition \ref{prop:orthogonaldesign_subset_recovery}.
\begin{proof}[Proof of Proposition \ref{prop:orthogonaldesign_subset_recovery}]
According to Theorem~\ref{subsetmainthm}, the subset support recovery result relies on two conditions: the minimum eigenvalue condition and the mutual incoherence condition. In the orthogonal design case, we will argue that both inequalities~\eqref{eigenvaluecon_orthogonal_one_pool} and~\eqref{eigenvaluecon_orthogonal_two_pool} hold in the one-pool case, and inequlaity~\eqref{subset_two_pool} is sufficient for both inequalities~\eqref{eigenvaluecon_orthogonal_two_pool} and~\eqref{mutualincohcon_orthogonal_two_pool} in the two-pool case.

For the one-pool case, the assumption~\eqref{subset_one_pool} implies that $f_i \neq 0,\ , \forall i \in [t]$. Note that the minimum eigenvalue condition~\eqref{eigenvaluecon_orthogonal_one_pool} is equivalent to $f_i \neq 0,\ , \forall i \in [t]$. Hence, the minimum eigenvalue condition holds. Furthermore, the mutual incoherence condition~\eqref{eigenvaluecon_orthogonal_two_pool} clearly holds. 

For the two-pool case, if $f_i = 0$ for some $i \in [t]$, then plugging into~\eqref{subset_two_pool} implies that $w_i^2 > 0$. Thus, $f_i$ and $w_i$ cannot be zero at the same time, implying that the eigenvalue condition~\eqref{eigenvaluecon_orthogonal_two_pool} holds. Note that inequality~\eqref{subset_two_pool} is equivalent to inequlaity~\eqref{mutualincohcon_orthogonal_two_pool}.

The remaining of the argument concerns the choice of $\lambda$. Note that Theorem~\ref{subsetmainthm} requires $\lambda$ to be lower-bounded for subset recovery (see inequality~\eqref{tdplambdacond1}). Taking the two-pool case as an example, we will show that when inequality~\eqref{lamtwo_orthogonal} holds, inequality~\eqref{tdplambdacond1} holds with high probability. Define
\begin{align*}
\begin{split}
Z_j = \Pbar_{\cdot j}^\top \Big(I - \Pbar_T(\Pbar_T^\top \Pbar_T)^{-1}\Pbar_T^\top \Big)\frac{\epsilon'}{n},\quad j \in T^c.
\end{split}
\end{align*}
Note that $\norm{\Pbar_{\cdot j}^\top \Big(I - \Pbar_T(\Pbar_T^\top \Pbar_T)^{-1}\Pbar_T^\top \Big)} \leq 1$ for all $j \in T^c, $ and $\epsilon' = \begin{pmatrix} \epsilon \\ \sqrt{\frac{\eta n}{m}}\epsilontil \end{pmatrix}$ has i.i.d.\ sub-Gaussian entries with parameter at most $\max\{1,\frac{\eta n}{mL}\}\sigma^2$. Thus, $Z_j$ is sub-Gaussian with parameter at most $\max\{1,\frac{\eta n}{mL}\}\frac{\sigma^2}{n^2}$. By a sub-Gaussian tail bound (cf.\ Lemma~\ref{lem:noniidconcen}), we then have
\begin{align*}
\begin{split}
\mathbb{P}\left(\max_{j \in T^c} |Z_j| \geq \delta_0 \right) \leq 2(n-t)\exp{\left(-\frac{n^2\delta_0^2}{2\max\{1,\frac{\eta n}{mL}\}\sigma^2}\right)}.
\end{split}
\end{align*}

Let $C'$ be a constant such that
\begin{align*}
\begin{split}
2(n-t)\exp{\left(-\frac{n^2\delta_0^2}{2\max\{1,\frac{\eta n}{mL}\}\sigma^2}\right)} = \exp{\left(-\frac{C'^2}{2}\right)},
\end{split}
\end{align*}
and define
\begin{align*}
\begin{split}
\delta_0 := \frac{\sigma}{n} \max\{1, \sqrt{\frac{\eta n}{mL}}\}\sqrt{\log 2(n-t)+C'^2}.
\end{split}
\end{align*}
Note that we want
\begin{align*}
\begin{split}
\frac{2\max_{j \in T^c} |Z_j|}{1 - \alpha'} \leq \lambda,
\end{split}
\end{align*}
which therefore occurs with probability at least $1 - e^{-\frac{C'^2}{2}}$ when 
\begin{align*}
\lambda \geq \frac{2\sigma}{n(1-\alpha')} \max\{1, \sqrt{\frac{\eta n}{mL}}\}\left(\sqrt{\log 2(n-t)}+C'\right) \geq \frac{2\delta_0}{1-\alpha'}.
\end{align*}

The proof for the one-pool case is similar, so we omit the details.
\end{proof}

Here is the proof of Proposition \ref{prop:orthogonaldesign_exact_recovery}.
\begin{proof}[Proof of Proposition \ref{prop:orthogonaldesign_exact_recovery}]
To simplify notation, define
\begin{align*}
\begin{split}
u_i & := \epsilon_i - \frac{r_i}{f_i}\epsilon_{i+t}, \\
v_i & := \epsilon_i - \frac{r_if_i}{f_i^2+\frac{\eta n}{m}w_i^2}\epsilon_{i+t} - \frac{\sqrt{\frac{\eta n}{m}}r_iw_i}{f_i^2+\frac{\eta n}{m}w_i^2}\epsilontil_i.
\end{split}
\end{align*}
Note that $u_i$ is $\sigma_{u_i}$-sub-Gaussian and $v_i$ is $\sigma_{v_i}$-sub-Gaussian, with variance parameters
\begin{align*}
\begin{split}
\sigma_{u_i} = \sqrt{1+\frac{r_i^2}{f_i^2}}\sigma,\quad  \sigma_{v_i} = \sqrt{1+\frac{r_i^2(L^2f_i^2+\frac{\eta n}{m}w_i^2)}{L^2(f_i^2 + \frac{\eta n}{m}w_i^2)^2}}\sigma.
\end{split}
\end{align*}

We now prove two technical lemmas:

\begin{lemma}[Concentration for non-identical sub-Gaussian random variables]\label{lem:noniidconcen}
Suppose $\{u_i\}_{i=1}^t$ are $\sigma_{u_i}$-sub-Gaussian random variables and $\{v_i\}_{i=1}^t$ are $\sigma_{v_i}$-sub-Gaussian random variables. Then the following inequalities hold:
\begin{align}\label{usineq}
P\left(\max_{1 \leq i \leq t}|u_i| > \delta_1\right) \leq 2t \exp\left(-\frac{\delta_1^2}{2\max_{1 \leq i \leq t}\sigma^2_{u_i}}\right),
\end{align}
\begin{align}\label{vsineq}
P\left(\max_{1 \leq i \leq t}|v_i| > \delta_1\right)  \leq 2t \exp\left(-\frac{\delta_1^2}{2\max_{1 \leq i \leq t}\sigma^2_{v_i}}\right).
\end{align}
\end{lemma}

\begin{proof}
Note that
\begin{align*}
\max_{1 \leq i \leq t} |u_i| = \max_{1 \leq i \leq 2t} u_i,
\end{align*}
where $u_{t+i} \coloneqq -u_i$, for $1 \leq i \leq t$. By a union bound, we have
\begin{align*}
\begin{split}
P\left(\max_{1 \leq i \leq t}|u_i| > \delta_1\right) &= P\left(\bigcup_{1 \leq i \leq 2t}\{u_i > \delta_1\}\right) \\
& \leq \sum_{1 \leq i \leq 2t} P\left(u_i \geq \delta_1\right) \\
&  = \sum_{1 \leq i \leq t} P\left(u_i \geq \delta_1\right) + \sum_{1 \leq i \leq t} P\left(u_{t+i} \geq \delta_1\right)\\
&  = \sum_{1 \leq i \leq t} P\left(u_i \geq \delta_1\right) + \sum_{1 \leq i \leq t} P\left(u_i \leq -\delta_1\right).
\end{split}
\end{align*}
For each $u_i$, we have the tail bounds
\begin{align*}
P(u_i > \delta_1) \leq \exp{\left(-\frac{\delta_1^2}{2\sigma_{u_i}^2}\right)},\quad
P(u_i < -\delta_1) \leq \exp{\left(-\frac{\delta_1^2}{2\sigma_{u_i}^2}\right)}.
\end{align*}
Altogether, we see that
\begin{align*}
P\left(\max_{1 \leq i \leq t}|u_i| > \delta_1\right) \leq 2\sum_{1 \leq i \leq t} \exp\left(-\frac{\delta_1^2}{2\sigma^2_{u_i}}\right) \leq 2t \exp\left(-\frac{\delta_1^2}{2\max_{1 \leq i \leq t}\sigma^2_{u_i}}\right).
\end{align*}
Similarly, we may obtain the desired concentration inequality for the $v_i$'s:
\begin{align*}
\begin{split}
P\left(\max_{1 \leq i \leq t}|v_i| > \delta_1\right)  \leq 2t \exp\left(-\frac{\delta_1^2}{2\max_{1 \leq i \leq t}\sigma^2_{v_i}}\right).
\end{split}
\end{align*}
\end{proof}

\begin{lemma}
\label{lem:orthogonaldesign_exact_recovery}
In the one-pool case, under the orthogonal design setting, suppose
\begin{align}\label{neqonepoolexactrec}
\min_{1 \leq i \leq t}|\gammastar_i| > (\sqrt{2}\sqrt{\log t} + c_1) \max_{1 \leq i \leq t}\sigma_{u_i} + n\lambda\left(1+\max_{1 \leq i \leq t}\frac{r_i^2}{f_i^2}\right),
\end{align}
where $\sigma_{u_i} = \sqrt{1+\frac{r_i^2}{f_i^2}}\sigma$. Then the gamma-min condition holds with probability at least $1-2e^{-c_1^2/2}$.

In the two-pool case, suppose
\begin{align}\label{neqtwopoolexactrec}
\min_{1 \leq i \leq t}|\gammastar_i| & > (\sqrt{2}\sqrt{\log{t}}+c_2)\max_{1 \leq i \leq t} \sigma_{v_i} + n\lambda\left(1+\max_{i \in [t]}\frac{r_i^2}{f_i^2+\frac{\eta n}{m}w_i^2}\right),
\end{align}
where $\sigma_{v_i} = \sqrt{1+\frac{r_i^2(L^2f_i^2+\frac{\eta n}{m}w_i^2)}{L^2(f_i^2 + \frac{\eta n}{m}w_i^2)^2}}\sigma$. Then the gamma-min condition holds with probability at least $1-2e^{-c_2^2/2}$.
\end{lemma}

We use inequality~\eqref{usineq} in Lemma~\ref{lem:noniidconcen}. Let $\delta_1 = \sqrt{2\log t + c_1^2} \max_{1 \leq i \leq t}\sigma_{u_i}$ where $c_1 \in (0,+\infty)$. Then with probability $1 - 2e^{-\frac{c_1^2}{2}}$, the following holds:
\begin{align*}
\max_{1 \leq i \leq t} |u_i| \leq \sqrt{2\log t + c_1^2} \max_{1 \leq i \leq t}\sigma_{u_i} \leq (\sqrt{2\log t} + c_1) \max_{1 \leq i \leq t}\sigma_{u_i}.
\end{align*}
In inequality~\eqref{vsineq}, take $\delta_2 = \sqrt{2\log t + c_2^2} \max_{1 \leq i \leq t}\sigma_{u_i}$ where $c_2 \in (0, +\infty)$. Then with probability $1 - 2e^{-\frac{c_2^2}{2}}$, the following holds:
\begin{align*}
\max_{1 \leq i \leq t} |v_i| \leq \sqrt{2\log t + c_2^2} \max_{1 \leq i \leq t}\sigma_{v_i} \leq (\sqrt{2\log t} + c_2) \max_{1 \leq i \leq t}\sigma_{v_i}.
\end{align*}
Combining these inequalities with conditions~\eqref{gammamincon_orthogonal_one_pool} and~\eqref{gammamincon_orthogonal_two_pool}, we obtain $G \leq \min_{i \in [t]}|\gammastar_i|$ with probability at least $1 - 2e^{-\frac{c_1^2}{2}}$ or at least $1 - 2e^{-\frac{c_2^2}{2}}$. Specifically, when we choose $c_1 = c_2 = 2.72$, we can achieve a probability guarantee of at least $95\%$ for the two statements.

Therefore, Proposition~\ref{prop:orthogonaldesign_exact_recovery} is proved by plugging the results from Lemma~\ref{lem:noniidconcen} into Lemma~\ref{lem:orthogonaldesign_exact_recovery}.
\end{proof}

Here is the proof of Proposition~\ref{RepeatCleanPoints}.
\begin{proof}[Proof of Proposition~\ref{RepeatCleanPoints}]
We will prove the proposition by comparing the three conditions in the two situations: adding one clean point and repeating multiple clean points. The conditions for adding one clean point are already provided in inequalities~\eqref{eigenvaluecon_orthogonal_two_pool}, \eqref{mutualincohcon_orthogonal_two_pool} and \eqref{gammamincon_orthogonal_two_pool} above.

We now provide the conditions for repeating multiple clean points. 
\begin{subequations}
The minimum eigenvalue condition becomes
\begin{equation}\label{e2rep}
    \lambda_{\min}(P_{X',TT}^\perp) = \min_{1 \leq i \leq t}{\frac{f_i^2 + \sum_{j=1}^{l_i}w_{ij}^2}{r_i^2 + f_i^2 + \frac{\eta n}{m}\sum_{j=1}^{l_i}w_{ij}^2}} = \frac{1}{\max_{1 \leq i \leq t}\frac{r_i^2}{f_i^2+\sum_{j=1}^{l_i}w_{ij}^2} + 1};
\end{equation}
the mutual incoherence condition becomes
\begin{equation}\label{mu2rep}
    \opnorm{P_{X',T^cT}^\perp(P_{X', TT}^\perp)^{-1}}_{\infty} = \max_{1 \leq i \leq t} \left|\frac{r_if_i}{f_i^2+\frac{\eta n}{m}\sum_{j=1}^{l_i}w_{ij}^2}\right|;
\end{equation}
and the gamma-min condition becomes
\begin{equation}\label{u2rep}
\begin{split}
    \left\| \gammahat - \gammastar \right\|_\infty & \leq \max_{1 \leq i \leq t}\left|\epsilon_i+ \frac{r_if_i}{f_i^2 +\frac{\eta n}{m} \sum_{j=1}^{l_i}w_{ij}^2}\epsilon_{i+t}  + \sum_{j=1}^{k_i}\frac{r_iw_{ij}}{f_i^2 + \frac{\eta n}{m}\sum_{j=1}^{l_i}w_{ij}^2}\frac{\epsilon_{i+t+p+j}}{L}\right| \\
    & + n\lambda\left(\max_{1 \leq i \leq t}\{\frac{r_i^2}{f_i^2 + \frac{\eta n}{m}\sum_{j=1}^{l_i}w_{ij}^2}\} +1\right).
\end{split}
\end{equation}
\end{subequations}
Compared with inequlaities~\eqref{eigenvaluecon_orthogonal_two_pool}, \eqref{mutualincohcon_orthogonal_two_pool} and \eqref{gammamincon_orthogonal_two_pool}, conditions~\eqref{e2rep},~\eqref{mu2rep} and ~\eqref{u2rep} replace $w_i^2$ by $\sum_{j = 1}^{l_i}w_{ij}^2$. Suppose the scale of the clean data points is bounded by $w_B$. Then adding one data point may not be enough to satisfy the three conditions. Thus, to achieve the same effect of a large scaled $|w_i|$ in inequalities~\eqref{eigenvaluecon_orthogonal_two_pool}, \eqref{mutualincohcon_orthogonal_two_pool} and \eqref{gammamincon_orthogonal_two_pool}, we need the number of repeated clean points to be at least $\left(\frac{|w_i|}{w_B}\right)^2$.
\end{proof}

\subsection{Sub-Gaussian design}\label{app:subgaussian-design}
In this section, we will present the support recovery results for sub-Gaussian design in Proposition~\ref{randomdesign_subset_recovery_prop} and Proposition~\ref{randomdesign_exact_recovery_prop}, and the comparisons of the three conditions in the one- and two-pool cases in Table~\ref{tab:random-comparison}. Later, we will provide the proofs of the propositions.
\subsubsection{Main results for sub-Gausian design}
\begin{proposition}
\label{randomdesign_subset_recovery_prop}
Suppose $\{x_j\}_{j \in T^c}$ and $\{\xtil_i\}_{i \in [m]}$, are i.i.d.\ sub-Gaussian with parameter $\sigma^2_x$ and covariance matrix $\Sigma \succ 0$. Further assume that $\|X_T\|_2 \le B_T$. For the one-pool case, suppose we choose $\lambda$ to satisfy inequality~\eqref{lamone_orthogonal} and the sample size satisfies
\begin{align}\label{random_subset_one_pool_assump_n}
\begin{split}
n & > t+ \max\Big\{ p+ C_1, \frac{4 c_1^2 \sigma_x^4 (p + C_1) \|\Sigma\|_2^2}{\lambda^2_{\min}(\Sigma)}, \\
& \sqrt{t}\left(\sqrt{p\|\Sigma\|_2} + c_2\sigma_2^2(\log n+ \sqrt{p\log n})\right)\left(1+\frac{2c_1\sigma_x^2 \|\Sigma\|_2}{\lambda_{\min}(\Sigma)}\right)\frac{B_T}{\lambda_{\min}(\Sigma)}  \Big\},
\end{split}
\end{align}
then the contaminated pool achieves subset support recovery with probability at least $1 - e^{-\frac{C^2}{2}} - 2e^{-C_1} - n^{-(c_2-1 )}$.

For the two-pool case, assume we choose $\lambda$ to satisfy~\eqref{lamtwo_orthogonal} and the sample sizes satisfy
\begin{align}
\label{random_subset_two_pool_assump_n}
\begin{split}
n &> \max\Big\{ t+m, \frac{t}{1+\eta} + \frac{\sqrt{t}}{1+\eta}\left(\sqrt{p\|\Sigma\|_2} + c_2\sigma_2^2(\log n+ \sqrt{p\log n})\right)\left(1+\frac{2c_1\sigma_x^2 \|\Sigma\|_2}{\lambda_{\min}(\Sigma)}\right)\frac{B_T}{\lambda_{\min}(\Sigma)}  \Big\}
\end{split}
\end{align}
and 
\[
m \geq \max\{1, 4c_1^2 \sigma_x^4\|\Sigma\|_2^2\}(p+C_1').
\]
Then adding clean points achieves subset support recovery with probability at least $1 - e^{-\frac{C'^2}{2}} - 2e^{-C_1'} - n^{-(c_2-1)}$.
\end{proposition}
As seen in Proposition~\ref{randomdesign_subset_recovery_prop}, the number of data points $n$ may be reduced by $1+\eta$ with the introduction of a second data pool.
 % The statement does not require a condition on $m$ if we choose $\eta = \frac{m}{n}$; otherwise, $m$ will need to grow with $\eta$. This is consistent with intuition: it is safe to weight the clean data pool more highly if we have enough clean data points. 
 Note that when $T$ is randomly chosen from $[n]$, we have $B_T = O(\sqrt{t}\|\Sigma\|_2)$, so inequalities~\eqref{random_subset_one_pool_assump_n} and ~\eqref{random_subset_two_pool_assump_n} require $\frac{t}{n}$ to be upper-bounded, and adding a second pool may weaker the upper bound to be $(1+\eta)$ than the upper bound for one-pool case.

We now present a result concerning exact support recovery:

\begin{proposition}
\label{randomdesign_exact_recovery_prop}
In the one-pool case, suppose inequality~\eqref{random_subset_one_pool_assump_n} holds. If
\begin{equation}\label{g3-random}
\min_{i \in T} |\gammastar_i| \geq  \frac{1}{b_{\min}}\left(2\sigma \sqrt{\log t+c}  + \frac{2\sigma\sqrt{t}}{(1-\alpha)} \left(\sqrt{\log 2(n-t)}+C\right) \right),
\end{equation}
then there exists a $\lambda$ to achieve exact recovery with probability at least $1 - 2e^{-c} - e^{-\frac{C^2}{2}} - 2e^{-C_1} - n^{-C_2}$.

For the two-pool case, suppose the assumptions in Proposition~\ref{randomdesign_subset_recovery_prop} hold, and
\begin{equation}\label{g4-random}
\min_{i \in T} |\gammastar_i| \geq \frac{1}{b'_{\min}}\left(2\sigma \sqrt{\log t+c}  + \frac{2\sigma\sqrt{t}}{(1-\alpha')} \max\{1, \sqrt{\frac{\eta n}{mL}}\}\left(\sqrt{\log 2(n-t)}+C'\right) \right).
\end{equation}
Then there exists a $\lambda$ to achieve exact recovery with probability at least $1 - 2e^{-c} - e^{\frac{-C'^2}{2}} - 2e^{-C_1'} - n^{-C_2}$.
\end{proposition}
Compared to Proposition~\ref{randomdesign_subset_recovery_prop}, Proposition~\ref{randomdesign_exact_recovery_prop} additionally requires the ``signal-to-noise" ratio to be large enough. We can show that  $b_{\min} \leq b'_{\min}$; thus, for an appropriate choice of $\eta$, the lower bound~\eqref{g3-random} is smaller than the bound~\eqref{g4-random}, so the gamma-min condition is improved.

We now briefly compare the three conditions for the one- and two-pool cases in the random design setting.

\begin{table}[htp]
\caption{Comparison between the two cases}
\begin{center}
\begin{tabular}{ cccc } 
\toprule
Condition & One-pool case & Two-pool case\\[1ex]
\midrule
Eigenvalue & $\lambda_{\min}\left(P^\perp_{X, TT}\right) = b_{\min}$ & $\lambda_{\min}\left(P^\perp_{X', TT}\right) = b'_{\min} \geq b_{\min}$ \\[1.5ex]
\midrule
Mutual incoherence & $\|-X_{T^c}((n-t)\Sigma)^{-1}X_T^\top\|_\infty$  & $\frac{\|-X_{T^c}((n-t)\Sigma)^{-1}X_T^\top\|_\infty}{1+\eta \frac{n}{n-t}}$ \\
[1.5ex]
\midrule
Gamma-min & $\min_{i}|\gamma_i^*| \geq \frac{2 \sigma \sqrt{\log t} + n\lambda\sqrt{t}}{b_{\min}}$ &$\min_{i}|\gamma_i^*| \geq \frac{2 \sigma \sqrt{\log t} + n\lambda\sqrt{t}}{b_{\min}'}$ \\[1.5ex]
\bottomrule
\end{tabular}
\label{tab:random-comparison}
\end{center}
\end{table}
In general, the eigenvalue condition is improved by adding a second pool. The mutual incoherence condition is improved in the two-pool case with large $m$ by a constant multiplier $\frac{1}{1+ \eta \frac{n}{m}}\ (\leq 1)$, and the gamma-min condition lower bound is improved by a constant $\frac{b_{\min}}{b_{\min}'}\ (\leq 1)$.

%%%%%%%%%%%%%%%%%%%%%%%%%%%%%%%%%%%%%%%%%%%%%%%%%%%%%%%%%%%
For the \textbf{eigenvalue condition}, the key result is that adding clean data points will not hurt, i.e., it makes the minimum eigenvalue smaller. A formal statement is provided in Proposition~\ref{EigIncrease}. Recall that
\begin{align*}
P_{X',TT}^\perp & = I - X'_T (X'^\top X')^{-1} X'_T, \\
P_{X,TT}^\perp & = I - X_T (X^\top X)^{-1} X_T^\top,
\end{align*}
where $X' = \left(\begin{array}{c} X \\ \sqrt{\frac{\eta n}{m}} \Xtil \end{array}\right)$, and we assume that $X^\top X$ is invertible.

\begin{proposition}[Comparison of minimum eigenvalue conditions]
\label{EigIncrease}
We have
\begin{equation*}
\lambda_{\min}(P_{X',TT}^\perp) \geq  \lambda_{\min}(P_{X,TT}^\perp).
\end{equation*}
\end{proposition}
Note that the result of Proposition~\ref{EigIncrease} does not require any assumptions on $\Xtil$ or $\eta$.
However, the degree of improvement depends on $\eta$, as seen in the proof. Usually when $n$ is small, increasing $\eta$ leads to a big jump of the minimum eigenvalue; when $n$ is large, increasing $\eta$ does not change the minimum eigenvalue much. A typical relationship between $\eta$ and $\lambda_{\min}\left(P^\perp_{X', TT}\right)$ can be seen in Figure \ref{fig:etabetamin}.
\begin{figure}[htp!]
  \centering
  \includegraphics[width=.4\columnwidth]{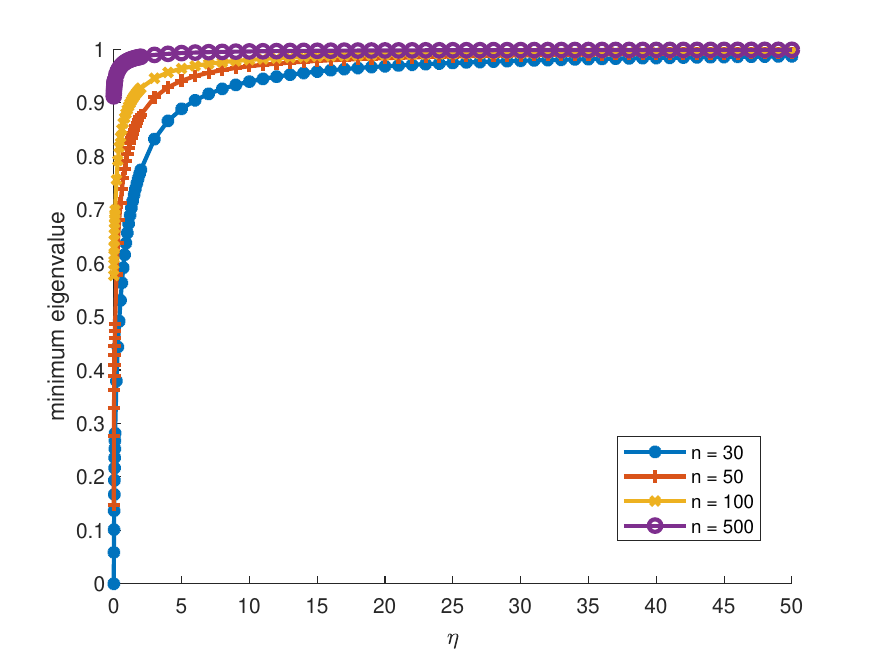}
  \caption{How does $\eta$ influence the minimum eigenvalue condition? The $x$-axis is the weight parameter $\eta$ and the $y$-axis is $\lambda_{\min}(P_{X',TT}^\perp)$. We take $t = 15, p = 20$, and $m = 5$, and vary $n$ from 30 to 500. Both pools are drawn randomly from $\mathcal{N}(\bm{0}, \rm{I}_p)$.}
  \label{fig:etabetamin}
\end{figure}

%%%%%%%%%%%%%%%%%%%%%%%%%%%%%%%%%%%%%%%%%%%%%%%%%%%%%%%%%%%
For \textbf{mutual incoherence} condition, it is possible to find settings for small $m$ that make the mutual incoherence condition worse. Consider the following example:

\begin{example}[Example where the mutual incoherence condition worsens]\label{counterexamplemutual}
Suppose
\begin{equation*}
X_T = \begin{bmatrix} -1.8271 &  -1.6954 & -1.1000 \\
0.3020 &  -1.4817 &  -0.2284 
\end{bmatrix},
\end{equation*}
\begin{equation*}
X_{T^c} = \begin{bmatrix}
-1.7680 & -0.0863 & 1.6822 \\
-0.5750 &  -1.1013 &  0.4749 \\
-0.6693 &  -0.6413 &  0.6126 \\
-0.3271 &  0.3060 & -1.0068 \\
0.6177 & 0.3941 & -2.6407 \\
-0.7001 &  2.3465 &  0.4309 \\
\end{bmatrix},
\end{equation*}
\begin{equation*}
\Xtil = \begin{bmatrix}
-1.8722 & 0.5154 & 0.1560 \\
-0.9036 & 0.6064 & -0.2540 
\end{bmatrix}.
\end{equation*}
Then
\begin{align*}
\| P_{X, T^cT}^\perp (P_{X, TT}^\perp)^{-1} \|_\infty & = 0.96 < 1 < \| P_{X', T^cT}^\perp (P_{X', TT}^\perp)^{-1} \|_\infty = 1.28.
\end{align*}
\end{example}

Despite this negative example, we can show that including a second pool helps when $m$ is large compared to $p$.
Recalling the assumption that $X_{T^c}^\top X_{T^c}$ is invertible, we can write
\begin{equation}\label{mutualsimplification1}
\begin{split}
    & P_{X,T^cT}^\perp(P_{X,TT}^\perp)^{-1}
    = - X_{T^c}\left(X_T^\top X_T + X_{T^c}^\top X_{T^c} \right)^{-1} X_T^\top \left( I - X_T \left(X_T^\top X_T + X_{T^c}^\top X_{T^c}\right)^{-1}X_T^\top \right)^{-1} \\
    & = - X_{T^c}\left(X_T^\top X_T + X_{T^c}^\top X_{T^c} \right)^{-1} X_T^\top \left( I + X_T\left(  X_{T^c}^\top X_{T^c} \right)^{-1} X_T^\top\right) \\
    & = - X_{T^c} \left(X_{T^c}^\top X_{T^c} \right)^{-1}\left(X_T^\top X_T(X_{T^c}^\top X_{T^c})^{-1} + I \right)^{-1} \left( I + X_T^\top X_T \left(  X_{T^c}^\top X_{T^c} \right)^{-1}\right)X_T^\top \\ 
    & = -X_{T^c}(X_{T^c}^\top X_{T^c})^{-1}X_T^\top.
\end{split}
\end{equation}
The first equality uses the definitions of $P_{X, T^cT}^\perp$ and $P_{X,TT}^\perp$, the second equality uses the Woodbury matrix identity~\cite{henderson1981deriving}, and the third equality follows from simple linear algebraic manipulations.

Similarly, we can simplify the mutual incoherence condition for the two-pool case, by replacing $X_{T^c}^\top X_{T^c}$ with $X_{T^c}^\top X_{T^c} + \eta \frac{n}{m} \Xtil^\top \Xtil$ in the inverse:
\begin{align}\label{mutualsimplification2}
P_{X',T^cT}^\perp(P_{X',TT}^\perp)^{-1} = -X_{T^c}\left(X_{T^c}^\top X_{T^c} + \eta \frac{n}{m} \Xtil^\top \Xtil\right)^{-1}X_T^\top,
\end{align}
where we know that $X_{T^c}^\top X_{T^c} + \eta \frac{n}{m} \Xtil^\top \Xtil$ must be invertible since $X_{T^c}^\top X_{T^c}$ is invertible.

Given these simplifications, it is easy to see that the difference between these two terms lies in the middle inverses. When $m$ is large, we have $(X_{T^c}^\top X_{T^c})^{-1} \approx ((n-t)\Sigma)^{-1}$ and $\left(X_{T^c}^\top X_{T^c} + \eta \frac{n}{m} \Xtil^\top \Xtil\right)^{-1} \approx \left(\left(n-t+\eta n\right)\Sigma\right)^{-1}$, where $\Sigma$ is the covariance matrix for the common distribution of $X_{T^c}$ and $\Xtil$. Therefore, the mutual incoherence parameter in the one-pool case is approximately equal to the mutual incoherence in the two-pool case scaled by $\left(1 + \eta\frac{n}{n-t}\right)^{-1}$, which immediately implies that adding a second data pool improves the mutual incoherence condition. This is stated formally in the following proposition:

\begin{proposition}[Comparison of mutual incoherence conditions]
\label{propasymmutual}
Let $B_T = O(\sqrt{t})$. In the one-pool case, if $n \geq t + \frac{c_1^2\sigma_x^4 (p+C_1)\|\Sigma\|^2}{\lambda_{\min}^2(\Sigma)}$, then
\begin{align*}
\left|\opnorm{X_{T^c}\frac{\Theta}{n-t} X_T^\top}_\infty - \opnorm{X_{T^c}\left(X_{T^c}^\top X_{T^c}\right)^{-1}X_T^\top}_\infty\right| 
 = O\left(t(n-t)^{-1}(\sqrt{p} + \sqrt{\log n})\right),
\end{align*}
with high probability.

In the two-pool case, if $n \geq t + \max\Big\{\frac{c_1^2\sigma_x^4 \|\Sigma\|^2}{\lambda_{\min}^2(\Sigma)},1\Big\}m$ and $m \geq \max\{1,c_1^2 \sigma^4_x (p+C_1')\|\Sigma\|_2^2\}$, then
\begin{align*}
\left|\opnorm{X_{T^c}\frac{\Theta}{n-t+\eta n} X_T^\top}_\infty - \opnorm{X_{T^c}\left(X_{T^c}^\top X_{T^c}+ \frac{\eta n}{m} \Xtil^\top \Xtil\right)^{-1}X_T^\top }_\infty\right| 
= O\left(t(n-t+\eta n)^{-1}(\sqrt{p} + \sqrt{\log n})\right),
\end{align*}
with high probability.
\end{proposition}

Proposition~\ref{propasymmutual} states that when $m$ and $n$ are sufficiently large, the one-pool mutual incoherence parameter is close to $\frac{\opnorm{X_{T^c}\Theta X_T^\top}_\infty}{n-t}$ and the two-pool mutual incoherence parameter is close to $\frac{\opnorm{X_{T^c}\Theta X_T^\top}_\infty}{n-t+\eta n}$. Since the second expression has a larger denominator, the mutual incoherence condition improves with the introduction of a second data pool with parameter $\eta > 0$.

%%%%%%%%%%%%%%%%%%%%%%%%%%%%%%%%%%%%%%%%%%%%%%%%%%%%%%%%%%%
For \textbf{gamma-min} condition, we need to compare the terms $G$ and $G'$. Note that inequalities~\eqref{g3-random} and~\eqref{g4-random} are equivalent to lower-bounding the ``signal-to-noise" ratio. The order of the lower bound for two-pool case is as same as the one-pool case, i.e., $\frac{\min_{i}|\gammastar_i|}{\sigma} \geq O(\sqrt{t \log{n}})$. However, adding a second pool improves the constant by having a factor of $\frac{1}{b_{\min}'}$ instead of $\frac{1}{b_{\min}}$. As established in Proposition~\ref{EigIncrease}, we have $b_{\min} \leq b_{\min}'$. Therefore, the lower bound in the two-pool case is smaller than the lower bound in the one-pool case.

%%%%%%%%%%%%%%%%%%%%%%%%%%%%%%%%%%%%%%%%%%%%%%%%%%%%%%%%%%%
Note that the \textbf{weight parameter $\eta$} shows up in all the three conditions.
However, recall that the mutual incoherence condition is not always improved by adding a second pool, unless $m$ is sufficiently large. Therefore, an appropriate conclusion is that once we have a large clean data pool, it is reasonable to place arbitrarily large weight on the second pool. On the other hand, if we have fewer clean data points, we cannot be as confident about the estimator obtained using the second pool alone. For example, in the orthogonal design, if we obtain clean points in the non-buggy subspace, the mutual incoherence condition is not improved no matter how large we make $\eta$. In addition, the gamma-min condition involves the randomness from noise, and in order to control the sparsity of $\gamma$, we need the regularizer $\lambda$ to match large $\eta$ (cf.\ inequality~\eqref{lamtwo_orthogonal}). Based on inequality~\eqref{g4-random}, we need the ``signal-to-noise" ratio, i.e., $\frac{n\lambda \sqrt{t}}{\sigma}$, to be sufficient large. If $\eta$ is too large, we cannot estimate relatively small components of $\gamma^*$. In summary, selecting $\eta$ too large or too small is not wise: If $\eta$ is too small, we do not improve the three conditions, whereas if $\eta$ is too large, the range of controllable ``signal-to-noise" ratios decays.

\subsubsection{Proofs for sub-Gaussian design}
In this section, we provide proofs of sub-Gaussian design. 
Here is the proof of Proposition~\ref{randomdesign_subset_recovery_prop}.
\begin{proof}[Proof of Proposition~\ref{randomdesign_subset_recovery_prop}]

We prove the results for the one- and two-pool cases sequentially. In each case, we begin with background calculations, and then analyze the eigenvalue condition followed by the mutual incoherence condition.

\textbf{For the one-pool case}, we know that $\lambda$ satisfies inequality~\eqref{lamone_orthogonal} with probability at least $1 - e^{-\frac{C^2}{2}}$.

Note that $x_j,\ j \in T^c$ are sub-Gaussian random vectors with parameter $\sigma_x$.
By Theorem 4.7.1 and Exercise 4.7.3 in Vershynin~\cite{vershynin2018high} and our assumption of $n$, we have
\begin{equation}\label{boundSampleCovarinace}
    \opnorm{\Sigma - \frac{X_{T^c}^\top X_{T^c}}{n-t}}_2 \leq c_1 \sigma^2_x \sqrt{\frac{p+C_1}{n-t}} \|\Sigma\|_2,
\end{equation}
with probability at least $1 - e^{-C_1}$. We will later use this bound multiple times to establish the eigenvalue condition and the mutual incoherence condition.

We first consider the eigenvalue condition. By the dual Weyl's inequality~\cite{horn1994topics}, we have $\lambda_{\min}(A+B) \geq \lambda_{\min}(A) + \lambda_{\min}(B)$ for any square matrices $A$ and $B$. Then 
\begin{align*}
\begin{split}
\lambda_{\min}\left(\frac{X_{T^c}^\top X_{T^c}}{n-t}\right) & = \lambda_{\min} \left(\frac{X_{T^c}^\top X_{T^c}}{n-t} - \Sigma + \Sigma\right) \\
& \geq \lambda_{\min}(\Sigma) + \lambda_{\min} \left(\frac{X_{T^c}^\top X_{T^c}}{n-t} - \Sigma\right) \\
& \geq \lambda_{\min}(\Sigma) - \left\|\frac{X_{T^c}^\top X_{T^c}}{n-t} - \Sigma\right\|_2,
\end{split}
\end{align*}
where the second inequality follows from the fact that $\lambda_{\min}(A) \leq \lambda_{\max}(A)$ for any square matrix $A$. Combining this with inequality~\eqref{boundSampleCovarinace} and taking $n \geq t + 4\frac{c_1^2 \sigma_x^4 (p + C_1)\|\Sigma\|_2^2}{\lambda^2_{\min}(\Sigma)}$ by assumption~\eqref{random_subset_one_pool_assump_n}, we have that 
\begin{align}
\label{lowerboundXTC}
\lambda_{\min}\left(\frac{X_{T^c}^\top X_{T^c}}{n-t}\right) \geq \lambda_{\min}(\Sigma) - c_1 \sigma^2_x \sqrt{\frac{p+C_1}{n-t}} \|\Sigma\|_2 \geq \frac{1}{2} \lambda_{\min}(\Sigma) > 0,
\end{align}
with probability $1 - e^{-C_1}$. We now derive the following result:
\begin{lemma}
\label{eigenClaim}
Suppose $X_{T^c}^\top X_{T^c}$ is invertible, where $X_T \in \real^{t \times p}$ and $X_{T^c} \in \real^{(n-t) \times p}$. Then 
\begin{align*}
\begin{split}
\lambda_{\min}\left(P_{X, TT}^\perp\right) & \geq 1 - \frac{\lambda_{\max}(X_T^\top X_T)}{\lambda_{\max}(X_T^\top X_T) + \lambda_{\min}(X_{T^c}^\top X_{T^c})} > 0,
\end{split}
\end{align*}
implying that the eigenvalue condition for the one-pool case holds.
\end{lemma}

\begin{proof}
Define $C = Q(I + Q^\top Q)^{-1}Q^\top $ and $Q \in \real^{s \times p}$, and suppose rank($Q$) = $r$. Let $Q = U S V^\top$ be the SVD, where $U\in \real^{t \times p}, V \in \real^{p \times p}$, and $S = \begin{bmatrix} J_{r \times r} & 0_{r \times (p-r)}\\ 0_{(t-r) \times r} & 0_{(t-r) \times (p-r)}\end{bmatrix}$. Here, $J$ is a diagonal matrix of positive singular values. Then 
\begin{align}
\begin{split}
    C &= USV^\top (I + VS^\top SV^\top )^{-1}VS^\top U^\top \\
    & = US(I + S^\top S)^{-1}S^\top U^\top \\
    & = U \begin{bmatrix}J_{r \times r} & 0_{r\times(p-r)} \\ 0_{(t-r)\times r} & 0_{(t-r) \times (p-r)}\end{bmatrix} \cdot \begin{bmatrix} (I + J^2)^{-1}_{r\times r} & 0_{r\times (p-r)} \\0_{(t-r)\times r} & I_{(p-r) \times (p-r)}\end{bmatrix} \cdot \begin{bmatrix}J_{r \times r} & 0_{r\times(p-r)} \\ 0_{(t-r)\times r} & 0_{(t-r) \times (p-r)}\end{bmatrix}U^\top  \\
    & = U\begin{bmatrix} (J(I + J^2)^{-1}J)_{r\times r} & 0_{r\times (p-r)} \\0_{(t-r)\times r} & 0_{(p-r) \times (p-r)}\end{bmatrix}U^\top .
\end{split}
\end{align}
Therefore, $\lambda_{\max}(C) = \frac{a^2_{\max}}{1 + a^2_{\max}}$, where $a_{\max}$ is the maximum singular value appearing in $J$. Also note that $a_{\max}^2$ is the maximum eigenvalue of $Q^\top Q$.

Following (16.51) in Seber~\cite{seber2008matrix}, given $X_{T^c}^\top X_{T^c}$ is invertible, there exists a non-singular matrix $A$ such that $AX_{T^c}^\top X_{T^c}A^\top  = I$ and $AX_T^\top X_TA^\top  = D$, where $D$ is diagonal matrix.

Note that 
\begin{align*}
X_T(X_T^\top X_T + X_{T^c}X_{T^c})^{-1}X_T^\top & =X_TA^\top (A(X_T^\top X_T + X_{T^c}X_{T^c})A^\top )^{-1}AX_T^\top  \\
& = X_TA^\top (AX_T^\top X_TA^\top  + I)AX_T^\top \\
& = Q(Q^\top Q + I)^{-1}Q^\top,
\end{align*}
where $Q:=X_TA^\top$.

Based on our earlier arguments, we know that the matrix under consideration has maximum eigenvalue $\frac{\lambda_{\max}(AX_T^\top X_TA^\top )}{1 + \lambda_{\max}(AX_T^\top X_TA^\top )}$. Since $AX_T^\top X_TA^\top $ is similar to $X_T^\top X_TA^\top A$, we have$\lambda_{\max}(AX_T^\top X_TA^\top ) = \lambda_{\max}(X_T^\top X_TA^\top A)$. Furthermore, we have $A^\top A = (X_{T^c}^\top X_{T^c})^{-1}$, implying that
\begin{align*}
\begin{split}
    \lambda_{\max}(AX_T^\top X_TA^\top )  & = \lambda_{\max}(X_T^\top X_T(X_{T^c}^\top X_{T^c})^{-1}) \\
    & \leq \max_v \frac{\norm{X_T^\top X_T(X_{T^c}^\top X_{T^c})^{-1}v}^2}{\norm{(X_{T^c}^\top X_{T^c})^{-1}v}^2} \cdot \max_v \frac{\norm{(X_{T^c}^\top X_{T^c})^{-1}v}^2}{\norm{v}^2}\\
    & \leq \frac{\lambda_{\max}(X_T^\top X_T)}{\lambda_{\min}(X_{T^c}^\top X_{T^c})}.
\end{split}
\end{align*}
Altogether, we have
\begin{align}\label{eigcompares}
\begin{split}
     \lambda_{\max}\left( X_T \left( X_T^\top X_T + X_{T^c}^\top X_{T^c}\right)^{-1} X_T^\top  \right)
     & \leq \frac{1}{1 + \lambda^{-1}_{\max}\left( X_T^\top X_T(X_{T^c}^\top X_{T^c})^{-1} \right)}\\
     & \leq \frac{1}{1 + \frac{\lambda_{\min}(X_{T^c}^\top X_{T^c})}{\lambda_{\max}(X_T^\top X_T)}}.
\end{split}
\end{align}
Finally, we may conclude that
\begin{align*}
\begin{split}
\lambda_{\min}\left(P_{X, TT}^\perp\right) 
&= \lambda_{\min}\left(I - X_T \left( X_T^\top X_T + X_{T^c}^\top X_{T^c}\right)^{-1}X_T^\top\right) \\
& = 1 - \lambda_{\max}\left( X_T \left( X_T^\top X_T + X_{T^c}^\top X_{T^c}\right)^{-1} X_T^\top  \right) \\
& \geq 1 -  \frac{1}{1 + \frac{\lambda_{\min}(X_{T^c}^\top X_{T^c})}{\lambda_{\max}(X_T^\top X_T)}} \\
& = 1 - \frac{\lambda_{\max}(X_T^\top X_T)}{\lambda_{\max}(X_T^\top X_T) + \lambda_{\min}(X_{T^c}^\top X_{T^c})}.
\end{split}
\end{align*}
Since $\lambda_{\min}(X_{T^c}^\top X_{T^c}) > 0$, we have $\lambda_{\min}\left(P_{X, TT}^\perp\right) < 1$, implying the desired result.  
\end{proof}

We now consider the mutual incoherence condition. By the triangle inequality, we have
\begin{equation*}
\begin{split}
\frac{1}{n-t}\infnorm{X_{T^c}\left(\frac{X_{T^c}^\top X_{T^c}}{n-t}\right)^{-1}X_T} \leq 
& \underbrace{\frac{1}{n-t}\opnorm{X_{T^c}\Theta X_T^\top - X_{T^c}\left(\frac{X_{T^c}^\top X_{T^c}}{n-t}\right)^{-1}X_T^\top }_\infty}_{\circled{1}} \\
& \quad + \underbrace{\frac{1}{n-t}\opnorm{X_{T^c}\Theta X_T^\top}_\infty}_{\circled{2}}.
\end{split}
\end{equation*}
We bound $\circled{1}$ and $\circled{2}$ separately. Note that
\begin{align*}
\circled{1} & = \frac{ \max_{j \in T^c} \left\|x_j^\top\left( \Theta - \left(\frac{X_{T^c}^\top X_{T^c}}{n-t}\right)^{-1} \right) X_T^\top \right\|_1}{n-t} \\
& \leq \frac{\sqrt{t}}{n-t} \max_{j \in T^c} \|x_j\|_2\opnorm{ \Theta - \left(\frac{X_{T^c}^\top X_{T^c}}{n-t}\right)^{-1} }_2\opnorm{X_T^\top}_2.
\end{align*}
In order to bound $\circled{1}$, we bound three parts separately. By assumption, we have $\opnorm{X_T^\top}_2 \leq B_T$. For $\max_{j \in T^c} \|x_j\|_2$, we leverage the Hanson-Wright inequality (Theorem 6.2.1 in~\cite{vershynin2018high}) and a union bound. By the Hanson-Wright inequality, we see that for $t > 0$, 
$$P\left( \|x_j\|_2^2 - \mathbb{E}[\|x_j\|_2^2] \geq t\right) \leq \exp\left\{-c \min \left(\frac{t^2}{\sigma_x^4 p}, \frac{t}{\sigma_x^2}\right)\right\},$$
where $c$ is an absolute constant. \\
By a union bound, we then have
\begin{equation*}
\begin{split}
    P\left(\max_{j \in T^c} \|x_j\|_2 \geq \sqrt{\mathbb{E}[\|x_j\|_2^2] + \Delta}\right)
    & = P\left(\max_{j \in T^c} \|x_j\|_2^2 \geq \mathbb{E}[\|x_j\|_2^2] + \Delta \right)\\
    & \leq \sum_{j \in T^c} P\left(\|x_j\|_2^2 \geq \mathbb{E}[\|x_j\|_2^2] + \Delta\right) \\
    & \leq (n-t)\exp\left\{-c \min \left(\frac{\Delta^2}{\sigma_x^4 p}, \frac{\Delta}{\sigma_x^2}\right)\right\}.
\end{split}
\end{equation*}
Setting $\Delta = c_2 \sigma_x^2\max\{\sqrt{p\log n}, \log n\}$ with $c_2 \geq 1$ so that we have $\min\left\{\frac{\Delta^2}{\sigma_x^4 p}, \frac{\Delta}{\sigma_x^2}\right\} \geq c_2 \log n$, we conclude that 
\begin{equation}\label{boundmaxx}
\begin{split}
\max_{j \in T^c} \|x_j\|_2 & \leq \sqrt{\mathbb{E}[\|x_j\|_2^2] + \Delta}\\
& \leq \sqrt{trace(\Sigma)} + \Delta \\
& \leq \sqrt{p\|\Sigma\|_2} + c_2 \sigma_x^2(\log n + \sqrt{p \log n}),
\end{split}
\end{equation}
with probability at least $1 - n^{-(c_2-1)}$, where $c_2 \geq \max\{2,2/c\}$.\\
To bound $\opnorm{ \Theta - \left(\frac{X_{T^c}^\top X_{T^c}}{n-t}\right)^{-1} }_2$, note that for two matrices $A$ and $B$, we have
\begin{equation*}
\opnorm{A^{-1} - B^{-1}}_2 \leq \frac{\opnorm{A-B}_2}{\lambda_{\min}(A)\lambda_{\min}(B)}.
\end{equation*}
Combining this fact with inequalities~\eqref{boundSampleCovarinace} and~\eqref{lowerboundXTC}, we obtain
\begin{align}
\label{Thetabound}
   \opnorm{ \Theta - \left(\frac{X_{T^c}^\top X_{T^c}}{n-t}\right)^{-1} }_2 & \leq  \frac{\left\|\Sigma - \frac{X_{T^c}^\top X_{T^c} }{n-t}\right\|_2}{\lambda_{\min}\left(\Sigma\right)\lambda_{\min}\left(\frac{X_{T^c}^\top X_{T^c}}{n-t}\right)} \leq \frac{2\left\|\Sigma - \frac{X_{T^c}^\top X_{T^c}}{n-t}\right\|_2}{\lambda^2_{\min}\left(\Sigma\right)} \notag \\
   & \leq \frac{2c_1 \sigma^2_x \sqrt{\frac{p+C_1}{n-t}} \|\Sigma\|_2}{\lambda^2_{\min}\left(\Sigma\right)}.
\end{align}
Altogether, we obtain the bound
\begin{equation}\label{onepool_closetopopulation}
\circled{1} \leq \frac{\sqrt{t}}{n-t} \left(\sqrt{p\|\Sigma\|_2} + c_2 \sigma_x^2(\log n + \sqrt{p \log n})\right)  \cdot \frac{2c_1 \sigma^2_x \sqrt{\frac{p+C_1}{n-t}} \|\Sigma\|}{\lambda^2_{\min}\left(\Sigma\right)} B_T.
\end{equation}

We now consider $\circled{2}$. Note that
\begin{align}\label{populationmutualinch}
\begin{split}
\frac{\opnorm{X_{T^c}\Theta X_T^\top}_\infty}{n-t} & = \frac{1}{n-t} \max_{j \in T^c}\|x_j^\top \Theta X_T^\top\|_1 \\
    & \leq \frac{\sqrt{t}}{n-t} \max_{j \in T^c}\|x_j^\top\|_2 \|\Theta\|_2 \|X_T^\top\|_2 \\
    & = \frac{\sqrt{t}}{n-t} \left(\sqrt{p\|\Sigma\|_2} + c_2 \sigma_x^2(\log n + \sqrt{p \log n})\right) \cdot \frac{1}{\lambda_{\min}(\Sigma)}B_T.
\end{split}
\end{align}
Therefore,
\begin{align*}
\circled{1} + \circled{2} &\leq \frac{\sqrt{t}}{n-t} \left(\sqrt{p\|\Sigma\|_2} + c_2 \sigma_x^2(\log n + \sqrt{p \log n})\right) \cdot \left(1 + \frac{2c_1 \sigma^2_x \sqrt{\frac{p+C_1}{n-t}} \|\Sigma\|_2}{\lambda_{\min}\left(\Sigma\right)}\right)\frac{B_T}{\lambda_{\min}(\Sigma)}.
\end{align*}
Finally, assuming $n$ satisfies the bound~\eqref{random_subset_one_pool_assump_n}, and taking a union bound over all the probabilistic statements appearing above, we conclude that the mutual incoherence condition holds with probability at least
$1 - e^{-\frac{C^2}{2}} - 2e^{-C_1} - n^{-(c_2-1)}$. This concludes the proof.

\textbf{For the two-pool case}, we will use the following inequalities:
\begin{equation*}
\begin{split}
    \opnorm{\Sigma - \frac{X_{T^c}^\top X_{T^c}}{n-t}}_2 & \leq c_1 \sigma^2_x \sqrt{\frac{p+C'_1}{n-t}} \|\Sigma\|_2, \\
    \opnorm{\Sigma - \frac{\Xtil^\top \Xtil}{m}}_2 & \leq c_1 \sigma^2_x \sqrt{\frac{p+C'_1}{m}}\|\Sigma\|_2,
\end{split}
\end{equation*}
with probablity at least $1 - 2e^{-C_1'}$. Combining these inequalities and using the triangle inequality, we obtain
\begin{equation}
\label{tradeoff}
\begin{split}
   \opnorm{\Sigma - \frac{X_{T^c}^\top X_{T^c} + \frac{\eta n}{m} \Xtil^\top \Xtil}{n-t + \eta n}}_2 &\leq \frac{n-t}{n-t + \eta n}\opnorm{\Sigma - \frac{X_{T^c}^\top X_{T^c}}{n-t}}_2 + \frac{\eta n}{n-t + \eta n}\opnorm{\Sigma - \frac{\Xtil^\top \Xtil}{m}}_2 \\
   \leq & c_1 \sigma^2_x\|\Sigma\|_2 \frac{n-t}{n-t + \eta n}\sqrt{\frac{p+ C_1'}{n-t}}  + c_1 \sigma^2_x \|\Sigma\|_2 \frac{\eta n}{n-t + \eta n}  \sqrt{\frac{p+C_1'}{m}}\\
   & \overset{n\geq t+m}{\leq} 2c_1 \sigma_x^2 \|\Sigma\|_2\sqrt{\frac{p+C_1'}{m}},
\end{split}
\end{equation}
with probability at least $1 - 2e^{-C_1'}$.

Analogous to Lemma~\ref{eigenClaim}, we can conclude that if $X_{T^c}^\top X_{T^c} + \frac{\eta n}{m}\Xtil^\top \Xtil$ is invertible, the eigenvalue condition satisfies
\begin{align*}
\lambda_{\min}(P_{X',TT}^\perp) \geq  1 - \frac{\lambda_{\max}(X_T^\top X_T)}{\lambda_{\max}(X_T^\top X_T) + \lambda_{\min}\left(X_{T^c}^\top X_{T^c} + \frac{\eta n}{m}\Xtil^\top \Xtil\right)} > 0.
\end{align*}
(This can be proved just by replacing $X_{T^c}^\top X_{T^c}$ with $X_{T^c}^\top X_{T^c} + \frac{\eta n}{m} \Xtil^\top \Xtil$ in the proof of Lemma~\ref{eigenClaim}.) However, since we further wish to bound the minimum eigenvalue from below by $\lambda_{\min}(\Sigma)/2$, to match the one-pool case and to be used in the proof for the mutual incoherence condition later, we will consider $X_{T^c}^\top X_{T^c} + \frac{\eta n}{m} \Xtil^\top \Xtil$ directly.

Note that 
\begin{align*}
\begin{split}
  \lambda_{\min}\left(\frac{X_{T^c}^\top X_{T^c} + \frac{\eta n}{m}\Xtil^\top \Xtil}{n-t+\eta n}\right) & = \lambda_{\min}\left(\frac{X_{T^c}^\top X_{T^c} + \frac{\eta n}{m}\Xtil^\top \Xtil}{n-t+\eta n} - \Sigma + \Sigma \right) \\
  & \geq \lambda_{\min}\left(\frac{X_{T^c}^\top X_{T^c} + \frac{\eta n}{m}\Xtil^\top \Xtil}{n-t+\eta n} - \Sigma \right) + \lambda_{\min}(\Sigma) \\
  & \geq \lambda_{\min}(\Sigma) - \left\|\frac{X_{T^c}^\top X_{T^c} + \frac{\eta n}{m}\Xtil^\top \Xtil}{n-t+\eta n} - \Sigma \right\|_2.
\end{split} 
\end{align*}
Thus, if we choose $m \geq 4c_1^2 \sigma^4_x (p+C_1')\|\Sigma\|_2^2$, we have 
\begin{align*}
\begin{split}
\lambda_{\min}\left(\frac{X_{T^c}^\top X_{T^c} + \frac{\eta n}{m}\Xtil^\top \Xtil}{n-t+\eta n}\right) & \geq \frac{1}{2} \lambda_{\min}(\Sigma) > 0,
\end{split} 
\end{align*}
with probability at least $1 - 2e^{-C_1'}$.

We now consider the mutual incoherence condition. Similar to the derivation of inequality~\eqref{Thetabound}, we have that 
\begin{equation*}
\begin{split}
   \opnorm{\Theta- \left(\frac{X_{T^c}^\top X_{T^c} + \frac{\eta n}{m} \Xtil^\top \Xtil}{n-t + \eta n}\right)^{-1}}_2
   &\leq  \frac{\left\|\Sigma - \frac{X_{T^c}^\top X_{T^c} + \eta \frac{n}{m}\Xtil^\top \Xtil}{(1+\eta)n-t}\right\|_2}{\lambda_{\min}\left(\Sigma\right)\lambda_{\min}\left(\frac{X_{T^c}^\top X_{T^c} + \eta \frac{n}{m}\Xtil^\top \Xtil}{(1+\eta)n-t}\right)}\\
   & \leq 2c_1 \sigma^2_x\frac{\|\Sigma\|_2}{\lambda^2_{\min}(\Sigma)} \sqrt{\frac{p+ C_1'}{m}}.
\end{split}
\end{equation*}
Combining this with inequality~\eqref{boundmaxx}, we obtain
\begin{equation*}
\begin{split}
        & \frac{\opnorm{X_{T^c}\Theta X_T^\top - X_{T^c}\left(\frac{X_{T^c}^\top X_{T^c} + \frac{\eta n}{m} \Xtil^\top \Xtil}{n-t + \eta n}\right)^{-1}X_T^\top }_\infty}{n-t+\eta n} \\
        & \qquad = \frac{ \underset{j \in T^c}{\max} \left\|x_j^\top \left( \Theta - \left(\frac{X_{T^c}^\top X_{T^c} + \frac{\eta n}{m} \Xtil^\top \Xtil}{n-t + \eta n}\right)^{-1} \right) X_T^\top \right\|_1}{n-t+\eta n} \\
        & \qquad \leq \frac{\sqrt{t}}{n-t+\eta n} \max_{j \in T^c} \|x_j\|_2  \cdot \opnorm{\Theta - \left(\frac{X_{T^c}^\top X_{T^c} + \frac{\eta n}{m} \Xtil^\top \Xtil}{n-t + \eta n}\right)^{-1}}_2 \opnorm{X_T^\top}_2 \\
        & \qquad \leq \frac{\sqrt{t}}{n-t+\eta n} \left( \sqrt{p\|\Sigma\|_2} + c_2 \sigma_x^2(\log n + \sqrt{p \log n})\right) \cdot 2c_1  \sigma^2_x\frac{\|\Sigma\|_2}{\lambda^2_{\min}(\Sigma)} \sqrt{\frac{p+ C_1'}{m}} B_T.\\
\end{split}
\end{equation*}
Therefore, together with the triangle inequality and inequality~\eqref{populationmutualinch}, we can bound the mutual incoherence parameter as follows:
\begin{align*}
& \frac{\opnorm{X_{T^c}\left(\frac{X_{T^c}^\top X_{T^c} + \frac{\eta n}{m} \Xtil^\top \Xtil}{n-t + \eta n}\right)^{-1}X_T^\top }_\infty}{{n-t+\eta n}} \\
&  \leq \frac{\opnorm{X_{T^c}\Theta X_T^\top - X_{T^c}\left(\frac{X_{T^c}^\top X_{T^c} + \frac{\eta n}{m} \Xtil^\top \Xtil}{n-t + \eta n}\right)^{-1}X_T^\top }_\infty}{{n-t+\eta n}}  + \frac{\opnorm{X_{T^c}\Theta X_T^\top}_\infty}{n-t+\eta n} \\
& \leq \frac{\sqrt{t}}{n-t+\eta n} \left( \sqrt{p\|\Sigma\|_2} + c_2 \sigma_x^2(\log n + \sqrt{p \log n})\right) \left(1 + 2c_1  \sigma^2_x\frac{\|\Sigma\|_2}{\lambda_{\min}(\Sigma)} \sqrt{\frac{p+ C_1'}{m}}\right) \frac{B_T}{\lambda_{\min}(\Sigma)}.
\end{align*}
By the assumption on $n$ in inequality~\eqref{random_subset_two_pool_assump_n}, the mutual incoherence condition therefore holds with probability $1 - e^{-\frac{C'^2}{2}} - 2e^{-C_1'} - n^{-(c_2-1)}$.
\end{proof}

Here is the proof of Proposition~\ref{randomdesign_exact_recovery_prop}.
\begin{proof}[Proof of Proposition~\ref{randomdesign_exact_recovery_prop}]
To achieve exact support recovery, we need all the three conditions to hold. The eigenvalue condition and the mutual incoherence condition have already been discussed in the analysis of subset support recovery in Appendix~\ref{randomdesign_subset_recovery_prop}, so it remains to analyze the gamma-min condition.

Recall that
\begin{align*}
G' = \| (P_{X', TT}^\perp)^{-1}P_{X', T\cdot}^\perp \epsilon'\|_\infty + n\lambda\opnorm{(P_{X', TT}^\perp)^{-1}}_{\infty}.
\end{align*}
To simplify notation, we define 
\begin{align*}
A := \|(P_{X', TT}^\perp)^{-1}P_{X', T\cdot}^\perp P_{X'}^\perp \epsilon'\|_\infty, \quad B := n\lambda\opnorm{(P_{X', TT}^\perp)^{-1}}_{\infty}.
\end{align*}
We also define the random variables $$Z_i \coloneqq e_i^\top( P_{X', TT}^\perp)^{-1}P_{X', T\cdot}^\perp P_{X'}^\perp \epsilon'.$$
Since $P_{X'}^\perp$ is a projection matrix and the maximum singular value of $P_{X', T\cdot}^\perp$ is smaller than the maximum singular value of $P_{X'}^\perp$'s, we have 
\begin{equation*}
\begin{split}
\opnorm{ (P_{X', TT}^\perp)^{-1}P_{X', T\cdot}^\perp P_{X'}^\perp}_2 & \leq \opnorm{(P_{X', TT}^\perp)^{-1}}_2 \leq \opnorm{(P_{X', TT}^\perp)^{-1}}_2 \leq \frac{1}{b'_{\min}}, 
\end{split}
\end{equation*}
for all $i \in T$. Note that $Z_i$ is a zero-mean sub-Gaussian random variable with parameter at most $\frac{\sigma}{b'_{\min}}$. By a sub-Gaussian tail bound, we then have 
\begin{align*}
P\left(\max_{1\leq i\leq t}|Z_i| > \frac{\sigma}{b'_{\min}}\left(\sqrt{2\log t} + \Delta\right)\right) \leq 2 e^{-\frac{\Delta^2}{2}}.
\end{align*}
Therefore, with probability at least $1 - 2e^{-c}$, we have $A \leq \frac{2\sigma \sqrt{\log t+c}}{b'_{\min}}.$ Note that $\|(P_{X', TT}^\perp)^{-1}\|_\infty \leq \sqrt{t} \|(P_{X', TT}^\perp)^{-1}\|_2 = \frac{\sqrt{t}}{b_{\min}'}.$
We can then immediately obtain the bound $B \leq \frac{2n\lambda\sqrt{t}}{b_{\min}'}$.

Combined with the fact that $\lambda \geq \frac{2\sigma}{n(1-\alpha')} \max\left\{1, \sqrt{\frac{\eta n}{mL}}\right\}\left(\sqrt{\log 2(n-t)}+C'\right)$, we then obtain 
\begin{align*}
G' \leq \frac{1}{b'_{\min}}\left(2\sigma \sqrt{\log t+c} + \frac{2\sigma\sqrt{t}}{(1-\alpha')} \max\left\{1, \sqrt{\frac{\eta n}{mL}}\right\}\left(\sqrt{\log 2(n-t)}+C'\right) \right).
\end{align*}
Thus, as long as $\min_{i \in T} |\gammastar_i|$ is greater than or equal to the RHS of the inequality above, the gamma-min condition holds with probability at least $1 - 2e^{-c} - e^{-\frac{C'^2}{2}}$. Consequently, the exact support recovery is achieved.

The proof of the one-pool case is similar as the proof of the two-pool case provided above, so we omit the details here. 
\end{proof}

Here is the proof of Proposition~\ref{EigIncrease}
\begin{proof}Proof of Proposition~\ref{EigIncrease}

By the Sherman-Morrison-Woodbury formula~\cite{henderson1981deriving}, we have
\begin{equation}\label{bmin}
\begin{split}
     & X_T\left( X^\top X + \frac{\eta n}{m} \Xtil^\top \Xtil \right)^{-1}X_T^\top\\
     &  = X_T\left( X^\top X \right)^{-1} X_T^\top  - \frac{\eta n}{m} X_T\left( X^\top X \right)^{-1}\Xtil^\top(I+\frac{\eta n}{m} \Xtil(X^\top X)^{-1}\Xtil^\top)^{-1} \Xtil\left( X^\top X \right)^{-1}X_T^\top.
\end{split}
\end{equation}

We now state and prove two useful lemmas:

\begin{lemma}\label{RestPSD}
Assume $X^\top X$ is invertible. Define 
\begin{align*}
A \coloneqq X_T\left( X^\top X \right)^{-1}\Xtil^\top(I+\frac{\eta n}{m} \Xtil(X^\top X)^{-1}\Xtil^\top)^{-1}\Xtil\left( X^\top X \right)^{-1}X_T^\top.
\end{align*}
Then $\lambda_{\min}\left(A\right) \geq 0$. Equality holds when $\Xtil\left( X^\top X \right)^{-1}X_T^\top$ is not full-rank.
\end{lemma}

\begin{proof}
First note that since $X^\top X$ is invertible and $\Xtil(X^\top X)^{-1}\Xtil^\top \succ 0$, the matrix $I+\frac{\eta n}{m} \Xtil(X^\top X)^{-1}\Xtil^\top$ is invertible. Note that
\begin{align*}
\forall y \in \real^t \neq 0,\quad  y^\top A y \geq 0,
\end{align*}
so the minimum eigenvalue of $A$ is nonnegative.

In order to study when the $\lambda_{\min} = 0$, let $z = \Xtil\left( X^\top X \right)^{-1}X_T^\top y$. When $y \neq 0$ and $\Xtil\left( X^\top X \right)^{-1}X_T^\top$ is full-rank, we have $z \neq 0$. Thus, if $\Xtil\left( X^\top X \right)^{-1}X_T^\top$ is full-rank, we have $\lambda_{\min}(A) > 0$. When $y \neq 0$ and $\Xtil\left( X^\top X \right)^{-1}X_T^\top$ is not full-rank, there exists $y \neq 0$ such that $z = 0$, which causes $y^\top A y = 0$ and $\lambda_{\min}(A) = 0$.
\end{proof}

\begin{lemma}\label{EigOneMinus}
The following equations holds:
\begin{align*}
\lambda_{\min}(P^\perp_{X, TT}) = 1 - \lambda_{\max}(X_T\left( X^\top X\right)^{-1}X_T^\top),
\end{align*}
\begin{align*}
\lambda_{\min}(P^\perp_{X', TT}) = 1 - \lambda_{\max}(X_T\left( X^\top X + \frac{\eta n}{m} \Xtil \Xtil^\top \right)^{-1}X_T^\top).
\end{align*}
\end{lemma}
\begin{proof}
    Since $X_T\left( X^\top X\right)^{-1}X_T^\top$ is symmetric positive semidefinite, we can write $X_T\left( X^\top X\right)^{-1}X_T^\top = Q \Lambda Q^\top$, where $Q$ is an orthogonal matrix and $\Lambda$ is a diagonal matrix with nonnegative diagonals. Then    \begin{equation*}
        I - X_T\left( X^\top X\right)^{-1}X_T^\top = Q( I - \Lambda) Q^\top.
    \end{equation*}
Furthermore, we have shown in inequality~\eqref{eigcompares} that 
    \begin{equation*}
        \lambda_{\max}\left(X_T(X^\top X)^{-1}X_T^\top\right) \leq \frac{1}{1 + \frac{\lambda_{\min}(X_{T^c}^\top X_{T^c})}{\lambda_{\max}(X_T^\top X_T)}}.
    \end{equation*}
Hence, the maximum diagonal in $\Lambda$ is upper-bounded by 1, and $I -\Lambda$ has all diagonal entries in the range $[0,1]$. Thus, we have shown that $\min{\diag(I- \Lambda)} = \max(\diag(\Lambda))$, implying the conclusion of the lemma.
\end{proof}

Returning to the proof of the proposition, we have
\begin{equation*}
\begin{split}
     & \lambda_{\max}\left(X_T\left( X^\top X + \frac{\eta n}{m} \Xtil^\top \Xtil \right)^{-1}X_T^\top\right) \\
     & \qquad \leq \lambda_{\max}\left(X_T\left( X^\top X \right)^{-1}X_T^\top\right) \\
     & \qquad \qquad -\frac{\eta n}{m}\lambda_{\min}\left((X_T\left( X^\top X \right)^{-1}\Xtil^\top (I+\frac{\eta n}{m} \Xtil(X^\top X)^{-1}\Xtil^\top)^{-1}\Xtil\left( X^\top X \right)^{-1}X_T^\top\right) \\
     & \qquad \stackrel{(i)}{\leq} \lambda_{\max}\left(X_T\left( X^\top X \right)^{-1}X_T^\top\right),
\end{split}
\end{equation*}
Here, $(i)$ comes from the fact that
\begin{align*}
\lambda_{\min}\left(X_T\left( X^\top X \right)^{-1}\Xtil^\top \cdot (I+\frac{\eta n}{m} \Xtil(X^\top X)^{-1}\Xtil^\top)^{-1}\Xtil\left( X^\top X \right)^{-1}X_T^\top\right) \geq 0,
\end{align*}
which follows from Lemma~\ref{RestPSD}. Furthermore, by Lemma \ref{EigOneMinus}, we have 
\begin{align*}
\lambda_{\min}\left(P^\perp_{X', TT}\right) = 1 - \lambda_{\max}\left(X_T\left( X^\top X + \frac{\eta n}{m} \Xtil^\top \Xtil \right)^{-1}X_T^\top\right)
\end{align*}
and
\begin{align*}
\lambda_{\min}\left(P^\perp_{X', TT}\right) = 1 - \lambda_{\max}\left(X_T\left( X^\top X + \frac{\eta n}{m} \Xtil^\top \Xtil \right)^{-1}X_T^\top\right).
\end{align*}
Altogether, we conclude that the minimum eigenvalue is at least improved by \\
$\frac{\eta n}{m}\lambda_{\min}\left(X_T\left( X^\top X \right)^{-1} \Xtil^\top(I+\frac{\eta n}{m} \Xtil(X^\top X)^{-1}\Xtil^\top)^{-1}\Xtil\left( X^\top X \right)^{-1}X_T^\top\right)$. 
\end{proof}

Here is the proof of Proposition~\ref{propasymmutual}.
\begin{proof}[Proof of Proposition~\ref{propasymmutual}]
The proof leverages arguments from the proof of Proposition~\ref{randomdesign_subset_recovery_prop} in Appendix~\ref{appendix_randomdesign_subset_recovery_prop}. The goal is to argue that when $n$ and $m$ are sufficiently large, the empirical quantities are close to their population-level versions. We will use Big-$O$ notation to simplify our discussion.

As already stated in inequality~\eqref{onepool_closetopopulation}, if $n \geq t + \frac{c_1^2\sigma_x^4 \|\Sigma\|^2}{\lambda_{\min}^2(\Sigma)}(p+C_1)$, then
\begin{multline*}
\frac{\opnorm{X_{T^c}\Theta X_T^\top - X_{T^c}\left(\frac{X_{T^c}^\top X_{T^c}}{n-t}\right)^{-1}X_T^\top }_\infty}{n-t} \\
\leq \frac{\sqrt{t}}{n-t} \left(\sqrt{p\|\Sigma\|_2} + c_2 \sigma_x^2(\log n + \sqrt{p \log n})\right)  \cdot \frac{2c_1 \sigma^2_x \sqrt{\frac{p+C_1}{n-t}} \|\Sigma\|}{\lambda^2_{\min}\left(\Sigma\right)} B_T.,
\end{multline*}
with probability at least $1 - e^{-C_1} - n^{-1}$, where $c_2 > \max\{2,2/c\}$.

Also for the two-pool case, if $n \geq t + \max\Big\{\frac{c_1^2\sigma_x^4 \|\Sigma\|^2}{\lambda_{\min}^2(\Sigma)},1\Big\}m$ and $m \geq \max\{1,c_1^2 \sigma^4_x (p+C_1')\|\Sigma\|_2^2\}$, we have
\begin{align*}
\begin{split}
& \frac{\opnorm{X_{T^c}\Theta X_T^\top - X_{T^c}\left(\frac{X_{T^c}^\top X_{T^c}+ \frac{\eta n}{m} \Xtil^\top \Xtil}{n-t+\eta n}\right)^{-1}X_T^\top }_\infty}{n-t+\eta n}\\
& \leq \frac{\sqrt{t}}{n-t+\eta n} \left( \sqrt{p\|\Sigma\|_2} + c_2 \sigma_x^2(\log n + \sqrt{p \log n})\right) \left(1 + 2c_1  \sigma^2_x\frac{\|\Sigma\|_2}{\lambda_{\min}(\Sigma)} \sqrt{\frac{p+ C_1'}{m}}\right) \frac{B_T}{\lambda_{\min}(\Sigma)},
\end{split}
\end{align*}
with probability at least $1 - 2e^{-C_1'} - n^{-1}$, where $c_2$ is defined in the same way as above. Noting that  $B_T \propto \sqrt{t}$ and using the triangle inequality, we conclude the proof.
\end{proof}

%%%%%%%%%%%%%%%%%%%%%%%%%%%%%Tuning Paramter Selection%%%%%%%%%%%%%%%%%%%%%%%%%%
\section{Proofs for Section~\ref{sec:tune}}
\label{AppSecTune}
In this section, we provide proofs and additional details for the results in Section~\ref{sec:tune}.
We will establish several auxiliary results in the process, which are stated and proved in Appendix~\ref{AppAux}. The flow of logic is outlined below:

Theorem~\ref{thm:choice-lambda} $\Leftarrow$ (Lemma~\ref{claim:exactrecovery-choice-lambda}, Lemma~\ref{lem:concentration-weaklydependent-orderstatistics}); \\
Lemma~\ref{claim:exactrecovery-choice-lambda} $\Leftarrow$ Theorem~\ref{eaxtsetmainthm}; \\
Lemma~\ref{lem:concentration-weaklydependent-orderstatistics} $\Leftarrow$ (Lemma~\ref{lem:PXmax}, Lemma~\ref{lem:sigmahat-sigmastar});\\
Lemma~\ref{lem:sigmahat-sigmastar} $\Leftarrow$ (Lemma~\ref{lem:concentration-iid-orderstatistics}, Lemma~\ref{lem:differences-orderstatistics});\\
Lemma~\ref{lem:differences-orderstatistics} $\Leftarrow$ Lemma~\ref{lem:PXmax}.\\
Corollary~\ref{cor:lambda-upper-bound-1} $\Leftarrow$ (Theorem~\ref{thm:choice-lambda}, Corollary~\ref{cor:choice-lambda}). 

We sometimes write $\gammahat(\lambda)$ to represent the estimator from Lasso-based debugging with tuning parameter $\lambda$. 

%%%%%%%%%%%%%%%%%%%%%%%%%%%%%%%%%%%%%%%%%%%%%%%%%%%%%%%%%%%%%%%%%%%%%%%%%%%%%%%%%%%%%%%%%%%%%%%%%%%%%%%%%%%%%%%%%%%%%%%%%%
\subsection{Proof of Theorem~\ref{thm:choice-lambda}}

We will first argue that the algorithm will stop, and then argue that all bugs are identified correctly when the algorithm stops. Finally, we will take a union bound over all the iterations in the while loop to obtain a probabilistic conclusion. 

\textbf{Algorithm~\ref{alg:choice-lambda} stops:}
Note that if we have an iteration $k$ such that $\lambdahat^k > 2\lambda^*$ and $C = 0$, then the algorithm must stop after at most $\lfloor \log_2 \frac{\lambda^u}{\lambda^*} \rfloor$ iterations. Otherwise, we know that $C = 1$ for all iterations $k$ such that $\lambdahat^k \geq \lambda^*$. Thus, after $k =  \lfloor\log_2\frac{\lambda^u}{\lambda^*}\rfloor$ iterations, we have
\[
\lambda^k = \frac{\lambda^u}{2^{\lfloor\log_2\frac{\lambda^u}{\lambda^*}\rfloor}} \in \left[\frac{\lambda^u}{2^{\log_2\frac{\lambda^u}{\lambda^*}}}, \frac{\lambda^u}{2^{\log_2\frac{\lambda^u}{\lambda^*}-1}}\right] =  [\lambda^*, 2\lambda^*].
\] 
As established in Lemma~\ref{claim:exactrecovery-choice-lambda}, we know that all true bugs will be identified with such a value of $\lambda^k$, so the remaining points are $(X^{(k)}, y^{(k)}) = (X_{T^c}, y_{T^c})$. Also note that
\[
\|P_{X_{T^c}}^\perp y_{T^c}\|_\infty = \|P_{X_{T^c}}^\perp (X_{T^c}\betastar + \epsilon_{T^c})\|_\infty = \|P_{X_{T^c}}^\perp \epsilon_{T^c}\|_\infty.
\] 
Hence, by Lemma~\ref{lem:concentration-weaklydependent-orderstatistics}, we have
\[
\|P_{X_{T^c}}^\perp \epsilon_{T^c}\|_\infty  <  \frac{5}{2}\frac{1}{ \bar{c}} \sqrt{\log 2n}\,\sigmahat.
\]
Therefore, the stopping criteria takes effect and the algorithm stops. 

\paragraph{Algorithm~\ref{alg:choice-lambda} correctly identifies all bugs:}
A byproduct of the preceding argument is that $\lambdahat > \lambda^*$. By Theorem~\ref{subsetmainthm}, we have $\supp(\gammahat^k) \subseteq \supp(\gammastar)$. Now suppose we are at a stage where $l$ of the $t$ bugs are flagged, where $l \in \{0,1,\dots,t\}$. 

If $l = t$, then $\bar{X} = X_{T^c}$. As argued preveiously, the algorithm stops with high probability. Hence, we output all of the bugs. 

Otherwise, we have $l \leq t-1$. Suppose this happens at the $k^{\text{th}}$ iteration. Then at least one bug remains in $(X^{(k)}, y^{(k)})$, and all the clean points are included. Let $S$ denote the corresponding row indices of $X$ and let $\gammastar_S$ denote the following subvector of $\gammastar$. Since bugs still remain, we must have $\min_{i\in S} |\gammastar_{S,i}| \geq \min_{i\in T} |\gammastar_i|$. Furthermore,
\[
\|P_{X^{(k)}}^\perp y^{(k)}\|_\infty = \|P_{X^{(k)}}^\perp (X^{(k)}\betastar + \gammastar_S+ \epsilon_S)\|_\infty = \|P_{X^{(k)}}^\perp (\gammastar_S+ \epsilon_S)\|_\infty.
\] 
By Lemma~\ref{lem:concentration-weaklydependent-orderstatistics}, we have
\[
\|P_{X^{(k)}}^\perp (\gammastar_S+ \epsilon_S)\|_\infty > \frac{5}{2} \frac{1}{\bar{c}}\sqrt{\log 2n} \, \sigmahat,
\] 
implying that $C = 0$. Thus, the procedure proceeds to the $(k+1)^{\text{st}}$ iteration. If for all $k$ such that $\lambdahat^k \geq 2\lambda^*$, bugs still remain, then $\lambdahat^k$ keeps shrinking until the $\lfloor\log_2\frac{\lambda^u}{\lambda^*}\rfloor^{\text{th}}$ iteration. Then the tuning parameter must lie in the interval $(\lambda^*,2\lambda^*]$, resulting in a value of $\gammahat$ such that $\supp(\gammahat) = \supp(\gammastar)$.

\paragraph{Probability by union bound:}
Now we study the probability for this algorithm to output a value of $\gammahat$ that achieves exact recovery. Firstly, the algorithm stops as long as Lemma~\ref{claim:exactrecovery-choice-lambda} and Lemma~\ref{lem:concentration-weaklydependent-orderstatistics} hold, which holds with probability at least $1 - \frac{3}{n-t} - 2\exp\left(-2\left(\frac{1}{2}-c_t-\nu\right)^2n\right)$. 

Secondly, consider the argument that the algorithm correctly identifies all bugs. For each iteration, the events $\{C = 0 \text{ if a bug still exists}\}$ and $\{C=1 \text{ if no bugs exist}\}$ hold as long as Lemma~\ref{claim:exactrecovery-choice-lambda} and Lemma~\ref{lem:concentration-weaklydependent-orderstatistics} hold, which happens with probability at least $ 1 - \frac{3}{n-t} - 2\exp\left(-2\left(\frac{1}{2}-c_t-\nu\right)^2n\right)$. If the algorithm has $K$ iterations, the probability that the algorithm flags all bugs is therefore at least $ 1 - \frac{3K}{n-t} - 2K\exp\left(-2\left(\frac{1}{2}-c_t-\nu\right)^2n\right)$ by a union bound. Since we have argued that $K \leq \log_2 \frac{\lambda^u}{\lambda(\sigma^*)} $, the desired statement follows.

%%%%%%%%%%%%%%%%%%%%%%%%%%%%%%%%%%%%%%%%%%%%%%%%%%%%%%%%%%%%%%%%%%%%%%%%%%%%%%%%%%%%%%%%%%%%%%%%%%%%%%%%%%%%%%%%%%%%%%%%
\subsection{Proof of Corollary~\ref{cor:lambda-upper-bound-1}}

According to the PDW procedure, we can set $\gammahat = \vec{0}$, solve for $\hat{z}$ via the zero-subgradient equation, and check whether $\|\hat{z}\|_\infty<1$, where $\zhat$ is a subgradient of $\|\gammahat\|_1$. The gradient of the loss function is equal to zero, which implies that 
\[
\zhat = \frac{1}{\lambda n}\|\Pbar^\top P_{X'}^\perp y'\|_\infty.
\]
Therefore, we see that $\|\zhat\|_\infty < 1$ for $\lambda > \frac{\|\bar{P}^\top P_{X'}^\perp y'\|_\infty}{n}$, which means the optimizer satisfies $\gammahat = \vec{0}$. Since $\lambda_u = \frac{2\|\bar{P}^\top P_{X'}^\perp y'\|_\infty}{n}$, the output with tuning parameter $\lambda_u$ gives $\gammahat(\lambda_u) = 0$. 

Note that
\begin{equation*}
\|\Pbar^\top P_{X'}^\perp y'\|_\infty = \|\Pbar^\top \Pbar \gammastar + \Pbar^\top P_{X'}^\perp \epsilon'\|_\infty \leq \|\Pbar^\top \Pbar \gammastar\|_\infty + \|P_{X'}^\perp \epsilon'\|_\infty
\end{equation*}
by the triangle inequality. 
The second term is bounded by $2\max\{1,\sqrt{\frac{\eta n}{mL}}\}\sqrt{\log 2n}\, \sigma^*$ with probability at least $ 1- \frac{1}{n}$, since $e_j^\top P_{X'}^\perp \epsilon'$ is Gaussian with variance at most $\max\{1,\sqrt{\frac{\eta n}{mL}}\}\sigma^*$. For the first term, we have 
\begin{align*}
\begin{split}
\|\Pbar^\top \Pbar \gammastar\|_\infty 
& = \left\|\Pbar^\top \Pbar\gammastar\right\|_{\infty} \\
& \overset{(i)}{\leq}  t \left\|\Pbar^\top \Pbar\right\|_{\max} \|\gammastar_T\|_\infty \\
% & \leq \\
& \overset{(ii)}{\leq} t \|\gammastar\|_\infty \\
& \leq \frac{Cc_\nu}{2} \sqrt{1-c_t} \sqrt{\log 2n}\, n^{c_n+\frac{1}{2}} \sigma^*,
\end{split}
\end{align*}
where $(i)$ holds because $\|v^\top \gammastar\|_1 = \sum_{i\in T} |v_i \gammastar_i| \leq t \|v\|_\infty \|\gammastar\|_\infty$ for any row $v$ of the matrix $\Pbar^\top \Pbar$, and $(ii)$ holds because $\Pbar^\top \Pbar$ is a submatrix of the projection matrix $P_{X'}^\perp$ and each entry of a projection matrix is upper-bounded by 1. Altogether, we obtain 
\[
\lambda_u  \leq \left[\max\left\{1,\sqrt{\frac{\eta n}{mL}}\right\}\frac{2\sqrt{\log 2n}}{n} + \frac{Cc_\nu}{2} \sqrt{1-c_t} \sqrt{\log 2n}\, n^{c_n+\frac{1}{2}} \right] \sigma^*.
\]
By a similar argument as in Theorem~\ref{thm:choice-lambda} and Corollary~\ref{cor:choice-lambda}, we know that Algorithm~\ref{alg:choice-lambda} stops with at most $\log_2 \frac{\lambda_u }{\lambda(\sigma^*)}$ with probability at least $1 - \frac{1}{n-t}$. Hence, 
\begin{align*}
\begin{split}
\log_2 \frac{\lambda_u}{\lambda(\sigma^*)} & = \log_2 \frac{\left[\max\{1,\sqrt{\frac{\eta n}{mL}}\} + \frac{Cc_\nu}{4} \sqrt{1-c_t} n^{c_n+\frac{3}{2}} \right] \frac{2\sqrt{\log 2n}}{n}\sigma^*}{\frac{4}{1-\alpha'}\sqrt{2\log 2n(1-c_t)}\frac{\left\|\Pbar_{T^c}^\perp\right\|_2}{n} \sigma^*} \\
& \overset{(1)}{\leq} \log_2 \frac{\left[\max\{1,\sqrt{\frac{\eta n}{mL}}\} + \frac{C}{4}  n^{c_n+\frac{3}{2}} \right] 2\sqrt{\log n}}{\frac{4}{1-\alpha'}\sqrt{2\log 2n}} \\
&  \overset{(2)}{\leq} \log_2 \frac{\left[\max\{1,\sqrt{\frac{\eta n}{mL}}\} + \frac{C}{4}n^{c_n+\frac{3}{2}} \right] }{2} \\
&  \leq c\left(\frac{3}{2}+c_n\right)\log_2 n + \max\left\{0,\frac{1}{2}\log_2 \frac{\eta n}{mL}-1\right\}, \\
\end{split}
\end{align*}
where $(1)$ comes from the fact that $\Pbar_{T^c}^\perp$ is a submatrix of $P_{X'}^\perp$, which has spectral norm 1 when $n \geq t+p+1$; and $(2)$ holds because $1-\alpha'<1$. To illustrate that $\|\Pbar_{T^c}^\perp\|_2 = 1$, note that it is sufficient to show $\|P_{X',T^c T^c}^\perp\|_2 =1 $ $P_{X',T^c T^c}^\perp$ is a principal matrix of $P_{X'}^\perp$. By interlacing theorem (\cite{hwang2004cauchy}), we know that $\lambda_{\max}(P_{X',T^c T^c}^\perp)$ is no less than the $(t+1)^{\text{st}}$ largest eigenvalue of $P_{X'}^\perp$, which is a projection matrix and therefore has $n-p$ eigenvalues equal to 1. Thus, if $t+1 \leq n-p$, i.e., $n\geq t + p + 1$, then $\|P_{X',T^c T^c}^\perp\|_2 = 1$.

Now that we have bounded the number of iterations, we consider probability that the statement holds. Note that $\epsilon'$ is sub-Gaussian and all the statements based on $\lambda(\sigma^*)$ hold with probability $1-\frac{1}{n-t}$. Compared to Theorem~\ref{thm:choice-lambda}, note that on each iteration, we have subset support recovery with probability $1-\frac{1}{n-t}$; and on iteration $\log_2 \frac{\lambda_u}{\lambda(\sigma^*)} $, we have exact support recovery with probability $1-\frac{1}{n-t}$. Thus, we conclude that Algorithm~\ref{alg:choice-lambda} outputs a value of $\lambdahat$ that achieves exact recovery with probability at least 
\[
 1 - \frac{5\left(c \log_2 n + \max\left\{0,\frac{1}{2}\log_2 \frac{\eta n}{mL}\right\} \right)}{n-t} - 2\left(c\log_2 n +  \max\left\{0,\frac{1}{2}\log_2 \frac{\eta n}{mL}\right\}\right)e^{-2\left(\frac{1}{2}-c_t-\nu\right)^2n}.
\]

%%%%%

\subsection{Proof of Proposition~\ref{prop:Xassumption-holds}}

We consider the three cases in Appendices~\ref{sec:Xgaussian-assumption-holds}, \ref{sec:Xsubgaussian-assumption-holds}, and~\ref{sec:Xconvexconcentration-assumption-holds}. 

Let $\Sigma = \mathbb{E}[x_i x_i^\top]$ and $\Theta = \Sigma^{-1}$, and assume that $X^{(k)}$ corresponds to some $X_S$ with rows indexed by $S$. Our goal is to prove that
\begin{equation}\label{eq:linfcov}
\left\|\frac{X_S \Sigma^{-1} X_S^{\top} }{p} - I\right\|_{\max} \le c\max\left\{\sqrt{\frac{\log |S|}{p}},\frac{\log |S|}{p} \right\},
\end{equation} 
\begin{equation}\label{eq:EqnSigmaEig}
\norm{\frac{X_S\top X_S}{|S|} - \Sigma} \le \frac{\lambda_{\min}(\Sigma)}{2},
\end{equation}
for at most $\log_2 \frac{\lambda_u}{\lambda^*}$ of such sets $S$. Note that $T^c \subseteq S \subseteq [n]$ holds with probability at least $1 - \frac{\log_2 \frac{\lambda_u}{\lambda^*}}{n-t}$.

%%%%%

\subsubsection{Proof of Proposition~\ref{prop:Xassumption-holds} for Gaussian case}\label{sec:Xgaussian-assumption-holds}

The spectral norm bound follows from standard results~\cite{vershynin2010introduction}, which holds for a fixed set $S$ with probability at least $1-e^{-|S|} \geq 1 - e^{-(n-t)}$. Note that Algorithm~\ref{alg:choice-lambda} runs for at most $\log_2 \frac{\lambda_u}{\lambda^*}$ iterations by Theorem~\ref{thm:choice-lambda}. Taking a union bound over all sets $S$, we obtain an overall probability of $1 - \log_2 \frac{\lambda_u}{\lambda^*} \, e^{-(1-c_t)n} \geq 1 -  e^{-\frac{n}{2} + \log \log_2 \frac{\lambda_u}{\lambda^*}}$. 

We now consider~\eqref{eq:linfcov}. Define $z_i = \Theta^{1/2}x_i$ for $1 \le i \le n$, so that
\begin{equation*}
X \Theta^{1/2}  = \begin{bmatrix} - z_{1}^\top  -\\ ... \\ - z_n^\top  - \end{bmatrix}.
\end{equation*} 

We know the $ \Theta^{1/2}x_i$'s are i.i.d.\ isotropic Gaussian random vectors. Hence, $z_i^\top z_i \sim \chi^2(p)$ satisfies
\begin{align*}
    \frac{\norm{z_i}^2}{p} - 1 \leq 4\sqrt{\frac{\log \frac{1}{\delta}}{p}},
\end{align*}
with probability at least $1 - \delta$. Similarly, we can bound $z_k^\top z_k$ and $(z_i+z_k)^\top (z_i + z_k)$. Since $z_i^\top z_k = \frac{1}{2}[(z_i+z_k)^\top (z_i + z_k) - z_i^\top z_i - z_k^\top z_k]$, we then have
\begin{align*}
\begin{split}
    \frac{\langle z_i, z_k\rangle}{p}  \leq 8\sqrt{\frac{\log \frac{1}{\delta}}{p}}, \quad \forall i \neq k,
\end{split}
\end{align*}
with probability at least $1 - \delta$.

We now choose $\delta = \frac{1}{n^{c}}$ for some $c > 2$ and take a union bound over all $n^2$ entries of the matrix $X\Theta X^\top$, to obtain 
\begin{align*}
\left\|\frac{X\Theta X^\top }{p} - I\right\|_{\max} & \le c\max\left\{\sqrt{\frac{\log n}{p}},\frac{\log n}{p} ,\right\} \label{eq:linfcov} 
\end{align*}
with probability at least $1 - \frac{1}{n^{c'-2}}$, where $c' > 2$.

Finally, note that for all $S \subseteq [n]$, we have
\[
\left\|\frac{X_S\Theta X_S}{p} - I\right\|_{\max} \leq \left\|\frac{X\Theta X}{p} - I\right\|_{\max}.
\]

%%%%%

\subsubsection{Proof of Proposition~\ref{prop:Xassumption-holds} for sub-Gaussian case}\label{sec:Xsubgaussian-assumption-holds}

By Lemma~\ref{LemConcEig}, inequality~\eqref{eq:EqnSigmaEig} holds for a fixed set $S$, with probability at least $1-e^{-c|S|} \geq 1 - e^{-c(n-t)}$ for some $c > 0$. Note that Algorithm~\ref{alg:choice-lambda} runs for at most $\log_2 \frac{\lambda_u}{\lambda^*}$ iterations. We then take a union bound over the possible subsets $T^c \subseteq S \subseteq [n]$ to reach a probability of at least $1 - \log_2 \frac{\lambda_u}{\lambda^*} \, e^{-c(1-c_t)n} \geq 1 - e^{-\frac{cn}{2} + \log \log_2 \frac{\lambda_u}{\lambda^*}}$. 

Next, we focus on verifying inequality~\eqref{eq:linfcov}. Assuming that the $x_i$'s are independent random vectors and the components of the $x_i$'s are independent of each other, our goal is to prove that
\begin{align*}
\left\| \frac{X \Theta X^\top }{p} - I \right\|_{\max} \lesssim \max\left\{\sqrt{\frac{\log n}{p}},\frac{\log n}{p}\right\},
\end{align*}
w.h.p., where $\Sigma = \Cov(x_i) = \Theta^{-1}=: D^2$ is a diagonal matrix.

Define $z_i = D^{-1}x_i$. Since the $z_i$'s are mutually independent with independent components,
we know that the vector $g_{ij} = (z_{i1}, ..., z_{ip}, z_{j1},...,z_{jp})^\top$, for $i \neq j$, also has independent components. Furthermore, the sub-Gaussian parameter of $g_{ij}$ is bounded by $l_{\max}= \max_{q=1}^p \frac{K}{d_q^2}$, where $K$ is the sub-Gaussian variance parameter of the $x_i$'s. This is because for a unit vector $u$, we have
\begin{align*}
\begin{split}
\mathbb{E}\left[e^{\lambda u^\top  g_{ij}}\right] &= \Pi_{q = 1}^{p}\mathbb{E}\left[e^{\lambda u_q z_{iq}}\right] \mathbb{E}\left[e^{\lambda u_{p+q} z_{jq}}\right] \\
& = \Pi_{q=1}^p \mathbb{E}\left[e^{\lambda \frac{u_q}{d_q} x_{iq}}\right] \mathbb{E}\left[e^{\lambda \frac{u_{p+q}}{d_q} x_{jq}}\right]\\
& \leq \Pi_{q=1}^p \mathbb{E}\left[e^{\lambda^2 \frac{u_q^2}{2d_q^2} K}\right] \mathbb{E}\left[e^{\lambda^2 \frac{u_{p+q}^2}{2d_q^2} K}\right]\\
& = \mathbb{E}\left[ e^{\sum_{q=1}^p \lambda^2 \frac{u_q^2 + u_{p+q}^2 }{2d_q^2} K}\right] \\
& \leq \mathbb{E}\left[ e^{\sum_{q=1}^p (u_q^2 + u_{p+q}^2) \frac{\lambda^2}{2} l_{\max}}\right] \\
& = \mathbb{E}\left[ e^{\frac{\lambda^2}{2} l_{\max}}\right].
\end{split}
\end{align*}
Since we have assumed that $\|\Sigma\|_2$ is bounded, the $d_q$'s are all bounded for each $q$, so $l_{\max}$ is bounded, as well.

Now let $A = \begin{bmatrix}
0_{p \times p} & I_{p \times p} \\ 0_{p \times p} & 0_{p \times p}
\end{bmatrix}$. By the Hanson-Wright inequality, with probability at least $1 -\delta$, we have
\begin{align}
\left|\frac{\langle z_i, z_j \rangle}{p}\right| = \frac{g_{ij}^\top  A g_{ij}}{p} \leq c_1 \sqrt{\frac{\log {\frac{2}{\delta}}}{p}},
\end{align}
where $c_1$ is a constant related to $l_{\max}$.

Now applying the Hanson-Wright inequality to the vector $z_i$, we have
\begin{align}
\left|\frac{\norm{z_i}^2}{p} - \frac{\mathbb{E}[\norm{z_i}^2]}{p}\right| \leq c_2 \max \left\{ \sqrt{\frac{\log{\frac{2}{\delta}}}{p}}, \frac{\log{\frac{2}{\delta}}}{p}\right\},
\end{align}
with probability at least $ 1- \delta$. Noting that $\mathbb{E}[\norm{z_i}^2] = tr (\Theta \Sigma) = p$, we will finally have
\begin{align*}
  \left|\frac{\norm{z_i}^2}{p} - 1\right| \leq c_2 \max \left\{\sqrt{\frac{\log{\frac{2}{\delta}}}{p}}, \frac{\log{\frac{2}{\delta}}}{p}\right\}.
\end{align*}

Plugging in $\delta = \frac{2}{n^3}$ and taking a union bound, we then conclude that
\begin{align*}
\left\| \frac{X \Theta X^\top}{p} - I\right\|_{\max} \leq 2\max\{c_1,c_2\} \max  \left\{\sqrt{\frac{\log{n}}{p}}, \frac{\log{n}}{p}\right\},
\end{align*}
with probability at least $1 - \frac{2}{n}$.

%%%%%
\subsubsection{Proof of Proposition~\ref{prop:Xassumption-holds} for convex concentration case}
\label{sec:Xconvexconcentration-assumption-holds}

Recall the following definition:

\begin{definition}[Convex concentration property]
Let $X$ be a random vector in $\mathbb{R}^d$. If for every 1-Lipschitz convex function $\varphi: \mathbb{R}^d \to \mathbb{R}$ such that $\mathbb{E}[\varphi(X)]<\infty$ and for every $t>0$, we have
\[
\mathbb{P}\left(|\varphi(X) - \mathbb{E}[\varphi(X)]| \geq t\right) \leq 2 \exp(-t^2/K^2),
\]
then $X$ satisfies the convex concentration property with constant $K$. 
\end{definition}

Suppose $x_i$ has the convex concentration property with parameter $K$. Note that
\begin{align*}
\left\|\frac{X\Theta X^\top }{p} - I\right\|_{\max} & = \max_{i,j} \left|e_i^\top  \left(\frac{X \Theta X^\top }{p} - I\right)e_j\right| \\
& = \max_{i,j} \left|\frac{x_i^\top  \Theta x_j}{p} - e_i^\top  e_j\right|.
\end{align*}
By Lemma~\ref{LemAda15}, we thus have the exponential tail bound
\begin{equation*}
\mprob\left(\left|\frac{x_i^\top  \Theta x_i}{p} - 1\right| \ge w\right) \leq 2 \exp \left(-\frac{1}{C}\min\left\{\frac{w^2p^2}{2K^4\|\Theta\|_F}, \frac{wp}{K^2 \|\Theta\|_2}\right\}\right),
\end{equation*}
for all $1 \le i \le p$, which implies that
\[
\left|\frac{x_i^\top  \Theta x_i}{p} - 1\right| \leq c K^2 \max \left\{ \sqrt{\frac{\log{\frac{2}{\delta}}}{p}}, \frac{\log{\frac{2}{\delta}}}{p}\right\},
\] 
with probability at least $ 1- \delta$. Taking $\delta = 2/n^3$, we then obtain 
\begin{align}\label{eq:diag}
\left|\frac{x_i^\top  \Theta x_i}{p} - 1\right| \leq c K^2 \max \left\{ \sqrt{\frac{\log{n}}{p}}, \frac{\log{n}}{p}\right\},
\end{align}
with probability at least $1 - \frac{2}{n^3}$.

Now we consider the off-diagonals $\frac{x_i \Theta x_j}{p}$, for $i \neq j$. We first rewrite
\[
\mathbb{P}\left(\left|\frac{x_i^\top \Theta x_j }{p} \right| \geq \Delta \right) = \mathbb{P} \left(\left|x_i^\top \frac{ \Theta x_j }{\norm{\Theta x_j }} \right| \geq \frac{\Delta p}{\norm{\Theta x_j }} \right).
\]
Conditioning on $\norm{\Theta x_j }$ for some $ w > 0$, we obtain
\begin{align*}
\mathbb{P}\left(\left|\frac{x_i^\top \Theta x_j }{p} \right| \geq \Delta \right) & = \mathbb{P} \left(\left|x_i^\top \frac{ \Theta x_j }{\norm{\Theta x_j }} \right| \geq \frac{\Delta p}{\norm{\Theta x_j }} \middle\vert \norm{\Theta x_j} \geq w \right)\mathbb{P} \left(\norm{\Theta x_j} \geq w \right) \\
& \quad + \mathbb{P} \left(\left|x_i^\top \frac{ \Theta x_j }{\norm{\Theta x_j }} \right| \geq \frac{\Delta p}{\norm{\Theta x_j }} \middle\vert \norm{\Theta x_j} < w \right) \mathbb{P} \left(\norm{\Theta x_j} < w \right).
\end{align*}
Since we have a convex 1-Lipschitz function mapping from $x_i$ to $x_i^\top \frac{ \Theta x_j }{\norm{\Theta x_j }}$, we can further upper-bound the probability using the convex concentration property:
\begin{align*}
\begin{split}
\mathbb{P}\left(\left|\frac{x_i^\top \Theta x_j }{p} \right| \geq \Delta \right) & \leq \mathbb{P} \left(\norm{\Theta x_j} \geq w \right)  + \mathbb{P} \left(\left|x_i^\top \frac{ \Theta x_j }{\norm{\Theta x_j }} \right| \geq \frac{\Delta p}{\norm{\Theta x_j }} \middle\vert \norm{\Theta x_j} < w \right) \\ 
& \leq \mathbb{P} \left(\norm{x_j} \geq \frac{w}{\norm{\Theta}} \right)  + \mathbb{P} \left(\left|x_i^\top \frac{ \Theta x_j }{\norm{\Theta x_j }} \right| \geq \frac{\Delta p}{w} \right) \\ 
& \overset{(1)}{\leq} \mathbb{P} \left(\norm{x_j} - \mathbb{E}[\norm{x_j}] \geq \frac{w}{\norm{\Theta}}  - \mathbb{E}[\norm{x_j}]  \right)  + 2 \exp\left(-\frac{\Delta^2 p^2}{w^2 K^2}\right) \\
& \overset{(2)}{\leq} \mathbb{P} \left(\norm{x_j} - \mathbb{E}[\norm{x_j}] \geq \frac{w}{\norm{\Theta}}  - \sqrt{\mathbb{E}[\norm{x_j}^2]}  \right)  + 2 \exp\left(-\frac{\Delta^2 p^2}{w^2 K^2}\right) \\
& \overset{(3)}{\leq} 2\exp\left(-\frac{\left(\frac{w}{\norm{\Theta}}  - \sqrt{tr(\Sigma)}\right)^2}{K^2}\right)  + 2 \exp\left(-\frac{\Delta^2 p^2}{w^2 K^2}\right) \\
& \leq 2\exp\left(-\frac{\left(\frac{w}{\norm{\Theta}}  - \sqrt{p \|\Sigma\|_2}\right)^2}{K^2}\right)  + 2 \exp\left(-\frac{\Delta^2 p^2}{w^2 K^2}\right), \\
\end{split}
\end{align*}
where (1) and (3) use the convex concentration property and (2) uses Jensen's inequality. The last inequality assumes that $w \geq \sqrt{p \|\Sigma\|_2}$, can be guaranteed if we choose $w$ sufficiently large.

Plugging $\Delta = c\max\left\{\frac{\log n}{p}, \sqrt{\frac{\log n}{p}}\right\}$ and $w = c' \left(\sqrt{p} + \sqrt{\log n}\right)$ into the above derivations, we then obtain
\[
\mathbb{P}\left(\left|\frac{x_i^\top \Theta x_j }{p} \right| \geq \Delta \right) \leq 2\exp\left(-\frac{c''\log n}{K^2}\right)  + 2 \exp\left(-c'''\frac{\max\{(\log n)^2, p\log n \}}{(p + \log n) K^2}\right).
\]
If $p > \log n$, then $2 \exp\left(-\frac{\max\{(\log n)^2, p\log n \}}{(p + \log n) K^2}\right) \leq 2 \exp\left(-\frac{ c''''\log n }{ K^2}\right) $; If $p \leq \log n$, then $2 \exp\left(-\frac{\max\{(\log n)^2, p\log n \}}{(p+ \log n) K^2}\right) \leq 2 \exp\left(-\frac{ c'''''\log n }{ K^2}\right) $. Hence, we have
\[
\mathbb{P}\left(\left|\frac{x_i^\top \Theta x_j }{p} \right| \geq \Delta \right) \leq 2 \exp\left(- C \log n \right).
\]
We can choose $c$ and $c'$ sufficiently large to ensure that $C > 2$. Combining this with inequality~\eqref{eq:diag} using a union bound, we finally obtain the desired result. 

%%%%%%%%%%%%%%%%%%%%%%%%%%%%%%%%%%%%%%%%%%%%%%%%%%%%%%%%%%%%%%%%%%%%%%%%%%%%%%%%%%%%%%%%%%%%%%%%%%%%%%%%%%%%%%%%%%%%%%%%
\subsection{Auxiliary lemmas}
\label{AppAux}

By Theorem~\ref{subsetmainthm}, we have the following corollary:

\begin{corollary}\label{cor:choice-lambda}
For two data pools, suppose the eigenvalue and mutual incoherence conditions hold. Let $\lambda \geq \lambda(\sigma^*).$
Then with probability $1- \frac{1}{n-t}$, we have $\supp(\gammahat) \subseteq \supp(\gammastar)$, and 
\begin{align}
    \infnorm{\gammahat(\lambda) - \gammastar} \leq  G'(\lambda).
\end{align}
\end{corollary}

\begin{proof}
Recall that the rule for regularizer selection in Theorem~\ref{subsetmainthm} is
\[
        \lambda \geq
        \frac{2}{1-\alpha'}\left\|\Pbar_{T^c}^\top \left(I-\Pbar_T (\Pbar_T^\top \Pbar_T)^{-1}\Pbar_T^\top \right)\frac{\epsilon'}{n}\right\|_\infty.
\]
Note that $e_j^\top \Pbar_{T^c}^\top \left(I-\Pbar_T (\Pbar_T^\top \Pbar_T)^{-1}\Pbar_T^\top \right)\frac{\epsilon'}{n}$ is sub-Gaussian with variance parameter $\max\{1,\frac{\eta n}{mL}\}\frac{\|\Pbar_{T^c}^\perp\|_2^2\sigma^{*2}}{n^2}$. We have
\[
\max_{j \in T^c} \left|e_j^\top \Pbar_{T^c}^\top \left(I-\Pbar_T (\Pbar_T^\top \Pbar_T)^{-1}\Pbar_T^\top \right)\frac{\epsilon'}{n}\right| \leq 4 \max\left\{1,\frac{\eta n}{mL}\right\}\sqrt{\log 2(n-t)}\frac{\|\Pbar_{T^c}^\perp\|_2}{n} \sigma^{*2},
\]
with probability at least $1 - \frac{1}{n-t}$. According to the definition of $\lambda(\sigma^*)$, we can further derive the bound for $\gammahat$, since
\[
    \infnorm{\gammahat - \gammastar} \leq  \| (P_{X', TT}^\perp)^{-1}P_{X', T\cdot}^\perp \epsilon'\|_\infty + 2n\lambda(\sigma^{*})\opnorm{(P_{X', TT}^\perp)^{-1}}_{\infty}.
\]
\end{proof}
%%%%%%%%%%%%%%%%%%%%%%%%%%%%%%%%%%%%%%%%%%%%%%%%%%%%%%%%%%%%

The following lemma suggests that if $\min_{i \in T} |\gammastar_i| \geq G'(2\lambda^*)$, then $\supp(\gammahat(\lambda)) = \supp(\gammastar)$ if we take $\lambda \in [\lambda^*, 2\lambda^*]$.

\begin{lemma}
\label{claim:exactrecovery-choice-lambda}
If $\min_{i \in T} |\gammastar_i| \geq G'(2\lambda^*)$, then taking $\lambda \in [\lambda^*,2\lambda^*]$ yields an estimator $\gammahat(\lambda)$ that satisfies $\supp(\gammahat(\lambda)) = \supp(\gammastar)$. 
\end{lemma}

\begin{proof}
According to Theorem~\ref{subsetmainthm}, for a regularizer $\lambda \in [\lambda^*, 2\lambda^*]$, we have $\gammahat_{T^c} = 0$ and $\infnorm{\gammahat(\lambda) - \gammastar} \leq  G'(\lambda)$.
If $\min_{i\in T} |\gammastar_i| \geq G'(2\lambda^*)$, then by the triangle inequality, we have
\[
|\gammahat_i| > \min_{i\in T} |\gammastar_i| - G'(\lambda) \geq G'(2\lambda^*) - G'(\lambda) \geq 0,
\]
for all $i \in T$.
\end{proof}

%%%%%%%%%%%%%%%%%%%%%%%%%%%%%%%%%%%%%%%%%%%%%%%%%%%%%%%%%%%%

We use $X_S$ to represent some $X^{(k)}$ for $S \subseteq [n]$, as shown in Algorithm~\ref{thm:choice-lambda}. In each loop of the algorithm, we know that the points in $S^c$ all lie in $T$ by the subset recovery result. Thus, $S \supseteq T^c$. Let $l = n -|S|$, and note that $0 \leq l \leq t$.

\begin{lemma}
\label{lem:PXmax}
Suppose Assumption~\ref{assump:X-choice-lambda} holds. If $\lambda_{\min}(\Sigma)$ and $\lambda_{\max}(\Sigma)$ are bounded, then
\begin{equation*}
\left\|P_{X_S}^\perp - \left(1-\frac{p}{n-l}\right)I\right\|_{\max} \leq C\frac{\max\{p, \sqrt{p \log (n-l)},\log (n-l)\}}{n-l}.
\end{equation*}
\end{lemma}

\begin{proof}  

Using the notation $\Theta = \Sigma^{-1}$ and $\Sigmahat = \frac{X_S^\top X_S}{|S|}$, we have
\begin{align*}
 \left\|P_{X_S}^\perp - \left(1-\frac{p}{|S|}\right)I_{|S|\times |S|}\right\|_{\max}  & = \left\|X_S(X_S^\top X_S)^{-1} X_S^\top  - \frac{p}{|S|}I\right\|_{\max} \\
& \le \left\|\frac{X_S(\Sigmahat)^{-1}X_S^\top }{|S|} - \frac{X_S\Theta X_S^\top }{|S|}\right\|_{\max}  + \left\|\frac{X_S\Theta X_S^\top }{|S|} - \frac{p}{|S|} I\right\|_{\max}.
\end{align*}
By assumption, we may bound the second term by
\begin{equation*}
\left\|\frac{X_S\Theta X_S^\top }{|S|} - \frac{p}{|S|}I\right\|_{\max} \le \frac{p}{|S|} \cdot c \max \left\{\sqrt{\frac{\log |S|}{p}},{\frac{\log |S|}{p}} \right\} = \frac{c \max\{\sqrt{p \log |S|}, \log |S|\}}{|S|}.
\end{equation*}
For the first term, we have
\begin{align*}
\begin{split}
\left\|\frac{X_S(\Sigmahat)^{-1}X_S^\top }{|S|} - \frac{X_S\Theta X_S^\top }{|S|}\right\|_{\max} &= \frac{1}{|S|} \left\|X_S\left((\Sigmahat)^{-1} - \Theta\right) X_S^\top \right\|_{\max} \\
&\le \norm{(\Sigmahat)^{-1} - \Theta} \cdot \max_{1 \le i \le |S|} \frac{1}{|S|} \|X_S^\top  e_i\|_2^2.
\end{split}
\end{align*}
We now have the bound
\begin{equation*}
\begin{split}
\norm{(\Sigmahat)^{-1} - \Theta} &\le  \frac{\frac{1}{2}\lambda_{\min}(\Sigma)}{\lambda_{\min}(\Sigma)\lambda_{\min}(\Sigmahat)} \\
& \leq \frac{\frac{1}{2}\lambda_{\min}(\Sigma)}{\lambda_{\min}(\Sigma)(\lambda_{\min}(\Sigma)-\frac{1}{2}\lambda_{\min}(\Sigma))} = \frac{1}{\lambda_{\min}(\Sigma)},
\end{split}
\end{equation*}
as well, where the second inequality holds by Weyl's Theorem (\cite{horn1994topics}): $\lambda(\Sigmahat) \geq \lambda(\Sigma) - \|\Sigma-\Sigmahat\|_2$. The basic idea for the first inequality is to use the multiplicativity of matrix norms to conclude that
\begin{align}\label{invBound}
\begin{split}
\norm{A^{-1} - B^{-1}}  & \le \norm{A^{-1} (A-B) B^{-1}} \\
& \le \norm{A^{-1}} \norm{A-B} \norm{B^{-1}} \\
&= \frac{\norm{A-B}}{\lambda_{\min}(A) \cdot \lambda_{\min}(B)}.
\end{split}
\end{align}
Hence, an upper bound on $\norm{A-B}$---which we obtain from our assumptions---together with minimum eigenvalue bounds on $A$ and $B$, implies an upper bound on $\norm{A^{-1} - B^{-1}}$.\\
Finally, we have
\begin{align*}
\max_{1 \le i \le |S|} \frac{1}{|S|} \|X_S^\top  e_i\|_2^2  & \le \max_{1 \le i \le |S|} \frac{1}{|S|} \cdot \frac{\|\Theta^{1/2} X_S^\top  e_i\|_2^2}{\lambda_{\min}^2(\Theta^{1/2})} \\
& = \frac{1}{\lambda_{\min}(\Theta)} \cdot \max_{1 \le i \le |S|} \frac{\|\Theta^{1/2} X_S^\top  e_i\|_2^2}{|S|} \\
& = \lambda_{\max}(\Sigma) \cdot \max_{1 \le i \le |S|} \frac{e_i^\top  X_S\Theta X_S^\top  e_i}{|S|} \\
& \le \lambda_{\max}(\Sigma) \cdot \left\|\frac{X_S \Theta X_S^\top }{|S|}\right\|_{\max}.
\end{align*}
By assumption, we have
\begin{align*}
\left\|\frac{X_S \Theta X_S^\top }{p} - I\right\|_{\max} & \le c \max \left\{ \sqrt{\frac{\log |S|}{p}}, \frac{\log |S|}{p} \right\}.
\end{align*}
Hence, rescaling and using the triangle inequality, we have
\begin{align*}
\left\|\frac{X_S \Theta X_S^\top }{|S|}\right\|_{\max} & \le \frac{p}{|S|} \left(\left\|\frac{X_S \Theta X_S^\top }{p} - I\right\|_{\max} + 1\right) \le \frac{p}{|S|} + \frac{p}{|S|}\max \left\{ \sqrt{\frac{\log |S|}{p}}, \frac{\log |S|}{p} \right\}.
\end{align*}
Altogether, we have the bound
\begin{equation*}
\left\|\frac{X_S(\Sigmahat)^{-1}X_S^\top }{|S|} - \frac{X_S\Theta X_S^\top }{|S|}\right\|_{\max} \le \frac{\lambda_{\max}(\Sigma)}{\lambda_{\min}(\Sigma)} \cdot \frac{p}{|S|} \left(1+\max \left\{ \sqrt{\frac{\log |S|}{p}}, \frac{\log |S|}{p} \right\} \right).
\end{equation*}
Finally, we have
\begin{multline*}
\frac{c\max\{\sqrt{p \log |S|},\log |S|\}}{|S|} + c'' \frac{p}{|S|}\left(1 + \max \left\{ \sqrt{\frac{\log |S|}{p}}, \frac{\log |S|}{p} \right\} \right) \\
\leq C\frac{\max\{p, \sqrt{p \log |S|},\log |S|\}}{|S|}.
\end{multline*}
This finishes the proof.
\end{proof}

%%%%%%%%%%%%%%%%%%%%%%%%%%%%%%%%%%%%%%%%%%%%%%%%%%%%%%%%%%%

We use $\alpha(k)$ to represent the $k^{\text{th}}$ order statistics of $|\epsilon_i|$, for $i \in T^c$, where $\alpha_{(1)} \leq \alpha_{(2)} \leq \cdots \leq \alpha_{(n-t)}$. 
\begin{lemma}
\label{lem:concentration-iid-orderstatistics}
For i.i.d.\ random variables $\{|\epsilon_i|\}_{i \in T^c}$, the $k^{\text{th}}$ order statistics, for any $k \in \{\frac{n-t}{2}, \dots, \frac{n}{2} \}$ satisfy
\[
c_\nu \sigma^* \leq \alpha(k) \leq C_\nu \sigma^*,
\]
with probability at least $1-2\exp\left(-2\left(\frac{1}{2}-c_t-\nu\right)^2n\right)$, for $\nu \in(0,\frac{1}{2})$ such that $\nu < \frac{1}{2} - c_t$. 
\end{lemma}

\begin{proof}
By the assumptions on the noise distribution, we have
\[
\nu =\mathbb{P}\left[|\epsilon_i| \leq c_\nu\sigma^*\right]\ \mbox{and}\ \nu =\mathbb{P}\left[|\epsilon_i| \geq C_\nu\sigma^*\right].
\]

Let $\xi_i$'s be i.i.d.\ Bernoulli variables such that
\begin{equation*}
\xi_i =
\begin{cases}
1 & \text{if $|\epsilon_i| \leq c_\nu\sigma^*$,} \\
0 & \text{otherwise}.
\end{cases}
\end{equation*}
Note that $t = c_t n$ for some positive constant $c_t \in (0,\frac{1}{2})$. We have 
\[
k - \nu (n-t) \geq \frac{n-t}{2} - \nu (n-t) = \frac{(1-c_t)(1-2\nu)}{2}n > 0
\]
and 
\[
\left(\frac{k}{n-t}-\nu\right)^2(1-c_t) \geq \left(\frac{1}{2}-\nu\right)^2(1-c_t) \geq \left(\frac{1-2\nu}{2}\right)\left(\frac{1-c_t-2\nu}{2}\right).
\]
By Hoeffding's inequality (\cite{hoeffding1994probability}), we then obtain
\begin{align*}
\mathbb{P}\left[\sum_{i=1}^{n-t}\xi_i \geq k\right] &= \mathbb{P}\left[\sum_{i=1}^{n-t}\xi_i - \nu(n-t) \geq k -\nu(n-t)\right] \\
& \leq \exp\left(-2\left(\frac{k}{n-t}-\nu\right)^2(n-t)\right) \\
& \leq \exp\left(-2\left(\frac{1}{2}-c_t-\nu\right)^2n\right),
\end{align*}
implying that
\[
\mathbb{P}\left[\alpha(k) \leq c_\nu \sigma^*\right] = \mathbb{P}\left[\sum_{i=1}^n\xi_i \geq k\right] \leq \exp\left(-2\left(\frac{1}{2}-c_t-\nu\right)^2n\right).
\]

Similarly, let $\eta_i$'s be i.i.d.\ Bernoulli variables such that
\begin{equation*}
\eta_i =
\begin{cases}
1 & \text{if $|\epsilon_i| \geq C_\nu \sigma^*$,} \\
0 & \text{otherwise.}
\end{cases}
\end{equation*}
Note that the assumption that $c_t < \frac{1}{2} - \nu $ gives us 
\[n-t-k - \nu (n-t) > n-c_t n - \frac{n}{2}  - \nu (1-c_t) n \geq \left(\frac{1}{2}-c_t - \nu \right)n > 0,\]
and 
\[
\left(1-\frac{k}{n-t}-\nu\right)^2(1-c_t) \geq \left(\frac{1}{2}-c_t - \nu\right)^2 \frac{n}{n-t} \geq \left(\frac{1}{2}-c_t - \nu\right)^2.
\]
Then by Hoeffding inequality, we obtain
\begin{align*}
\mathbb{P}\left[\sum_{i=1}^{n-t}\eta_i \geq n-t-k\right] &= \mathbb{P}\left[\sum_{i=1}^{n-t}\eta_i - \nu(n-t)\geq n-t-k - \nu(n-t)\right]\\
& \leq \exp\left(-2\left(1-\frac{k}{n-t}-\nu\right)^2(n-t)\right) \\
& \leq \exp\left(-2\left(\frac{1}{2}-c_t-\nu\right)^2n\right), 
\end{align*}
so that
\[
\mathbb{P}\left[\alpha(k) \geq C_\nu \sigma^*\right] \leq \exp\left(-2\left(\frac{1}{2}-c_t-\nu\right)^2n\right).
\]
\end{proof}

%%%%%%%%%%%%%%%%%%%%%%%%%%%%%%%%%%%%%%%%%%%%%%%%%%%%%%%%%%%

\begin{lemma}
\label{lem:differences-orderstatistics}
Suppose the assumptions of Lemma~\ref{lem:PXmax} hold and 
\[
n^{1-2c_n} \geq \max\left\{\frac{32C^2}{1-c_t}\log(2n)\, (p^2+ \log^2 n), \quad \left(\frac{24}{c_\nu}\right)^{\frac{1}{c_n}}\right\},
\]
and 
\[
\max_{i \in S} |\gammastar_S| \le \frac{c_\nu C}{2}\sqrt{1-c_t} \sqrt{ \log 2n}\, \frac{n^{1/2+c_n}}{t} \sigma^*,
\]
for some constant $c_n \in (0,\frac{1}{2})$.
Then the $k^{\text{th}}$ order statistic of $|P_{X_{S}}^\perp (\gammastar_S + \epsilon_{S})|$ and the $k^{\text{th}}$ order statistic of $\left|\left(1-\frac{p}{|S|}\right) (\gammastar_S+\epsilon_{S})\right|$ have differences of at most $\frac{\bar{c}}{4}\sigma^*,$ for any $k \in [|S|]$, with probability at least $1-\frac{1}{n-t}$.
\end{lemma}

\begin{proof}
Recall that $l = n - |S|$. Now consider the sequences $\{z_i = |e_i^\top P_{X_S}^\perp (\gammastar_S+\epsilon_S)|\}_{i=1}^{n-l}$ and $\left\{w_i = \left|\left(1-\frac{p}{n-l}\right) (\gammastar_{S,i}+\epsilon_{S,i})\right|\right\}_{i=1}^{n-l}$. By the triangle inequality, we have
\begin{align*}
\begin{split}
    |z_i - w_i| &\leq \left|e_i^\top \left(P_{X_S}^\perp-\left(1-\frac{p}{n-l}\right)I\right)(\gammastar_S+\epsilon_S)\right|\\
    &\leq \left|\underbrace{e_i^\top \left(P_{X_S}^\perp-\left(1-\frac{p}{n-l}\right)I\right)\gammastar_S}_{v_i}\right| + \left|\underbrace{e_i^\top \left(P_{X_S}^\perp-\left(1-\frac{p}{n-l}\right)I\right)\epsilon_S}_{u_i}\right|,
\end{split}    
\end{align*}
for $i =1,\dots,n-l$.

Since $u_i$ is sub-Gaussian with parameter at most $\norm{(P_{X_S}^\perp)_{i\cdot}-e_i^\top \left(1-\frac{p}{n-l}\right)}^2\sigma^{*2}$, we can upper-bound the maximum of $\{|u_i|\}$. With probability at least $ 1- \frac{1}{n-t}$, we have 
\begin{align*} 
\begin{split}
\max_{i\in S}{|u_i|} & \leq 2\sqrt{ \log 2(n-l)} \sigma^* \norm{(P_{X_S}^\perp)_{i\cdot}-e_i^\top \left(1-\frac{p}{n-l}\right)} \\
& \leq 2\sqrt{\log 2(n-l)} \sigma^* \sqrt{n-l}\left\|P_{X_S}^\perp-\left(1-\frac{p}{n-l}\right)\right\|_{\max}\\
& \leq 2C \sqrt{\log 2(n-l)} \frac{(\sqrt{p}+\sqrt{\log (n-l)})^2}{\sqrt{n-l}}\sigma^*,
\end{split}
\end{align*}
where the last inequality follows by Lemma~\ref{lem:PXmax}. Further note that since $n^{1-2c_n} \geq \frac{32C^2}{1-c_t}\log (2n)\, (p^2+ \log^2 n)$ for some $c_n \in (0,\frac{1}{2})$, we have $\max_{i\in S}{|u_i|} \leq \frac{1}{n^{c_n}}\sigma^*$ .

For the $v_i$'s, we have
\begin{align}\label{eq:max-PXS-gammaS}
\begin{split}
    \max_{i \in S}|v_i| &\overset{(i)}{\leq} t \left\|P_{X_S}^\perp-\left(1-\frac{p}{n-l}\right)\right\|_{\max} \max_{i \in S} |\gammastar_S| \\
    & \overset{(ii)}{\leq} \sqrt{\frac{t^2}{n(1-c_t)}} \frac{(\sqrt{p}+\sqrt{\log (n-l)})^2}{\sqrt{n-l}} \max_{i \in S} |\gammastar_S| \\
    & \overset{(iii)}{\leq} \frac{1}{2C}\sqrt{\frac{1}{1-c_t}} \frac{t}{n^{1/2+c_n}}\frac{1}{\sqrt{\log 2n}} \max_{i \in S} |\gammastar_S|,
\end{split}
\end{align}
where $(i)$ holds because $|a^\top \gammastar_S| \leq \|a\|_\infty \|\gammastar_S\|_\infty |\supp(\gammastar_S)|$ for any vector $a$, $(ii)$ holds by Lemma~\ref{lem:PXmax}, and $(iii)$ holds by our assumption on $n$.
Combining this with the assumption that $\max_{i \in S} |\gammastar_S| \le \frac{c_\nu C}{4}\sqrt{1-c_t} \sqrt{ \log 2n}\, \frac{n^{1/2+c_n}}{t} \sigma^*$, we obtain $\max_{i\in S}|v_i| \leq \frac{c_\nu}{8}\sigma^*$. Finally, using the fact that $n \geq \left(\frac{24}{c_\nu}\right)^{\frac{1}{c_n}}$, we obtain
\[
|z_i - w_i| \leq \frac{c_\nu}{6}\sigma^*,
\]
with probability at least $1-\frac{1}{n-t}$.

We then use the following lemma:

\begin{lemma}
\label{claim:median-shift}
For two sequences $a_1,\dots,a_n$ and $b_1, \dots,b_n$ such that $|a_i - b_i| \leq c$ for some positive number $c$, the $j^{\text{th}}$ order statistics of $\{a_i\}$ and $\{b_i\}$, denoted by $\alpha_a(j)$ and $\alpha_b(j)$, satisfy
\begin{align}\label{eq:statistics-bound}
|\alpha_a(j) - \alpha_b(j)| \leq c.
\end{align}
\end{lemma}

\begin{proof}
Without loss of generality, suppose $a_1 \leq a_2 \leq \cdots \leq a_n$. If there exists $j \in [n]$ such that inequality~\eqref{eq:statistics-bound} does not hold, then we have either $a_j > c+\alpha_b(j)$ or $a_j < \alpha_b(j) -c$. If the first case occurs, we have
\begin{equation*}
a_n \geq \cdots \geq a_j > c+ \alpha_b(j) \geq c+ \alpha_b(j-1) \geq \cdots c+ \alpha_b(1).
\end{equation*}
Pick a number $z$ between $c+\alpha_b(j)$ and $a_j$. We see that at least $j$ of the $b_i$'s, denoted by $\vec{b}_\downarrow$, are smaller than $z-c$; and at least $n-j+1$ of $a_i$'s, denoted by $\vec{a}_\uparrow$, are greater than $z$. This means that at most $j-1$ of $a_i$'s are no larger than $z$. Note that for the $\vec{b}_\downarrow$, the components of the corresponding vector $\vec{a}_\downarrow$ are within a distance of $c$, so the elements of $\vec{a}_\downarrow$ must be at most $z$. However, this contradicts the fact that at most $j-1$ of the $a_i$'s are at most $z$.  This concludes the proof. 
\end{proof}

From Lemma~\ref{claim:median-shift}, we can compare the order statistics of sequences $\{z_i\}_{i=1}^n$ and $\{w_i\}_{i=1}^n$ and conclude that they have differences of at most $\frac{\bar{c}}{6}\sigma^*$, with probability at least $1-\frac{1}{n-t}$. 
\end{proof}

%%%%%%%%%%%%%%%%%%%%%%%%%%%%%%%%%%%%%%%%%%%%%%%%%%%%%%%%%%%
\begin{lemma}
\label{lem:sigmahat-sigmastar}
Suppose the conditions of Lemma~\ref{lem:concentration-iid-orderstatistics} and Lemma~\ref{lem:differences-orderstatistics} hold, and also $\min_{i\in T} |\gammastar_i| > 4 \sqrt{\log(2n)}\sigma^*$. Then
\begin{align*}
    \left(c_\nu - \frac{|S|}{|S|-p}\frac{c_\nu}{6}\right)\sigma^*\leq \sigmahat \leq \left(\frac{|S|}{|S|-p} \frac{c_\nu}{6} + C_\nu\right)\sigma^*,
\end{align*}
with probability at least $1 - 2\exp\left(-2\left(\frac{1}{2}-c_t-\nu\right)^2n\right) - \frac{2}{n-t}$.
\end{lemma}

\begin{proof}
Let $M_P (S)$ denote the median of $|P_{X_{S}}^\perp (\gammastar_S + \epsilon_{S})|$. By Lemma~\ref{lem:differences-orderstatistics}, we know that $M_P(S)$ is close to the median of $\left|\left(1-\frac{p}{|S|}\right)(\gammastar_S+\epsilon_S)\right|$. Thus, it remains to analyze the median of $\{|\gammastar_i + \epsilon_i|\}_{i\in S}$. 

Note that for $j \in T^c$, we have $|\gammastar_j + \epsilon_j| = |\epsilon_j|$. Therefore, for all $j \in S \cap T^c = T^c$, we have $|\gammastar_j + \epsilon_i|_\infty \leq 2\sqrt{\log 2n}\, \sigma^*$, with probability at least $1-\frac{1}{n}$. 

For $i \in T \cap S$, by the assumption that $\min_{i \in T}|\gammastar_i| > 4\sqrt{\log 2n}\, \sigma^*$, we have $|\gammastar_i + \epsilon_i| \geq |\gammastar_i| - |\epsilon_i| > 2\sqrt{\log 2n}\, \sigma^*$. Therefore, the median of $|\gammastar_S + \epsilon_S|$ is actually the $k^{\text{th}}$ order statistics of $|\epsilon_{T^c}|$ for some $\{k \in  \frac{n-t}{2}, \dots,  \frac{n}{2} \}$. By Lemma~\ref{lem:differences-orderstatistics}, we have
\[
\left(1-\frac{p}{|S|}\right)\alpha(k) - \frac{c_\nu}{6} \sigma^* \leq M_P(S) \leq \left(1-\frac{p}{|S|}\right)\alpha(k) + \frac{c_\nu}{6} \sigma^*.
\]
In Algorithm~\ref{alg:choice-lambda}, at some iteration $k$, we have $\sigmahat = \frac{|S|}{|S|-p} M_P(S)$, where $S$ is the corresponding set of indices of $\left(\supp(\gammahat^{(k)})\right)^c$. Thus,
\[
\alpha(k) - \frac{|S|}{|S|-p} \frac{c_\nu}{6}\sigma^* \leq \sigmahat \leq \alpha(k) + \frac{|S|}{|S|-p} \frac{c_\nu}{6}\sigma^*.
\]
Combining this with Lemma~\ref{lem:concentration-iid-orderstatistics}, we have
\begin{align*}
    \left(c_\nu - \frac{|S|}{|S|-p}\frac{c_\nu}{6}\right)\sigma^*\leq \sigmahat \leq \left(\frac{|S|}{|S|-p} \frac{c_\nu}{6} + C_\nu\right)\sigma^*,
\end{align*}
with probability at least $1 - 2\exp\left(-2\left(\frac{1}{2}-c_t-\nu\right)^2n\right) - \frac{2}{n-t}$.
\end{proof}

%%%%%%%%%%%%%%%%%%%%%%%%%%%%%%%%%%%%%%%%%%%%%%%%%%%%%%%%%%%

\begin{lemma}
\label{lem:concentration-weaklydependent-orderstatistics}
Suppose $n \geq 12p$,
\[
\min_{i\in T} |\gammastar_i| \geq \frac{5}{4}\left(\frac{c_\nu + 5C_\nu}{\bar{c}}\right) \sqrt{\log 2n}\, \sigma^*,
\]
and inequality~\eqref{eq:max-PXS-gammaS} holds. Then
\begin{align}\label{eq:cleandata}
\|P_{X_{T^c}}^\perp \epsilon_{T^c}\|_\infty < \frac{5}{2\bar{c}}\sqrt{\log 2n} \sigmahat,
\end{align}
and for any $\gammastar_S$ such that $S\cap T \neq \emptyset$, we have
\begin{align}\label{eq:ifgammaexists}
\|P_{X_S}^\perp(\gammastar_S+\epsilon_S)\|_\infty > \frac{5}{2\bar{c}}\sqrt{\log 2n} \sigmahat,
\end{align}
with probability at least $1 - \frac{3}{n-t} - 2\exp\left(-2\left(\frac{1}{2}-c_t-\nu\right)^2n\right)$.
\end{lemma}

\begin{proof}
We first establish the bound on $\|P_{X_{T^c}}^\perp \epsilon_{T^c}\|_\infty$. Note that $e_j^\top P_{X_{T^c}}^\perp \epsilon_{T^c}$ is Gaussian with variance at most $\max\limits_{j\in T^c}(P_{X_{T^c}}^\perp)_{jj}$, so
\[
\|P_{X_{T^c}}^\perp \epsilon_{T^c}\|_\infty = \max\limits_{j \in T^c} |e_j^\top P_{X_{T^c}}^\perp \epsilon_{T^c}| \leq \max\limits_{j}(P_{X_{T^c}}^\perp)_{jj} 2\sqrt{\log 2(n-l)}\sigma^* \leq 2\sqrt{\log 2n}\,\sigma^*,
\]
with probability at least $1-\frac{1}{n-t}$. In addition, Lemma~\ref{lem:sigmahat-sigmastar} implies that
\[
\|P_{X_{T^c}}^\perp \epsilon_{T^c}\|_\infty \leq 2\sqrt{\log 2n} \frac{1}{\left(- \frac{c_\nu}{6}\frac{|S|}{|S|-p} + c_\nu\right)}\sigmahat \leq 2\sqrt{\log 2n} \frac{1}{\left(- \frac{1}{6}\frac{|S|}{|S|-p} + 1\right)\bar{c}}\sigmahat.
\]
For $n \geq 12 p$, we therefore conclude the bound~\eqref{eq:cleandata}.

Now consider $\gammastar_S$ with nonzero elements, i.e., $S \supset T^c$. We have
\begin{align*}
\begin{split}
\|P_{X_S}^\perp(\gammastar_S+ \epsilon_S)\|_\infty &\geq \max\limits_{i \in S} |e_i^\top P_{X_S}^\perp \gammastar_S| -  \|P_{X_S}^\perp \epsilon_S\|_\infty \\
& \geq \max\limits_{i \in S} |e_i^\top P_{X_S}^\perp \gammastar_S| - 2\sqrt{\log 2n} \, \sigma^*,
\end{split}
\end{align*}
with probability at least $1 - \frac{1}{n-t}$. We now split $P_{X_S}^\perp$ into $P_{X_S}^\perp - (1-\frac{p}{n-l})I$ and $(1-\frac{p}{n-l})I$. By the triangle inequality, we have
\begin{align*}
\begin{split}
\max\limits_{i \in [n-l]} \left|e_i^\top P_{X_S}^\perp \gammastar_S\right| &\geq  \max\limits_{i \in [n-l]} \left|e_i^\top \left(1-\frac{p}{n-l}\right)I \gammastar_S\right| - \max\limits_{i \in [n-l]} \left|e_i^\top \left(P_{X_S}^\perp - \left(1-\frac{p}{n-l}\right)I\right) \gammastar_S\right|\\
& \geq \left(1-\frac{p}{n-l}\right)\|\gammastar_S\|_\infty - \max\limits_{i \in [n-l]} \left|\underbrace{e_i^\top \left(P_{X_S}^\perp - \left(1-\frac{p}{n-l}\right)I\right) \gammastar_S}_{v_i}\right|.\\
\end{split}
\end{align*} 
Plugging this into the result from inequality~\eqref{eq:max-PXS-gammaS}, we then obtain
\begin{align*}
\max\limits_{i \in [n-l]} \left|e_i^\top P_{X_S}^\perp \gammastar_S\right| &\geq \left(1-\frac{p}{n-l}\right)\|\gammastar_S\|_\infty - \frac{c_\nu}{8} \sigma^*.
\end{align*}
Therefore, we have 
\begin{align*}
\begin{split}
\|P_{X_S}^\perp(\gammastar_S+\epsilon_S)\|_\infty &\geq \left(1-\frac{p}{n-t}\right) \min\limits_{i\in T}|\gammastar_i| - (2 \sqrt{\log 2n} + c_\nu/8) \sigma^*.
\end{split}
\end{align*}
By the assumption that $n \geq 12p$ and Lemma~\ref{lem:sigmahat-sigmastar}, we then obtain
\begin{align*}
\|P_{X_S}^\perp(\gammastar_S+\epsilon_S)\|_\infty& \geq \frac{5}{6} \min\limits_{i\in T}|\gammastar_i| - \frac{(2 \sqrt{\log 2n} + c_\nu/8)}{c_\nu - \frac{|S|}{|S|-p}\frac{c_\nu}{6}} \sigmahat \\
& \geq \frac{5}{6} \min\limits_{i\in T}|\gammastar_i| - \frac{(2 \sqrt{\log 2n} + c_\nu/8)}{c_\nu - \frac{c_\nu}{5}} \sigmahat \\
& \geq \frac{5}{6} \min\limits_{i\in T}|\gammastar_i| - \frac{13}{6}\frac{ \sqrt{\log 2n}}{\frac{4c_\nu}{5}} \sigmahat.
\end{align*}

Thus, $\|P_{X_S}^\perp(\gammastar_S+\epsilon_S)\|_\infty \geq \frac{5}{2\bar{c}}\sqrt{\log 2n}\, \sigmahat$ if $\min\limits_{i \in T}|\gammastar_i|$ satisfies
\[
\min_{i\in T} |\gammastar_i| \geq \sqrt{\log 2n}\, \sigmahat \left(\frac{3}{\bar{c}}+\frac{13}{4c_\nu}\right).
\]
This can be further achieved according to Lemma~\ref{lem:sigmahat-sigmastar} if 
\[
\min_{i\in T} |\gammastar_i| \geq \sqrt{\log 2n}\, \sigma^* \left(\frac{3}{\bar{c}}+\frac{13}{4c_\nu}\right)\left(C_\nu + \frac{c_\nu}{6}\frac{|S|}{|S|-p}\right).
\]
Also note that by the assumption of $\min_{i\in T}|\gamma_i|$, we have  
\[
\min_{i\in T} |\gammastar_i| \geq \frac{5}{4}\left(\frac{c_\nu+5C_\nu}{\bar{c}}\right) \sqrt{\log 2n}\, \sigma^* \geq \sqrt{\log 2n}\, \sigma^* \left(\frac{3}{\bar{c}}+\frac{13}{5c_\nu-\bar{c}}\right)\left(C_\nu + \frac{c_\nu}{6}\frac{|S|}{|S|-p}\right).
\]
This concludes the proof. 
\end{proof}

%%%%%%%%%%%%%%%%%%%%%%%%%%%%%%%%%%%%%%%%%%%%%%%%%%%%%%%%%%%
\begin{lemma} [Theorem 2.5 in Adamczak~\cite{Ada15}]
\label{LemAda15}
Suppose $X$ is a zero-mean random vector in $\real^n$ satisfying the convex concentration property with constant $K$. Then for any fixed matrix $A \in \real^{n \times n}$ and any $w > 0$, we have
\begin{align*}
\mprob\left(|X^\top  AX - \E[X^\top  AX]| \ge w\right) 
\le 2\exp\left(-\frac{1}{C} \min\left\{\frac{w^2}{2K^4 \|A\|_F^2}, \; \frac{w}{K^2 \norm{A}}\right\}\right).
\end{align*}
\end{lemma}

%%%%%%%%%%%%%%%%%%%%%%%%%%%%%%%%%%%%%%%%%%%%%%%%%%%%%%%%%%%

\begin{lemma}
\label{LemConcEig}
Suppose $X \in \real^{n \times p}$ has i.i.d.\ rows from a zero-mean distribution satisfying the convex concentration property with constant $K$. Then
\begin{equation*}
\norm{\frac{X^\top X}{n} - \E\left[\frac{X^\top X}{n}\right]} \le c\frac{\lambda_{\min}(\Sigma)}{2},
\end{equation*}
with probability at least $1- \exp(-n)$.
\end{lemma}

\begin{proof}
Note that for any fixed unit vector $u \in \real^p$, the map $\varphi: x \mapsto \inprod{x}{u}$ is convex and 1-Lipschitz. Hence, by the definition of the convex concentration property, each $x_i^\top  u$ is sub-Gaussian with parameter proportional to $K$. In fact, this is enough to show the desired matrix concentration result (cf.\ Vershynin~\cite{vershynin2010introduction}). We omit the details.
\end{proof}

%%%%%%%%%%%%%%%%%%%%%%%%%%%%%%%%%%%%%%%%%%%%%%%%% Active debugging %%%%%%%%%%%%%%%%%%%%%%%%%%%%%%%%%%%%%%%%%%%%%%%

\section{Appendix for Section~\ref{sec:active}}
\label{AppSecActive}

In this sectopm, we provide proofs and additional details for the results in Section~\ref{sec:active}.

\subsection{Proof of Theorem~\ref{thm:active-recover}}\label{app:sign}
We will prove a stronger results here, which implies Theorem~\ref{thm:active-recover}. This is actually mentioned by Remark~\ref{remark:sign-thm-active}.
\begin{theorem}\label{thm:active-recover-sign}
With respect to $D$, the bug generator, who has attacking budgets no more than $t$, cannot fail the sign support recovery if only if~\eqref{EqnRSN} holds. That failure of sign support recovery, $\sign(\gammahat) \neq \sign(\gammastar)$, means either $\gammahat_j \neq 0$ for some $j\in T^c$ or $\gammahat_i \gammastar_i \leq 0$ for some $i \in T$.
\end{theorem}
\begin{proof}[Proof of Theorem~\ref{thm:active-recover}]
We will use the following lemma to prove Theorem~\ref{thm:active-recover}.

\begin{lemma}\label{prop:rnp-support}
The following two properties are equivalent:
\begin{enumerate}
\item[(a)] For any vector $\gammastar \in \real^d$ with support $K$, the constraint-based optimization has all solutions $\gammahat$ satisfying $\sign(\gammahat) = \sign(\gammastar)$.
\item[(b)] The matrix $\Pbar(D)$ satisfies the restricted nullspace property with respect to $K$.
\end{enumerate}
\end{lemma}

\begin{proof}[Proof of Lemma~\ref{prop:rnp-support}]
We first prove $(b) \implies (a)$. This immediately follows Theorem 7.8 in~\cite{wainwright2019high} since $(b) \implies \gammastar = \gammahat$ for any vector $\gammastar$ with $\supp(\gammastar) = K$, it thus implies $(b) \implies \sign(\gammahat) = \sign(\gammastar)$.  Or we can show it directly as follow. Suppose $(a)$ doesn't hold. Then, we have $\Delta := \gammastar - \gammahat \neq 0$. By the constraint and the objective, it also needs to satisfy that $\Delta \in Null(\Pbar(D))$ and 
\begin{align*}
\|\gammastar - \Delta\|_1 = \|\gammahat\|_1 \leq \|\gammastar\|_1 = \|\gammastar_K\|_1.
\end{align*}
Therefore, we have
\begin{align*}
\|\gammastar_K\|_1 - \|\Delta_K\|_1 + \|\Delta_{K^c}\|_1 \leq \|\gammastar_K - \Delta_K\|_1 + \|\Delta_{K^c}\|_1 \leq \|\gammastar_K\|_1,
\end{align*}
which means a nonzero $\Delta \in Null(\Pbar) \cap \mathbb{C}^A$ and causes a contradiction. Thus when $(b)$ is true, $(a)$ holds as well.  

From now on to the end of the proof, we will abuse notation by using $\Pbar$ to represent $\Pbar(D)$. The remaining thing is to prove $(a) \implies (b)$. We will prove by contradiction. If $(b)$ doesn't hold, then there exists a nonzero $\Delta$ such that $\Pbar \Delta = 0$ and $\|\Delta_{K^c}\|_1 \leq \|\Delta_K\|_1$. We consider a $\gammastar$ with $\gammastar_K = \Delta_K$ and $\gammastar_{K^c} = \vec{0}$. Let $\gammahat$ be the optimizer given this $\gammastar$. By $(a)$, we shall have $\sign(\gammahat) = \sign(\gammastar) = \sign\left(\begin{bmatrix}\Delta_K\\\vec{0}_{(n-t)\times 1}\end{bmatrix}\right)$. The idea is to construct a $\gamma'$ that has no larger $\ell_1$ norm than $\gammahat$ and has support not equal to $K$, which contradicts with $(a)$, and therefore, $(b)$ must hold.

Consider $\gamma' = \gammahat - c \cdot \Delta$ where $c = \frac{\gammahat_i}{\Delta_i}$ for $i = \arg\min_{j\in K} \frac{\gammahat_j}{\Delta_j}$. Since $\Delta$ is a nonzero vector, we must have $\Delta_l \neq 0$ for some $l \in K$. Therefore, we have $c$ being positive finite, $\gamma'_i = 0$ and $|\gammahat_j| \geq c |\Delta_j|$ for all $j \in K$. Therefore, we further get
\begin{align*}
\Pbar(\gammastar - \gamma') = \Pbar (\gammastar - \gammahat + c\Delta) = \Pbar (\gammastar - \gammahat) = 0,
\end{align*}
as well as
\begin{align*}
\|\gamma'\|_1 & = \|\gammahat_K - c\cdot \Delta_K \|_1 + \|\gammahat_{K^c} - c\cdot \Delta_{K^c} \|_1 \\
& \overset{(i)}{=} \|\gammahat_K\|_1 - c\|\Delta_K \|_1 + c\|\Delta_{K^c} \|_1 \\
& \overset{(ii)}{\leq} \|\gammahat\|_1,
\end{align*}
where $(i)$ is because $\sign(\gammahat_K) = \sign(\Delta_K), c > 0, |\gammahat_K| \geq c|\Delta_K|$ and $\gammahat_{K^c} = 0$, $(ii)$ is because $\Delta \in \mathbb{C}(K)$. Hence, we find a $\gamma'$ to have smaller or equal $\ell_1$ norm than $\gammahat$. This contradicts with the fact that all the solutions have support $K$ or $\gammahat$ is the optimal solution. Therefore, $(b)$ must hold and $(a) \implies (b)$.
\end{proof}
We first prove that~\eqref{EqnRSN} is sufficient. For any $|K| \leq t$ and $K \subseteq [n]$, we know that $Null(\Pbar(D)) \cap \mathbb{C}(K) = \{0\}$. Then by Proposition~\ref{prop:rnp-support}, we conclude that $sign(\gammahat) = sign(\gammastar)$ with $\supp(\gammastar) = K$ for any subset $K$ of size no more than $t$.

We second prove that~\eqref{EqnRSN} is necessary. Note that for any subset $K$ of size less equal to $t$, we have $\sign(\gammahat) = \sign(\gammastar)$ with $\supp(\gammastar) = K$. By Proposition~\ref{prop:rnp-support}, it means $\Pbar(D)$ satisfies the restricted nullspace property for any such $K$. Therefore $Null(\Pbar(D)) \cap \mathbb{C}^A = \{\vec{0}\}$.
\end{proof}

Theorem~\ref{thm:active-recover} immediately holds from Theorem~\ref{thm:active-recover-sign}.

\subsection{Proof of Remark~\ref{remark:active-condition}}
We will prove the statement in Remark~\ref{remark:active-condition} here. 
\begin{proposition}\label{prop:active-condition}
The subspace $Null(\Pbar(D))$ is equivalent to $\{u \in \real^n  \mid \exists v \in \real^p, \mbox{ such that } u = Xv, X_Dv = 0\}$.
\end{proposition}
\begin{proof}[Proof of Proposition~\ref{prop:active-condition}]
We first prove $Null(\Pbar(D)) \supseteq \{u\in \real^n  \mid \exists v \in \real^p, \mbox{ such that } u = Xv, X_Dv = 0\}$.  Let $u = \left(X+M^\top X_D\right)v$ for some $v \in \real^p$, where $M \in \real^{m \times p}$ contains $m$ rows stacked with the canonical vectors indexed by $D$ so that $MX = X_D$. We have
\begin{align*}
\left(I - X\left(X^\top X +  X_D^\top X_D \right)^{-1}X^\top\right) u &= u - X\left(X^\top X + \frac{\eta n}{m} X_D^\top X_D \right)^{-1}X^\top \left(X+\frac{\eta n}{m}M^\top X_D\right)v \\
&= \frac{\eta n}{m}M^\top X_D v.
\end{align*}
Besides,  we have 
\begin{align*}
\begin{split}
X_D\left(X^\top X + \frac{\eta n}{m} X_D^\top X_D \right)^{-1}X^\top u &= X_D\left(X^\top X +  X_D^\top X_D \right)^{-1}X^\top \left(X+M^\top X_D\right)v \\
& = X_D v.
\end{split}
\end{align*}
Therefore $X_Dv = 0, u = Xv \implies u \in Null(\Pbar(D))$.

Secondly we prove $Null(\Pbar(D)) \subseteq \{u\mid \exists v \in \real^d, \mbox{ such that } u = Xv, X_Dv = 0\}$.
Let $u$ be some vector in $\mathbb{N}(X_D)$. Then we have 
\begin{align}\label{eq:active1}
u = X\left(X^\top X +  X_D^\top X_D \right)^{-1}X^\top u,
\end{align}
and
\begin{align}\label{eq:active2}
 X_D \left(X^\top X + X_D^\top X_D \right)^{-1}X^\top u = 0.
\end{align}
By~\eqref{eq:active2}, we have $\left(X^\top X + X_D^\top X_D \right)^{-1}X^\top u = v$ for some $v \in Null(X_D)$. Plugging this back to~\eqref{eq:active1}, we have $u = Xv$. Hence, we have $u \in \{u \mid \exists v \in \real^d, \mbox{ such that } u = Xv, X_D v = 0\}$.
\end{proof}

\subsection{Proof of Theorem~\ref{thm:milp}}
Here we prove the proof of Theorem~\ref{thm:milp}.
We write the minimax MILP here again.  
\begin{align}\label{eq:milp-proof}
\min_{\xi \in \{0,1\}^n} \quad  \max_{\substack{a,a^+,a^-, u, u^+, u^- \in \real^n, v \in \real^d \\ z,w \in \{0,1\}^n}} & \sum_{j=1}^n a_j^+ - a_j^-,  \\
\mbox{subject to } & u = Xv, \label{line:constraint}\\
& u = u^+ - u^-,  a = u^+ + u^-,  u^+, u^- \geq 0,  u^+ \leq  z, \ u^- \leq (\mathds{1}_n-z), \label{line:u-abs} \\
& \sum_{i=1}^n w_i \leq t, \label{line:number-bugs} \\
& a^+ \leq w, \ a^- \leq \mathds{1}_n-w, a = a^+ + a^-, a^+ \geq 0, a^- \geq 0, \label{line:a-obj} \\
& \sum_{i=1}^n \xi_i \leq m \ i = 1, \dots, n,\\
& u \leq  \mathds{1}_n - \xi, u \geq - (\mathds{1}_n - \xi) \label{line:constraint-xi-XDv}.
\end{align} 

\begin{proof}[Proof of Theorem~\ref{thm:milp}]
We first argue that if~\eqref{eq:minimax} has the unique solution of $(u,v) = (\vec{0},\vec{0})$, then~\eqref{EqnRSN} holds and thus the debugger can add $m$ points indexed by $D$ to achieve support recovery.  
\begin{equation}\label{eq:minimax}
\begin{split}
\min_{\substack{D \in [n], \\ |D| \leq m}} \max_{K \subseteq [n], |K| \leq t, u \in \real^n, v \in \real^d} & \|u_K\|_1 - \|u_{K^c}\|_1, \\
\mbox{subject to } & u = Xv, X_D v = 0, \|u\|_\infty \leq 1. 
\end{split}
\end{equation}
Suppose~\eqref{EqnRSN} doesn't hold. Then there exists $K\subseteq [n], |K| \leq t$ and a nonzero vector $u'$ such that $u'=Xv, X_Dv= 0$ and $\|u'_K\|_1 \geq \|u'_{K^c}\|_1$. And $\frac{u'}{\|u'\|_2}$ satisfies $\|u'\|_\infty \leq 1$. This contradicts with that~\eqref{eq:minimax} has the unique solution of $(u,v) = (\vec{0},\vec{0})$, then~\eqref{EqnRSN} holds. This concludes our first part of the proof.

Now we argue that the MILP is equivalent to~\eqref{eq:minimax}. Equation~\eqref{line:constraint} is inherited from original constraint. 
Equations in~\eqref{line:u-abs} are equivalent to $a = |u|$. Note that $u^+, u^-$ respectively correspond to the positive and negative parts of $u$. If $z_i = 0$, then $u_i^+ = 0$, $u^-_i \leq 1$ and $u^-_i = -u_i$. If $z_i = 1$, then $u^-_i = 0$, $u^+_i \leq 1$ and $u^+_i = u_i$. 
The vector $w$ indicates $K$ in~\eqref{eq:minimax}.  If $w_i=1$, then $i \in K$ otherwise $i \in K^c$. Therefore, equation~\eqref{line:number-bugs} restricts the attacking budget to $t$. Then, equations in~\eqref{line:a-obj} are equivalent to $a_i^+ = |u_i|, a_i^- = 0$ for $i \in K$ and $a_i^- = |u_i|, a_i^+ = 0$ for $i \in K^c$. Therefore, the objective function corresponds to $\|u_K\|_1 - \|u_{K^c}\|_1$.

Note that the variable in the first layer is $\xi$. If $\xi_i = 1$, it means the debugger queries the point $x_i$. And the constraint $X_D v = 0$ is replaced by~\eqref{line:constraint-xi-XDv}. This is because $x_i^\top v = 0 \Leftrightarrow u_i = 0$. If $\xi_j = 0$, then $u_j$ just needs to satisfy $|u_j| \leq 1$.

Therefore, we have shown that the MILP is equivalent to~\eqref{eq:minimax} and thus conclude Theorem~\ref{thm:milp}.

\end{proof}

\vfill

\bibliographystyle{spmpsci}      
\bibliography{refdebugging}

\begin{thebibliography}{10}
\providecommand{\url}[1]{{#1}}
\providecommand{\urlprefix}{URL }
\expandafter\ifx\csname urlstyle\endcsname\relax
  \providecommand{\doi}[1]{DOI~\discretionary{}{}{}#1}\else
  \providecommand{\doi}{DOI~\discretionary{}{}{}\begingroup
  \urlstyle{rm}\Url}\fi

\bibitem{Ada15}
Adamczak, R.: A note on the {H}anson-{W}right inequality for random vectors
  with dependencies.
\newblock Electronic Communications in Probability \textbf{20} (2015)

\bibitem{cadamuro2016debugging}
Cadamuro, G., Gilad-Bachrach, R., Zhu, X.: Debugging machine learning models.
\newblock In: ICML Workshop on Reliable Machine Learning in the Wild (2016)

\bibitem{chakraborty2018adversarial}
Chakraborty, A., Alam, M., Dey, V., Chattopadhyay, A., Mukhopadhyay, D.:
  Adversarial attacks and defences: A survey.
\newblock arXiv preprint arXiv:1810.00069  (2018)

\bibitem{fergus2009semi}
Fergus, R., Weiss, Y., Torralba, A.: Semi-supervised learning in gigantic image
  collections.
\newblock In: NIPS, vol.~1, p.~2. Citeseer (2009)

\bibitem{foygel2014corrupted}
Foygel, R., Mackey, L.: Corrupted sensing: Novel guarantees for separating
  structured signals.
\newblock IEEE Transactions on Information Theory \textbf{60}(2), 1223--1247
  (2014)

\bibitem{henderson1981deriving}
Henderson, H.V., Searle, S.R.: On deriving the inverse of a sum of matrices.
\newblock Siam Review \textbf{23}(1), 53--60 (1981)

\bibitem{hoeffding1994probability}
Hoeffding, W.: Probability inequalities for sums of bounded random variables.
\newblock In: The Collected Works of Wassily Hoeffding, pp. 409--426. Springer
  (1994)

\bibitem{horn1994topics}
Horn, R.A., Johnson, C.R.: Topics in Matrix Analysis.
\newblock Cambridge University Press (1994)

\bibitem{huber2009robust}
Huber, P., Ronchetti, E.: Robust Statistics.
\newblock Wiley Series in Probability and Statistics. Wiley (2011)

\bibitem{hwang2004cauchy}
Hwang, S.G.: Cauchy's interlace theorem for eigenvalues of {H}ermitian
  matrices.
\newblock The American Mathematical Monthly \textbf{111}(2), 157--159 (2004)

\bibitem{meinshausen2009lasso}
Meinshausen, N., Yu, B.: Lasso-type recovery of sparse representations for
  high-dimensional data.
\newblock The Annals of Statistics \textbf{37}(1), 246--270 (2009)

\bibitem{nguyen2013robust}
Nguyen, N.H., Tran, T.D.: Robust {L}asso with missing and grossly corrupted
  observations.
\newblock IEEE Transactions on Information Theory \textbf{4}(59), 2036--2058
  (2013)

\bibitem{ravikumar2010high}
Ravikumar, P., Wainwright, M.J., Lafferty, J.D.: High-dimensional {I}sing model
  selection using $\ell_1$-regularized logistic regression.
\newblock The Annals of Statistics \textbf{38}(3), 1287--1319 (2010)

\bibitem{rousseeuw2006computing}
Rousseeuw, P.J., Van~Driessen, K.: Computing lts regression for large data
  sets.
\newblock Data mining and knowledge discovery \textbf{12}(1), 29--45 (2006)

\bibitem{sasai2020robust}
Sasai, T., Fujisawa, H.: Robust estimation with {L}asso when outputs are
  adversarially contaminated.
\newblock arXiv preprint arXiv:2004.05990  (2020)

\bibitem{seber2008matrix}
Seber, G.A.F.: A Matrix Handbook for Statisticians, vol.~15.
\newblock John Wiley \& Sons (2008)

\bibitem{she2011outlier}
She, Y., Owen, A.B.: Outlier detection using nonconvex penalized regression.
\newblock Journal of the American Statistical Association \textbf{106}(494),
  626--639 (2011)

\bibitem{slawski2017linear}
Slawski, M., Ben-David, E.: Linear regression with sparsely permuted data.
\newblock arXiv preprint arXiv:1710.06030  (2017)

\bibitem{tang2016class}
Tang, Y., Richard, J.P.P., Smith, J.C.: A class of algorithms for mixed-integer
  bilevel min--max optimization.
\newblock Journal of Global Optimization \textbf{66}(2), 225--262 (2016)

\bibitem{veit2017learning}
Veit, A., Alldrin, N., Chechik, G., Krasin, I., Gupta, A., Belongie, S.:
  Learning from noisy large-scale datasets with minimal supervision.
\newblock In: Proceedings of the IEEE conference on computer vision and pattern
  recognition, pp. 839--847 (2017)

\bibitem{vershynin2010introduction}
Vershynin, R.: Introduction to the non-asymptotic analysis of random matrices.
\newblock arXiv preprint arXiv:1011.3027  (2010)

\bibitem{vershynin2018high}
Vershynin, R.: High-Dimensional Probability: An Introduction with Applications
  in Data Science, vol.~47.
\newblock Cambridge University Press (2018)

\bibitem{wainwright2009sharp}
Wainwright, M.J.: Sharp thresholds for high-dimensional and noisy sparsity
  recovery using $\ell_1$-constrained quadratic programming ({L}asso).
\newblock IEEE Transactions on Information Theory \textbf{55}(5), 2183--2202
  (2009)

\bibitem{wainwright2019high}
Wainwright, M.J.: High-Dimensional Statistics: A Non-Asymptotic Viewpoint,
  vol.~48.
\newblock Cambridge University Press (2019)

\bibitem{xu2014exact}
Xu, P., Wang, L.: An exact algorithm for the bilevel mixed integer linear
  programming problem under three simplifying assumptions.
\newblock Computers \& operations research \textbf{41}, 309--318 (2014)

\bibitem{zeng2014solving}
Zeng, B., An, Y.: Solving bilevel mixed integer program by reformulations and
  decomposition.
\newblock Optimization Online pp. 1--34 (2014)

\bibitem{zhang2018training}
Zhang, X., Zhu, X., Wright, S.: Training set debugging using trusted items.
\newblock In: Proceedings of the AAAI Conference on Artificial Intelligence,
  vol.~32 (2018)

\end{thebibliography}

\end{document}